\documentclass[twoside,11pt]{article}

\usepackage[preprint]{jmlr2e}
\usepackage{listings}
\usepackage{booktabs}
\usepackage{amsmath}
\usepackage{xspace}
\usepackage{xcolor}
\usepackage{wrapfig}
\usepackage{colortbl}
\usepackage{makecell}
\usepackage{etoc}
\etocdepthtag.toc{mtchapter}
\etocsettagdepth{mtchapter}{subsection}
\etocsettagdepth{mtappendix}{none}
\usepackage{enumitem,algorithm,algorithmic}
\usepackage{subfigure}
\usepackage{multirow}
\usepackage{longtable}
\usepackage{tikz}
\usetikzlibrary{arrows.meta, positioning, decorations.markings}

\usepackage[T1]{fontenc}
\usepackage{lmodern} %
\usepackage{url}
\definecolor{darkblue}{rgb}{0, 0, 0.5}
\definecolor{beaublue}{rgb}{0.74, 0.83, 0.9}
\definecolor{gainsboro}{rgb}{0.86, 0.86, 0.86}
\definecolor{kleinblue}{rgb}{0,0.18,0.65}
\hypersetup{hidelinks}

\usepackage{amsmath,amsfonts,bm}

\newcommand{\ind}{\perp\!\!\!\!\perp}
\newcommand{\notind}{\not \! \ind}

\newcommand{\varset}[1]{\mathcal{#1}}    %
\def\eqref#1{equation~\ref{#1}}

\def\1{\bm{1}}

\def\vp{{\bm{p}}}

\def\vu{{\bm{u}}}
\def\vv{{\bm{v}}}
\def\vw{{\bm{w}}}
\def\vx{{\bm{x}}}

\def\mU{{\bm{U}}}
\def\mV{{\bm{V}}}
\def\mW{{\bm{W}}}
\def\mX{{\bm{X}}}

\def\mZ{{\bm{Z}}}

\DeclareMathAlphabet{\mathsfit}{\encodingdefault}{\sfdefault}{m}{sl}
\SetMathAlphabet{\mathsfit}{bold}{\encodingdefault}{\sfdefault}{bx}{n}

\def\gA{{\mathcal{A}}}

\def\gC{{\mathcal{C}}}
\def\gD{{\mathcal{D}}}

\def\gF{{\mathcal{F}}}
\def\gG{{\mathcal{G}}}

\def\gI{{\mathcal{I}}}

\def\gM{{\mathcal{M}}}
\def\gN{{\mathcal{N}}}

\def\gP{{\mathcal{P}}}

\def\gR{{\mathcal{R}}}
\def\gS{{\mathcal{S}}}

\def\gU{{\mathcal{U}}}

\def\gW{{\mathcal{W}}}
\def\gX{{\mathcal{X}}}

\newcommand{\R}{\mathbb{R}}
\newcommand{\N}{\mathbb{N}}

\DeclareMathOperator*{\argmax}{arg\,max}

\newenvironment{myquotation}{\setlength{\leftmargini}{0em}\quotation}{\endquotation}
\newcommand{\runin}[1]{\par\addvspace{1.5ex}\noindent\textit{#1}\quad}
\newcommand{\para}[1]{\runin{{#1}}}
\newtheorem{assumption}{Assumption}

\usepackage{footmisc} %

\usepackage{etoolbox}
\makeatletter
\IfFormatAtLeastTF{2025-06-01}{%
  \IfDocumentMetadataTF{}{%
    \@ifundefined{H@refstepcounter}{}{%
      \@ifundefined{subfig@oldcaption}
        {\patchcmd\@caption}
        {\patchcmd\subfig@oldcaption}
        {\Hy@raisedlink{\expandafter\hyper@@anchor\expandafter{\@currentHref}{\relax}}}{}%
        {}%
        {\PackageWarning{main}{Legacy caption anchor not found; no patch applied}}%
    }%
  }%
}{}
\makeatother

\definecolor{codegreen}{rgb}{0,0.6,0}
\definecolor{codegray}{rgb}{0.5,0.5,0.5}
\definecolor{codepurple}{rgb}{0.58,0,0.82}
\definecolor{backcolour}{rgb}{0.95,0.95,0.92}

\definecolor{ensogreen}{RGB}{130, 179, 102}
\definecolor{ensoblue}{RGB}{108, 142, 191}
\definecolor{ensoorange}{RGB}{215, 155, 0}

\newcommand{\ourtitle}{Discovering and Reasoning of Causality in the Hidden World with Large Language Models\xspace}
\newcommand{\ourshorttitle}{Discovering and Reasoning of Causality in the Hidden World}
\newcommand{\ourauthors}{Liu, Chen, Liu, Gong, Cheng, Han, and Zhang}
\newcommand{\ourauthorsfull}{Chenxi Liu, Yongqiang Chen, Tongliang Liu, Mingming Gong, James Cheng, Bo Han, and Kun Zhang}
\newcommand{\ours}[0]{\text{COAT}\xspace}
\newcommand{\ourst}[0]{\text{COAT}\xspace}
\newcommand{\oursMB}[0]{\text{COAT-MB}\xspace}
\newcommand{\oursPAG}[0]{\text{COAT-PAG}\xspace}
\newcommand{\oursGAS}[0]{\text{COAT-ADJ}\xspace}
\newcommand{\oursADJ}[0]{\oursGAS}
\newcommand{\oursElicit}[0]{Active Causal Representation Elicitation\xspace}
\newcommand{\oursfull}[0]{\textbf{C}ausal representati\textbf{O}n \textbf{A}ssistan\textbf{T}\xspace}

\newcommand{\MB}{\text{MB}}

\newcommand{\apple}{\text{AppleGastronome}\xspace}
\newcommand{\pain}{\text{Neuropathic}\xspace}
\newcommand{\chatgpt}{\text{GPT-3.5}\xspace}
\newcommand{\gptf}{\text{GPT-4}\xspace}
\newcommand{\gptfo}{\text{GPT-4o}\xspace}
\newcommand{\llama}{\text{LLaMA2}\xspace}
\newcommand{\llamathree}{\text{LLaMA3}\xspace}
\newcommand{\mistral}{\text{Mistral}\xspace}

\newcommand{\std}[1]{\textit{\footnotesize{$\pm$#1}}}
\newcommand{\codeblock}[1]{ \colorbox{black!5}{\texttt{#1}} }

\newcommand{\starrightarrow}{ { \ * \! \! \rightarrow \ } }
\newcommand{\starleftarrow}{ { \ \leftarrow  \! \! * \ } }
\newcommand{\circarrowcirc}{ {  \ \circ  \! \! - \! \!  \circ  \  } }
\newcommand{\stararrowcirc}{ {  \ *  \! \! - \! \!  \circ  \  } }
\newcommand{\circarrowstar}{ {  \ \circ  \! \! - \! \!  *  \  } }
\newcommand{\circrightarrow}{ { \ \circ \! \! \rightarrow \ } }

\newcommand{\edit}[1]{{{#1}}} %
\newcommand{\minor}[1]{{{#1}}} %

\usepackage{lastpage}
\jmlrheading{}{2026}{}{}{}{}{\ourauthorsfull}

\ShortHeadings{\ourshorttitle}{\ourauthors}
\firstpageno{1}

\begin{document}

\title{\ourtitle}

\newcommand{\addrHKBU}{$^1$~TMLR Group, Hong Kong Baptist University, Kowloon, Hong Kong SAR}
\newcommand{\addrMBZUAI}{$^2$~Mohamed bin Zayed University of Artificial Intelligence, Abu Dhabi, UAE}
\newcommand{\addrCMU}{$^3$~Department of Philosophy, Carnegie Mellon University, Pittsburgh, PA 15213, USA}
\newcommand{\addrSYD}{$^4$~School of Computer Science, The University of Sydney, NSW 2008, Australia}
\newcommand{\addrUMelb}{$^5$~School of Mathematics and Statistics, The University of Melbourne, Victoria 3010, Australia}
\newcommand{\addrCUHK}{$^6$~The Chinese University of Hong Kong, Shatin, Hong Kong SAR}

\author{\name Chenxi Liu$^{*~1}$ \email cscxliu@comp.hkbu.edu.hk 
       \AND
       \name Yongqiang Chen$^{*~2,~3}$   \email yqchen24@gmail.com 
       \AND
       \name Tongliang Liu$^{~4,~2}$  \email tongliang.liu@sydney.edu.au 
       \AND
       \name Mingming Gong$^{~5,~2}$  \email mingming.gong@unimelb.edu.au  
       \AND
       \name James Cheng$^{~6}$   \email jcheng@cse.cuhk.edu.hk
       \AND
       \name Bo Han$^{\ddag~1}$   \email bhanml@comp.hkbu.edu.hk
       \AND
       \name  Kun Zhang$^{~3,~2}$   \email kunz1@cmu.edu  
       \AND
       \addr \addrHKBU \\
       \addr \addrMBZUAI \\ 
       \addr \addrCMU \\
       \addr \addrSYD  \\
       \addr \addrUMelb \\
       \addr \addrCUHK  
}

\maketitle

\renewcommand{\thefootnote}{\fnsymbol{footnote}}
\begin{NoHyper} %
\footnotetext{\small \textsuperscript{*} These authors contributed equally to this work. \quad \textsuperscript{\ddag} Corresponding author.}
\end{NoHyper}
\renewcommand{\thefootnote}{\arabic{footnote}}

\begin{abstract}%
Revealing hidden causal variables alongside the underlying causal mechanisms is essential to the development of science.
Despite the progress in the past decades, existing practice in causal discovery (CD) heavily relies on high-quality measured variables, which are usually given by human experts. 
In fact, the \textit{lack of well-defined high-level variables behind unstructured data} has been a longstanding roadblock to a broader real-world application of CD.
This procedure can naturally benefit from an automated process that can suggest potential hidden variables in the system. 
Interestingly, Large language models (LLMs) are trained on massive observations of the world and have demonstrated great capability in processing unstructured data.
To leverage the power of LLMs, we develop a new framework termed \oursfull (\ours) that incorporates the rich world knowledge of LLMs to propose useful measured variables for CD with respect to high-value target variables on their paired unstructured data.
Instead of directly inferring causality with LLMs, \ours constructs feedback from intermediate CD results to LLMs to refine the proposed variables.
Given the target variable and the paired unstructured data, we first develop \oursMB that leverages the predictivity of the proposed variables to iteratively uncover the Markov Blanket of the target variable.
Built upon \oursMB, \oursPAG further extends to uncover a more complete causal graph, i.e., Partial Ancestral Graph, by iterating over the target variables and actively seeking new high-level variables.
Moreover, the reliable CD capabilities of \ours also extend the debiased causal inference to unstructured data by discovering an adjustment set.
We establish theoretical guarantees for the CD results and verify that \ours can effectively and reliably make the most use of LLM knowledge for CD across $7$ realistic benchmarks and real-world case studies, including analysis of human reviews and diagnosis of neuropathic and brain tumors.
\end{abstract}


\section{Introduction}
\label{sec:intro}
Science originates along with identifying important variables and revealing their causal relations~\citep{Hanson1958PatternsOD,sci_revolution}.
Despite the progress in the past decades, existing causal discovery (CD) approaches  mainly rely on high-quality measured variables, which are usually given by human experts~\citep{causation_1st,auto_causal,causal_learn_survey}. 
However, the desired variables and their measurements are usually unavailable in a wide range of real-world applications. For example, Amazon sellers who want to analyze the factors related to user ratings only have raw user reviews, which are written according to the underlying user preferences for certain product characteristics. 
The high-level variables, such as the preferred product characteristics, reflected from the low-level raw text reviews, are useful for causal analysis on how to improve purchase rates of the products.
Therefore, \textit{the lack of high-quality high-level variables} has been a longstanding impediment to broader real-world applications of CD or causality-inspired methods~\citep{causal_rep_learn}.

The recently emerged Large Language Models (LLMs)~\citep{gpt3,chatgpt,llama,openai2023gpt4} offer a new opportunity to mitigate the gap~\citep{causal_llm_frontier}.
Trained from massive observations of the world, LLMs demonstrate impressive capabilities in \textit{comprehending unstructured inputs}, and combining the rich knowledge learned from pretraining to resolve a variety of general tasks~\citep{spark_AGI}.
Early studies suggest that incorporating LLMs as a \textit{direct reasoner} can effectively answer commonsense causal questions, demonstrating the great potential in resolving the bottlenecks in causal learning and reasoning~\citep{understand_causality_llm,causal_llm_frontier,cma}.
Nevertheless, the reliability of LLMs in directly reasoning the causal structure behind any specific data-generating process remains a debate~\citep{causalparrots,cladder,corr2cau,understand_causality_llm} due to a series of drawbacks of LLMs~\citep{llm_hallucination,gpt4bias,reversal_curse}.
The other works also try to leverage the prior knowledge of LLMs to assist causal discovery with respect to \textit{a fixed set of given variables}~\citep{llm_query_tool,cma}, while overlooking the valuable insight from unstructured data, e.g., image or text. 
In fact, the existing use of LLMs for CD still sticks to the traditional paradigm of CD, while surprisingly overlooking the identifiability of causal structure that is essential for the faithfulness of CD results~\citep{causation_1st,auto_causal,causal_learn_survey}. Furthermore, it is also unclear how LLMs can be integrated into causal inference tasks~\citep{elements_ci,bareinboim2022pearl}. Hence, it raises a challenging research question:
\begin{myquotation}
    \textit{How can LLMs reliably assist with discovering and reasoning the causality behind the observed world?}
\end{myquotation}
To answer the question, we consider a general setting ubiquitous in real-world applications, where the data consists of a set of high-value numerical variables $\varset{U}$, such as customer ratings and medical diagnoses, and their paired unstructured data $\mX$, such as customer reviews or medical images.
Built upon the setting, we further decompose and approach the general question into actionable sub-questions:
\begin{enumerate}
    \item[\textbf{Q1.}] As the first step, we begin by identifying factors that are useful for predicting the high-value target variable, which are immediately valuable for a variety of applications. Specifically, given one scalar target variable $Y\in \varset{U}$, how can we reliably acquire a set of high-level factors from unstructured data $\mX$ that is a Markov Blanket of $Y$?
    \item[\textbf{Q2.}] Furthermore, we extend our analysis to reveal more causal insights on the underlying high-level factors. Specifically, given a set of variables $\varset{U}$, how can we maximally reveal the causal information in the Partial Ancestral Graph (PAG) $\gP_\varset{U}$ by extending the nodes with factors from unstructured data $\mX$?
    \item[\textbf{Q3.}] Built upon the high-level factors identified from low-level observations data, we then explore another important application of causality in real-world applications, i.e., causal effect estimation: How can the elicitation of causal variables facilitate treatment effect estimation?
\end{enumerate}

\begin{table}[tbp]
\resizebox{\textwidth}{!}{%
\begin{tabular}{lll}
\toprule
                                      & \textbf{Content}                                                    & \textbf{Clarification}                \\ \midrule
\multirow{3}{*}{Method}               & COAT-MB (Algorithm~\ref{alg:coat-main}; Section~\ref{subsec:coat-mb})                  & inherited                     \\
                                  & COAT-PAG (Algorithm~\ref{alg:coat-search}; Section~\ref{subsec:introduce-coat-pag})                  & \textbf{new}                 \\
                                      & COAT-ADJ (Algorithm~\ref{alg:coat-t-y}; Section~\ref{subsec:intro-coat-adj})                      & \textbf{new}                 \\ \midrule
\multirow{3}{*}{\begin{tabular}[c]{@{}l@{}}Theoretical \\ Analysis\end{tabular}} 
                                      & Identifiability on Markov Blanket (Section~\ref{subsec:MB-identification})            & inherited                     \\
                                      & Identifiability on Partial Ancestral Graphs (Section~\ref{subsec:pag-identifiability}) & \textbf{new}                 \\
                                      & Identifiability on Adjustment Sets (Section~\ref{subsec:adjustment-set})         & \textbf{new}                 \\  \midrule
\multirow{8}{*}{Experiments}          & Apple Gastronome (setting 1) (Section~\ref{subsec:applesetting1})               & inherited                     \\
                                      & Neuropathic data set (Section~\ref{sec:pain_bench})                       & inherited                     \\
                                      & Three case studies (Section~\ref{sec:realwold-casestudy})                   & inherited                     \\
                                      & Apple Gastronome (setting 2) (Section~\ref{subsec:applesetting2})                 & \textbf{new}                 \\
                                      & Hillstrom and Retail Hero data sets (Section~\ref{subsec:CI-ATE})             & \textbf{new}                 \\
                                      & Apple Gastronome constructed with different LLMs           & \textbf{new} \\
                                      & Error Propagation Stress Test and Analysis    & \textbf{new} \\
                                      & Empirical comparison between COAT-MB, -PAG, and -ADJ      & \textbf{new} \\
\bottomrule
\end{tabular}}
\caption{\edit{Comparison of main results with conference version~(\citeyear{causalcoat2024}).}}
\label{tab:conf-compare}
\end{table}

To answer the aforementioned questions, we investigate a new interplay between CD and LLMs, and propose a different paradigm termed \oursfull (\ours): rather than directly reasoning for causality on existing variables, we harness LLMs to actively augment data sets by proposing and parsing useful high-level variables from unstructured data, which we term as \textit{\oursElicit}.
In general, the \ours framework employs two key components. (\textit{i}) LLMs: these models can understand unstructured input $\mX$ like text and images. Therefore, they are capable of proposing useful candidate factors $\varset{W}$ from the underlying possible factors $\varset{O}$ based on observation and instruction, and can parse out actual factor values for each unstructured observation. (\textit{ii}) Guaranteed CD Methods for structured data: these numerical algorithms have rigorous identifiability about when and how the causal information can be revealed. We utilize them to parse and construct feedback to refine the candidate factors $\varset{W} \subseteq \varset{O}$ proposed by LLMs. The feedback is the essential drive for \oursElicit, with which \ours is able to combine the best of LLMs and traditional CD methods to uncover the desired causality from the unstructured data.

\edit{
For the sub-question \textbf{Q1}, we introduce \oursMB as a reformulation of the previous work~\citep{causalcoat2024} on revealing Markov Blanket (MB) factors.
}
As stated in Algorithm~\ref{alg:coat-main}, in each iteration \oursMB first uses LLMs to propose a set of candidate factors, and then uses CD methods to construct feedback based on the predictivity of the target variable. The feedback filters out samples where the target variable can not be well explained by existing factors discovered in previous rounds. We also establish conditions on when LLMs can effectively identify the subset $\varset{W} \subseteq \varset{O}$ as a Markov Blanket of the given target variable $Y\in\varset{U}$ (Section~\ref{subsec:MB-identification}).

\begin{figure}[t]
    \subfigtopskip=2pt
    \centering
    \subfigure[]{
        \begin{minipage}[c][2.5cm][c]{0.3\linewidth}
            \centering
            \begin{tikzpicture}[
                baseline=(current bounding box.center),
                node distance=0.6cm,
                every node/.style={circle, draw, minimum size=6mm, inner sep=1pt},
                nonmbset/.style={circle, draw, fill=red!10, minimum size=6mm, inner sep=1pt},
                edge/.style={->, >=Latex, semithick},
                boxnode/.style={rectangle, draw, fill=black!20, minimum size=6mm, inner sep=1pt}
            ]
            \node (X3) [nonmbset] {$X_3$};
            \node[above right=0.2cm and 0.6cm of X3] (X1) {$X_1$};
            \node[below right=0.2cm and 0.6cm of X3] (X2) {$X_2$};
            \node[below right=0.2cm and 0.6cm of X1] (Y) {$Y$};
            
            \draw[edge] (X3) -- (X1);
            \draw[edge] (X3) -- (X2);
            \draw[edge] (X1) -- (Y);
            \draw[edge] (X2) -- (Y);
            \end{tikzpicture}
        \end{minipage}
        \label{fig:Q2_example_a}
    }
    \subfigure[]{
        \begin{minipage}[c][2.5cm][c]{0.3\linewidth}
            \centering
            \begin{tikzpicture}[
                baseline=(current bounding box.center),
                node distance=0.6cm,
                every node/.style={circle, draw, minimum size=6mm, inner sep=1pt},
                edge/.style={->, >=Latex, semithick},
                boxnode/.style={rectangle, draw, fill=black!20, minimum size=6mm, inner sep=1pt},
                oedge/.style={line width=0.8pt, decoration={markings, 
                    mark=at position 0.1 with {\node[circle, draw, fill=white, inner sep=0pt, minimum size=4pt] {};}, 
                    mark=at position 0.9 with {\node[circle, draw, fill=white, inner sep=0pt, minimum size=4pt] {};}}, 
                    postaction=decorate},
                circnormaledge/.style={line width=0.8pt, decoration={markings, 
                    mark=at position 0.1 with {\node[circle, draw, fill=white, inner sep=0pt, minimum size=4pt] {};}, 
                    mark=at position 0.99 with {\arrow{Latex}}}, 
                    postaction=decorate, shorten >=2pt},
            ]
            \node (Y) {$Y$};
            \node[above left=0.2cm and 0.6cm of Y] (X1) {$X_1$};
            \node[below left=0.2cm and 0.6cm of Y] (X2) {$X_2$};
   
            \draw[oedge] (X1) -- (Y);
            \draw[oedge] (X2) -- (Y);
            \draw[oedge] (X1) -- (X2);
            \end{tikzpicture}
        \end{minipage}
        \label{fig:Q2_example_b}
    }
    \subfigure[]{
        \begin{minipage}[c][2.5cm][c]{0.3\linewidth}
            \centering
            \begin{tikzpicture}[
                baseline=(current bounding box.center),
                node distance=0.6cm,
                every node/.style={circle, draw, minimum size=6mm, inner sep=1pt},
                edge/.style={->, >=Latex, semithick},
                nonmbset/.style={circle, draw, fill=red!10, minimum size=6mm, inner sep=1pt},
                oedge/.style={line width=0.8pt, decoration={markings, 
                    mark=at position 0.1 with {\node[circle, draw, fill=white, inner sep=0pt, minimum size=4pt] {};}, 
                    mark=at position 0.9 with {\node[circle, draw, fill=white, inner sep=0pt, minimum size=4pt] {};}}, 
                    postaction=decorate},
                circnormaledge/.style={line width=0.8pt, decoration={markings, 
                    mark=at position 0.1 with {\node[circle, draw, fill=white, inner sep=0pt, minimum size=4pt] {};}, 
                    mark=at position 0.99 with {\arrow{Latex}}}, 
                    postaction=decorate, shorten >=2pt},
            ]
            \node (X3) [nonmbset] {$X_3$};
            \node[above right=0.2cm and 0.6cm of X3] (X1) {$X_1$};
            \node[below right=0.2cm and 0.6cm of X3] (X2) {$X_2$};
            \node[below right=0.2cm and 0.6cm of X1] (Y)  {$Y$};
            
            \draw[oedge] (X3) -- (X1);
            \draw[oedge] (X3) -- (X2);
            \draw[circnormaledge] (X1) -- (Y);
            \draw[circnormaledge] (X2) -- (Y);
            
            \end{tikzpicture}
        \end{minipage}
        \label{fig:Q2_example_c}
    }
    \vspace{-0.1in}
    \caption{An example where PAG is less informative when limited in $\MB(Y)$. (a) is the true causal structure indicating causal directions; (b) and (c) are the partial ancestral graphs. The circle mark ($\circ$) means it is undetermined to be arrow head or tail. The node out of $\MB(Y)$ is colored. }
    \label{fig:Q2_example}
\end{figure}

\edit{For the sub-question \textbf{Q2}, we extend \ours framework with a new approach beyond the previous Markov Blanket setting.}
Essentially, Markov Blanket factors only cover limited causal information of the high-level variables. Shown as in Figure~\ref{fig:Q2_example}, the skeleton between two Markov Blanket factors $X_1, X_2 \in \text{MB}(Y)$ and their arrow heads towards $Y$, cannot be revealed without $X_3 \notin \text{MB}(Y)$. Therefore, it motivates us to develop \oursPAG to tackle \textbf{Q2}. \oursPAG can elicit helpful factor set $\varset{W}$ from unstructured input $\mX$ to further reveal causal structure (in the partial ancestral graph) among the input variable set $\varset{U}$. In Section~\ref{subsec:pag-identifiability}, we show that the skeleton and arrow-head information within any subset $\varset{U}' \subseteq \varset{V}\triangleq \varset{W} \cup \varset{U}$ (where $\varset{W}$ is produced by \oursPAG) can be maximally revealed if $ \MB(\varset{U}' )\subseteq \varset{V}$.

\edit{For the sub-question \textbf{Q3}}
, we further extend \ours framework to causal inference. Specifically, we propose \oursGAS that finds a PAG $\gP_\varset{V}$ where the adjustment sets within any subset $\varset{U}'\subseteq \varset{V}$ are preserved if $\MB(\varset{U}' )\subseteq \varset{V}$, as stated in Algorithm~\ref{alg:coat-t-y}.
In this paper, we employ the Generalized Adjustment Criterion~\citep{perkovi2018complete}, a sound and complete graphical characterization for adjustments sets, to process the PAG returned by \oursGAS. By this approach, we show that \oursGAS can achieve human-level unbiased treatment effect estimation from unstructured data (Section~\ref{subsec:adjustment-set}).

This paper builds upon the preliminary conference work~\citep{causalcoat2024}, which focuses on identifying a set of Markov Blanket factors (Section~\ref{sec:causal_feedback}~and~\ref{subsec:MB-identification}). 
However, 
\edit{as summarized in Table~\ref{tab:conf-compare}},
the present paper is significantly extended in several aspects. 
\begin{itemize}
    \item In Section~\ref{subsec:introduce-coat-pag}, we further extend the original framework to refine causal information about the input structured variable set $\varset{U}$ with unstructured data $\mX$. 
    We propose a new approach \oursPAG, as stated in Algorithm~\ref{alg:coat-search}, to augment the node set  $\varset{U}$ by proposing factors $\varset{W}$ from $\mX$. In Section~\ref{subsec:pag-identifiability}, we show that \oursPAG can produce an augmented node set $\varset{V}$ where the skeleton and arrow heads among any subset $\varset{U}' \subseteq \varset{V}$ are maximally revealed if $\MB(\varset{U}') \subseteq \varset{V}$ (Proposition~\ref{prop:arrowhead-preserve}). In Corollary~\ref{cor:arrow-tails-in-discriminating-path}, we show that \oursPAG can uncover information about discriminating paths. Furthermore, we also provide additional evaluation for this setting in Section~\ref{subsec:applesetting2}.
    \item In Section~\ref{subsec:intro-coat-adj}, we extend the original framework to construct adjustment sets from unstructured data $\mX$ and the input variable set $\varset{U}$ for causal effect estimation from a treatment $T$ to an effect $Y$. We propose a novel approach \oursGAS, as stated in Algorithm~\ref{alg:coat-t-y}, to augment the node set  $\varset{U}$ by proposing factors $\varset{W}$ from $\mX$. In Section~\ref{subsec:adjustment-set}, through Proposition~\ref{prop:Adjustment-Sets-Preservation} and other related propositions, we show that on the partial ancestral graph over the augmented node set $\varset{V}$, all valid adjustment sets within any subset $\varset{U}' \subseteq \varset{V}$ can be read off by graphical criterion %
    if $\MB(\varset{U}') \subseteq \varset{V}$.
\end{itemize}

This paper is organized as follows. In Section~\ref{sec:Preliminary-and-Related-Work}, we briefly review related work, and elaborate the novelty and motivation of our research. 
In Section~\ref{sec:setting}, we introduce our \ours framework. We first explain the key components, then describe the variant and extensions including \oursPAG and \oursGAS. 
In Section~\ref{subsec:theory}, we establish the identifiability results for the proposed methods. 
In Section~\ref{sec:apple_bench}, we present comprehensive evaluation, including synthetic data for eliciting the Markov Blanket and extending Partial Ancestral Graphs, ablation studies as well as the results on several realistic benchmarks. 
In Section~\ref{sec:realwold-casestudy}, we perform case studies of \ours on the real-world unstructured data including images, database that requires coding, and sequential news data. 

\section{Related Work}
\label{sec:Preliminary-and-Related-Work}

\subsection{Causal Methods for Structured Data}

\runin{Causal discovery} This topic aims to uncover the underlying causal structure from observational data~\citep{causation_1st}. It is critical to both real-world applications and scientific discoveries~\citep{Pearl1999CausalDF,book_why}. 
Such structures are usually represented by Directed Acyclic Graphs, where each arrow starts from a cause to an effect.

Classical methods include constrained-based methods~\citep{auto_causal,causation_1st,fci} that incorporate conditional independence tests to find Markov Equivalence class. There are also different methods with different assumptions, like constrained functional~\citep{shimizu2006linear,zhang2012identifiability,hoyer2008nonlinear}, multiple domain data~\citep{huang2020causal,pmlr-v80-yang18a,NEURIPS2020_f8b7aa3a,mooij2020joint,NEURIPS2022_46a12649}, and non-stationary data~\citep{pmlr-v89-malinsky19a,pmlr-v97-huang19g,huang2020causal,liu2023causal}.
There are also some works aiming to address realistic challenges, like structure models with arbitrary distribution~\citep{Chen_Xie_Qiao_Hao_Zhang_Cai_2022}, sub-sampled time series~\citep{wulearning}, and latent hierarchical causal structure ~\citep{huang2022latent,tin,NEURIPS2023_bdeab378,li2024causal}.

Despite recent theoretical and empirical improvements~\citep{Glymour2019ReviewOC,causal_learn_survey}, most existing causal discovery approaches rely on well-structured data with human-crafted factors as inputs. For example, the identification of causal direction can still be challenging when there are confounders not included in data~\citep{montagna2023assumption}. In this work, we show how to improve causal identification by introducing new variables from the paired unstructured data.

\runin{Causal effect estimation} This topic aims to estimate the causal influence of one variable on another.
The causal effect can be formulated as the average treatment effect in the potential outcome framework~\citep{rubin1974estimating,rosenbaum1983central,holland1986statistics}, or represented by an interventional distribution in the structural causal model~\citep{goldszmidt1992rank,pearl1995theory,causality}.
While the random control experiment is the gold standard for the inference of causal effect~\citep{rubin1978bayesian}, it is appealing to conduct estimation through nonexperimental data~\citep{pearl1995causal}.
One approach is to utilize instrumental variables that influence the effect variable only through the treatment variable~\citep{angrist1996identification,angrist2009mostly}. For instance, Mendelian randomization analyzes causal effects among clinical factors with genetic variants~\citep{burgess2015mendelian}. Another approach is to estimate the interventional distribution by conditional distributions through a set of covariates satisfying certain graphical criterion~\citep{pearl2011graphical,shpitser2012validity}. 
In this paper, we employ the Generalized Adjustment Criterion~\citep{perkovi2018complete} as a showcase to demonstrate the usefulness of the proposed \ours framework in constructing valid adjustment sets with unstructured data.

\subsection{Causal Methods for Unstructured Data}

\runin{Causal representation learning} This topic aims to develop approaches to identifying latent high-level variables (like location and color of an object) from low-level observations such as images~\citep{causal_rep_learn,zhang2024causal}. Such identification can be feasible with certain conditions, like auxiliary information~\citep{pmlr-v89-hyvarinen19a,pmlr-v108-khemakhem20a,pmlr-v119-locatello20a,pmlr-v162-kong22a,zhang2024causal}, sparsity~\citep{klindt2021nonlineardisentanglementnaturaldata,pmlr-v177-lachapelle22a,NEURIPS2022_63d3bae2,NEURIPS2022_6801fa3f}, or interventional data~\citep{ahuja2023interventional,zhang2024identifiability,NEURIPS2023_97fe251c,pmlr-v202-squires23a,buchholz2024learning}.
Recent works also generalize the causal representation learning to graph-structured data~\citep{ciga,gala,gmt}, and discuss the optimization challenges in realizing causal representation learning~\citep{pair,feat}.
In this work, we open up a different setting with LLMs, which we termed as \textit{\oursElicit}, to iteratively augment tabular data with constructed interpretable factors to improve causal structure identification.

\runin{Causal reasoning with LLMs} This topic aims to analyze the causal relations implied by the semantics given a text.
One research line is \textit{text mining of causal relations}~\citep{Drury2022ASO,cui2024odyssey,hosseini2021predicting}. In particular, some works focus on the identification of causal relations among events specified in documents~\citep{gao2019modeling,weber2020causal,liu2021everything}. 
Another line is to \textit{understand the causal reasoning ability of LLMs}. \citet{causal_llm_frontier} find that LLMs can recover the pairwise causal relations based on their textual description. \citet{llm_anticausal} find that it is crucial to incorporate prompts aligned with the underlying causal story for LLMs to do pairwise causal relation inference.
\citet{willig2022can,causalparrots} find that LLMs can not understand causality but simply retell the causal knowledge contained in the training data.

Recent works explore causal reasoning with LLMs in trustworthy, explainable, and high-stake scenario~\citep{Han2025TrustworthyML,zhou2025explainable,li2025survey}. \citet{chi2024unveiling} propose goal-oriented causal reasoning with retrieval-augmented generation system to involve external reliable knowledge. \citet{khatibi2025cdf} incorporate causal insights to refine the query iteratively for multi-hop reasoning.
In this work, instead of reasoning about existing causal knowledge, \ours focuses on learning novel causal insights from data, which is complementary to these approaches.

\runin{Causal learning with LLMs} This topic aims to draw novel causal insights on data with LLMs.
Preliminary studies~\citep{understand_causality_llm,corr2cau,cladder,causalparrots} find that LLMs can hardly provide satisfactory answers for discovering new knowledge or decision-making tasks. \citet{wu2025llm} suggest that LLMs should be a non-decisional support in causal discovery.
One promising direction is to incorporate the causal discovery results by LLMs as a prior or constraint~\citep{lmprior,imperfect_llm_expert,llm_query_tool,cma} to improve the performance of data-driven causal discovery algorithms.
\citet{vashishtha2023causal} find that causal order is a more stable interface for utilizing imperfect expert knowledge from LLMs. 
There are also some active learning approaches. \citet{lampinen2023passive} shows that transformer-based agents can learn causal strategies passively if intervention is allowed during tests. \citet{li2025can} incorporates LLMs to augment numerical causal discovery by intervention targeting. \citet{wang2025causal} develops the Causal Copilot framework to enable rigorous causal analysis for non-experts.

Despite promising advances in integrating LLMs with causal discovery, current approaches still focus on artificially curated structured data with a fixed variable set, and the potential of paired unstructured data in each sample is less explored.
In this work, \ours focuses on actively extending the variables from unstructured data. As we shall show later, this approach can reliably improve the structure identification with guarantees.

\section{The \oursfull Framework}
\label{sec:setting}

In this section, we present the \ours framework, which leverages an LLM to serve as a representation assistant and actively elicits the desired causal variables on unstructured data. 
The key drive for the \oursElicit is based on the feedback constructed from the intermediate causal discovery results with CD methods.
We will begin with revealing the underlying Markov Blanket for a given target variable, and then extend to Partial Ancestral Graph, and the desirable Adjustment Set for treatment effect estimation.

\runin{Notation}
In this paper, random variables are denoted by capital Roman letters (e.g., $U_1, U_2, \dots, U_m$); random vectors are expressed by bold capital Roman letters (e.g., $\mathbf{U}$); and sets of random variables are written by calligraphic letters (e.g., $\mathcal{U}$). Consistency in the base letter reflects their relationship: for instance, $\mathbf{U} = [U_1, U_2, \dots, U_m]^\top$ and $\mathcal{U} = \{U_1, U_2, \dots, U_m\}$. We summarize the symbols with special meanings in this paper in Appendix~\ref{appdx:notation}.

For presenting edges in PAGs, this paper follows the practice in \citet{ZHANG20081873} and \citet{perkovi2018complete}. Each end of one edge is associated with one of three possible marks: tail ($-$), arrow ($>$), and circle ($\circ$). For example, edges can be bi-directed $\leftrightarrow$, directed $\rightarrow$, partially directed $\circrightarrow$, and undirected $\circarrowcirc$. An asterisk ($*$) would be used in a rule if it can be applied to any of the three marks.

\edit{
\subsection{Preliminary and Overview}

In this paper, we aim to reveal three types of causal representations that are of practical value for discovery from unstructured data: Markov Blanket, Partial Ancestral Graph, and Adjustment Sets, respectively. 
In these settings, the observed unstructured data is denoted as $\mX \in \mathcal{X}$ (e.g., natural language, images) and an LLM serves as a transformation function $\Psi: \mathcal{X} \to \mathcal{Z}$, parameterized by a set of prompts. Let ${\mZ} := \Psi(\mX;\mathbf{p})$ be the structured vector converted from an unstructured data $\mX$ according to prompt $\mathbf{p}$.

\runin{Markov Blanket } 
The \emph{Markov Blanket} (MB) of a target variable $Y$\minor{~(here $Y$ is an arbitrary variable of interest)}, denoted $\mathbf{MB}(Y)$, is the minimal subset of factors such that $Y$ is conditionally independent of others:
\begin{equation}
    Y \perp (\mathcal{Z} \setminus \mathbf{MB}(Y)) \mid \mathbf{MB}(Y).
\end{equation}
Equivalently, $\mathbf{MB}(Y)$ comprises the parents, children, and spouses (parents of common children) of $Y$ in the underlying causal DAG. 

\oursMB leverages the LLM to iteratively analyze data $\mathcal{D} \subseteq \{(\mathbf{x}_i, y_i)\}_{i=1}^n$, and \minor{suggests} new factors by itself.  The goal is to ensure $\mathbf{MB}(Y) \subseteq \mZ$ such that no factor outside the set of the discovered factors provides additional information on $Y$.

\begin{figure}[t]
    \centering
    \def\imgH{0.15\textheight}
    \subfigure[partial ancestral graph]{
        \begin{minipage}[c]{0.3\linewidth}
            \centering
            \includegraphics[height=\imgH]{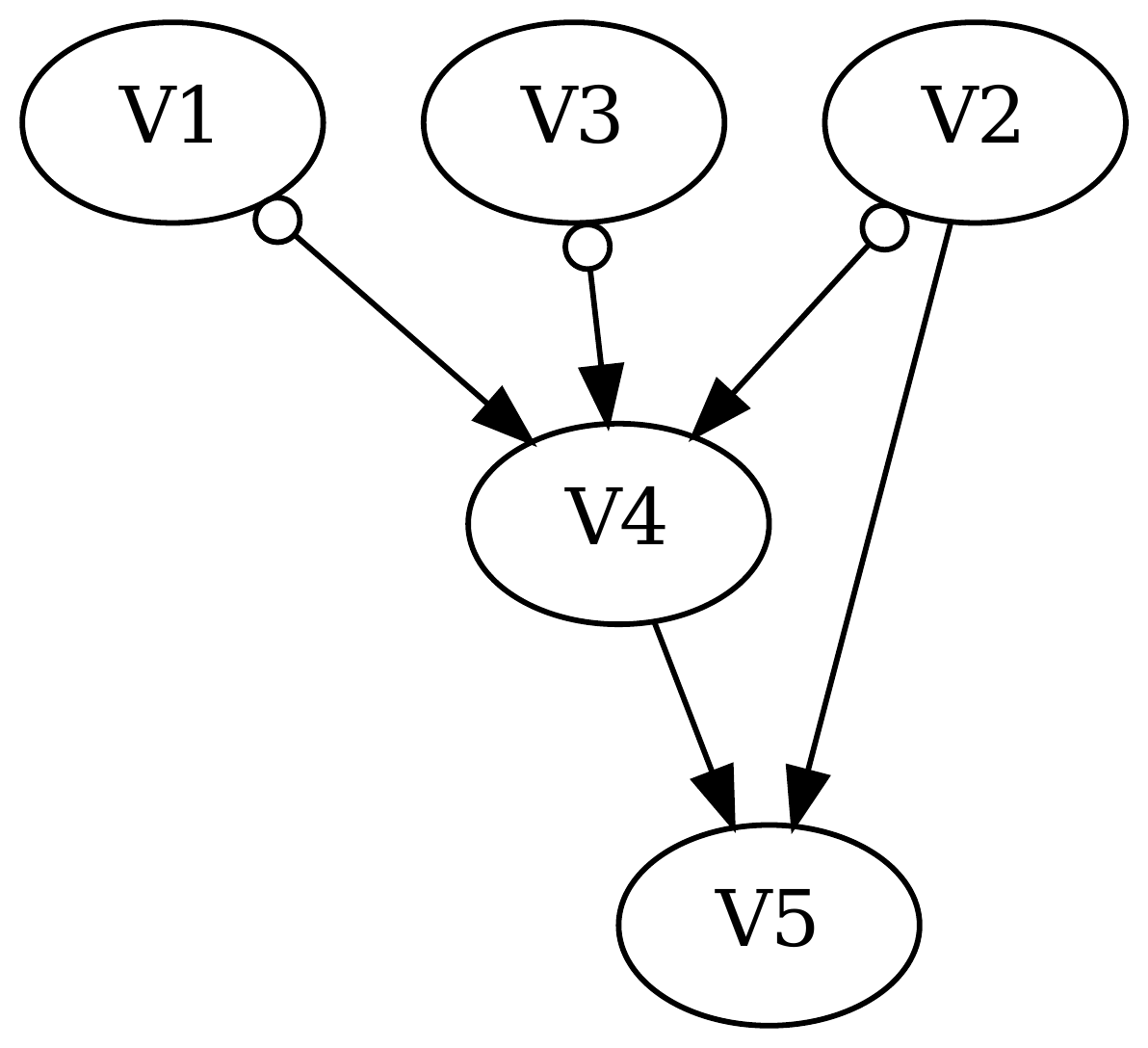}
        \end{minipage}
        \label{fig:pag-examples-p}
    }%
    \subfigure[mixed graph]{
        \begin{minipage}[c]{0.3\linewidth}
            \centering
            \includegraphics[height=\imgH]{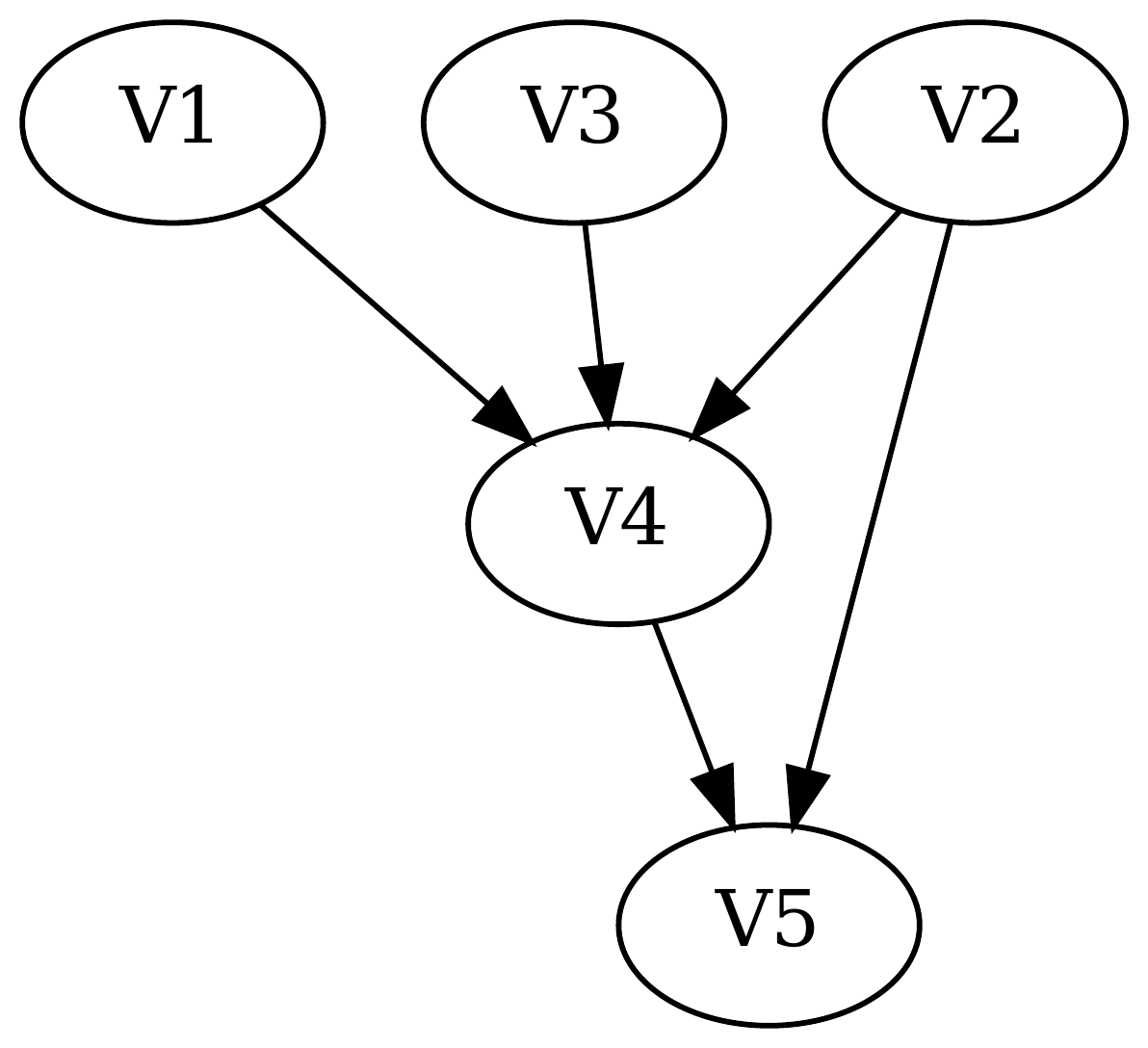}
        \end{minipage}
        \label{fig:pag-examples-m1}
    }%
    \subfigure[mixed graph]{
        \begin{minipage}[c]{0.3\linewidth}
            \centering
            \includegraphics[height=\imgH]{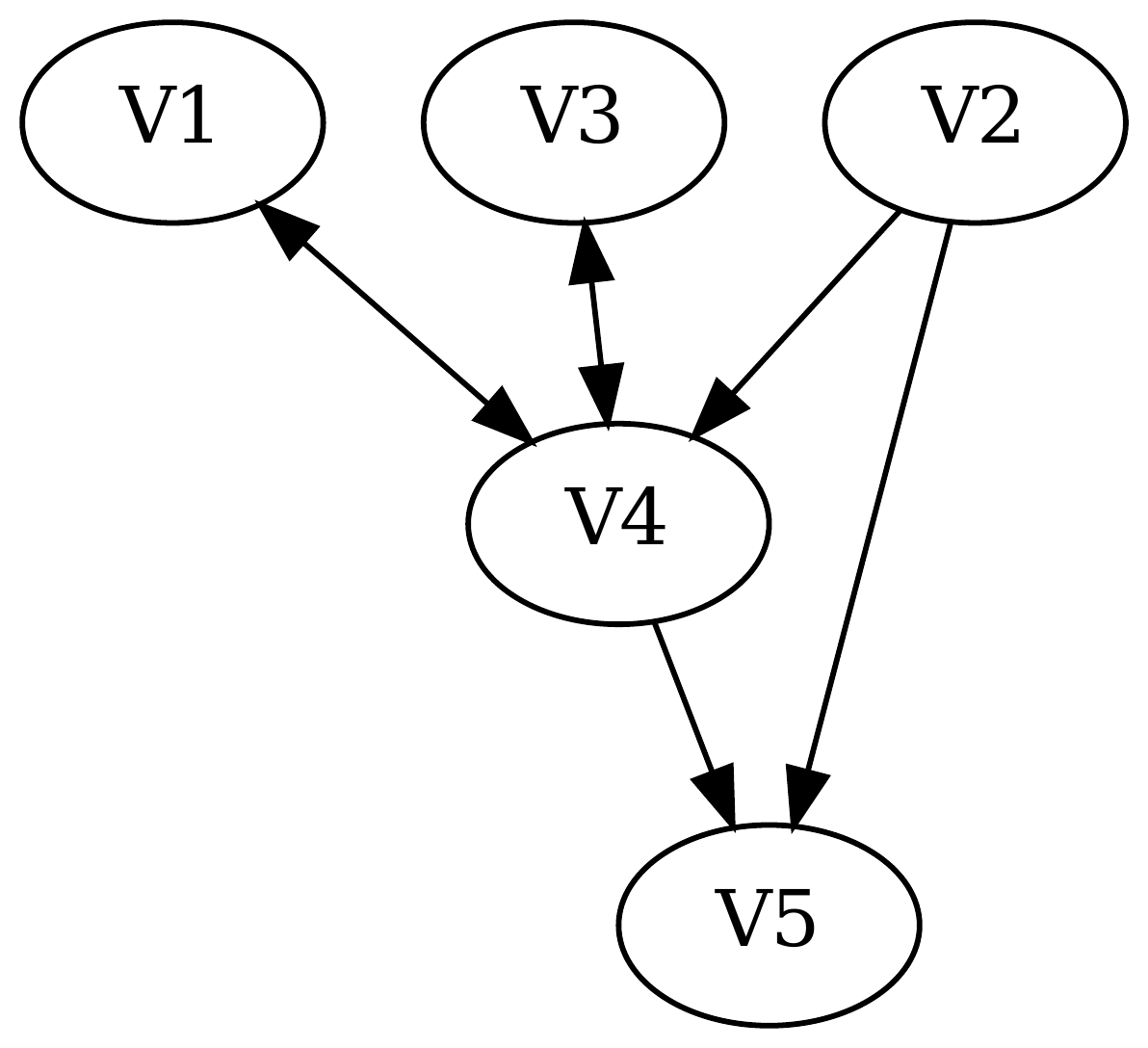}
        \end{minipage}
        \label{fig:pag-examples-m2}
    }

    \caption{\edit{Examples of Partial Ancestral Graph and two Markov-equivalent mixed graphs.}}
\end{figure}
 
\runin{Partial Ancestral Graphs} 
A \emph{Maximal Ancestral Graph} (\textbf{MAG}) is typically used to represent the causal structure among a set of observed variables $\mathbf{O}$ when there are possibly unobserved latent variables $\mathbf{L}$. As shown in Figure~\ref{fig:pag-examples-m1} and~\ref{fig:pag-examples-m2}, it has following edge types: \textit{directed} edge ($\rightarrow$) indicates clear causal relations; 
\textit{bi-directed} edge ($\leftrightarrow$) indicates latent confounder between two variables.
A \emph{Partial Ancestral Graph} (\textbf{PAG}) $\mathcal{P}_{\mU}$ over variables $\mathbf{O}$ is used to represent a Markov equivalent class of MAGs that share the same set of conditional independencies. It has two additional edge types: \textit{partially directed} edge ($\circrightarrow$) for possible causal relations, i.e., could be directed or bi-directed; 
\textit{non-directed} edge ($\circarrowcirc$) for the uncertain relations. The circle ($\circ$) mark means that there exist some Markov equivalent MAGs with different edge marks. 
For example, one can find the circle mark at $V1$ at Figure~\ref{fig:pag-examples-p} is due to the inconsistent marks in the two MAGs.
When given the numerical samples of $\mathbf{O}$, the Fast Causal Inference~\citep{fci} algorithm constructs $\mathcal{P}$ by performing conditional independence tests.

In \oursPAG, given an initial structured variables $\mU$, the algorithm extends the node set to $\mU':=\mU \cup \mZ$ by proposing new factors, where each factor can be treated as a function of unstructured data $\mX$. The goal is to maximally reveal the causal information among $\mU$. For example, some non-directed edge in $\mathcal{P}_{\mU}$ could be directed in $\mathcal{P}_{\mU'}$.

\runin{Adjustment Sets} 
One important question in causal inference is to estimate the causal effect from a treatment variable $T$ to an outcome variable $Y$ given a causal graph $\gG$ over a node set $\mV$ with $\{T, Y\} \subseteq \mV$. The association between $T$ and $Y$ can be confounded by other variables, e.g., a confounder $V_1$ that is a shared causal parent of both $T$ and $Y$ (i.e., $T \leftarrow V_1 \rightarrow Y$). In such cases, instead of estimating the raw association, we condition on an \emph{adjustment set} to block non-causal paths while preserving causal pathways. For instance, adjusting for $V_1$ can be conducted by \textit{inverse probability weighting}~\citep{robins2000marginal}: one models the propensity score $\pi(V_1) = P(T=1 \mid V_1)$
and reweights each observation by $ w = \frac{T}{\pi(V_1)} + \frac{1-T}{1-\pi(V_1)}$
 to create a pseudo-population in which $T$ is independent of $V_1$. The average treatment effect (ATE) can then be estimated as
\[
\hat{\tau}_{\text{ATE}} = \frac{1}{n} \sum_{i=1}^{n} \left( \frac{T_i Y_i}{\pi(V_{1i})} - \frac{(1-T_i)Y_i}{1-\pi(V_{1i})} \right),
\]
which provides a consistent estimator of $\mathbb{E}[Y \mid do(T=1)] - \mathbb{E}[Y \mid do(T=0)]$ under correct propensity score specification and positivity.

In \oursADJ, given $(T, Y) \subseteq \mU$ and unstructured data $\mX$, we discover factors $\mZ$ from $\mX$ to construct valid adjustment sets within $\mV := \mU \cup \mZ$. Based on the estimated causal structure $\mathcal{P}_{\mV}$, one could check identifiability conditions (e.g., the absence of unobserved confounding along back-door paths) and select adjustment sets using generalized graphical criteria. In this paper, we focus on the \textit{Generalized Graphical Criteria}~\citep{perkovi2018complete}, please see Definition~\ref{def:Generalized-adjustment-criterion} for details.
}

\subsection{Active Causal Representation Elicitation}
\label{sec:method}
\label{sec:problem_def}

\runin{Data}
We are given \textit{unstructured data} (e.g., text or images) as a high-dimensional random vector $\mX$, and \textit{structured data} (e.g., stars rated by a customer or the tumor type of a patient) as a set of random variables of interest $\varset{U} = \{U_1, \dots, U_m\}$. The input data set $\gD = \{\vx^{(j)}, \vu^{(j)}\}^n_{j=1}$ contains $n$ samples independently drawn from the joint distribution over $(\mX,\mU)$, where $\mU$ is an $m$-dim random vector that consists of variables in $\varset{U}$. There is no prior assumption on the relations between $\mX$ and $\mU$. 

\runin{LLM as a representation assistant} 
Given the unstructured data $\mX$, we aim to make use of the rich knowledge of LLMs to assist in converting meaningful unstructured information in a structured form. Formally speaking, we seek a mapping $h: \mX \mapsto \mW$ which serves as a \textit{structured} representation $\mW=h(\mX)$.
We present the following definition:

\begin{definition}[Representation Assistant]\label{def:llm-ra}
A \textbf{Representation Assistant} for the unstructured data $\mX$ with domain $\gX$ is defined by a set of high-level factors $\varset{W} \triangleq \{W_1, \dots, W_k\}$. Each factor $W_i$ is induced by a prompt instruction (e.g., a natural language query), which specifies a deterministic mapping $\vw_i: \gX \to \gC_i$ via an LLM $\Psi$. Here, $\gC_i \subseteq \R$ is a predefined numerical set (discrete or continuous), and $W_i = \vw_i(\mX)$ is a random variable with range $\gC_i$. The structured representation is the concatenation:
$$
h\big(\mX\big) = \mW = \Big( \, W_1, \, W_2, \, \cdots \, , W_k \, \Big).
$$
\end{definition}

\begin{figure*}[t]
  \centering
  \includegraphics[width=\textwidth]{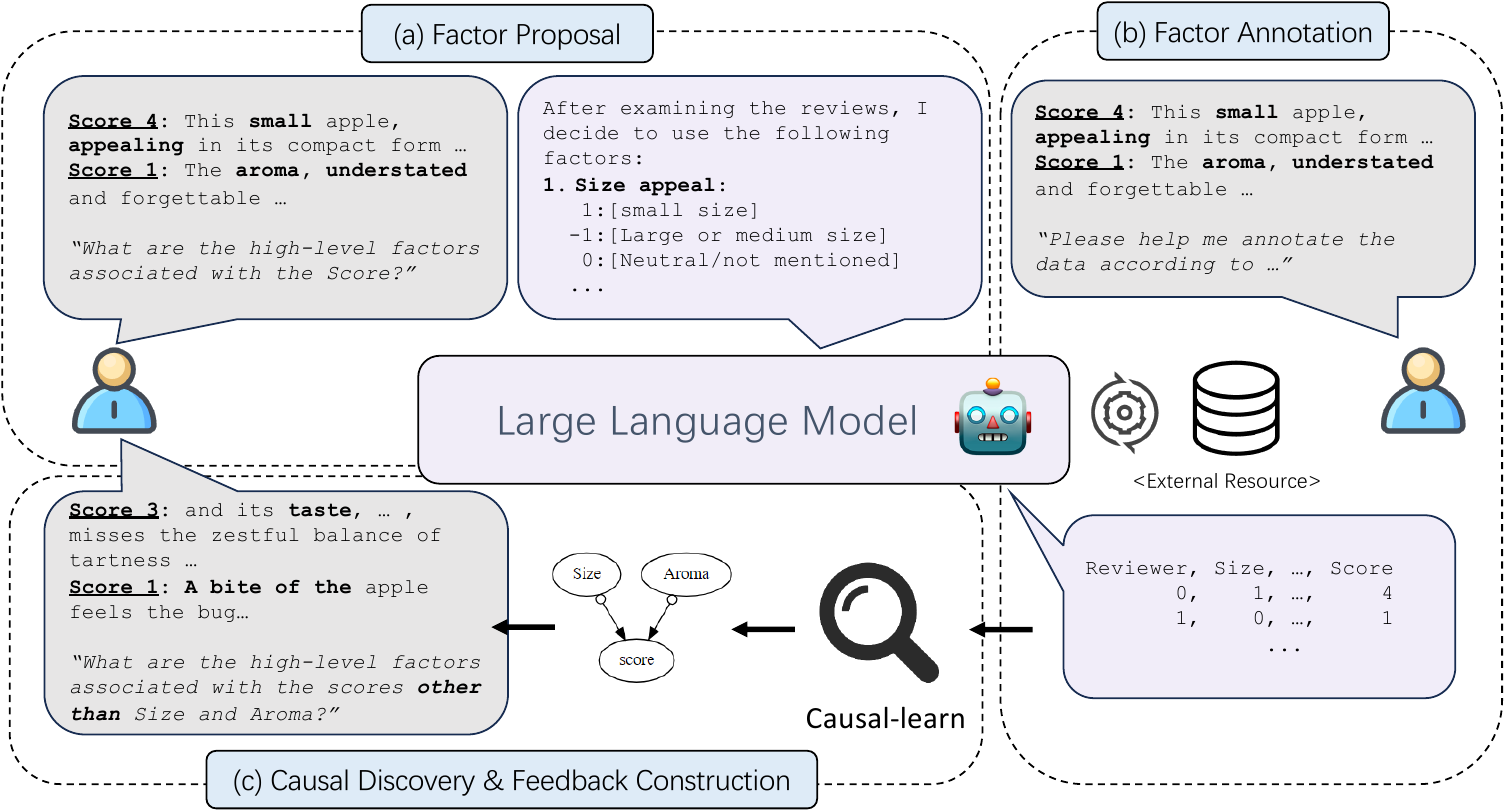}
  \caption{Illustration of \ours framework in eliciting the Markov Blanket of the rating score from the paired text reviews.
  \ours first (a) adopts an LLM to read, comprehend, and relate the rich knowledge about reviews on apples. The LLM needs to propose a series of candidate factors such as apple sizes and smells, along with some meta-information such as annotation guidelines. 
  Based on the candidate factors, \ours then (b) prompts another LLM to annotate the unstructured review into structured data. (c) The CD algorithm then finds causal relations among the factors, and constructs feedback based on samples where the ratings can not be well explained by the existing factors. By looking into the new samples, the LLM is expected to use more related knowledge to uncover the desired causal factors. 
  }
  \label{fig:illustration}
\end{figure*}

Each factor is defined by a specific prompt instruction so that an LLM can decide a value based on a sample $\vx$ (drawn from $\mX$).
For example, given a customer review on an apple $\vx^{(j)}$, an LLM-proposed factor $\vw_i$ with $\gC=\left\{ -1, 0, 1\right\}$ could be described as: $\vw_i(\vx^{(j)})=1$ if the customer appreciates the sweetness of the apple; $\vw_i(\vx^{(j)})=-1$ if  the customer is disappointed about the sweetness; and $\vw_i(\vx^{(j)})=0$ if the sweetness has not been mentioned. In addition, a factor instruction can require LLM to utilize external tools for further processing. 
\runin{Active causal representation elicitation}
In this work, we propose \oursfull (\ours) framework to actively elicit representations consisting of factors $\varset{W}$ to reveal causality in different settings. %
\begin{itemize}
    \item Given a target variable $Y \in \varset{U}$, \oursMB (Section~\ref{sec:causal_feedback}) finds factors $\varset{W}$ that serves as a Markov Blanket of $Y$ with respect to $\mX$, i.e., $Y \ind \mX \mid \varset{W}$. That is, $\varset{W}$ captures all relevant information about $Y$ from $\mX$;
    \item Furthermore, based on \oursMB, \oursPAG (Section~\ref{subsec:introduce-coat-pag}) extends the structured input set to $\varset{V} \triangleq \varset{W} \cup \varset{U}$ so that the causal information, i.e., skeleton or arrow heads, among  $\varset{U}$ can be maximally revealed in the partial ancestral graph $\gP_\varset{V}$ with the node set $\varset{V}$;
    \item In causal inference setting, for a treatment-effect pair $T,Y \in \varset{U}$, \oursGAS (Section~\ref{subsec:intro-coat-adj}) extends structured input set to $\varset{V} \triangleq \varset{W} \cup \varset{U}$ so that valid adjustment sets within $\varset{U}$ for the causal effect on $T \rightarrow Y$ can be found by graphical criterion on the partial ancestral graph $\gP_\varset{V}$ with the node set $\varset{V}$. 
\end{itemize}

To introduce the key components of \ours framework, we take the illustration in Figure~\ref{fig:illustration} as a running example, where \ours aims to elicit useful factors to serve as a Markov Blanket of a target variable $Y$ (rating score) with respect to text reviews through multiple rounds of iteration. 
We use the superscript $t$ to refer to the input and output of LLM at the $t$-th round. We denote the union of results from the first $t$ rounds by the superscript ``$\le t$''.

\runin{Factor proposal} To obtain useful high-level factors from the rich world knowledge of LLMs, \ours employs a prompt $\vp$ and a few samples $\widehat{\gD} \subseteq \gD$. 
Specifically, in Figure~\ref{fig:prompt_example}, $\vp$ contains three components: \emph{samples}, \emph{instructions}, and \emph{format control}. To encourage LLM to focus on the information related to the target variable $Y$, \textit{samples} are grouped by the value of $Y$. The \textit{instruction} requires $\Psi$ to give each proposed factor $\vw_i$ a concrete description of the mapping $\vw_i(\cdot)$, like how to decide the factor values. 
In addition, the metadata about the task, such as the task description and context, can also be incorporated if available.
The prompt $\vp$ essentially imitates the typical process where human experts abstract and define high-level variables. 
The set of resulting factors in the $t$-th round, defined with natural language by the LLM $\Psi$, is denoted as $\gW^t  = \Psi(\vp^t, \widehat{\gD}^t)$. We merge them with factors proposed in the previous rounds to update the set of all factors $\gW^{\le t}  = \gW^1 \cup \cdots \cup \gW^t$.

\begin{figure*}[t]
    \centering
    \vspace{-0.1in}
    \includegraphics[width=1\textwidth]{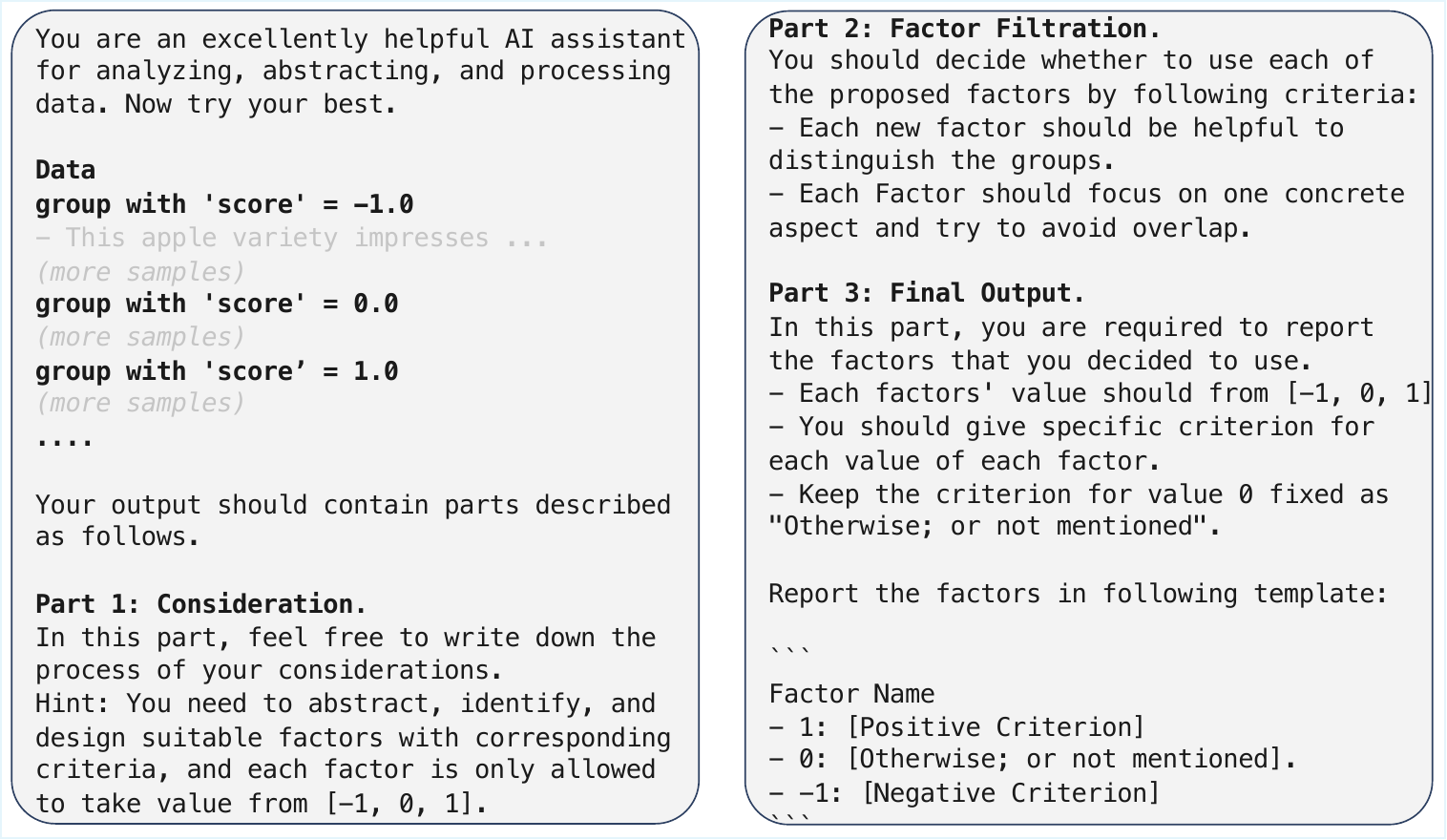}
    \vspace{-0.1in}
    \caption{Illustration of the prompt template for factor proposal in \ours.}
    \label{fig:prompt_example}
\end{figure*}

\runin{Factor parsing}
Once we obtain the candidate factors, we then collect the values of the factors from the unstructured observations. In prior works, they are usually collected from human experts according to the given factors~\citep{causation_1st}.
To do so, another LLM $\Psi_s$ is instructed to read the annotated guidelines of each variable $\vw_i$ and parse the unstructured observations into structured or tabular data, with its entity at the $i$-th column and the $j$-th row:
\begin{equation}\label{eq:factor_parsing}
    \vw_i\big( \, \vx^{(j)} \, \big) := \Psi_s\big( \, \vx^{(j)}, \, \vw_i, \, \vp_p \, \big),
\end{equation}
where $\vp_p$ refers to the additional instruction to prompt $\Psi_s$ to parse the observed data.%

When the data curation of the proposed factors requires additional domain-specific knowledge/skills (e.g., intervening on the external environments) that the LLMs do not acquire, we could fetch $\vw_i(\vx^{(j)})$, the value of each factor $\vw_i$ for each sample $\vx^{(j)}$, through some external process~\citep{toolformer,llm_agent}. For example, studying the causes of a disease requires annotating relevant symptoms from diagnosis records and conducting additional medical checks~\citep{neuropanic}.
Our experiments show that \ours can effectively extract the hidden factors under both schemes.

\runin{Causal discovery} Given the candidate factors $\varset{W}^{\le t}$ with their associated factor values on data set $\gD$, a CD algorithm $\gA$ (e.g., FCI~\citep{fci}) is used to reason about the causal structure based on the parsed data:
\begin{equation}
    {\gG}^t=\gA \\ \Big( \ \varset{W}^{\le t}\cup\{Y\} \ \Big),
\end{equation}
where ${\gG}^t$ is the discovered causal structure. In general, any CD approach with suitable theoretical assumptions could be used for $\gA$. More detailed theoretical discussions are presented in Section~\ref{subsec:MB-identification}. In this work, we demonstrate the idea of \ours via the FCI algorithm~\citep{auto_causal} as it is flexible with respect to different functional classes of the underlying generation process, allowing for the existence of latent confounders~\citep{fci}, which aligns well with our objective.

\runin{Improving factor proposal with causal feedback} 
LLMs require proper prompts to fully unlock their capabilities~\citep{cot,llm_self-improve,prompt_reasoning_survey}. 
When it comes to factor proposing, it is also hard for LLMs to propose all factors at once. 
Nevertheless, from the causal discovery results, we could find useful information and thus provide feedback to further improve the factor proposal:
\begin{equation}\label{eq:feedback}
    \Big( \, \widehat{\gD}^{t+1}, \, \vp^{t+1} \, \Big)= \gF\Big( \, {\gG}^t, \, \gD, \, \vp^t \, \Big),
\end{equation}
where $\gF$ samples specific examples from $\gD$ and constructs new prompts according to the results of $\gA$ for the next round of factor proposal.
For example, FCI is able to imply the existence of latent confounders, from which we could refine $\vp$ to prompt $\Psi$ to focus on the corresponding factors.

\begin{figure}
    \vspace{-0.1in}
    \centering
    \subfigure[Parent]{
        \begin{minipage}[c][2.5cm][c]{0.22\linewidth}
            \centering
            \begin{tikzpicture}[
                baseline=(current bounding box.center),
                node distance=0.6cm,
                every node/.style={circle, draw, minimum size=6mm, inner sep=1pt},
                edge/.style={->, >=Latex, semithick},
                observed/.style={circle, draw, fill=black!10, minimum size=6mm, inner sep=1pt},
                oedge/.style={line width=0.8pt, decoration={markings, 
                    mark=at position 0.1 with {\node[circle, draw, fill=white, inner sep=0pt, minimum size=4pt] {};}, 
                    mark=at position 0.9 with {\node[circle, draw, fill=white, inner sep=0pt, minimum size=4pt] {};}}, 
                    postaction=decorate},
                circnormaledge/.style={line width=0.8pt, decoration={markings, 
                    mark=at position 0.1 with {\node[circle, draw, fill=white, inner sep=0pt, minimum size=4pt] {};}, 
                    mark=at position 0.99 with {\arrow{Latex}}}, 
                    postaction=decorate, shorten >=2pt},
            ]
            \node (new) {$U$};
            \node[below right=0.6cm and 0.2cm of new] (Y) [observed] {$Y$};
            
            \draw[edge] (new) -- (Y);
            \end{tikzpicture}
        \end{minipage}
    }
    \subfigure[Child]{
        \begin{minipage}[c][2.5cm][c]{0.22\linewidth}
            \centering
            \begin{tikzpicture}[
                baseline=(current bounding box.center),
                node distance=0.6cm,
                every node/.style={circle, draw, minimum size=6mm, inner sep=1pt},
                edge/.style={->, >=Latex, semithick},
                observed/.style={circle, draw, fill=black!10, minimum size=6mm, inner sep=1pt},
                oedge/.style={line width=0.8pt, decoration={markings, 
                    mark=at position 0.1 with {\node[circle, draw, fill=white, inner sep=0pt, minimum size=4pt] {};}, 
                    mark=at position 0.9 with {\node[circle, draw, fill=white, inner sep=0pt, minimum size=4pt] {};}}, 
                    postaction=decorate},
                circnormaledge/.style={line width=0.8pt, decoration={markings, 
                    mark=at position 0.1 with {\node[circle, draw, fill=white, inner sep=0pt, minimum size=4pt] {};}, 
                    mark=at position 0.99 with {\arrow{Latex}}}, 
                    postaction=decorate, shorten >=2pt},
            ]
            \node (new) {$U$};
            \node[below right=0.6cm and 0.2cm of new] (Y) [observed] {$Y$};
            
            \draw[edge] (Y) -- (new);
            \end{tikzpicture}
        \end{minipage}
    }
    \subfigure[Confounded]{
        \begin{minipage}[c][2.5cm][c]{0.22\linewidth}
            \centering
            \begin{tikzpicture}[
                baseline=(current bounding box.center),
                node distance=0.6cm,
                every node/.style={circle, draw, minimum size=6mm, inner sep=1pt},
                edge/.style={->, >=Latex, semithick},
                observed/.style={circle, draw, fill=black!10, minimum size=6mm, inner sep=1pt},
                oedge/.style={line width=0.8pt, decoration={markings, 
                    mark=at position 0.1 with {\node[circle, draw, fill=white, inner sep=0pt, minimum size=4pt] {};}, 
                    mark=at position 0.9 with {\node[circle, draw, fill=white, inner sep=0pt, minimum size=4pt] {};}}, 
                    postaction=decorate},
                circnormaledge/.style={line width=0.8pt, decoration={markings, 
                    mark=at position 0.1 with {\node[circle, draw, fill=white, inner sep=0pt, minimum size=4pt] {};}, 
                    mark=at position 0.99 with {\arrow{Latex}}}, 
                    postaction=decorate, shorten >=2pt},
            ]
            \node (new) {$U$};
            \node[below right=0.6cm and 0.2cm of new] (Y) [observed] {$Y$};
            \node[above right=0.6cm and 0.2cm of Y] (W) [observed] {$W$};
            
            \draw[edge] (new) -- (Y);
            \draw[edge] (W) -- (Y);
            \draw[edge] (W) -- (new);
            \end{tikzpicture}
        \end{minipage}
    }
    \subfigure[Spouse]{
        \begin{minipage}[c][2.5cm][c]{0.22\linewidth}
            \centering
            \begin{tikzpicture}[
                baseline=(current bounding box.center),
                node distance=0.6cm,
                every node/.style={circle, draw, minimum size=6mm, inner sep=1pt},
                edge/.style={->, >=Latex, semithick},
                observed/.style={circle, draw, fill=black!10, minimum size=6mm, inner sep=1pt},
                oedge/.style={line width=0.8pt, decoration={markings, 
                    mark=at position 0.1 with {\node[circle, draw, fill=white, inner sep=0pt, minimum size=4pt] {};}, 
                    mark=at position 0.9 with {\node[circle, draw, fill=white, inner sep=0pt, minimum size=4pt] {};}}, 
                    postaction=decorate},
                circnormaledge/.style={line width=0.8pt, decoration={markings, 
                    mark=at position 0.1 with {\node[circle, draw, fill=white, inner sep=0pt, minimum size=4pt] {};}, 
                    mark=at position 0.99 with {\arrow{Latex}}}, 
                    postaction=decorate, shorten >=2pt},
            ]
            \node (new) {$U$};
            \node[below right=0.6cm and 0.2cm of new] (Y) [observed] {$Y$};
            \node[above right=0.6cm and 0.2cm of Y] (W) [observed] {$W$};
            
            \draw[edge] (Y) -- (W);
            \draw[edge] (new) -- (W);
            \end{tikzpicture}
        \end{minipage}
    }
    \caption{Illustration of variables that could be discovered with \ours. 
    Let $Y$ be the target variable, and $W$ be a factor that has been discovered, and also assume a latent variable $U=\widehat{\vw}(\mX) \in \MB(Y)$. Conditioning on $W$ facilitates the discovery of $U$. 
    }
    \label{fig:scm}
    \vspace{0.1in}
\end{figure}

\subsection{\ours for Markov Blanket} 
\label{sec:causal_feedback}
\label{subsec:coat-mb}
\begin{algorithm}[t]
        \caption{\ourst for Markov Blanket (\oursMB)}
        \label{alg:coat-main}
    \textbf{Input:} Data set $\gD = \big\{ \, (\vx^{(j)}, \, y_j ) \, \big\}_{j=1}^{n}$; LLM for factor proposal $\Psi$; Model for factor parsing $\Psi_s$; causal discovery algorithm $\gA$; Maximal rounds $t_{\text{Max}}$.
    \begin{enumerate}
        \item Let $\gM = \emptyset$ be the initial Markov Blanket. Let $\varset{W}=\Big\{ \,  W_1, \, \cdots \, , \, W_m \Big\}$ be the set of all the proposed factors from $\mX$. Initially, $m=0$.%
        \item Repeat with loop index $t$ starting from $1$, stop when $t=t_{\text{Max}}$ or $\gM$ is unchanged:
            \begin{enumerate}
                \item Do clustering on samples with $\varset{M}$; use random partition if $\varset{M}$ is empty.
                \item Take $\widehat{\gD}^t$ as the cluster with highest conditional entropy w.r.t. $Y$.
                \item LLM $\Psi$ propose new factors $W_{m+1}, \cdots, W_{m'}$, then update $\varset{W}$ and set $m=m'$.
                \item For each new factor, LLM $\Psi_s$ parse out its values on $\big\{\vx^{(j)}\big\}_{j \in [n]}$.
                \item if $\varset{M}$ is not $Y$'s Markov Blanket w.r.t. $\varset{W}$, update $\gM$ with $\gA$.
            \end{enumerate} 
        \item Get causal graph $\gG$ over $\{Y\}\cup\varset{W}$ with $\gA$.
    \end{enumerate}
    \textbf{Output:} Factors $\varset{W}$, Markov Blanket set $\varset{M}$, causal graph $\gG$.
\end{algorithm}

In this section, \edit{we restate our previous result~\citep{causalcoat2024}} to identify a set of \textit{structured} high-level factors $\varset{M}$ from \textit{unstructured} observations $\mX$ that serve as a Markov Blanket of the given target variable $Y$. The basic idea is intuitive: we iteratively verify the proposed factors from LLMs, and, in the meantime, encourage LLMs to propose diverse factors based on their observations. All proposed factors will be put into the factor pool $\varset{W}$ to serve as proxies of the unstructured variable $\mX$ in statistical testing. After each iteration, we update the choice of Markov Blanket $\varset{M} \subseteq \varset{W}$ with respect to the current factor pool, i.e., $Y \ind \varset{W} \setminus \varset{M}  \mid \varset{M}$.

The concrete procedure is stated in Algorithm~\ref{alg:coat-main}. In the $t$-th loop, we first conduct clustering on the current Markov Blanket set. The intuition is to strengthen the correlation between $Y$ and the potential factors. If there exists such a potential factor, say, $\widehat{\vw}$, that means: 
\begin{equation}\label{eq:markov_property}
     H\big( \, Y \, | \, \varset{M} \, \big) > H\Big(Y \, | \, \varset{M}, \, \widehat{\vw}(\mX)\Big),
\end{equation}
where $H(\cdot)$ refers to the entropy. 
This is because $Y$ should be independent of other factors from $\mX$ conditioning on the Markov Blanket.
As shown in Figure~\ref{fig:scm}, finding factors satisfying 
Inequality~\ref{eq:markov_property} progressively expands the discovered factors and pushes $\varset{M}$ to a valid Markov Blanket. 
Therefore, to find the desired factor, we are motivated to select a suitable $\widehat{\gD}^{t}$ for the current iteration such that $\widehat{\gD}^{t}=\text{$\argmax$}_{\widehat{\gD}\subseteq \gD} H_{\widehat{\gD}} (Y|\varset{M}),$
where $\varset{M}$ can not well explain $Y$. %
This problem can be interpreted as selecting the hardest data set $\widehat{\gD}^*$ for predicting $Y$ with existing factors.
In our experiments, we implement the classification via clustering with respect to $\varset{M}$. The clustering elicits $C$ groups $\widehat{\gD}_c:=\big\{\vx_i \, \, \text{for} \, \, i \in \gI_c\big\}$: $\gI_1, \cdots, \gI_C = \text{K-Means}\big(\varset{M}\big)$.
We then take the group of samples with the largest conditional entropy to construct the feedback. In the rest of the $t$-th loop, we utilize the \ours components stated in Section~\ref{sec:method} to instruct LLMs to propose candidate factors and then parse out their values on each observation. The Markov Blanket can be obtained by statistical testing or the causal graph from $\gA$.

In practice, many factors, such as the LLM capabilities, data faithfulness, and prompt templates, could affect the success of Algorithm~\ref{alg:coat-main}. Therefore, in Section~\ref{subsec:MB-identification}, we formalize the required assumptions on the LLMs' ability, and establish a theoretical guarantee on the factor identification for Markov Blanket.

\subsection{\ours for Partial Ancestral Graph} 
\label{subsec:introduce-coat-pag}

The causal information may not be maximally revealed due to insufficient node set. For example, as we have seen in Figure~\ref{fig:Q2_example}, we could further reveal 
the skeleton and arrow heads among $\text{MB}(Y)$ if we can include one additional factor $X_3 \notin \text{MB}(Y)$. 

\begin{figure}[t]
    \centering
    \subfigure[Ground-truth causal graph]{
        \begin{minipage}[c][3cm][c]{.4\linewidth} %
            \centering
            \vspace*{-0.2in} %
            \begin{tikzpicture}[
                baseline=(current bounding box.center),
                node distance=0.6cm,
                every node/.style={circle, draw, minimum size=6mm, inner sep=1pt},
                edge/.style={->, >=stealth, semithick},
                dottededge/.style={->, >=stealth, semithick, dotted}
            ]
            \node (BN) {$B_N$};
            \node[right=of BN] (B1) {$B_1$};
            \node[right=of B1] (Y) {Y};
            \node[right=of Y] (A) {A};
            \node[above left=0.2cm and 0.6cm of BN] (C) {C};
            \node[below left=0.2cm and 0.6cm of BN] (D) {D};
            \draw[edge] (C) -- (BN);
            \draw[edge] (D) -- (BN);
            \draw[dottededge] (BN) -- (B1);
            \draw[edge] (B1) -- (Y);
            \draw[edge] (Y) -- (A);
            \end{tikzpicture}
        \end{minipage}
        \label{fig:Example_for_motivating_coatpag_1}
    }
    \subfigure[PAG over $\MB^{(N)}(Y)$]{
        \begin{minipage}[c][3cm][c]{.4\linewidth} %
            \centering
            \vspace*{-0.2in} %
            \begin{tikzpicture}[
                baseline=(current bounding box.center),
                node distance=0.6cm,
                every node/.style={circle, draw, minimum size=6mm, inner sep=1pt},
                oedge/.style={line width=0.8pt, decoration={markings, 
                    mark=at position 0.1 with {\node[circle, draw, fill=white, inner sep=0pt, minimum size=4pt] {};}, 
                    mark=at position 0.9 with {\node[circle, draw, fill=white, inner sep=0pt, minimum size=4pt] {};}}, 
                    postaction=decorate}
            ]
            \node (BN) {$B_N$};
            \node[right=of BN] (B1) {$B_1$};
            \node[right=of B1] (Y) {Y};
            \node[right=of Y] (A) {A};
            \draw[oedge] (BN) -- (B1);
            \draw[oedge] (B1) -- (Y);
            \draw[oedge] (Y) -- (A);
            \end{tikzpicture}
        \end{minipage}
        \label{fig:Example_for_motivating_coatpag_2}
    }
    \caption{An example where arrow heads cannot be revealed within $\MB^{(N)}(Y)$.}
    \label{fig:motivating_coatpag}
\end{figure}

One straightforward approach is to apply \oursMB recursively to the discovered variables to obtain $\text{MB}^{(2)}(Y)$, i.e., the Markov Blanket of the Markov Blanket of $Y$. 
Although it helps revealing the skeleton information, the arrow heads information can still be missing. Essentially, for any positive integer $N$, there exists a causal graph whose arrow heads cannot be revealed by obtaining $\text{MB}^{(N)}(Y)$: as shown in Figure~\ref{fig:motivating_coatpag}(a), simply considering a chain structure $B_N \rightarrow \cdots B_1 \rightarrow Y \rightarrow A $ with a V-structure at the end away from $Y$: $C \rightarrow B_N \leftarrow D$. When only $\text{MB}^{(N)}(Y)$ are observed, as shown in Figure~\ref{fig:motivating_coatpag}(b), $C$ and $D$ are latent variables; and thus there is no V-structure in the PAG restricted on $\text{MB}^{(N)}(Y)$. Consequently, only the skeleton can be identified.

Motivated by the observation above, we propose a method called \oursPAG to adaptively decide on which node we should apply \oursMB to expand the current node set by its Markov Blanket.
As we stated in Algorithm~\ref{alg:coat-search}, the purpose of \oursPAG is to ensure both the skeleton and the arrow heads within the initial factor set can be maximally revealed in the PAG on the output factor set.
In the first step, it simply expands the factor set $\varset{U}$ to $\MB(\varset{U})$ to ensure the correctness of the skeleton. In the following steps, it conducts Breadth-First Search to apply \oursMB on the nodes that are adjacent to an undetermined edge head ($\circ$).

\begin{algorithm}[t]
        \caption{\ourst for Partial Ancestral Graph (\oursPAG)}
        \label{alg:coat-search}
    \textbf{Input:} Data set $\gD = \big\{ \, (\vx^{(j)}, \, \vu^{(j)}) \, \big\}_{j=1}^{n}$; LLM for factor proposal $\Psi$; Model for factor parsing $\Psi_s$; causal discovery algorithm $\gA$; Maximal rounds $t_{\text{Max}}$. %
    \begin{enumerate}
        \item Initialize $\varset{W}$ with $\MB(\varset{U}) \setminus \varset{U}$ by performing Algorithm~\ref{alg:coat-main} with $\Psi$, $\Psi_s$, $\gA$, and $t_{\text{Max}}$.
        \item Let $\gP_\varset{V}$ be the partial ancestral graph over $\varset{V}\triangleq\varset{U} \cup \varset{W}$, and initialize $\varset{S}_{\text{mark}} = \varset{U}$.
        \item Let $\varset{S}_{\text{target}}=\Big\{ \, \alpha \in \varset{V} \setminus \varset{S}_{\text{mark}} \mid \text{$\gP_\varset{V}$ contains edges like } \alpha \circ \! \! - \! \! * \beta \text{ for some $\beta \in \varset{V}$} \, \Big\}$
        \item Repeat with loop when $\varset{S}_{\text{target}}$ is not empty:
            \begin{enumerate}
                \item Take one arbitrary element $\alpha$ from $\varset{S}_{\text{target}}$.
                \item Extend $\varset{W}$ with $\text{MB}(\alpha) \setminus \varset{V}$ by performing Algorithm~\ref{alg:coat-main} on $\alpha$.
                \item Put $\alpha$ into $\varset{S}_{\text{mark}}$, then update $\gP_\varset{V}$ and $\varset{S}_{\text{target}}$.
            \end{enumerate}
    \end{enumerate}
    \textbf{Output:} Factor set $\varset{W}$, causal graph $\gP_\varset{V}$ on $\varset{V}\triangleq\varset{U} \cup \varset{W}$.
\end{algorithm} 
\begin{figure}[t]
    \centering
    \subfigure[Ground-truth causal graph]{
        \begin{minipage}[c][2.5cm][c]{0.9\linewidth}
            \centering
            \begin{tikzpicture}[
                baseline=(current bounding box.center),
                node distance=0.6cm,
                every node/.style={circle, draw, minimum size=6mm, inner sep=1pt},
                edge/.style={->, >=Latex, semithick},
                boxnode/.style={rectangle, draw, fill=black!20, minimum size=6mm, inner sep=1pt}
            ]
            \node (E) {E};
            \node[above right=0.2cm and 0.6cm of E] (C) {C};
            \node[below right=0.2cm and 0.6cm of E] (D) {D};
            \node[above right=0.2cm and 0.6cm of D] (B) {B};
            \node[right=of B] (Y) {Y};
            \node[right=of Y] (A) {A};
            
            \draw[edge] (E) -- (C);
            \draw[edge] (E) -- (D);
            \draw[edge] (C) -- (B);
            \draw[edge] (D) -- (B);
            \draw[edge] (B) -- (Y);
            \draw[edge] (Y) -- (A);
            \end{tikzpicture}
        \end{minipage}
        \label{fig:coatpag_example-COATPAG_gt}
    }

    \subfigure[Loop 0: $\text{S}_\text{target} = \left\{A, B\right\}$]{
        \begin{minipage}[c][1.8cm][c]{0.4\linewidth}
            \centering
            \begin{tikzpicture}[
                baseline=(current bounding box.center),
                node distance=0.6cm,
                every node/.style={circle, draw, minimum size=6mm, inner sep=1pt},
                targetnode/.style={circle, draw, fill=blue!10, minimum size=6mm, inner sep=1pt},
                oedge/.style={line width=0.8pt, decoration={markings, 
                    mark=at position 0.1 with {\node[circle, draw, fill=white, inner sep=0pt, minimum size=4pt] {};}, 
                    mark=at position 0.9 with {\node[circle, draw, fill=white, inner sep=0pt, minimum size=4pt] {};}}, 
                    postaction=decorate}
            ]
            \node (B) [targetnode] {B};
            \node[right=of B] (Y)  {Y};
            \node[right=of Y] (A) [targetnode] {A};
            
            \draw[oedge] (B) -- (Y);
            \draw[oedge] (Y) -- (A);
            \end{tikzpicture}
        \end{minipage}
        \label{fig:coatpag_example-COATPAG_loop0}
    }
    \hfill
    \subfigure[Loop 1: $\text{S}_\text{target} = \left\{B\right\}$]{
        \begin{minipage}[c][1.8cm][c]{0.4\linewidth}
            \centering
            \begin{tikzpicture}[
                baseline=(current bounding box.center),
                node distance=0.6cm,
                every node/.style={circle, draw, minimum size=6mm, inner sep=1pt},
                targetnode/.style={circle, draw, fill=blue!10, minimum size=6mm, inner sep=1pt},
                oedge/.style={line width=0.8pt, decoration={markings, 
                    mark=at position 0.1 with {\node[circle, draw, fill=white, inner sep=0pt, minimum size=4pt] {};}, 
                    mark=at position 0.9 with {\node[circle, draw, fill=white, inner sep=0pt, minimum size=4pt] {};}}, 
                    postaction=decorate}
            ]
            \node (B) [targetnode] {B};
            \node[right=of B] (Y)  {Y};
            \node[right=of Y] (A)  {A};
            
            \draw[oedge] (B) -- (Y);
            \draw[oedge] (Y) -- (A);
            \end{tikzpicture}
        \end{minipage}
        \label{fig:coatpag_example-COATPAG_loop1}
    }

    \subfigure[Loop 2: $\text{S}_\text{target} = \left\{C, D\right\}$]{
        \begin{minipage}[c][2.5cm][c]{0.4\linewidth}
            \centering
            \begin{tikzpicture}[
                baseline=(current bounding box.center),
                node distance=0.6cm,
                every node/.style={circle, draw, minimum size=6mm, inner sep=1pt},
                targetnode/.style={circle, draw, fill=blue!10, minimum size=6mm, inner sep=1pt},
                oedge/.style={line width=0.8pt, decoration={markings, 
                    mark=at position 0.1 with {\node[circle, draw, fill=white, inner sep=0pt, minimum size=4pt] {};}, 
                    mark=at position 0.9 with {\node[circle, draw, fill=white, inner sep=0pt, minimum size=4pt] {};}}, 
                    postaction=decorate}
            ]
            \node (B) {B};
            \node[above left=0.2cm and 0.6cm of B] [targetnode] (C) {C};
            \node[below left=0.2cm and 0.6cm of B] [targetnode] (D) {D};
            \node[right=of B] (Y)  {Y};
            \node[right=of Y] (A)  {A};
            
            \draw[oedge] (C) -- (B);
            \draw[oedge] (D) -- (B);
            \draw[oedge] (C) -- (D);
            \draw[oedge] (B) -- (Y);
            \draw[oedge] (Y) -- (A);
            \end{tikzpicture}
        \end{minipage}
        \label{fig:coatpag_example-COATPAG_loop2}
    }
    \hfill
    \subfigure[Loop 3: $\text{S}_\text{target} = \left\{D, E\right\}$]{
        \begin{minipage}[c][2.5cm][c]{0.4\linewidth}
            \centering
            \begin{tikzpicture}[
                baseline=(current bounding box.center),
                node distance=0.6cm,
                every node/.style={circle, draw, minimum size=6mm, inner sep=1pt},
                targetnode/.style={circle, draw, fill=blue!10, minimum size=6mm, inner sep=1pt},
                oedge/.style={line width=0.8pt, decoration={markings, 
                    mark=at position 0.1 with {\node[circle, draw, fill=white, inner sep=0pt, minimum size=4pt] {};}, 
                    mark=at position 0.9 with {\node[circle, draw, fill=white, inner sep=0pt, minimum size=4pt] {};}}, 
                    postaction=decorate},
                circnormaledge/.style={line width=0.8pt, decoration={markings, 
                    mark=at position 0.1 with {\node[circle, draw, fill=white, inner sep=0pt, minimum size=4pt] {};}, 
                    mark=at position 0.99 with {\arrow{Latex}}}, 
                    postaction=decorate, shorten >=2pt},
                normaledge/.style={->, >=Latex, semithick}
            ]
            
            \node (E)  [targetnode] {E};
            \node[above right=0.2cm and 0.6cm of E] (C) {C};
            \node[below right=0.2cm and 0.6cm of E]  [targetnode] (D) {D};
            \node[above right=0.2cm and 0.6cm of D] (B) {B};
            \node[right=of B] (Y) {Y};
            \node[right=of Y] (A) {A};
            
            \draw[oedge] (E) -- (C);
            \draw[oedge] (E) -- (D);
            \draw[circnormaledge] (C) -- (B);
            \draw[circnormaledge] (D) -- (B);
            \draw[normaledge] (B) -- (Y);
            \draw[normaledge] (Y) -- (A);
            \end{tikzpicture}
        \end{minipage}
        \label{fig:coatpag_example-COATPAG_loop3}
    }
    
    \caption{An example where \oursPAG is conducted in several loops. $\text{S}_\text{target}$ are colored.}
    \label{fig:coatpag_example}
\end{figure}

The procedure of \oursPAG is illustrated via the example in Figure~\ref{fig:coatpag_example}. The skeleton is identified at the beginning; however, the arrow heads are present only after factor $E$ is included. The presence of $E$ will introduce a V-structure $C \rightarrow B \leftarrow D$ and thus several edges can be oriented. 

In Section~\ref{subsec:pag-identifiability}, we formally show that for any subset $\varset{U}'$ of the extended factor set $\varset{V}$ from Algorithm~\ref{alg:coat-search}, the skeleton and arrow heads will be maximally revealed if $\MB(\varset{U}') \subseteq \varset{V}$. For arrow tails, we show that they can be maximally revealed if it is in a discriminating path within $\varset{U}'$.

\subsection{\ours Extension on Adjustment Set}
\label{subsec:intro-coat-adj}

With the reliable causal representation elicitation frameworks established in the previous sections, we are able to further apply \ours approach to causal inference. Specifically, we consider two variables $T,Y \in \varset{U}$, where $T$ is the treatment variable, $Y$ is the effect variable, and $\mX$ is the associated unstructured data. The goal is to propose a set of factors from $\mX$ to serve as an adjustment set for the total causal effect estimation from $T$ to $Y$, when it is possible. We seek an algorithm to output a factor set $\varset{W}$ so that an adjustment set for $(T,Y)$ can be constructed by applying existing graphical criterion over the causal graph $\gP_\varset{V}$ on $\varset{V} \triangleq \varset{W} \cup \varset{U}$.%

The construction of an adjustment set is not always feasible even when all factors are given. For example, according to the sound and complete graphical characterization by~\cite{perkovi2018complete}, if a PAG does not satisfy the \textit{amenability} condition, then no adjustment set can be found in this graph. Therefore, one inevitable assumption is the existence of such adjustment sets that meet the graphical criterion.%

\oursGAS is given in Algorithm~\ref{alg:coat-t-y}: \oursGAS first includes all the important factors in the graph (step 1 to 3), and then further extends the node set to reveal the causal structure (step 4). The most important part is to maintain the $\varset{S}_{\text{path}}$ defined in step 2. The purpose of this set is to track and include the intermediate nodes on the path between $T$ and $Y$, as well as their possible descendants. As we shall show in Section~\ref{subsec:adjustment-set}, including the intermediate nodes can preserve the amenability, and including their possible descendants can help to exclude invalid nodes. In addition, Step 1 provides a necessary initialization for constructing $\varset{S}_{\text{path}}$; Step 3 is to exhaust $\varset{S}_{\text{path}}$ iteratively; and Step 4 reveals the causal structure among those factors.

After applying \oursGAS, the adjustment sets can be constructed by searching subsets that meet the generalized adjustment criterion~\citep{perkovi2018complete}. In Section~\ref{subsec:adjustment-set}, for the output factor set $\varset{W}$ (and also $\varset{V} \triangleq \varset{W} \cup \varset{U}$), we show that all the valid adjustment sets relative to $(T,Y)$ in $\varset{U}' \subseteq \varset{V}$ will be preserved if $\MB(\varset{U}') \subseteq \varset{V}$.

\begin{algorithm}[t]
        \caption{\ourst for Adjustment Sets (\oursGAS)}
        \label{alg:coat-t-y}
    \textbf{Input:} Data set $\gD = \big\{ \, (\vx^{(j)}, \, \vu^{(j)})\, \big\}_{j=1}^{n}$; specifications of $T, Y \in \varset{U}$; LLM $\Psi$ for factor proposal; a factor parsing model $\Psi_s$; a causal discovery algorithm $\gA$; and maximal rounds $t_{\text{Max}}$. %
    \begin{enumerate}
        \item Initialize $\varset{W}$ with $\MB(\varset{U}) \setminus \varset{U}$ by performing Algorithm~\ref{alg:coat-main} with $\Psi$, $\Psi_s$, $\gA$, and $t_{\text{Max}}$; and initialize $\varset{S}_{\text{mark}}$ by $\varset{U}$. $\gP_\varset{V}$ is the causal graph over $\varset{V} \triangleq \varset{U} \cup \varset{W}$ estimated by $\gA$.
        \item Let $\varset{S}_{\text{path}}=\Big\{ \, \alpha \in \text{PossDe}(\beta, \gP_\varset{V}) \setminus \varset{S}_{\text{mark}} \mid \text{ $\beta$ is on a path between $T$ and $Y$ in $\gP_\varset{V}$} \, \Big\}$
        \item Repeat with loop when $\varset{S}_{\text{path}}$ is not empty:
            \begin{enumerate}
                \item Take one arbitrary element $\alpha$ from $\varset{S}_{\text{path}}$.
                \item Extend $\varset{W}$ with $\text{MB}(\alpha) \setminus \varset{V}$ by performing Algorithm~\ref{alg:coat-main} on $\alpha$.
                \item Put $\alpha$ into $\varset{S}_{\text{mark}}$, then update $\gP_\varset{V}$ and $\varset{S}_{\text{path}}$.
            \end{enumerate}
        \item Update $\varset{W}$ and $\gP_\varset{V}$ by applying Algorithm~\ref{alg:coat-search} treating $\varset{V}$ as structured data, i.e., using the augmented data set $\gD' \triangleq \big\{ \, (\vx^{(j)}, \, \vv^{(j)}) \, \big\}_{j=1}^{n}$.%
    \end{enumerate}
    \textbf{Output:} Factor set $\varset{W}$, causal graph $\gP_\varset{V}$.
\end{algorithm} 

\subsection{Practical Discussions} %
\label{sec:practice_discussion}
Supplementary to the aforementioned methodology discussion, in this section, we elaborate more on the practical considerations of \ours.

\edit{
\runin{Instruction following}
In practice, strict compliance with output format specifications is critical for pipeline reliability. 
The violation of output format can be (1) proposing a factor in HTML format instead of the markdown format as required in Figure~\ref{fig:prompt_example}; or (2) passing unsupported measurements when using external tools as shown in Listing~\ref{lst:enso_factor_code} in Section~\ref{sec:realwold-casestudy}. 
These violations make parts of the output invalid, and consequently influence the overall quality.
To alleviate this issue, one may need to provide more examples on valid format and documents about external tools, or employ stronger LLMs to handle complex operations.
}

\runin{Modern causal discovery} We use FCI in this paper in order to illustrate the idea. 
To attain identifiability better than the Markov equivalent class, one can choose more advanced CD methods under different assumptions, see Section~\ref{sec:Preliminary-and-Related-Work} if interested.
We also discuss some cases where we need to handle them properly in practice with LLMs.

\runin{Factor filtering} LLMs may output several factors with similar semantics or exhibit multicollinearity in the annotated data, which will hinder the causal discovery process. To mitigate the issue, one could do factor filtering that adopts PCA or early conditional independence tests given the currently discovered variables in the Markov Blanket to detect and eliminate these variables.

\runin{Factor pool} LLMs may discover useful factors in early rounds while being discarded. For example, the underlying spouse variables of the target label $Y$ may be independent of $Y$ without conditioning on their common children variables. To resolve the issue, we could introduce a factor pool that stores the candidate variables proposed in the past, and replay the variables that have not been passed by conditional independence tests with existing variables in the Markov Blanket for a double check.

\runin{Cost analysis} The key component in the cost of \ours is the number of candidate factors that are proposed by LLMs from unstructured data. For each proposed candidate factor, the cost, $c_\text{candidate}$, is on evaluating the factor values on each of the $N$ samples, i.e., $c_\text{candidate}=N \cdot c_\text{sample}$. For example, if the text data in each sample has an average length of 1K tokens (including the instruction prompt), the complete annotation of $N=1000$ samples will cost 1M tokens (the number of output tokens can be ignored with a suitable prompt). This will cost about \$0.15 USD with GPT-4o-mini, or \$0.04 USD with Qwen-turbo.

The required number for candidate factors is determined by both LLMs' capacity and the nature of the underlying causal structure. Ideally, if all the proposed factors are valid, according to the first step in Algorithm~\ref{alg:coat-search}, the minimal factor number is $|\MB(\varset{U})|-|\varset{U}|$ for extending a PAG with original node set $\varset{U}$. These nodes in $\MB(\varset{U})$ are necessary for refining the skeleton. 

The upper bound of the number for candidate factors is influenced by the \textit{long directed chains} in causal graph where each node on it is only connected with its unique predecessor and unique successor. For example, in the special case demonstrated by Figure~\ref{fig:motivating_coatpag}, to reveal the arrow head from $Y$ to $A$, \ours will extend the node set by $B_1, \cdots, B_N$, and then $C$ and $D$. However, such a directed chain can be arbitrarily long, which will increase the cost of the algorithm.

\edit{

\begin{figure}
    \centering
    \includegraphics[width=0.75\linewidth]{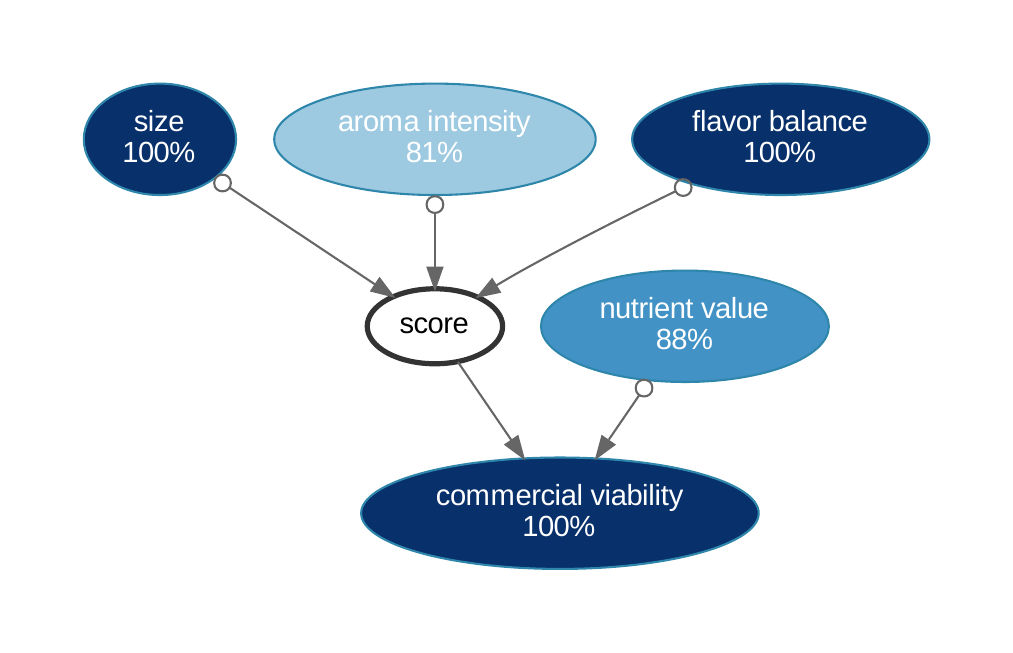}
    \caption{\edit{An illustration of \oursMB with uncertainty estimation. Deeper colors of discovered factors indicate lower uncertainty.}}
    \label{fig:uncertaintyCOATMB}
\end{figure}

\runin{Bootstrap Uncertainty Estimation} 
We discuss how to provide uncertainty estimation for \oursMB algorithm based on stability~\citep{meinshausen2010stability}. The same idea can be applied to \oursPAG and \oursADJ.

At each round $t$, with current numerical data $\mathcal{D}$, we perform $B$ bootstrap resampling to estimate the 
\emph{selection frequency} of each candidate factor $W_i \in \varset{W}$.
Specifically, for the $b$-th bootstrap sample $\gD^{(b)} \sim \text{Bootstrap}(\gD)$, 
we obtain an estimated Markov Blanket $\gM^{(b)} = \gA(\gD^{(b)}, Y, \varset{W})$. 
The selection frequency for factor $W_i$ is defined as:
\begin{equation}
    f(W_i) = \frac{1}{B} \sum_{b=1}^{B} \mathbb{I}\left[ W_i \in \gM^{(b)} \right],
    \label{eq:selection-frequency}
\end{equation}
where $\mathbb{I}[\cdot]$ is the indicator function. 
The frequency vector $\gF_t = [f(W_1), \cdots, f(W_m)]^\top$ serves as an 
\emph{uncertainty measure}, and the final Markov Blanket at round $t$ is determined by a threshold $\tau \in (0,1]$:
\begin{equation}
    \gM_t = \left\{ W_i \in \varset{W} : f(W_i) \geq \tau \right\}.
    \label{eq:mb-threshold}
\end{equation}

In Figure~\ref{fig:uncertaintyCOATMB}, we present the proposed factors and their uncertainty at the final round. We apply \oursMB on \apple with $B=16$, $\tau=0.8$, and  $ |\mathcal{D}^{(b)}| / |\mathcal{D}|=0.5$.
The bootstrap procedure relies only on the sample values of the target (``score'') and the proposed factors, e.g., ``size''.
The computation load depends on the method for finding Markov Blanket.
}

\runin{Extensions to improve efficiency} To reduce the required candidate factor number, and to also control the cost budget, we provide the following practical suggestions:
\begin{itemize}
    \item \textbf{Incorporate multi-view unstructured data.} This will enlarge the pool of all possible factors that can be proposed by LLMs, and can reduce the length of the \textit{long directed chains} mentioned above.
    \item \textbf{Utilize priori knowledge.} Sometimes it is possible to acquire reliable prior knowledge. For example,  if we know that the unstructured data is generated before some nodes in the original PAG are generated, then these factors cannot be causes of the proposed factors.
    \item \textbf{Cost budget.} The loop at step 4 in Algorithm~\ref{alg:coat-search} can be terminated when the cost budget (on resources or time) is exhausted. The early stop will leave some arrow heads not oriented in the output PAG.
\end{itemize}

\runin{Extensions to experiment design} In existing causal discovery literature, there are methods that leverage data sets from different domains or node sets~\citep{triantafillou2015constraint,mooij2020joint,huang2020causal}, as well as strategies for optimal experimental design~\citep{agrawal2019abcd,tigas2022interventions,toth2022active}. 

In the experimental-design methods, the causal structure can be learned by actively collecting samples from an intervened distribution. For example, in Figure~\ref{fig:coatpag_example-COATPAG_loop0}, one can apply an intervention on node $Y$ to distinguish the causal directions in the two undetermined edges. Instead of collecting new samples, \ours\ \minor{proposes} three new factors: $C$, $D$, and $E$ from the paired unstructured data in existing samples, and therefore the causal directions can be revealed in the extended causal graph in Figure~\ref{fig:coatpag_example-COATPAG_loop3}. The \ours framework can integrate with these approaches in new perspectives:
\begin{enumerate}
    \item By expanding the subgraph around an intervention target, \ours can resolve undetermined edges using unstructured data, reducing the need for redundant experiments.
    \item With the expanded node set and the refined causal structure, \ours delivers richer input to methods for experimental design\minor{, and} ultimately yields more valuable experimental data sets.
\end{enumerate}

\section{Identifiability Analysis }\label{subsec:theory}

\begin{table}[t]
\centering
{\small
\setlength{\tabcolsep}{8pt}
\renewcommand{\arraystretch}{1.25}
\resizebox{\columnwidth}{!}{%
\begin{tabular}{@{}p{9.6cm} p{6cm}@{}}
\toprule
\textbf{Assumption} & \textbf{Notes} \\
\midrule
\multicolumn{2}{l}{\textbf{I. LLM Capability Assumptions}} \\
Propose valid factors, i.e., $p_\Psi > 0$ and  $C_\Psi > 0$ \footnotesize{(Definition~\ref{asp:LLM_good})}. & \multirow{1}{*}{Validated in Figure~\ref{fig:apple_llm_cap}~(c)} and Table~\ref{tab:full_results_causal_metrics_on_apple} \\
Accurate annotation with independent noise \footnotesize{(Assumption~\ref{ass:Faithful_Markov})}. & Validated in Figure~\ref{fig:apple_llm_cap}~(a-b) and Table~\ref{tab:ind_test_annotation} \\
\midrule
\multicolumn{2}{l}{\textbf{II. Causal Discovery Assumptions}~\citep{fci}} \\
Faithfulness and Markov condition    & \multirow{2}{*}{\makecell{Depends on the specific algorithm \\ (e.g., FCI, LiNGAM)}} \\
Validity of conditional independence tests &  \\
\bottomrule
\end{tabular}%
}}
\caption{\edit{Summary of Core Assumptions for COAT Methods.}}
\label{tab:assumptions-core}
\vspace{0.3em}
\footnotesize
\textit{Note:} These assumptions are shared across COAT-MB, COAT-PAG, and COAT-ADJ.
\end{table} 
In this section, we provide theoretical discussions and analyze the identifiability of different variants of \ours in the corresponding settings.
\edit{
As shown in Table~\ref{tab:assumptions-core}, we summarize the key assumptions required by the three \ours methods, which are grouped into LLM capability assumptions and causal discovery assumptions.
}

\subsection{On Markov Blanket}
\label{subsec:MB-identification}

Given a new factor $\vw_{k+1}$, with the current representation as \[h_{[k]}(\mX) = \Big(\,\vw_1\left(\mX\right), \, \vw_2\left(\mX\right), \, \cdots\, , \, \vw_k\left(\mX\right) \, \Big),\] and \ours tests:
        \begin{equation}\label{eq:mutual_info_down:notind}
            Y \notind \vw_{k+1}(\mX) \mid h_{[k]}(\mX).
        \end{equation} 
\ours also requires the following usual condition about distribution and causal graph:
\begin{assumption}[Faithful and Markov conditions, from~\cite{causation_1st}] 
\label{ass:Faithful_Markov}For any disjoint non-empty subsets $A, B, C \subset \mathcal{W}^{\leq t} \cup \{Y\} $, $A$ and $B$ are d-separated by $C$ on the causal graph iff $A \ind B \mid C$ on the factors' distribution. All conditional independencies are preserved after factor parsing.
\end{assumption}

The annotation from a poor model could introduce an additional ``measurement error'' on the true factor values, disturbing the true distribution and thus violating the above assumption. The measurement error injected through LLM-based parsing may be of independent interest to the literature of causal discovery~\citep{Glymour2019ReviewOC,causal_learn_survey}. 
Empirically, as one shall see in Section~\ref{subsec:applesetting1}, predominant LLMs can conduct effective factor annotation.
If Assumption~\ref{ass:Faithful_Markov} holds, the conditional mutual information between $Y$ and $\mX$ given the desired factors decreases:

\begin{proposition}\label{prop:mutual_info_down}
    Under Assumption~\ref{ass:Faithful_Markov}, if condition~\ref{eq:mutual_info_down:notind} holds, then
    for Markov Blanket $\gS \subseteq [k+1]$ of $Y$, i.e., $Y \ind h_{[k+1] \setminus \gS}(\mX)  \mid  h_{\gS}(\mX)$,
    we have the following about conditional mutual information:
        \begin{equation}\label{eq:mutual_info_down:result}
            I\Big( Y; \mX \mid h_{\gS}(\mX) \Big) = I\Big( Y; \mX \mid h_{[k+1]}(\mX) \Big) < I\Big( Y; \mX \mid h_{[k]}(\mX) \Big).
        \end{equation}
\end{proposition}

We provide an initial exploration under what conditions LLMs can identify target-related factors and how the ability of an LLM influences this procedure.
\begin{definition}[Ability of LLMs] \label{asp:LLM_good}
    Given a suitable prompt about current factors and data, the LLM $\Psi$ has non-zero probability $p_\Psi > 0$ to propose a new factor $\vw_{k+1}$ that satisfies condition~\ref{eq:mutual_info_down:notind} and 
        \begin{equation}
            \frac{I\big( Y; \mX \mid h_{[k+1]}(\mX) \big)}{I\big( Y; \mX \mid h_{[k]}(\mX) \big)} < 1-C_{\Psi},
        \end{equation}
    for some positive constant $C_{\Psi}$ whenever $I\left( Y; \mX \mid h_{[k]}(\mX) \right) > 0$. Note that $h_{[0]}(\mX)=\phi$, hence we also use $I\left( Y; \mX \mid h_{[0]}(\mX) \right)$ to refer to $I\left( Y; \mX\right)$. We use $p$ instead of $p_\Psi$ when the context is clear.
\end{definition}
We further explain the intuition behind Def~\ref{asp:LLM_good}: the \emph{Perception Score} $p$ captures the LLM's responsiveness to the given prompts and the feedback; the \emph{Capacity Score} $C_\Psi$ captures the quality of the factors proposed by the LLM.
Empirically, the two scores are used to estimate the abilities of the predominant LLMs in Section~\ref{subsec:applesetting1}.
Theoretically, we use them to characterize the influence of prompt templates, the LLM responsiveness, and the quality of factors on the performance of \ours:
\begin{proposition}[Characterization for Factor Identification Process] \label{prop:cha_f_i}
   With Assumption~\ref{ass:Faithful_Markov}, for any small number $\epsilon, \delta \in (0,\frac{1}{2})$,  perception score $p>0$, capacity score $C_\Psi>0$, with  $t$ \ours rounds that
        \begin{equation} \label{eq:cha_f_i:round}
            \sqrt{t} > \frac{|z_\delta|\sqrt{1-p}}{2 \sqrt{p}} \left( 1 +  \sqrt{1+\frac{4 \log{\epsilon} }{z_\delta^2 (1-p) \log{(1-C_\Psi)}}} \, \right),
        \end{equation}
    where $z_\delta$ is the $\delta$-quantile  of the standard normal distribution, we have 
        \begin{equation}
            \Pr\left(  \frac{I\big( Y; \mX \mid h_{\le t}(\mX) \big)}{I\left( Y; \mX  \right)} < \epsilon \right) \ge 1 - \delta.
        \end{equation}
\end{proposition}

The proof is given in Appendix~\ref{appdx:proof}. %
Prop.~\ref{prop:cha_f_i} gives a guarantee on identifying a Markov Blanket. Intuitively, Prop.~\ref{prop:cha_f_i} also characterizes the influence of prompt templates, the LLM responsiveness, and the quality of factors on the performance of COAT via the two proposed measures: $p$ and $C_\Psi$. When both of them are positive, \ours will converge exponentially:
\begin{proposition}[Rate of Convergence]\label{prop:rate_of_convergence}
    With Assumption~\ref{ass:Faithful_Markov}, for any small number $\epsilon, \delta \in (0,\frac{1}{2})$, perception score $p>0$, capacity score $C_\Psi>0$, after $t$ \ours rounds, the following inequality holds with probability at least $1-\delta$:
    \begin{equation}\label{eq:rate_of_convergence}
        \frac{I\big(Y;\mX\mid h_{\le t}(\mX)\big)}{I(Y;\mX)} \le \left(\frac{1}{1-C_\Psi}\right)^{-tp-z_{\delta}\sqrt{tp(1-p)}}
    \end{equation}
\end{proposition}

\edit{
\para{Finite-Sample Error Control.}
The previous discussion assumes perfect conditional independence oracles\minor{. Here} we analyze the Type-I error rate with a finite sample of size $n$. 
Assume that $W$ is the candidate factor to be tested, $\gM$ is the current choice of Markov blanket of the target variable $Y$. 
The following proposition quantifies the sample size required to simultaneously control the Type~I error (false positive, probability $\alpha$) and the Type~II error (false negative, probability $\beta$) under Gaussianity.

\begin{proposition}[Sample size analysis for Markov blanket]
\label{prop:finite-sample-mb}
Assume that the joint distribution of $(Y, \mathcal{M}, W)$ is multivariate Gaussian. Let $\rho$ be the true partial correlation between target $Y$ and  candidate factor $W$ given the current choice of Markov blanket $\gM$. $\rho \neq 0$ indicates that $W$ carries additional information about $Y$ beyond $\gM$. Suppose a conditional independence test based on the Fisher \(z\)-transform~\citep{fisher1921probable} of $\rho$ is performed at significance level \(\alpha \in (0,1)\). Then, to detect the dependence with power \(1-\beta\), the sample size \(n\) must satisfy
\begin{equation}
    n \;\ge\; \frac{4\,(z_{1-\alpha/2} + z_{1-\beta})^2}{\bigl[\ln\bigl(\frac{1+\rho}{1-\rho}\bigr)\bigr]^2} \;+\; |\mathcal{M}| \;+\; 3,
\end{equation}
where \(z_q\) denotes the \(q\)-quantile of the standard normal distribution and \(|\mathcal{M}|\) is the cardinality of the conditioning set.
\end{proposition}

The proof is given in Appendix~\ref{appdx:proof}. The above bound reveals that the required sample size grows linearly with the cardinality of the current Markov blanket $|\gM|$. Intuitively, conditioning on more variables consumes additional degrees of freedom, which increases the variance of the estimation. In Section~\ref{sec:practice_discussion}, we provide the bootstrap stability selection procedure and show the stability of \ours.
}

\para{Causal structure identification} It is clear that LLMs are not involved in the causal discovery process, which is mainly executed by causal discovery methods such as FCI. Therefore, the CD guarantees the identifiability of the final causal graph over the LLM-proposed factors. 
The concrete assumptions required for identifiability depend on the specific CD used in COAT. For instance, the FCI algorithm requires faithfulness of the data distribution with respect to the true causal graph~\citep{fci}.
In our experiments, we also verify that the structured data annotated by LLMs has a high accuracy and little noise, which is friendly to the CD assumption.
In general, one could switch to another CD in \ours, while using different CD may require different assumptions. For example, LiNGAM~\citep{shimizu2006linear} requires the relations among variables to be linear and non-Gaussian models.  Empirically, we find that \ours with LiNGAM works very well (Appendix~\ref{appdx:ligam_coat}).

\subsection{On Partial Ancestral Graphs}
\label{subsec:pag-identifiability}

In this section, we analyze how \oursPAG improves causal structure identification by leveraging unstructured data $\mX$. As introduced in Section~\ref{subsec:introduce-coat-pag}, \oursPAG extends the original structured variables $\varset{U}$ by defining a new set of variables $\varset{W}$ from $\mX$. 
The augmented node set enables more conditional independence test involving nodes in $\varset{U}$, and therefore yields a more informative partial ancestral graph (PAG).
We will demonstrate: in the PAG $\gP_\varset{V}$ over the extended node set $\varset{V} \triangleq \varset{W} \cup \varset{U}$, the information in $\mX$ will maximally refine the causal structure among variables from $\varset{U}$. 

We start with the following necessary formulation. Let $\varset{O}$ be the set of all possible factors that can be defined based on $\mX$, and we have $\varset{W} \subseteq \varset{O}$. %
Then, the maximally observable node set is $\varset{E} \triangleq \varset{O} \cup \varset{U}$. We expect that the causal structure among $\varset{U}$ in $\gP_\varset{E}$ will be preserved in $\gP_\varset{V}$. We formalize this idea in the following definition:

\begin{definition}
    For three vertex sets $\varset{U} \subseteq \varset{V} \subseteq \varset{E}$, and let $\gP_\varset{V}$ and $\gP_\varset{E}$ be the maximally informative PAG for $\varset{V}$ and $\varset{E}$ respectively. We say arrow heads (or arrow tails) in $\gP_\varset{E}$ are preserved by $\gP_\varset{V}$ within $\varset{U}$, if for $\alpha, \beta \in \varset{U}$:
    \begin{enumerate}
        \item $\alpha$ and $\beta$ are adjacent in $\gP_\varset{V}$ if and only if they are adjacent in $\gP_\varset{E}$; and
        \item $\alpha \leftarrow \!\! * \ \beta$ (or $\alpha - \!\! * \ \beta$ ) occurs in  $\gP_\varset{V}$ if and only if it occurs in $\gP_\varset{E}$.
    \end{enumerate}
\end{definition}

In the next proposition, we state under what condition, one can have arrow head preservation in a node subset $\varset{U}$ in the PAG augmented by Algorithm~\ref{alg:coat-search}. %

\begin{proposition} \label{prop:arrowhead-preserve}
    Let $\varset{W}$ be the node set returned by applying Algorithm~\ref{alg:coat-search} on $(\mX, \varset{U})$. If a set $\varset{U}'$ satisfies $\MB(\varset{U}') \subseteq \varset{V}$, then arrow heads in $\gP_\varset{E}$ are preserved by $\gP_\varset{V}$ within $\varset{U}'$.
\end{proposition}

As a special case, by the first step in Algorithm~\ref{alg:coat-search}, one can see that the original structured variable set $\varset{U}$ satisfies $\MB(\varset{U}) \subseteq \varset{V}$, therefore, by Proposition~\ref{prop:arrowhead-preserve}, arrow heads in $\gP_\varset{E}$ are preserved by $\gP_\varset{V}$ within $\varset{U}$.
The following result will be used later. It says that the edges in discriminating path in desired $\mU$ can be clearly oriented.

\begin{corollary} \label{cor:arrow-tails-in-discriminating-path}
Let $\varset{W}$ be the node set returned by applying Algorithm~\ref{alg:coat-search} on $(\mX, \varset{U})$, and $\MB(\varset{U}') \subseteq \varset{V}$.
Let $u$ be a discriminating path between $\theta$ and $\gamma$ in $\gP_\varset{E}$ within $\varset{U}'$, then for any node $\alpha$ between $\theta$ and $\gamma$ that is adjacent to $\gamma$, the edge is either $\alpha \rightarrow \gamma$ or  $\alpha \leftrightarrow \gamma$ in $\gP_\varset{V}$.
\end{corollary}

\begin{proof}
    According to Proposition~\ref{prop:arrowhead-preserve}, the arrow heads in $\gP_\varset{E}$ are preserved by $\gP_\varset{V}$ within $\varset{U}'$. If the corollary doesn't hold, we have $\alpha \circrightarrow \gamma$ in $\gP_\varset{V}$, then either of the FCI orientation rules $\gR 1$ and $\gR 4$ from \citet{ZHANG20081873} can be applied to orient $\alpha \circrightarrow \gamma$. However, this contradicts with that $\gP_\varset{V}$ is maximally informative.
\end{proof}

\subsection{On Adjustment Sets}
\label{subsec:adjustment-set}

\edit{
To begin with, we introduce the essential concept and definition:

\begin{definition}[Visibility~\citep{zhang2008causal,perkovi2018complete}] \label{def:visibility}
    Let $X, Y$ be nodes in a mixed graph $\mathcal{G}$. 
    \begin{itemize}
        \item If $\mathcal{G}$ is a DAG or a CPDAG, any directed edge $X \rightarrow Y$ is said to be \textbf{visible}.
        \item If $\mathcal{G}$ is a MAG or a PAG, a directed edge $X \rightarrow Y$ is \textbf{visible} if there exists a node $C$ (where $C \neq X$ and $C \neq Y$) such that:
        \begin{enumerate}
            \item $C$ is not adjacent to $Y$; and
            \item there is an edge between $C$ and $X$ with an arrowhead pointing into $X$ (i.e., $C \rightarrow X$ or $C \leftrightarrow X$), \textbf{or} there is a collider path from $C$ to $X$ where:
            \begin{itemize}
                \item every edge on the path has an arrowhead pointing into $X$; and
                \item every non-endpoint vertex on this path is a parent of $Y$.
            \end{itemize}
        \end{enumerate}
        \item Otherwise, the directed edge $X \rightarrow Y$ is called \textbf{invisible}.
    \end{itemize}
\end{definition}

\begin{definition}[Amenability~\citep{perkovi2018complete}]
    For $T,Y \in \varset{V}$, $\gP_\varset{V}$ is \textbf{amenable} relative to $(T, Y)$ if every proper possibly directed path from $T$ to $Y$ in $\gP_\varset{V}$ starts with a visible edge out of $T$.
\end{definition}
}

We consider the Generalized adjustment criterion~\citep{perkovi2018complete} defined as follow:

\begin{definition}[Generalized adjustment criterion~\citep{perkovi2018complete}] \label{def:Generalized-adjustment-criterion}
    Let $\varset{Z} \subseteq \varset{V} \setminus \{T,Y\}$ be a node set in $\gP_\varset{V}$. $\varset{Z}$ satisfies the generalized adjustment criterion if the following three conditions hold:
    \begin{itemize} 
        \item[~] (\textbf{Amenability}) $\gP_\varset{V}$ is amenable relative to $(T,Y)$; and
        \item[~] (\textbf{Forbidden Set}) $\varset{Z} \cap \text{Forb}(T,Y,\gP_\varset{V})=\emptyset$; and
        \item[] (\textbf{Blocking}) all proper definite status non-causal paths from $T$ to $Y$ are blocked by $\varset{Z}$.
    \end{itemize}
    Note that $\text{Forb}(T,Y,\gP_\varset{V}):= \cup_{\alpha\in \text{S}(T,Y, \gP_\varset{V})} \text{PossDe}(\alpha,\gP_\varset{V})$, and $\text{S}(T,Y, \gP_\varset{V})$ is the set of all nodes in a proper directed path from $T$ to $Y$ in $\gP_\varset{V}$ (exclude $T$ itself).
\end{definition}

\edit{For the first key condition, i.e., Amenability, }
we show the Algorithm~\ref{alg:coat-t-y} can produce a PAG $\gP_\varset{V}$  that is amenable to $(T,Y)$ when $\gP_\varset{E}$ is amenable to $(T,Y)$. If $\gP_\varset{E}$ is not amenable, that means we cannot find an adjustment set even if all possible factors from given unstructured data are given.

\begin{proposition} \label{prop:amenable-preservation}
    Let $\varset{W}$ be the node set returned by Algorithm~\ref{alg:coat-t-y} on $T$ and $Y$. If $\gP_\varset{E}$ is amenable relative to $(T, Y)$, then $\gP_\varset{V}$ is amenable relative to $(T, Y)$.
\end{proposition}

The next proposition will be used in later proof. It says that we can identify some special arrow tails in the PAG $\gP_\varset{V}$ produced by  Algorithm~\ref{alg:coat-t-y}.

\begin{proposition} \label{prop:arrow-tails-in-path-T-Y}
    Let $\varset{W}$ be the node set returned by Algorithm~\ref{alg:coat-t-y} on $T$ and $Y$. If $\alpha$ and $\gamma$ are two adjacent nodes in a possibly directed path from $T$ to $Y$ in $\gP_\varset{V}$, then the arrow tails in  $\gP_\varset{E}$ are preserved in $\gP_\varset{V}$ within $\{\alpha, \gamma\}$.
\end{proposition}

With the Proposition~\ref{prop:arrow-tails-in-path-T-Y}, we can identify definite status paths between $T$ and $Y$. A definite status path consists of colliders and definite non-colliders, where a definite non-collider means there is at least one edge out of this node. Therefore, we need to identify arrow tails by Proposition~\ref{prop:arrow-tails-in-path-T-Y}.

\begin{corollary} \label{cor:definite-status-path}
    Let $\varset{W}$ be the node set returned by Algorithm~\ref{alg:coat-t-y} on $T$ and $Y$. Let $\vu$ be a path from $T$ to $Y$ in $\gP_\varset{E}$. $\vu$ is a definite status in $\gP_\varset{V}$ if and only if it is definite status in $\gP_\varset{E}$.
\end{corollary}

\begin{proof}
    We have $\vu \subseteq \MB^{(2)}(\vu) \subseteq  \varset{V}$ by the Loop at step 3 in Algorithm~\ref{alg:coat-t-y}. By Proposition~\ref{prop:arrowhead-preserve}~and~\ref{prop:arrow-tails-in-path-T-Y}, its arrow tails and heads are preserved.
\end{proof}

With above preparation, in the following proposition, we show that if a node set $\varset{U}'$ have its Markov Blanket in the PAG $\gP_\varset{V}$ produced by  Algorithm~\ref{alg:coat-t-y}, an adjustment set can be found if it is within the node set $\varset{U}'$.  As a special case, by the first step in Algorithm~\ref{alg:coat-t-y}, one can see that the original structured variable set $\varset{U}$ satisfies $\MB(\varset{U}) \subseteq \varset{V}$, therefore, it must satisfy the condition described in Proposition~\ref{prop:Adjustment-Sets-Preservation}. In practice, one can keep applying Algorithm~\ref{alg:coat-t-y} multiple times to expand $\varset{U}$ for searching possible adjustment sets.

\begin{proposition}[Adjustment Sets Preservation]
    \label{prop:Adjustment-Sets-Preservation}
    Let $\varset{W}$ be the node set returned by Algorithm~\ref{alg:coat-t-y} on $(T, Y)$. If $\varset{Z} \subseteq \varset{U}'$ s.t. $\MB(\varset{U}') \subseteq \varset{V}$, then $\varset{Z}$ is an adjustment set relative to $(T, Y)$ in $\gP_\varset{V}$, %
    if and only if, it is an adjustment set relative to $(T, Y)$ in $\gP_\varset{E}$.
\end{proposition}

\section{Empirical Analysis of \ours}
\label{sec:apple_bench}
We evaluate the proposed \ours framework in multiple settings with different data sets.
\begin{itemize}
    \item For the identifiability assumptions related to LLM capability (Definition~\ref{asp:LLM_good}), we use synthetic data (Section~\ref{subsec:applesetting1}) to examine the capabilities of $10$ predominant LLMs including \gptfo~\citep{openai2024gpt4o}, Claude-3-Opus~\citep{anthropic2024claude3}, \llamathree-70b~\citep{meta2024llama3}, and \mistral-Large~\citep{mistral}. The results are shown in Figure~\ref{fig:quant_ca}.
    \item For \oursMB method (Algorithm~\ref{alg:coat-main}), we conduct experiments on one synthetic benchmark (Section~\ref{subsec:applesetting1}) to evaluate the quality of the proposed Markov blanket factors from textual reviews; furthermore, we use one realistic clinical data set (Section~\ref{sec:pain_bench}) to draw Neuropathic factors using external tools and symptom-level diagnosis.
    \item For \oursPAG method (Algorithm~\ref{alg:coat-search}), we conduct experiments on one synthetic benchmark (Section~\ref{subsec:applesetting2}) to determine the causal direction between two correlated variables. The direction is not identifiable within the partial ancestral graph on the input node set unless suitable factors are proposed.
    \item For \oursGAS method (Algorithm~\ref{alg:coat-t-y}), we conduct experiments on two realistic data sets about marketing treatment with customers' social media posts (Section~\ref{subsec:CI-ATE}). The quality of adjustment sets are evaluated by the adjusted causal effect.
    \item We evaluate the annotation error in Section~\ref{subsec:applesetting1} (Figure~\ref{fig:apple_attribute_acc} and~\ref{fig:apple_match_acc}); and the hyperparameter sensitivity in Section~\ref{subsec:ablation-study}.
\end{itemize}

\subsection{Eliciting Markov Blanket }
\label{subsec:applesetting1}

We consider the scenario where a rating score $Y$ will be assigned to each apple by gastronomes. 
Each apple has its own attributes, including size, aroma, and taste. Each gastronome pays attention to a subset of the above three attributes. They will write a review according to their preference and give the rating score.
As shown in Figure~\ref{fig:causal_graph_apple-Ground-truth}, we prepare different high-level factors: 3 parents of $Y$ (size, aroma, and taste), one child of $Y$ (market potential), and one spouse of $Y$ (nutrition value). These factors form a Markov blanket of $Y$. In addition, we also prepared one disturbing factor, juiciness, that is related to $Y$ but not a part of this blanket. 
A good method is expected to propose the five high-level factors (up to semantical meanings) and exclude the disturbing factor.

\begin{figure}[t]
    \centering
    \subfigure[Ground truth]{
        \includegraphics[height=0.2\textheight, trim=40 40 40 35, clip]{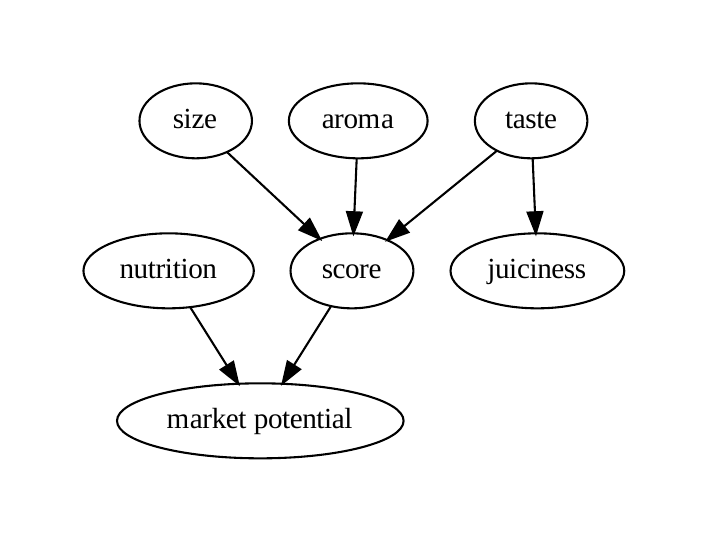}
        \label{fig:causal_graph_apple-Ground-truth}
    } 
    \subfigure[GPT-4 META]{
        \includegraphics[height=0.25\textheight, trim=40 40 40 35, clip]{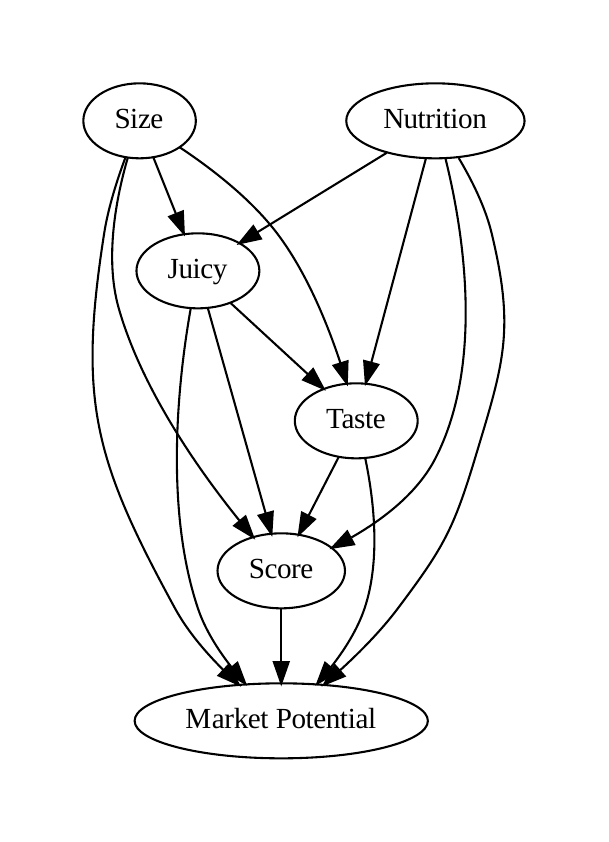}
    } \\
    \subfigure[\chatgpt \ours]{
        \includegraphics[height=0.2\textheight, trim=40 40 40 35, clip]{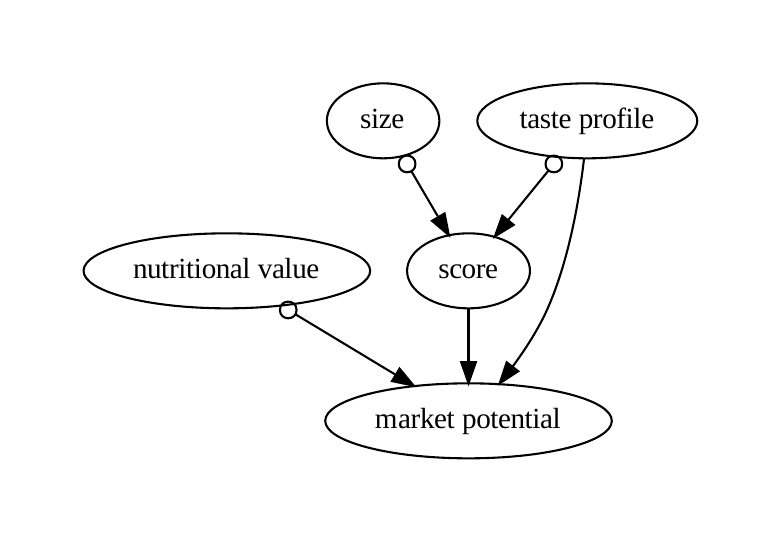}
    } 
    \subfigure[Claude-3-Opus \ours]{
        \includegraphics[height=0.2\textheight, trim=40 40 40 35, clip]{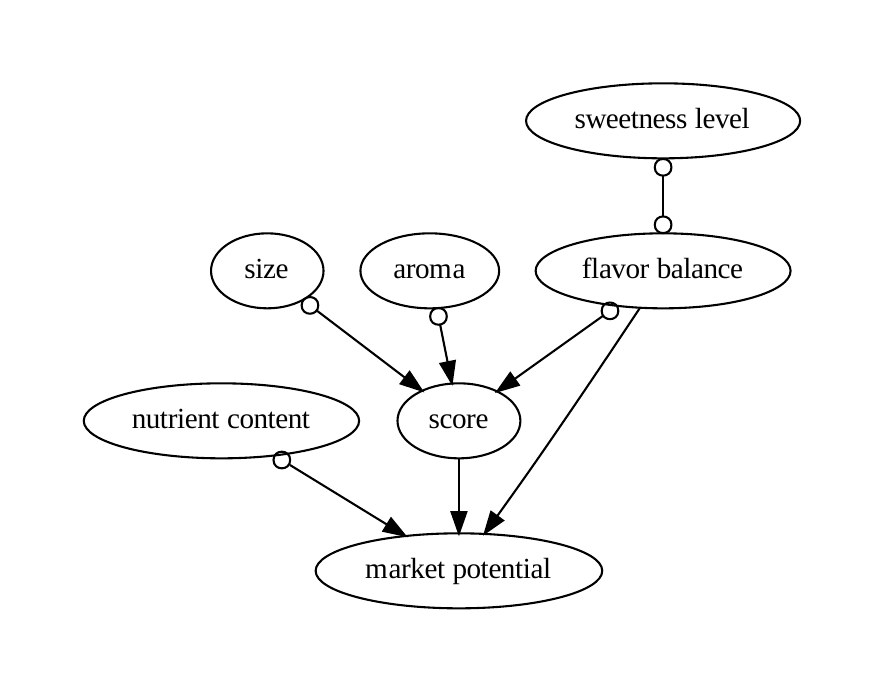}
    }
    \vspace{-0.1in}
    \caption{The discovered causal graphs in \apple. Compared to the ground truth results, directly adopting LLMs to reason the causal relations can easily elicit many false positive edges. In contrast, the relations recovered by \ours have a high precision and recall. The directed edge between ``taste'' and ``juiciness'' can not be recovered by \ours because of the limitations of FCI.}
    \label{fig:causal_graph_apple}
\end{figure}

\runin{Setup and Benchmark Construction}
For each sample, we first generate the numerical values of each factor according to Figure~\ref{fig:causal_graph_apple-Ground-truth}. These are ground-truth values and will not be accessible to the algorithm. Then, we generate the according review using \gptf by feeding the predefined interpretation of each factor value. For example, $\text{taste}=1$ means the apple is sweet, $\text{taste}=-1$ means the apple is sour, and  $\text{taste}=0$ means the gastronome doesn't care the taste.  We generated 200 samples for LLMs' analysis and annotation. In appendix, we show the concrete prompt template for this procedure in Figure~\ref{fig:apple_gen_prompt}, and several review examples in Figure~\ref{fig:apple_example_appdx}.

\runin{Baselines and Evaluation} In this setting, we mainly employ two different uses of LLMs as the baselines: \textbf{META} is the zero-shot factor proposal given only the context to LLMs; and \textbf{DATA} additionally gives some samples of raw observations, which is an ablation of \ours without the \emph{feedback} module, i.e., only one \ours round.  For causal relation recovery, we follow~\citet{causal_llm_frontier} that prompt LLMs to reason for the causal direction of each pair of the discovered variables by \textbf{DATA}.

We evaluate the ability on factor proposal based on three metrics: \emph{MB}, \emph{NMB}, and \emph{OT}. 
\emph{MB} means the desired factor forming the Markov Blanket of Y. 
\edit{The ground truth of the benchmark is given in Figure~\ref{fig:causal_graph_apple-Ground-truth}. The \emph{MB} factors include \texttt{size}, \texttt{aroma}, \texttt{taste}, \texttt{market potential}, and \texttt{nutrition}.}
\emph{NMB} means the undesired factor relevant to data but not in \emph{MB}. 
\edit{In this benchmark, it is \texttt{juiciness}.}
\emph{OT} means the unexpected factors irrelevant to data. 
\edit{They are factors that are outside of the six factors mentioned. For example, an LLM may propose \texttt{color} as a factor whose information can be rarely observed in samples. }
We also present the corresponding recall, precision, and F1 with respect to $\MB(Y)$.

\runin{Finding 1: \ours reliably \minor{elicits} Markov Blanket} 
Empirically, LLMs with CoT can be aware of high-level factors behind data (lower \emph{OT} than \emph{META}) but still struggle to distinguish the desired factors in Markov Blanket (higher \emph{NMB} than \ours).
\ours is more resistant to the ``disturbing'' factor, which is supported by the lower \emph{NMB} column. \ours filters out irrelevant factors from LLMs' prior knowledge that are not reflected by the data, which is supported by the lower \emph{OT} column. COAT robustly encourages LLM to find more expected factors through the feedback, which is supported by the higher \emph{MB} column.

\begin{table}[p]
    \centering
    {\small
    \setlength{\tabcolsep}{3pt}
    \resizebox{\textwidth}{!}{%
\begin{tabular}{@{}llcccccc@{}}
\toprule 
LLMs           & Method & MB          & NMB         & OT          & Recall      & Precision   & F1          \\ \midrule
GPT-4o         & META   & 3.00\std{0.82} & 1.00\std{0.00} & 4.00\std{0.82} & 0.60\std{0.16} & 0.37\std{0.09} & 0.46\std{0.12} \\
               & DATA   & 3.00\std{0.00} & 0.67\std{0.47} & 0.00\std{0.00} & 0.60\std{0.00} & 0.83\std{0.12} & 0.69\std{0.04} \\
               & DATA+CoT  & 4.67\std{0.58}  & 1.00\std{0.00}  & 0.33\std{0.58}  & 0.93\std{0.12}  & 0.78\std{0.10}  & 0.85\std{0.10} \\
               & COAT   & 4.00\std{0.82} & 0.00\std{0.00} & 0.00\std{0.00} & 0.80\std{0.16} & 1.00\std{0.00} & 0.88\std{0.10} \\ \midrule
GPT-4          & META   & 2.67\std{0.94} & 0.67\std{0.47} & 2.33\std{0.47} & 0.53\std{0.19} & 0.46\std{0.08} & 0.49\std{0.13} \\
               & DATA   & 3.00\std{0.00} & 0.33\std{0.47} & 0.00\std{0.00} & 0.60\std{0.00} & 0.92\std{0.12} & 0.72\std{0.04} \\
               & DATA+CoT   & 4.33\std{0.58} & 0.83\std{0.29} & 0.17\std{0.29} & 0.87\std{0.12} & 0.81\std{0.02} & 0.84\std{0.06} \\
               & COAT   & 4.00\std{0.82} & 0.33\std{0.47} & 0.00\std{0.00} & 0.80\std{0.16} & 0.93\std{0.09} & 0.85\std{0.11} \\ \midrule
GPT-3.5        & META   & 3.33\std{1.25} & 0.33\std{0.47} & 4.33\std{1.25} & 0.67\std{0.25} & 0.42\std{0.12} & 0.51\std{0.15} \\
               & DATA   & 2.67\std{0.47} & 0.67\std{0.47} & 0.00\std{0.00} & 0.53\std{0.09} & 0.81\std{0.14} & 0.64\std{0.10} \\
               & DATA+CoT  & 5.00\std{0.00} & 1.00\std{0.00} & 1.33\std{0.58} & 1.00\std{0.00} & 0.68\std{0.05} & 0.81\std{0.04} \\
               & COAT   & 3.67\std{0.47} & 0.00\std{0.00} & 0.00\std{0.00} & 0.73\std{0.09} & 1.00\std{0.00} & 0.84\std{0.07} \\ \midrule
Mistral-Large  & META   & 2.00\std{0.82} & 0.67\std{0.47} & 2.67\std{0.94} & 0.40\std{0.16} & 0.37\std{0.09} & 0.38\std{0.12} \\
               & DATA   & 3.00\std{0.00} & 0.33\std{0.47} & 0.00\std{0.00} & 0.60\std{0.00} & 0.92\std{0.12} & 0.72\std{0.04} \\
               & DATA+CoT  & 4.33\std{0.58}  & 1.00\std{0.00}  & 0.67\std{0.58}  & 0.87\std{0.12}  & 0.73\std{0.07}  & 0.79\std{0.05} \\
               & COAT   & 4.33\std{0.47} & 0.00\std{0.00} & 0.00\std{0.00} & 0.87\std{0.09} & 1.00\std{0.00} & 0.93\std{0.05} \\ \midrule
Mistral-Medium & META   & 3.00\std{0.00} & 0.67\std{0.47} & 1.67\std{1.25} & 0.60\std{0.00} & 0.59\std{0.13} & 0.59\std{0.07} \\
               & DATA   & 3.00\std{0.00} & 0.67\std{0.47} & 0.00\std{0.00} & 0.60\std{0.00} & 0.83\std{0.12} & 0.69\std{0.04} \\
               & DATA+CoT   & 4.33\std{0.58} & 1.00\std{0.00} & 0.67\std{0.58} & 0.87\std{0.12} & 0.73\std{0.07} & 0.79\std{0.05} \\
               & COAT   & 4.67\std{0.47} & 0.00\std{0.00} & 0.00\std{0.00} & 0.93\std{0.09} & 1.00\std{0.00} & 0.96\std{0.05} \\ \midrule
\llamathree-70b    & META   & 2.67\std{0.47} & 0.33\std{0.47} & 4.67\std{0.47} & 0.53\std{0.09} & 0.35\std{0.06} & 0.42\std{0.07} \\
               & DATA   & 2.67\std{1.25} & 0.33\std{0.47} & 0.00\std{0.00} & 0.53\std{0.25} & 0.93\std{0.09} & 0.63\std{0.21} \\
               & DATA+CoT  & 2.67\std{0.58}  & 0.67\std{0.58}  & 1.00\std{0.00}  & 0.53\std{0.12}  & 0.62\std{0.13}  & 0.57\std{0.11} \\
               & COAT   & 3.67\std{0.47} & 0.33\std{0.47} & 0.00\std{0.00} & 0.73\std{0.09} & 0.93\std{0.09} & 0.81\std{0.06} \\ \midrule
\llama-70b    & META   & 2.33\std{0.47} & 0.67\std{0.47} & 4.67\std{1.25} & 0.47\std{0.09} & 0.32\std{0.07} & 0.37\std{0.06} \\
               & DATA   & 2.33\std{0.94} & 0.67\std{0.47} & 0.00\std{0.00} & 0.47\std{0.19} & 0.75\std{0.20} & 0.57\std{0.20} \\
               & DATA+CoT  & 3.00\std{1.73}  & 0.67\std{0.58}  & 0.33\std{0.58}  & 0.60\std{0.35}  & 0.71\std{0.34}  & 0.65\std{0.35} \\
               & COAT   & 3.00\std{0.00} & 0.67\std{0.47} & 0.00\std{0.00} & 0.60\std{0.00} & 0.83\std{0.12} & 0.69\std{0.04} \\ \midrule
Qwen-1.5-110B  & META   & 2.00\std{0.00} & 1.00\std{0.00} & 4.00\std{0.00} & 0.40\std{0.00} & 0.29\std{0.00} & 0.33\std{0.00} \\
               & DATA   & 3.00\std{0.82} & 1.00\std{0.00} & 0.00\std{0.00} & 0.60\std{0.16} & 0.74\std{0.05} & 0.66\std{0.12} \\
               & DATA+CoT  & 3.67\std{1.53}  & 0.67\std{0.58}  & 0.00\std{0.00}  & 0.73\std{0.31}  & 0.83\std{0.17}  & 0.77\std{0.23} \\
               & COAT   & 4.00\std{0.82} & 0.33\std{0.47} & 0.00\std{0.00} & 0.80\std{0.16} & 0.93\std{0.09} & 0.85\std{0.11} \\ \midrule
DeepSeek-V2    & META   & 2.33\std{0.47} & 1.00\std{0.00} & 3.00\std{0.82} & 0.47\std{0.09} & 0.37\std{0.03} & 0.41\std{0.04} \\
               & DATA   & 3.33\std{0.47} & 0.67\std{0.47} & 0.00\std{0.00} & 0.67\std{0.09} & 0.85\std{0.11} & 0.74\std{0.05} \\
               & DATA+CoT$^1$  & 5.00\std{0.00}  & 1.00\std{0.00}  & 0.67\std{0.58}  & 1.00\std{0.00}  & 0.75\std{0.07}  & 0.86\std{0.04} \\
               & COAT   & 3.67\std{0.47} & 0.67\std{0.47} & 0.00\std{0.00} & 0.73\std{0.09} & 0.87\std{0.09} & 0.78\std{0.02} \\ \midrule
Claude-3-Opus  & META   & 2.00\std{0.00} & 1.00\std{0.00} & 3.00\std{1.41} & 0.40\std{0.00} & 0.35\std{0.07} & 0.37\std{0.04} \\
               & DATA   & 3.33\std{0.47} & 0.33\std{0.47} & 0.00\std{0.00} & 0.67\std{0.09} & 0.93\std{0.09} & 0.77\std{0.02} \\
               & DATA+CoT  & 2.67\std{0.58}  & 0.33\std{0.58}  & 1.00\std{1.00}  & 0.53\std{0.12}  & 0.67\std{0.14}  & 0.59\std{0.13} \\
               & COAT   & 5.00\std{0.00} & 0.33\std{0.47} & 0.00\std{0.00} & 1.00\std{0.00} & 0.94\std{0.08} & 0.97\std{0.04} \\
\bottomrule
\end{tabular}}
    }
    \caption{Full Results on the Apple Gastronome Benchmark.}
    \label{tab:full_results_on_apple}
{\footnotesize \textit{ $^1$DeepSeek-V2 was no longer available when we added this baseline, so we used DeepSeek-V2.5 instead.}}
\end{table}

\runin{Finding 2: \ours reliably recovers the causal relationships}
We present quantitative and qualitative results in Table~\ref{table:apple_rel} and Figure~\ref{fig:causal_graph_apple} respectively. Compared to directly adopting LLMs to reason the causal relations, \ours significantly boosts the causal relation recovery. 
Meanwhile, \ours maintains high performances based on various LLMs, which further demonstrates the effectiveness of the causal feedback in \ours to improve the robustness of this system. 
In fact, the causal feedback focuses on making maximal use of the rich knowledge of LLMs, and reducing the reliance on the reasoning capabilities of different LLMs, to assist with causal discovery. 

\begin{table}[t]
    \centering
    {\small
    \resizebox{\textwidth}{!}{
        \begin{tabular}{ll|ccccc}
        \toprule
        LLM            &  Method  & SHD         & SID         & Recall      & Precision   & F1          \\ \midrule
        GPT-4o         & pairwise & 5.33\std{2.62} & 0.00\std{0.00} & 1.00\std{0.00} & 0.45\std{0.07} & 0.62\std{0.07} \\
                       & COAT     & 1.00\std{0.00} & 0.00\std{0.00} & 1.00\std{0.00} & 0.79\std{0.03} & 0.89\std{0.02} \\ \midrule
        GPT-4          & pairwise & 8.67\std{2.05} & 1.33\std{1.89} & 0.93\std{0.09} & 0.35\std{0.05} & 0.51\std{0.06} \\
                       & COAT     & 2.33\std{0.47} & 1.00\std{1.41} & 0.93\std{0.09} & 0.70\std{0.09} & 0.79\std{0.04} \\ \midrule
        GPT-3.5        & pairwise & 5.67\std{4.03} & 3.00\std{2.94} & 0.67\std{0.24} & 0.49\std{0.36} & 0.55\std{0.32} \\
                       & COAT     & 2.00\std{0.82} & 2.67\std{2.05} & 0.75\std{0.20} & 0.63\std{0.21} & 0.68\std{0.21} \\ \midrule
        Mistral-Large  & pairwise & 7.33\std{1.89} & 0.00\std{0.00} & 1.00\std{0.00} & 0.38\std{0.03} & 0.55\std{0.03} \\
                       & COAT     & 2.00\std{0.00} & 1.33\std{1.25} & 0.85\std{0.11} & 0.74\std{0.05} & 0.78\std{0.02} \\ \midrule
        Mistral-Medium & pairwise & 6.00\std{1.63} & 2.00\std{0.00} & 0.51\std{0.13} & 0.37\std{0.04} & 0.42\std{0.07} \\
                       & COAT     & 0.33\std{0.47} & 0.00\std{0.00} & 1.00\std{0.00} & 0.94\std{0.08} & 0.97\std{0.04} \\ \midrule
        \llamathree-70b    & pairwise & 4.00\std{1.41} & 0.33\std{0.47} & 0.92\std{0.12} & 0.51\std{0.09} & 0.65\std{0.11} \\
                       & COAT     & 2.33\std{0.47} & 3.67\std{1.25} & 0.75\std{0.00} & 0.70\std{0.07} & 0.72\std{0.04} \\ \midrule
        \llama-70b    & pairwise & 3.33\std{1.25} & 1.00\std{1.41} & 0.89\std{0.16} & 0.48\std{0.11} & 0.62\std{0.13} \\
                       & COAT     & 1.00\std{0.82} & 0.33\std{0.47} & 0.89\std{0.16} & 0.81\std{0.14} & 0.84\std{0.14} \\ \midrule
        Qwen-1.5-110B  & pairwise & 6.67\std{2.62} & 2.33\std{1.25} & 0.62\std{0.03} & 0.38\std{0.09} & 0.47\std{0.07} \\
                       & COAT     & 4.00\std{1.63} & 7.00\std{4.32} & 0.47\std{0.34} & 0.47\std{0.34} & 0.70\std{0.10} \\ \midrule
        DeepSeek-V2    & pairwise & 5.00\std{1.41} & 1.33\std{1.89} & 0.93\std{0.09} & 0.47\std{0.05} & 0.62\std{0.06} \\
                       & COAT     & 1.33\std{1.25} & 2.33\std{1.70} & 0.85\std{0.11} & 0.87\std{0.19} & 0.85\std{0.14} \\ \midrule
        Claude-3-Opus  & pairwise & 7.33\std{1.25} & 0.33\std{0.47} & 0.93\std{0.09} & 0.40\std{0.05} & 0.56\std{0.07} \\
                       & COAT     & 1.33\std{0.47} & 0.00\std{0.00} & 1.00\std{0.00} & 0.79\std{0.06} & 0.88\std{0.04} \\ \bottomrule
        \end{tabular}}
    }
    \caption{Causal relation extraction results in \apple.}
    \label{table:apple_rel}
    \vspace{0.2in}
\end{table} 
\runin{Finding 3: LLMs are competent in proposing potential high-level factors}
As discussed in Sec.~\ref{subsec:theory}, there are two crucial abilities for LLMs in identifying potential high-level factors. The first one is to be aware of the existence of potential factors, and the second is to synthesize and describe these factors. Inspired by this observation, we propose two novel metrics to quantify LLMs' causal ability:
a) Perception that quantifies the ratio of \textit{valid} factors (satisfying Prop.~\ref{prop:mutual_info_down}) proposed by LLMs in each round; b) Capacity that measures the effective mutual information drop in Assumption~\ref{asp:LLM_good}. As shown in Figure~\ref{fig:quant_ca}, LLMs differ largely on the perception score while comparably on the capacity score. 
\edit{
In Table~\ref{tab:full_results_causal_metrics_on_apple} (Appendix~\ref{appdx:apple_results}), we also show the two metrics as well as the mutual information in each \ours iteration.
}

\begin{figure}[t]
    \subfigtopskip=2pt
    \centering
    \subfigure[Apple attributes Acc.]{
    \vspace{-0.8in}
        \includegraphics[width=0.3\textwidth]{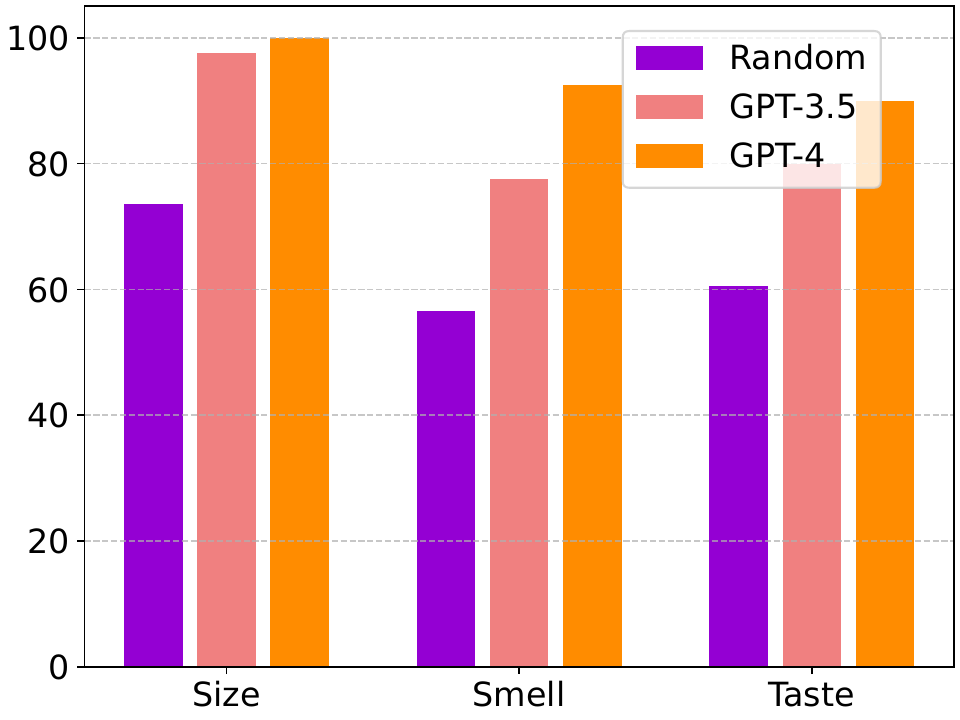}
        \label{fig:apple_attribute_acc}
    }
    \subfigure[Preference matchness Acc.]{
        \includegraphics[width=0.3\textwidth]{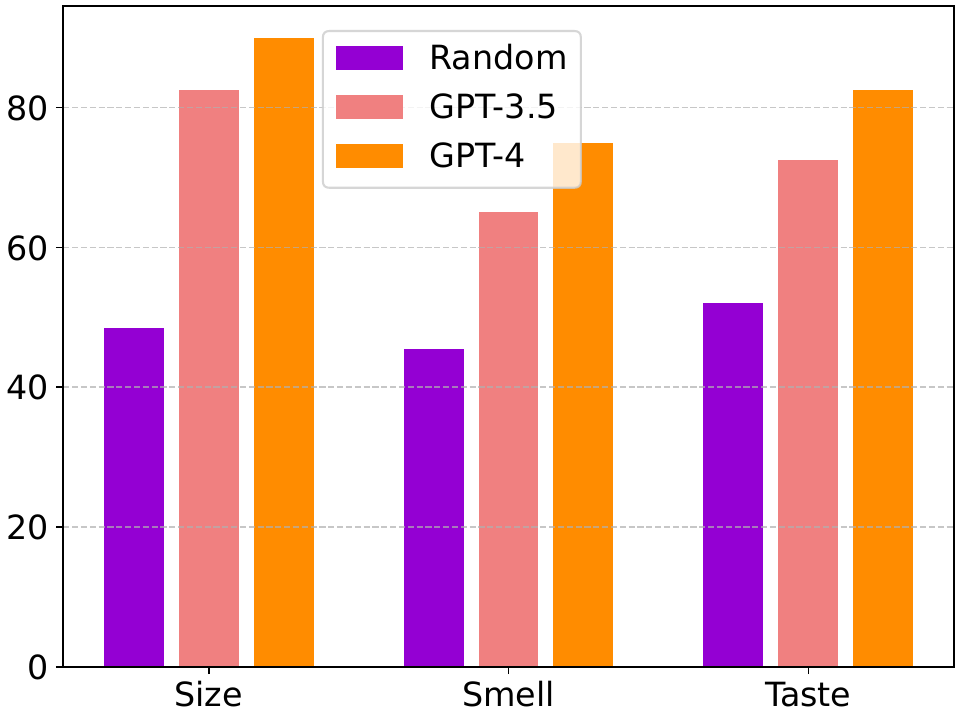}
        \label{fig:apple_match_acc}
     }
     \subfigure[LLM Ability Scatters ]{
        \includegraphics[width=0.32\textwidth]{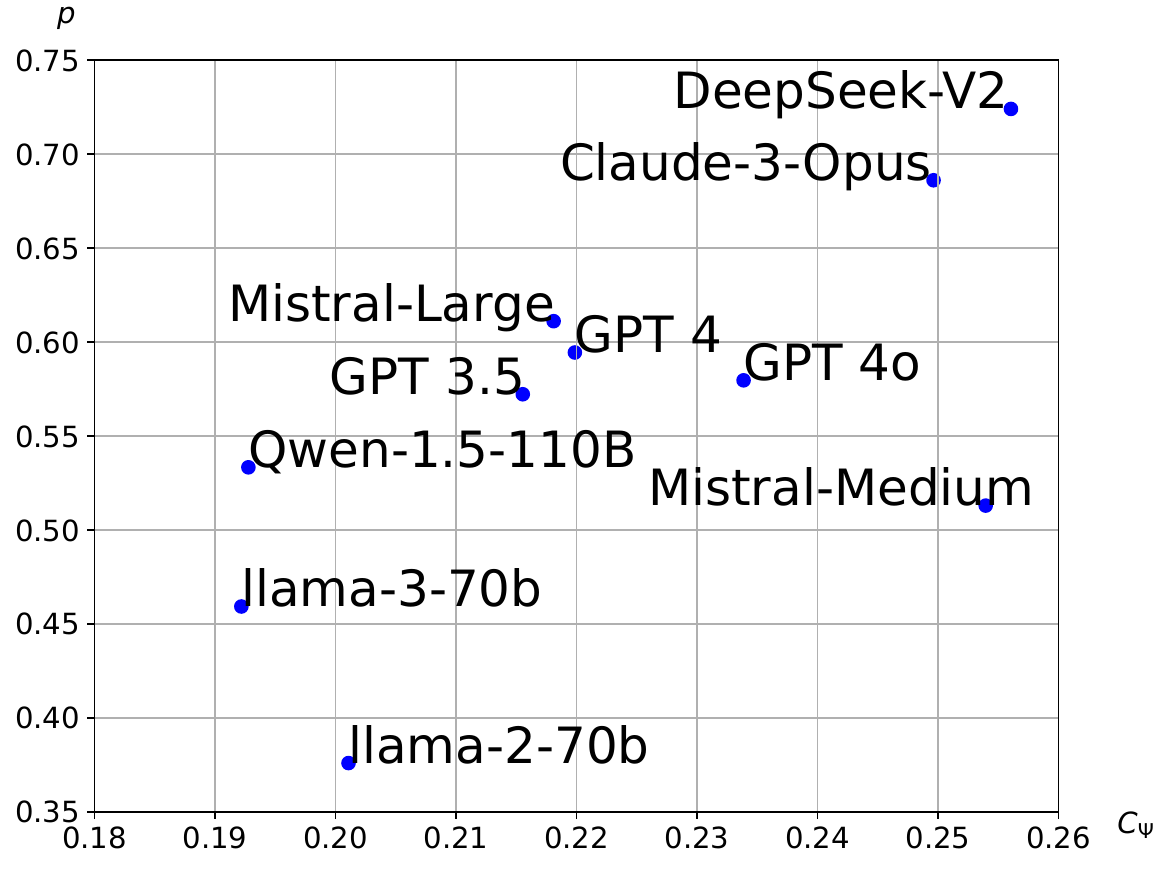}
        \label{fig:quant_ca}
    }
    \vspace{-0.1in}
    \caption{Quantitative evaluation of the causal capabilities of LLMs in \ours.}
    \label{fig:apple_llm_cap}
\end{figure}

\runin{Finding 4: LLMs can conduct effective factor annotation}
As shown in Figure~\ref{fig:apple_llm_cap}, both \chatgpt and \gptf annotate objective attributes well. Regarding subjective human preferences, the performances are still relatively high. Empirically, LLMs will not introduce new confounders. Table~\ref{table:apple_ind} presents the independence testing results between the annotation noises and the annotated features, and the noises themselves. It can be found that, the introduced noises are independent of the attributes, therefore, will not introduce much additional interference to the causal discovery procedure.

\begin{table}[t]
    \centering
    {\small
        \begin{tabular}{llcccccc}
            \toprule
            LLM                       & Test Object & T       & P-value \\
            \midrule
            \multirow{2}{*}{GPT-4}    & Feature     & 0.2828  & 0.9997  \\
                                      & Noise     & 0.0327  & 0.9962  \\
            \midrule
            \multirow{2}{*}{GPT-3.5} & Feature       & 0.4803  & 0.0325  \\
                                      & Noise       & 0.0446 & 0.9962  \\
            \bottomrule
        \end{tabular}%
        }
    \caption{Independence tests of the annotation noises with annotated features and other noises in \apple.}
    \label{tab:ind_test_annotation}
    \label{table:apple_ind}
\end{table}

\subsection{Eliciting Partial Ancestral Graph}
\label{subsec:applesetting2}

Similar to the previous setting, we consider the scenario where the rating score $Y$ and one additional factor $T$ (like taste) are observed. The PAG in such two-node case is always $T \circarrowcirc Y$ if they are correlated, because we can only conduct one CI test given only two variables. And thus, we have little information about edge marks. The goal of this setting is to extend the node set from unstructured data, so that the skeleton and orientation in the original graph can be further refined.
For example, in Figure~\ref{fig:coatpag_example-COATPAG_loop2} and Figure~\ref{fig:coatpag_example-COATPAG_loop3}, by extending node $E$, the skeleton between node $C$ and node $D$ is refined, and the orientation on nodes $B$, $Y$, and $A$ are identified.

\runin{Setup and Benchmark Construction}
As shown in Figure~\ref{fig:setting2_graph_5_graphs}, we design five different causal graphs with respect to different FCI orientation rules. Each causal graph consists of 4 nodes: 2 numerical variables $V_1$ and $V_2$ that will be in the original PAG, and two factors $X_1$ and $X_2$ that are expected to be elicited from unstructured data. We consider two types of causal relationships between $V_1$ and $V_2$: (1) direct cause, i.e., $V1 \rightarrow V2$; and (2) latent confounding, i.e., $V1 \leftrightarrow V2$. Given $X_1$ and $X_2$, the edge marks can be exactly identified in all graphs except for graph 5 which is $V1 \circrightarrow V2$ in the extended PAG. 

We choose $V_2$ to be score across all graphs, and randomly select 3 factors from Figure~\ref{fig:causal_graph_apple-Ground-truth} to be the remaining nodes. Other factors will be generated as independent variables. The concrete specification for each graph is listed in Table~\ref{Tab:apple-pag-specification}. In this setting, we generate 2000 samples using GLM-4-Flash, a light-weight LLM with a free API. The review generating prompt is constructed in the same way as the setting in Section~\ref{subsec:applesetting1}.

\begin{table}[tbp]
\centering
{\small
\resizebox{\textwidth}{!}{%
\begin{tabular}{@{}llllll@{}}
\toprule
        & V1               & V2    & X1               & X2               & FCI Orientation Rules \\  \midrule
graph 1 & Taste            & Score & Nutrition        & Market Potential & $\gR$0,  $\gR$1         \\
graph 2 & Taste            & Score & Juiciness        & Size             & $\gR$0                 \\
graph 3 & Smell            & Score & Market Potential & Nutrition        & $\gR$0, $\gR$4          \\
graph 4 & Juiciness        & Score & Nutrition        & Smell            & $\gR$0, $\gR$2, $\gR$4   \\
graph 5 & Market Potential & Score & Taste            & Juiciness        & $\gR$3                 \\   \bottomrule
\end{tabular}%
}
}
\caption{The variable specification for different PAGs.}
\label{Tab:apple-pag-specification}
\end{table} 

\begin{figure}[t]
    \subfigtopskip=2pt
    \centering
    \subfigure[Causal graph 1]{
        \begin{minipage}[c][2.5cm][c]{0.28\linewidth}
            \centering
            \begin{tikzpicture}[
                baseline=(current bounding box.center),
                node distance=0.6cm,
                every node/.style={circle, draw, minimum size=6mm, inner sep=1pt},
                edge/.style={->, >=Latex, semithick},
                observed/.style={circle, draw, fill=black!10, minimum size=6mm, inner sep=1pt},
            ]
            \node (v1) [observed] {$V_1$};
            \node[above left=0.2cm and 0.6cm of v1] (x2) {$X_2$};
            \node[below left=0.2cm and 0.6cm of v1] (x1) {$X_1$};
            \node[below right=0.2cm and 0.6cm of v1] (v2) [observed]  {$V_2$};
            
            \draw[edge] (x1) -- (v2);
            \draw[edge] (x1) -- (v1);
            \draw[edge] (x2) -- (v1);
            \draw[edge] (v1) -- (v2);
            \end{tikzpicture}
        \end{minipage}
        \label{fig:setting2_graph_1_r01}
    }
    \subfigure[Causal graph 2]{
        \begin{minipage}[c][2.5cm][c]{0.28\linewidth}
            \centering
            \begin{tikzpicture}[
                baseline=(current bounding box.center),
                node distance=0.6cm,
                every node/.style={circle, draw, minimum size=6mm, inner sep=1pt},
                edge/.style={->, >=Latex, semithick},
                observed/.style={circle, draw, fill=black!10, minimum size=6mm, inner sep=1pt},
                biedge/.style={<->, >=Latex, semithick},
            ]
            \node (v1) [observed] {$V_1$};
            \node[above right=0.2cm and 0.6cm of v1] (v2) [observed]  {$V_2$};
            \node[left=0.6cm of v1] (x1) {$X_1$};
            \node[above left=0.2cm and 0.6cm of v2] (x2) {$X_2$};
            
            \draw[edge] (x2) -- (v2);
            \draw[edge] (x1) -- (v1);
            \draw[biedge] (v1) -- (v2);
            \end{tikzpicture}
        \end{minipage}
        \label{fig:setting2_graph_2_r0}
    }\\
    \subfigure[Causal graph 3]{
        \begin{minipage}[c][2.5cm][c]{0.28\linewidth}
            \centering
            \begin{tikzpicture}[
                baseline=(current bounding box.center),
                node distance=0.6cm,
                every node/.style={circle, draw, minimum size=6mm, inner sep=1pt},
                edge/.style={->, >=Latex, semithick},
                observed/.style={circle, draw, fill=black!10, minimum size=6mm, inner sep=1pt},
                biedge/.style={<->, >=Latex, semithick},
            ]
            \node (x2) {$X_2$};
            \node [right=0.6cm of x2] (x1) {$X_1$};
            \node [above right=0.2cm and 0.6cm of x1] (v1) [observed] {$V_1$};
            \node [below right=0.2cm and 0.6cm of v1] (v2) [observed] {$V_2$};
            
            \draw[biedge] (v1) -- (v2);
            \draw[biedge] (v1) -- (x1);
            \draw[edge] (x2) -- (x1);
            \draw[edge] (x1) -- (v2);
            \end{tikzpicture}
        \end{minipage}
        \label{fig:setting2_graph_3_r04}
    }
    \subfigure[Causal graph 4]{
        \begin{minipage}[c][2.5cm][c]{0.28\linewidth}
            \centering
            \begin{tikzpicture}[
                baseline=(current bounding box.center),
                node distance=0.6cm,
                every node/.style={circle, draw, minimum size=6mm, inner sep=1pt},
                edge/.style={->, >=Latex, semithick},
                observed/.style={circle, draw, fill=black!10, minimum size=6mm, inner sep=1pt},
                biedge/.style={<->, >=Latex, semithick},
            ]
            \node (x2) {$X_2$};
            \node [right=0.6cm of x2] (x1) {$X_1$};
            \node [above right=0.2cm and 0.6cm of x1] (v1) [observed] {$V_1$};
            \node [below right=0.2cm and 0.6cm of v1] (v2) [observed] {$V_2$};
            
            \draw[edge] (v1) -- (v2);
            \draw[biedge] (v1) -- (x1);
            \draw[edge] (x2) -- (x1);
            \draw[edge] (x1) -- (v2);
            \end{tikzpicture}
        \end{minipage}
        \label{fig:setting2_graph_4_r024}
    }
    \subfigure[Causal graph 5]{
        \begin{minipage}[c][2.5cm][c]{0.28\linewidth}
            \centering
            \begin{tikzpicture}[
                baseline=(current bounding box.center),
                node distance=0.6cm,
                every node/.style={circle, draw, minimum size=6mm, inner sep=1pt},
                edge/.style={->, >=Latex, semithick},
                observed/.style={circle, draw, fill=black!10, minimum size=6mm, inner sep=1pt},
                biedge/.style={<->, >=Latex, semithick},
            ]
            \node (v1) [observed]  {$V_1$};
            \node [right=0.6cm of v1] (v2) [observed]  {$V_2$};
            \node [above left=0.2cm and 0.6cm of v1] (x1) {$X_1$};
            \node [below left=0.2cm and 0.6cm of v1] (x2) {$X_2$};

            \draw[edge] (v1) -- (v2);
            \draw[edge] (v1) -- (x1);
            \draw[edge]  (x2) to [bend right] (v2); 
            \draw[biedge] (x2) -- (v1);
            \draw[biedge]  (x1) to [bend left] (v2); 
            
            \end{tikzpicture}
        \end{minipage}
        \label{fig:setting2_graph_5_r3}
    }
    \caption{The five causal graphs used in the setting on extending PAG. The filled nodes are present in the original PAG.}
    \label{fig:setting2_graph_5_graphs}
    \vspace{0.1in}
\end{figure}

\runin{Baselines and Evaluation} We consider two baselines in this setting. The first considered one is \textbf{DATA+CoT} that prompts LLMs to decide the edge marks by analyzing a portion of samples from each value combination of $V_1$ and $V_2$. The second one is \textbf{DATA+FCI}, the ablation version of single round COAT without verification or iteration procedure.

For evaluation, we report the precision, recall, and F1 score for arrow heads identification between $V_1$ and $V_2$, and the purpose is to confirm our previous identifiability results. We also report the exact edge mark accuracy, to see whether each method can predict arrow head, arrow tail and circle mark in the edge between $V_1$ and $V_2$.

\runin{Finding 5: \ours reliably improves edge orientation by extending useful nodes} We present the averaged result over all 5 graphs in Table~\ref{tab:apple-PAG}\minor{(a)}. In the arrow-head metrics, \ours shows best performance and lowest standard deviation compared with the CoT baseline and its own ablation version. Compared with \emph{DATA+COT}, \emph{\ours} has significant improvement, indicating the effectiveness of \ours iteration.

\runin{Finding 6: LLMs' reasoning fails when reality deviates from its prior} Instead of extending the node set in PAG, \emph{DATA+CoT} improves orientation by LLMs' own reasoning ability with unstructured samples from each $(V1, V2)$ value combination. It can achieve competitive results when the causal graph aligns with commonsense, e.g., in graph 1, taste is the cause of the rating score. However, its performance can drop severely when the structure is different. For example, in graph 2, there is a latent confounder between taste and rating score; in graph 5, the market potential is the cause of the rating score. In the latter example, LLMs tend to believe the rating score determines the market potential. Therefore, it might be important to incorporate statistical methods to avoid subtle bias from prior.   

\begin{table}[!htb]
    \centering
    \vspace{0.5em}
    \textbf{(a) Average over graphs}\par\vspace{0.3em}
    \begin{tabular}{@{}cccccc@{}}
    \toprule
                                 &          & \textbf{Edge Mark} & \multicolumn{3}{c}{\textbf{Arrow Head}} \\ \cmidrule(lr){3-3} \cmidrule(l){4-6}
                                 &          & \textbf{Accuracy}  & \textbf{Precision} & \textbf{Recall} & \textbf{F1} \\
    \midrule
    \multirow{3}{*}{DeepSeek-V3} & DATA+CoT & 48 \std{40} & 52 \std{36} & 54 \std{34} & 53 \std{35} \\
                                 & DATA+FCI & 66 \std{19} & 88 \std{11} & 76 \std{22} & 80 \std{17} \\
                                 & COAT     & 88 \std{8}  & 96 \std{9}  & 88 \std{11} & 91 \std{9}  \\ \cmidrule(l){1-6}
    \multirow{3}{*}{GLM-4-Plus}  & DATA+CoT & 60 \std{16} & 74 \std{17} & 70 \std{22} & 69 \std{16} \\
                                 & DATA+FCI & 54 \std{30} & 68 \std{27} & 58 \std{25} & 61 \std{24} \\
                                 & COAT     & 68 \std{19} & 76 \std{17} & 74 \std{19} & 75 \std{19} \\ \cmidrule(l){1-6}
    \multirow{3}{*}{GPT-4.1-mini}& DATA+CoT & 43 \std{24} & 80 \std{45} & 65 \std{42} & 70 \std{41} \\
                                 & DATA+FCI & 63 \std{35} & 73 \std{28} & 66 \std{34} & 66 \std{34} \\
                                 & COAT     & 74 \std{24} & 93 \std{11} & 77 \std{26} & 76 \std{25} \\
    \bottomrule
    \end{tabular}
    \caption{Results on extending Partial Ancestral Graph.}\label{tab:apple-PAG}
\end{table}

\begin{table}[!htb]
    \addtocounter{table}{-1}
    \renewcommand{\theHtable}{\thetable-contB}
    \centering
    \vspace{0.5em}
    \textbf{(b) Graph 1}\par\vspace{0.3em}
    \begin{tabular}{@{}cccccc@{}}
    \toprule
                                 &          & \textbf{Edge Mark} & \multicolumn{3}{c}{\textbf{Arrow Head}} \\ \cmidrule(lr){3-3} \cmidrule(l){4-6}
                                 &          & \textbf{Accuracy}  & \textbf{Precision} & \textbf{Recall} & \textbf{F1} \\
    \midrule
    \multirow{3}{*}{DeepSeek-V3} & DATA+CoT & 80 \std{45} & 80 \std{45}  & 80 \std{45}  & 80 \std{45} \\
                                 & DATA+FCI & 80 \std{45} & 80 \std{45}  & 80 \std{45}  & 80 \std{45} \\
                                 & COAT     & 100 \std{0} & 100 \std{0}  & 100 \std{0}  & 100 \std{0} \\ \cmidrule(l){1-6}
    \multirow{3}{*}{GLM-4-Plus}  & DATA+CoT & 50 \std{0}  & 90 \std{22}  & 100 \std{0}  & 93 \std{15} \\
                                 & DATA+FCI & 80 \std{45} & 80 \std{45}  & 80 \std{45}  & 80 \std{45} \\
                                 & COAT     & 60 \std{55} & 60 \std{55}  & 60 \std{55}  & 60 \std{55} \\ \cmidrule(l){1-6}
    \multirow{3}{*}{GPT-4.1-mini}& DATA+CoT & 50 \std{0}  & 100 \std{0}  & 100 \std{0}  & 100 \std{0} \\
                                 & DATA+FCI & 100 \std{0} & 100 \std{0}  & 100 \std{0}  & 100 \std{0} \\
                                 & COAT     & 100 \std{0} & 100 \std{0}  & 100 \std{0}  & 100 \std{0} \\
    \bottomrule
    \end{tabular}
    \caption{Results on extending Partial Ancestral Graph (continued).}
\end{table}

\begin{table}[!htb]
    \addtocounter{table}{-1}
    \renewcommand{\theHtable}{\thetable-contC}
    \centering
    \vspace{0.5em}
    \textbf{(c) Graph 2}\par\vspace{0.3em}
    \begin{tabular}{@{}cccccc@{}}
    \toprule
                                 &          & \textbf{Edge Mark} & \multicolumn{3}{c}{\textbf{Arrow Head}} \\ \cmidrule(lr){3-3} \cmidrule(l){4-6}
                                 &          & \textbf{Accuracy}  & \textbf{Precision} & \textbf{Recall} & \textbf{F1} \\
    \midrule
    \multirow{3}{*}{DeepSeek-V3} & DATA+CoT & 30 \std{45} & 30 \std{45}  & 30 \std{45} & 30 \std{45} \\
                                 & DATA+FCI & 80 \std{27} & 100 \std{0}  & 80 \std{27} & 87 \std{18} \\
                                 & COAT     & 80 \std{27} & 100 \std{0}  & 80 \std{27} & 87 \std{18} \\ \cmidrule(l){1-6}
    \multirow{3}{*}{GLM-4-Plus}  & DATA+CoT & 40 \std{42} & 60 \std{55}  & 40 \std{42} & 47 \std{45} \\
                                 & DATA+FCI & 90 \std{22} & 100 \std{0}  & 90 \std{22} & 93 \std{15} \\
                                 & COAT     & 100 \std{0} & 100 \std{0}  & 100 \std{0} & 100 \std{0} \\ \cmidrule(l){1-6}
    \multirow{3}{*}{GPT-4.1-mini}& DATA+CoT & 50 \std{0}  & 100 \std{0}  & 50 \std{0}  & 67 \std{0}  \\
                                 & DATA+FCI & 100 \std{0} & 100 \std{0}  & 100 \std{0} & 100 \std{0} \\
                                 & COAT     & 100 \std{0} & 100 \std{0}  & 100 \std{0} & 100 \std{0} \\
    \bottomrule
    \end{tabular}
    \caption{Results on extending Partial Ancestral Graph (continued).}
\end{table}

\begin{table}[!htb]
    \addtocounter{table}{-1}
    \renewcommand{\theHtable}{\thetable-contD}
    \centering
    \vspace{0.5em}
    \textbf{(d) Graph 3}\par\vspace{0.3em}
    \begin{tabular}{@{}cccccc@{}}
    \toprule
                                 &          & \textbf{Edge Mark} & \multicolumn{3}{c}{\textbf{Arrow Head}} \\ \cmidrule(lr){3-3} \cmidrule(l){4-6}
                                 &          & \textbf{Accuracy}  & \textbf{Precision} & \textbf{Recall} & \textbf{F1} \\
    \midrule
    \multirow{3}{*}{DeepSeek-V3} & DATA+CoT & 20 \std{45} & 20 \std{45}  & 20 \std{45} & 20 \std{45} \\
                                 & DATA+FCI & 40 \std{22} & 80 \std{45}  & 40 \std{22} & 53 \std{30} \\
                                 & COAT     & 80 \std{27} & 100 \std{0}  & 80 \std{27} & 87 \std{18} \\ \cmidrule(l){1-6}
    \multirow{3}{*}{GLM-4-Plus}  & DATA+CoT & 80 \std{27} & 100 \std{0}  & 80 \std{27} & 87 \std{18} \\
                                 & DATA+FCI & 40 \std{22} & 80 \std{45}  & 40 \std{22} & 53 \std{30} \\
                                 & COAT     & 50 \std{50} & 60 \std{55}  & 50 \std{50} & 53 \std{51} \\ \cmidrule(l){1-6}
    \multirow{3}{*}{GPT-4.1-mini}& DATA+CoT & 63 \std{25} & 100 \std{0}  & 75 \std{29} & 83 \std{19} \\
                                 & DATA+FCI & 50 \std{41} & 75 \std{50}  & 63 \std{48} & 67 \std{47} \\
                                 & COAT     & 50 \std{41} & 75 \std{50}  & 63 \std{48} & 67 \std{47} \\
    \bottomrule
    \end{tabular}
    \caption{Results on extending Partial Ancestral Graph (continued).}
\end{table}

\begin{table}[!htb]
    \addtocounter{table}{-1}
    \renewcommand{\theHtable}{\thetable-contE}
    \centering
    \vspace{0.5em}
    \textbf{(e) Graph 4}\par\vspace{0.3em}
    \begin{tabular}{@{}cccccc@{}}
    \toprule
                                 &          & \textbf{Edge Mark} & \multicolumn{3}{c}{\textbf{Arrow Head}} \\ \cmidrule(lr){3-3} \cmidrule(l){4-6}
                                 &          & \textbf{Accuracy}  & \textbf{Precision} & \textbf{Recall} & \textbf{F1} \\
    \midrule
    \multirow{3}{*}{DeepSeek-V3} & DATA+CoT & 100 \std{0}  & 100 \std{0}  & 100 \std{0}  & 100 \std{0} \\
                                 & DATA+FCI & 80 \std{27}  & 100 \std{0}  & 100 \std{0}  & 100 \std{0} \\
                                 & COAT     & 90 \std{22}  & 100 \std{0}  & 100 \std{0}  & 100 \std{0} \\ \cmidrule(l){1-6}
    \multirow{3}{*}{GLM-4-Plus}  & DATA+CoT & 70 \std{45}  & 80 \std{45}  & 70 \std{45}  & 73 \std{43} \\
                                 & DATA+FCI & 20 \std{27}  & 40 \std{55}  & 40 \std{55}  & 40 \std{55} \\
                                 & COAT     & 60 \std{42}  & 80 \std{45}  & 80 \std{45}  & 80 \std{45} \\ \cmidrule(l){1-6}
    \multirow{3}{*}{GPT-4.1-mini}& DATA+CoT & 50 \std{0}   & 100 \std{0}  & 100 \std{0}  & 100 \std{0} \\
                                 & DATA+FCI & 40 \std{55}  & 40 \std{55}  & 40 \std{55}  & 40 \std{55} \\
                                 & COAT     & 60 \std{22}  & 90 \std{22}  & 80 \std{45}  & 73 \std{43} \\
    \bottomrule
    \end{tabular}
    \caption{Results on extending Partial Ancestral Graph (continued).}
\end{table}

\begin{table}[!htb]
    \addtocounter{table}{-1}
    \renewcommand{\theHtable}{\thetable-contF}
    \centering
    \vspace{0.5em}
    \textbf{(f) Graph 5}\par\vspace{0.3em}
    \begin{tabular}{@{}cccccc@{}}
    \toprule
                                 &          & \textbf{Edge Mark} & \multicolumn{3}{c}{\textbf{Arrow Head}} \\ \cmidrule(lr){3-3} \cmidrule(l){4-6}
                                 &          & \textbf{Accuracy}  & \textbf{Precision} & \textbf{Recall} & \textbf{F1} \\
    \midrule
    \multirow{3}{*}{DeepSeek-V3} & DATA+CoT & 10 \std{22} & 30 \std{45}  & 40 \std{55}  & 33 \std{47} \\
                                 & DATA+FCI & 50 \std{0}  & 80 \std{45}  & 80 \std{45}  & 80 \std{45} \\
                                 & COAT     & 90 \std{22} & 80 \std{45}  & 80 \std{45}  & 80 \std{45} \\ \cmidrule(l){1-6}
    \multirow{3}{*}{GLM-4-Plus}  & DATA+CoT & 60 \std{55} & 60 \std{55}  & 60 \std{55}  & 60 \std{55} \\
                                 & DATA+FCI & 40 \std{55} & 40 \std{55}  & 40 \std{55}  & 40 \std{55} \\
                                 & COAT     & 70 \std{45} & 80 \std{45}  & 80 \std{45}  & 80 \std{45} \\ \cmidrule(l){1-6}
    \multirow{3}{*}{GPT-4.1-mini}& DATA+CoT & 0 \std{0}   & 0 \std{0}    & 0 \std{0}    & 0 \std{0}   \\
                                 & DATA+FCI & 25 \std{29} & 50 \std{58}  & 25 \std{50}  & 25 \std{50} \\
                                 & COAT     & 60 \std{22} & 100 \std{0}  & 40 \std{55}  & 40 \std{55} \\
    \bottomrule
    \end{tabular}
    \caption{Results on extending Partial Ancestral Graph (continued).}
\end{table}

\edit{

\begin{figure}[t]
    \centering
    \subfigure[Ground Truth]{
        \begin{minipage}[c]{0.3\linewidth}
            \centering
            \includegraphics[width=\linewidth,height=0.3\textheight,keepaspectratio]{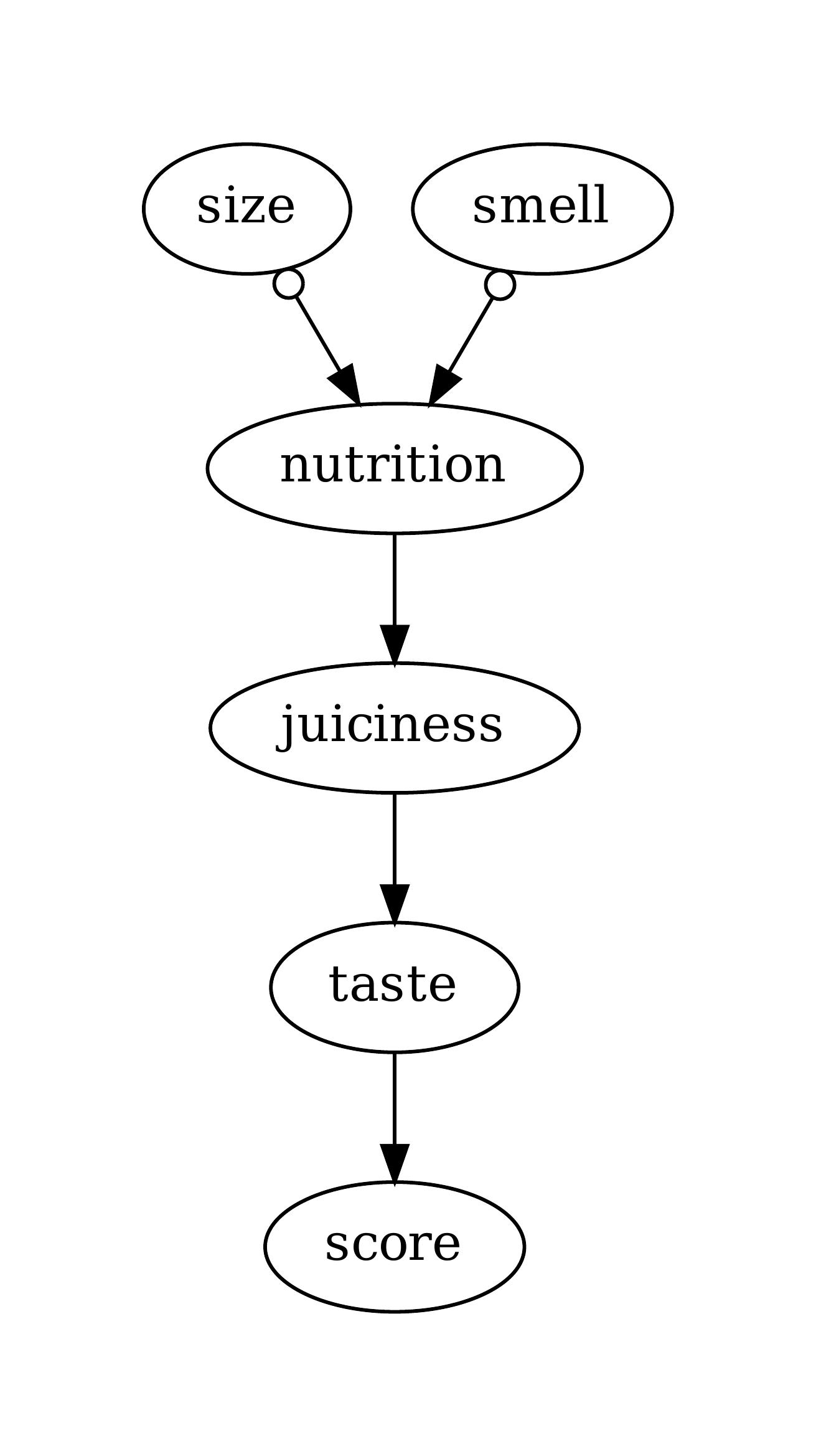}
        \end{minipage}
        \label{fig:compare-pag-gt}
    }%
    \subfigure[\oursMB]{
        \begin{minipage}[c]{0.3\linewidth}
            \centering
            \includegraphics[width=\linewidth,height=0.3\textheight,keepaspectratio]{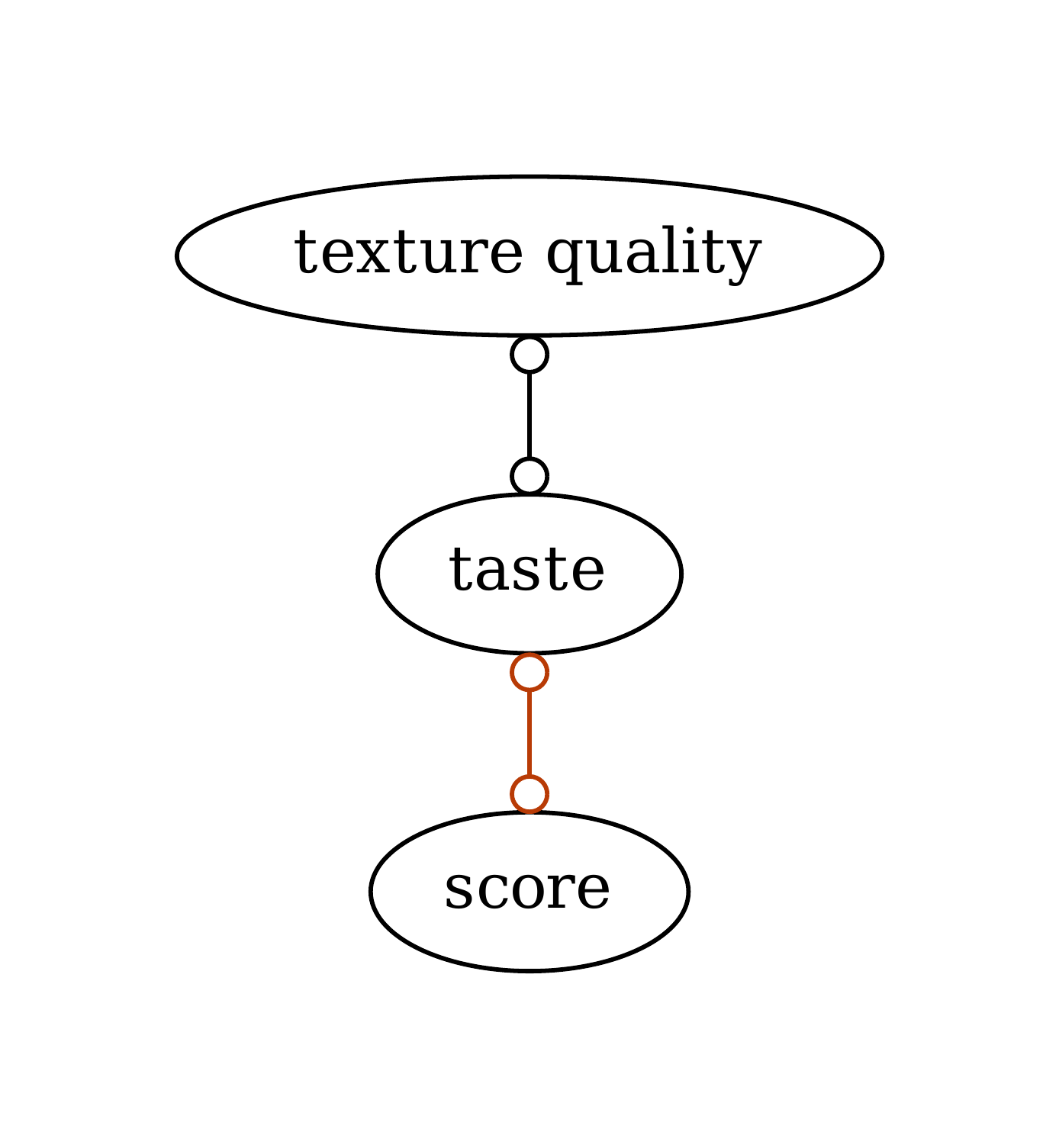}
        \end{minipage}
        \label{fig:compare-pag-mb}
    }%
    \subfigure[\oursPAG]{
        \begin{minipage}[c]{0.3\linewidth}
            \centering
            \includegraphics[width=\linewidth,height=0.3\textheight,keepaspectratio]{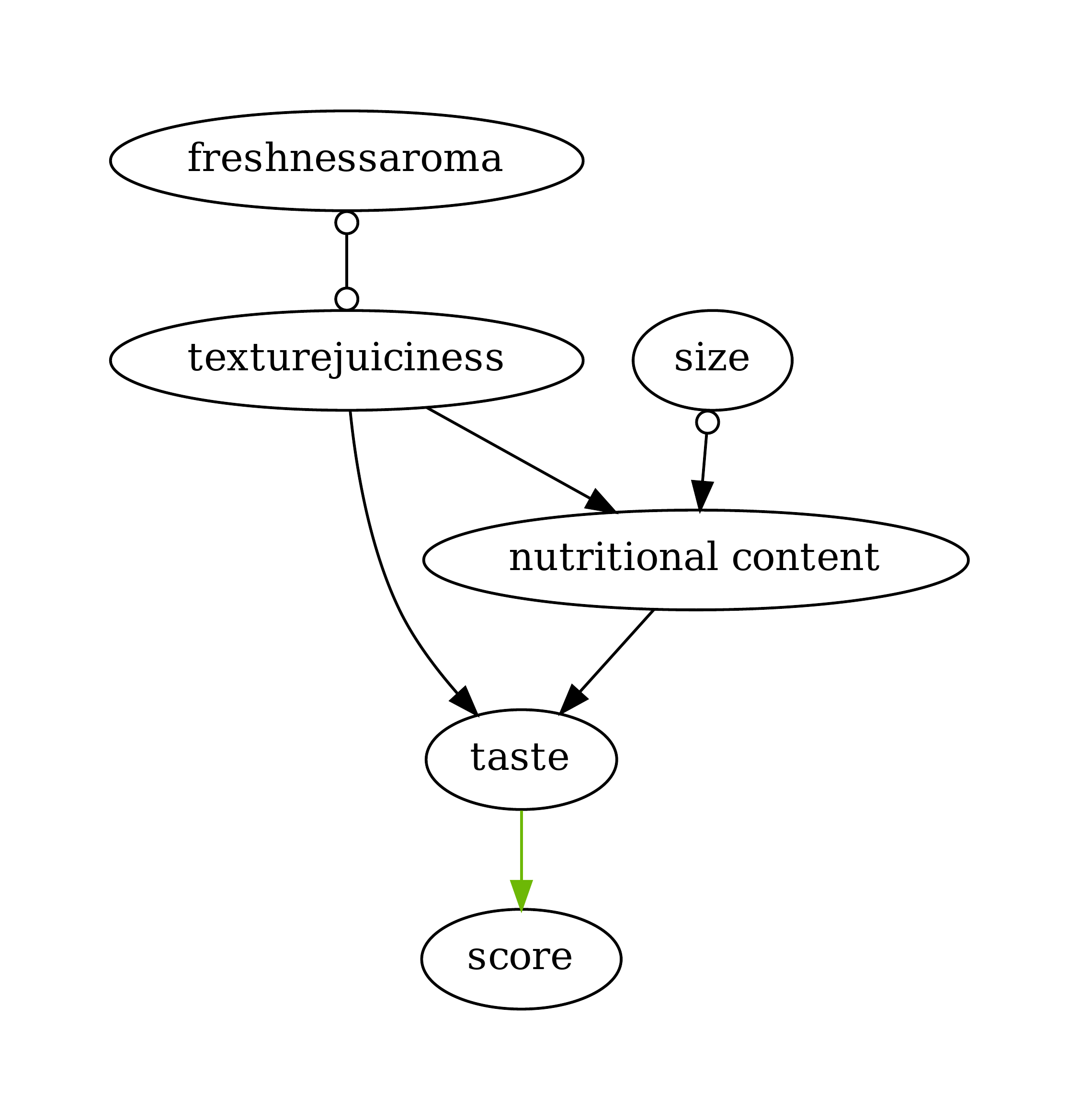}
        \end{minipage}
        \label{fig:compare-pag-pag}
    }

    \caption{\edit{The comparison between \oursMB and \oursPAG in the causal discovery setting.}}
    \label{fig:compare-mb-pag}
\end{figure}
 
\runin{Comparison between \oursMB and \oursPAG} We empirically demonstrate the differences between the two algorithms in a variant of \apple data set which has a different causal structure as shown in Figure~\ref{fig:compare-mb-pag}~(a). 
The initial variable set is $\{\texttt{score}, \texttt{taste}\}$. 
When applying \oursMB on \texttt{taste}, it introduces one more factor \texttt{textual quality} that is semantically close to \texttt{juiciness}; 
however, the causal direction between \texttt{taste} and \texttt{score} is not identified, as the red edge shown in Figure~\ref{fig:compare-mb-pag}~(b).  
On the contrary, applying \oursPAG reveals the causal direction in Figure~\ref{fig:compare-mb-pag}~(c), which is colored by green, by extending the uncertain factors in the graph. 
These results demonstrate that \oursPAG can adaptively extend necessary factors that may be beyond the Markov Blanket to reveal causal information among the initial variables.  
}

\subsection{Evaluation on \pain Data Set}
\label{sec:exp}
\label{sec:pain_results}
\label{sec:pain_bench}

\begin{figure*}[t]
    \centering
    \subfigure[Faithful ground truth]{
        \includegraphics[width=0.5\textwidth, trim=40 40 35 35, clip]{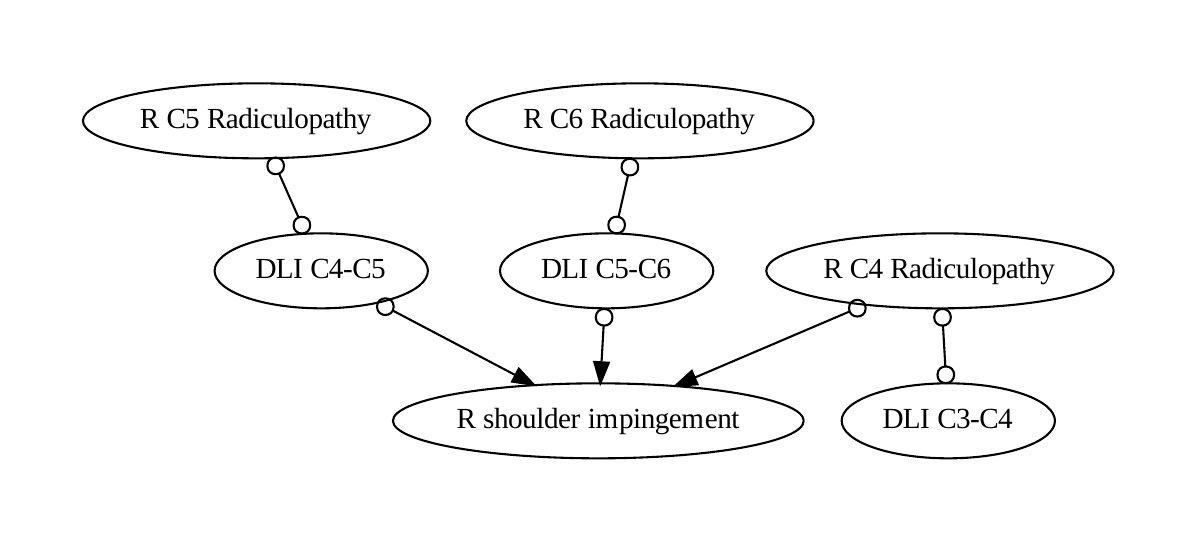}
        \label{fig:faithful_ground_truth}
    }  
    \subfigure[GPT-4 reasoning]{
        \includegraphics[width=0.45\textwidth, trim=40 40 35 35, clip]{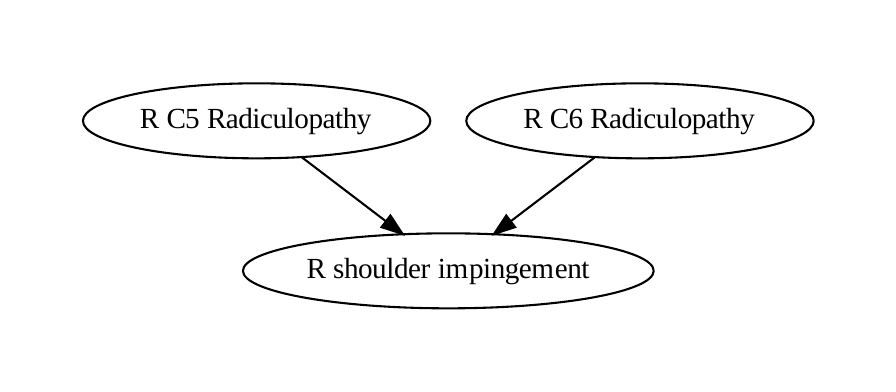}
    }
    \subfigure[GPT-4 \ours]{
        \includegraphics[width=0.8\textwidth, trim=40 40 35 35, clip]{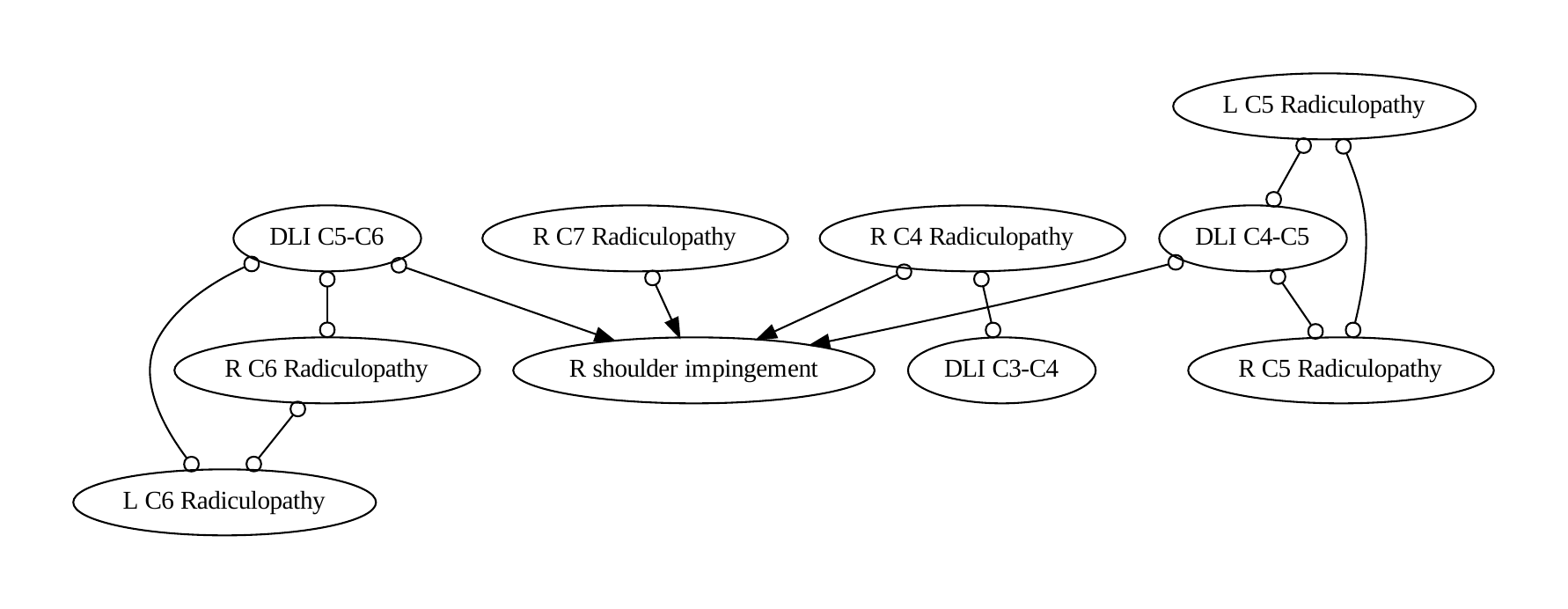}
    }%
    \caption{The discovered causal graphs in \pain. (b) shows the result based on directly prompting LLM to reason for the causal relations among all factors. Disconnected ones are dropped. }
    \label{fig:causal_graph_neuro}
\end{figure*}

\runin{Benchmark construction} In the \pain benchmark, we convert the data set into a clinical diagnosis task. In the original data set, there are three levels of causal variables, including the symptom level, radiculopathy level, and the pathophysiology level.
In the experiments, we mainly consider the target variable of right shoulder impingement.
When generating the clinical diagnosis notes as $\vx$ using \gptf, we avoid any mention of variables other than symptoms. We generated 100 samples for LLMs' analysis; since the number of possible factors is finite, we generate 1000 tabular data for CI tests.

As we intend to leverage the \pain benchmark to simulate the real-world diagnosis, after the factor proposal stage, we directly incorporate external tools to measure the values of the candidate factors.
More details about the construction of the \pain are given in Appendix~\ref{appdx:pain_construction}.

\runin{Evaluation and baselines} In \pain, we adopt a similar evaluation protocol and the baselines as in \apple. Nevertheless, due to the faithfulness issue of the original data set~\citep{neuropanic}, for the evaluation of causal relation discovery, we mainly conduct a qualitative comparison between the ground truth that is faithful to the data, against the baselines and \ours.

\begin{table}[t]
    \centering
    {\small
    \resizebox{0.7\textwidth}{!}{
        \begin{tabular}{ll|ccccc}
            \toprule
            \multirow{2}{*}{LLM}          & \multirow{2}{*}{Method} & \multicolumn{5}{c}{Factor Proposal}                         \\
                                          &                         & PA                                  & AN & OT & Acc  & F1   \\
            \midrule
            \multirow{4}{*}{GPT-4}        & Meta                    & 3                                   & 5  & 6  & 0.91 & 0.59 \\
                                          & DATA                    & 2                                   & 2  & 0  & 0.95 & 0.50 \\
                                          & DATA+CoT  & 3 & 4 & 13 & 0.81 & 0.35\\
                                          & \ours                   & 3                                   & 6  & 3  & 0.96 & 0.80 \\
            \midrule
            \multirow{4}{*}{GPT-3.5}     & Meta                    & 3                                   & 5  & 6  & 0.91 & 0.59 \\
                                          & DATA                    & 3                                   & 5  & 4  & 0.94 & 0.67 \\
                                          & DATA+CoT  & 2 & 2 & 3 & 0.91 & 0.36\\
                                          & \ours                   & 3                                   & 5  & 2  & 0.96 & 0.77 \\
            \midrule
            \multirow{4}{*}{\makecell{\llama-70b}}   & Meta                    & 2                                   & 4  & 5  & 0.91 & 0.53 \\
                                          & DATA                    & 3                                   & 3  & 1  & 0.95 & 0.60 \\
                                          & DATA+CoT  & 2 & 4 & 7 & 0.88 & 0.47\\
                                          & \ours                   & 3                                   & 6  & 2  & 0.97 & 0.86 \\
            \midrule
            \multirow{4}{*}{\makecell{\llama-13b}}   & Meta                    & 1                                   & 3  & 6  & 0.88 & 0.40 \\
                                          & DATA                    & 3                                   & 6  & 4  & 0.95 & 0.75 \\
                                          & DATA+CoT  & 0 & 1 & 10 & 0.81 & 0.12\\
                                          & \ours                   & 3                                   & 6  & 2  & 0.97 & 0.86 \\
            \midrule
            \multirow{4}{*}{\makecell{\llama-7b}}    & Meta                    & 1                                   & 1  & 17 & 0.72 & 0.08 \\
                                          & DATA                    & 3                                   & 6  & 3  & 0.96 & 0.80 \\
                                          & DATA+CoT  & 0 & 0 & 10 & 0.79 & $-$\\
                                          & \ours                   & 3                                   & 6  & 2  & 0.97 & 0.86 \\
            \midrule
            \multirow{4}{*}{\makecell{\mistral-Medium}} & Meta                    & 3                                   & 6  & 3  & 0.96 & 0.80 \\
                                          & DATA                    & 3                                   & 3  & 2  & 0.94 & 0.66 \\
                                          & DATA+CoT  & 3 & 5 & 8 & 0.88 & 0.53\\
                                          & \ours                   & 3                                   & 6  & 2  & 0.97 & 0.86 \\
            \bottomrule
        \end{tabular}}
    }
    \caption{Factor proposal in \pain.}
    \label{table:neuropathic}
\end{table}

\runin{Factor proposal} The quantitative results on \pain benchmark are given in Table~\ref{table:neuropathic}. PA, AN, and OT refer to the parents, ancestors, and others, respectively. Accuracy and F1 measure the recovery of the causal ancestors. Similarly, we can find that \ours consistently outperforms all of the baselines regardless of which LLMs are incorporated. In particular, even with the weakest backbone model, i.e., \llama-7b, \ours can still effectively leverage the intrinsic rich knowledge and beat the baselines with more powerful LLMs.

\runin{Causal relation recovery} Figure~\ref{fig:faithful_ground_truth} shows the causal graph obtained by FCI running on the original data, where we can find that several causal relations cannot hold on the data. As shown in Figure~\ref{fig:causal_graph_neuro}, when using LLMs to perform the reasoning, LLMs cannot identify the faithfulness issues. In contrast, \ours can imply faithful causal insights.

\subsection{Eliciting Adjustment Set for Causal Inference}
\label{subsec:CI-ATE}

\runin{Benchmark and baselines} We use two realistic data sets about marketing for estimating Average Treatment Effect. The original ones are randomized tabular data sets: Hillstrom~(\citeyear{radcliffe2008hillstrom}): the effect of mailing on customers' visiting to a website; and Retail Hero~(\citeyear{x5_retail_hero}): the effect of Short Message on customers' purchasing a product. We follow the evaluation setting introduced by~\citet{dhawan2024end}, where $\mX$ is a synthesized social-media post, and $T$ and $Y$ are confounded. In Hillstrom, the confounder is the \texttt{new customer indicator}, and in Retail Hero it is the \texttt{customer's age}. 

We consider two baselines: (1) \textbf{DATA+IPW}: LLMs will conduct reasoning to find the possible factors that will be utilized to estimate the ATE. (2) \textbf{DATA+FCI+IPW}: the proposed factors will be further verified by FCI algorithm based on the identified causal graphs.

In addition, we also include the human expert results as a reference. For the baseline of human expert results, a set of human-crafted factor descriptions is provided to predict the factor values by bag-of-word sentence encoders, or by LLMs~\citep{dhawan2024end}. For \ours method, we do not include this prior knowledge. The input is only $(\mX,T,Y)$. The factor descriptions will be identified by Algorithm~\ref{alg:coat-t-y} with access to an LLM.

\runin{Causal graphs interpretation} The PAGs produced by \ours with GLM-4-Plus are displayed in Figure~\ref{fig:causal_graph_ate}. In Figure~\ref{fig:causal_graph_ate_hill}, the direction between mailing and visit is confirmed by the elicited factors, while in Figure~\ref{fig:causal_graph_ate_retail} the skeleton is missing, indicating a faithfulness issue in the data set. Interestingly, in Figure~\ref{fig:causal_graph_ate_hill}, \ours doesn't directly include the {new customer indicator} factors, but alternatively proposes some related ones like {shopping frequency} or {spending amount}, providing an unexpected point of view. In Figure~\ref{fig:causal_graph_ate_retail}, the {customer's age} factor is included in the graph, and is directly connected with the SMS variable, while its connection to the purchase variable is missing, indicating their correlation may not be very strong.

\begin{figure*}[t]
    \centering
    \subfigure[hillstrom]{
        \includegraphics[width=0.4\textwidth, trim=40 40 35 35, clip]{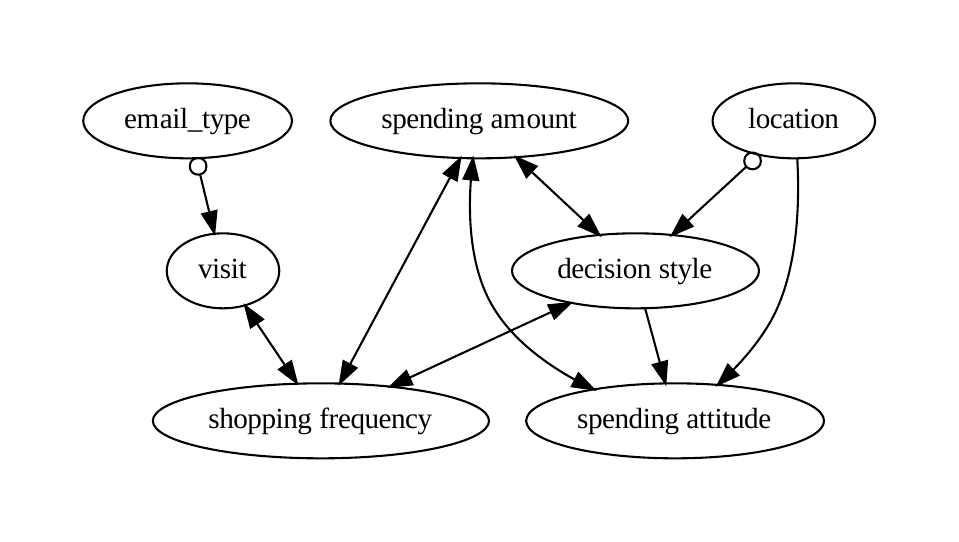}
        \label{fig:causal_graph_ate_hill}
    }
    \subfigure[retailhero]{
        \includegraphics[width=0.55\textwidth, trim=40 40 35 35, clip]{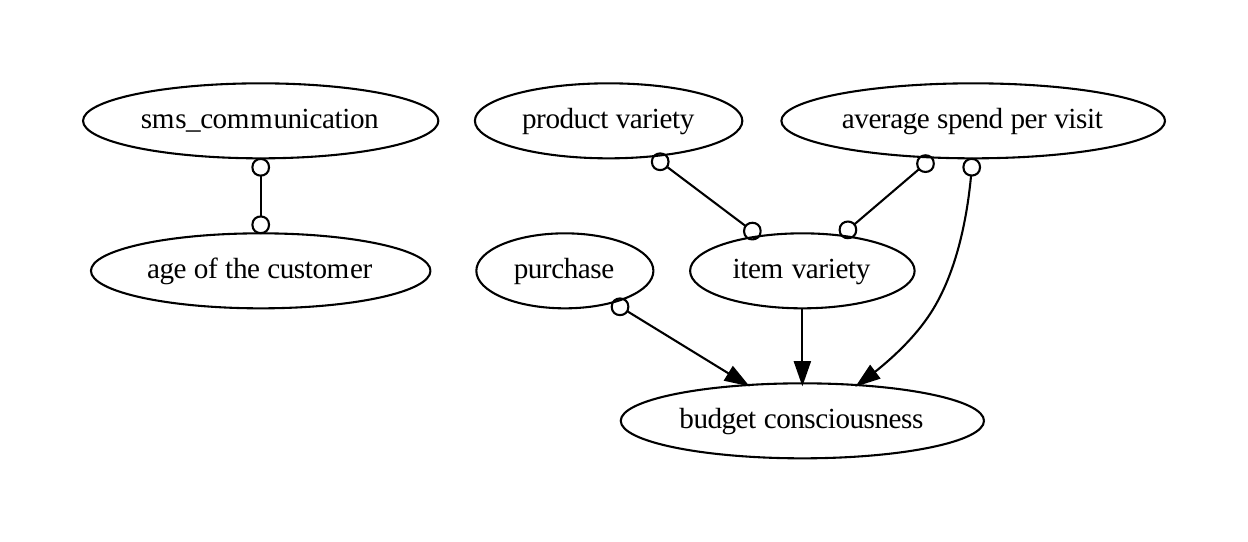}
        \label{fig:causal_graph_ate_retail}
    }
    \caption{The Causal Graph Identified by Algorithm~\ref{alg:coat-t-y} on two marketing data sets.}
    \label{fig:causal_graph_ate}
\end{figure*}

\runin{Treatment effect estimation} The ATE estimation results are presented in Table~\ref{table:CI-ATE}. For each experiment, we choose the largest valid adjustment sets from the generated PAG for evaluation, if there are multiple such sets, we calculate their averaged ATE. Experiments are repeated 3 times. Compared with baselines, \ours has lowest RMSE loss. Compared with human-expert results, we find the factor quality of \ours is also competitive.

\begin{table}[t]
    \centering
    {\small
    \resizebox{\textwidth}{!}{
        \begin{tabular}{cc|cccc} \toprule
            \multicolumn{1}{l}{}           & \multicolumn{1}{l}{} & \multicolumn{2}{c}{Hillstrom} & \multicolumn{2}{c}{Retail Hero} \\ 
            Factor Provider                & Method               & ATE (pp)          & RMSE      & ATE (pp)           & RMSE       \\  \midrule
            \multirow{1}{*}{-}           & Ground Truth         & 6.09              & -         & 3.32               & -          \\ \midrule
            \multirow{3}{*}{\makecell{Human Expert\\~\citep{dhawan2024end}}}                             & \makecell{Bag-of-Words}         & 7.57\std{1.37}       & 2.23      & 2.61\std{2.08}        & 2.42       \\ 
             & N-MC IPW             & 4.81\std{0.80}       & 1.51      & 1.85\std{2.01}        & 2.49       \\
                                           & N-IPW                & 5.23\std{1.00}       & 1.32      & 3.83\std{1.29}        & 1.39       \\ \midrule
            \multirow{3}{*}{DeepSeek-V3}                    & DATA+IPW              & 4.31\std{0.12}       & 1.78      & -1.68\std{1.72}      & 5.29       \\ 
                                                            & DATA+FCI+IPW          & 4.73\std{0.68}       & 1.52      & 2.29\std{0.39}      & 1.10       \\ 
                                                            & COAT(IPW)             & 5.70\std{0.26}       & \textbf{0.47}      & 3.16\std{0.47}      & \textbf{0.50}       \\  \midrule
            \multirow{3}{*}{GLM-4-Plus}                     & DATA+IPW              & 0.18\std{0.88}       & 5.98      & 73.21\std{15.93}    & 71.68      \\ 
                                                            & DATA+FCI+IPW          & 5.62\std{0.31}       & 0.57      & 2.38\std{0.51}      & 1.07       \\
                                                            & COAT(IPW)             & 5.92\std{0.17}       &\textbf{ 0.23}      & 2.39\std{0.25}      & \textbf{0.96 }      \\ \midrule
            \multirow{3}{*}{GPT-4.1-mini}                   & DATA+IPW              & 3.70\std{1.93}       & 3.07      & 1.04\std{1.42}      & 2.68       \\ 
                                                            & DATA+FCI+IPW          & 4.11\std{0.37}       & 2.01      & 2.40\std{0.62}      & 1.11       \\
                                                            & COAT(IPW)             & 4.79\std{1.38}       & \textbf{1.90}      & 2.54\std{0.06}      &\textbf{ 0.78}       \\ 
            \bottomrule
            \end{tabular}}
    }
    \caption{ATE Estimation Results.}
    \label{table:CI-ATE}
    \vspace{0.1in}
    \footnotesize
    \textit{Note:} ATE is reported in percentage points (pp), representing the absolute increase in conversion rate.
    \vspace{0.2in}
\end{table} 
\edit{

\begin{figure}[tbp]
    \centering
    \subfigure[Ground Truth]{
        \begin{minipage}[c]{0.3\linewidth}
            \centering
            \includegraphics[width=\linewidth,height=0.3\textheight,keepaspectratio]{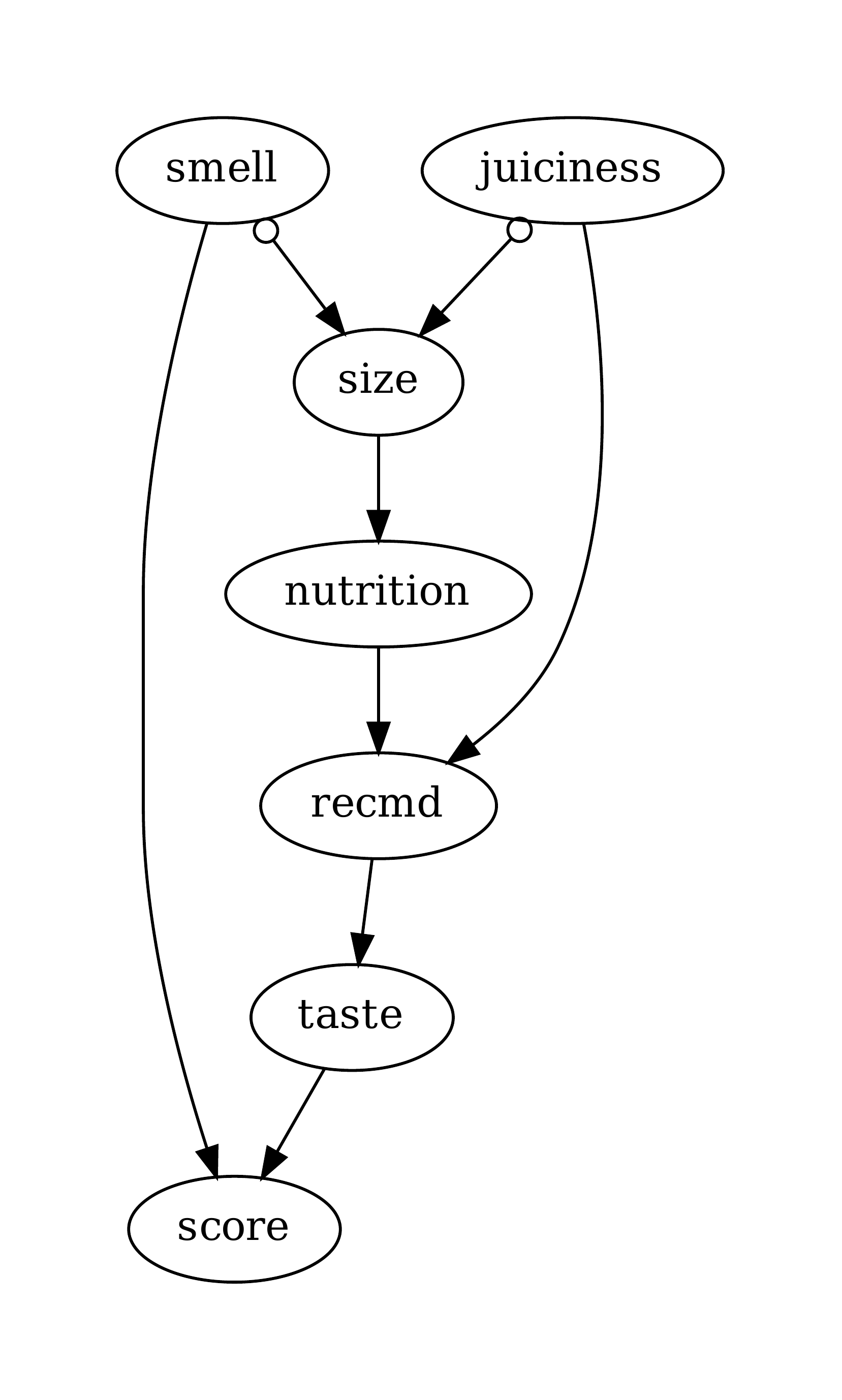}
        \end{minipage}
        \label{fig:compare-adj-gt}
    }%
    \subfigure[\oursMB]{
        \begin{minipage}[c]{0.3\linewidth}
            \centering
            \includegraphics[width=\linewidth,height=0.3\textheight,keepaspectratio]{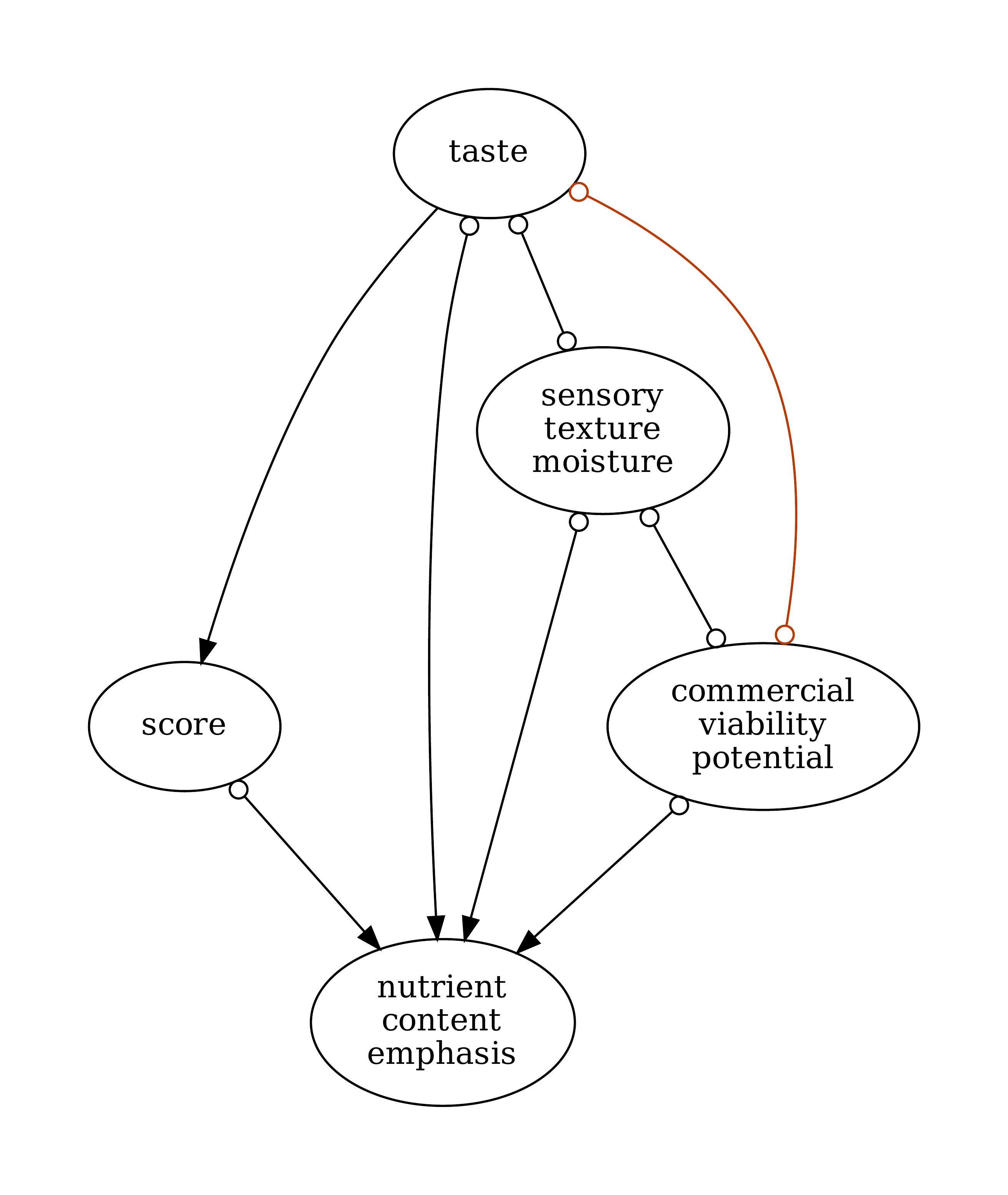}
        \end{minipage}
        \label{fig:compare-adj-mb}
    }%
    \subfigure[\oursGAS]{
        \begin{minipage}[c]{0.3\linewidth}
            \centering
            \includegraphics[width=\linewidth,height=0.3\textheight,keepaspectratio]{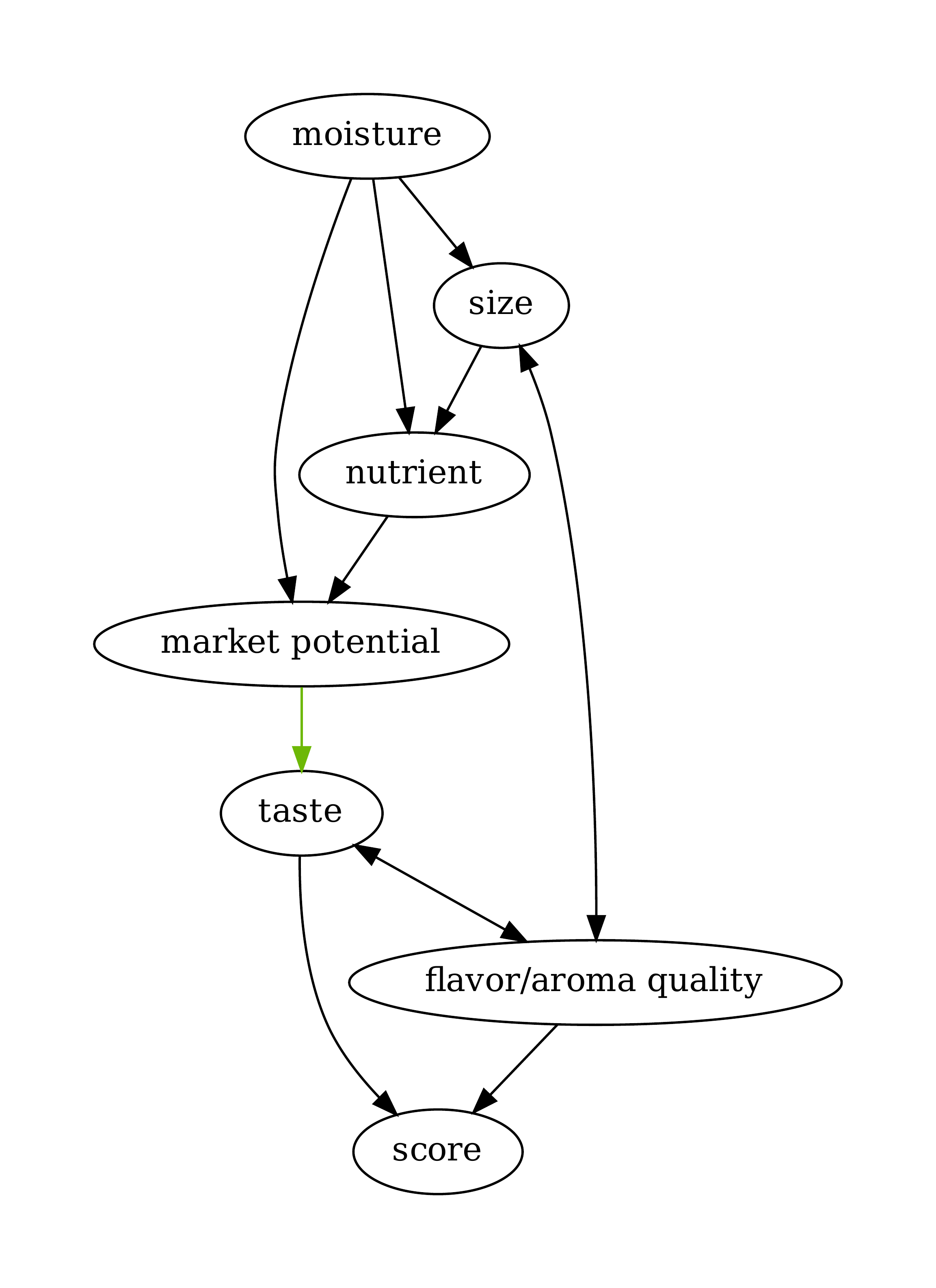}
        \end{minipage}
        \label{fig:compare-adj-adj}
    }

    \caption{\edit{The comparison between \oursMB and \oursGAS in the causal inference setting.}}
\end{figure}
 
\runin{Compare \oursMB and \oursADJ} We empirically demonstrate the differences between the two algorithms in a variant of \apple data set which has a different causal structure as shown in Figure~\ref{fig:compare-adj-gt}. 
The task is to estimate the causal effect from \texttt{taste} to \texttt{score} with their numerical values as well as the unstructured textual data.

After applying \oursMB on \texttt{taste}, the causal graph on the extended node set is shown in Figure~\ref{fig:compare-adj-mb}. This graph does not meet the graphical requirement on visibility because the edge $\text{taste} \rightarrow \text{score}$ is \textit{invisible}. The red edge in the graph doesn't have an arrow head into \texttt{taste}.
On the other hand, the edge is \textit{visible} (see Definition~\ref{def:visibility}) due to the green edge in Figure~\ref{fig:compare-adj-adj} found by \oursADJ, and one can proceed to search adjustment sets based on graphical criteria~\citep{maathuis2015generalized,perkovi2018complete}, like $\{\texttt{flavor/aroma quality}, \texttt{market potential}\}$.

}

\subsection{Ablation Studies}
\label{subsec:ablation-study}
\runin{Group Size in Prompt} In the COAT prompt, several samples are given and grouped by the values of the target variable. The samples in each group are randomly selected to a fixed number (like 3 samples per group). Empirically, we keep it to be 3 throughout all experiments (sometimes smaller than 3 if samples are not enough). In practice, it is mainly constrained by the LLM's context length.

\begin{table}[t]
\centering
{\footnotesize
\setlength{\tabcolsep}{2pt}
\begin{tabular}{lcccccccc}
\toprule
Method  & Cluster Size              & Group Size    & MB                & NMB               & OT                & Recall            & Precision         & F1                \\ \midrule
META    & -                         & -             & 2.67\std{0.94}    & 0.67\std{0.47}    & 2.33\std{0.47}    & 0.53\std{0.19}    & 0.46\std{0.08}    & 0.49\std{0.13}    \\
DATA    & -                         & 3             & 3.00\std{0.00}    & 0.33\std{0.47}    & 0.00\std{0.00}    & 0.60\std{0.00}    & 0.92\std{0.12}    & 0.72\std{0.04}    \\
\ours   & \texttt{len(factor)+1}    & 3             & 4.00\std{0.82}    & 0.33\std{0.47}    & 0.00\std{0.00}    & 0.80\std{0.16}    & 0.93\std{0.09}    & 0.85\std{0.11}    \\
\ours   & \texttt{len(factor)+1}    & 1             & 4.67\std{0.58}    & 0.00\std{0.00}    & 0.00\std{0.00}    & 0.93\std{0.12}    & 1.00\std{0.00}    & 0.96\std{0.06}    \\
\ours   & 2                         & 3             & 3.67\std{1.53}    & 0.00\std{0.00}    & 0.00\std{0.00}    & 0.73\std{0.31}    & 1.00\std{0.00}    & 0.82\std{0.22}    \\
\bottomrule
\end{tabular}}
\caption{Results about Ablation Study on Hyperparameters.} \label{table:ablation_hyperpara}
\end{table}

\runin{The Number of Clusters} When constructing feedback, we first use clustering to separate the data set and then find the cluster where the target variable is not explained well by current factors. Empirically, we set the number of clusters to be one plus the number of current factors.

We conduct ablation studies of COAT with GPT-4 using different hyperparameters. As shown in Table~\ref{table:ablation_hyperpara}, one can observe that COAT is not sensitive to these hyperparameters and performs more robustly than the baselines under different hyperparameter setups.

\edit{

\begin{table}[t]
\centering
{\footnotesize
\setlength{\tabcolsep}{2pt}
\begin{tabular}{@{}>{\raggedright\arraybackslash}p{2.2cm}>{\raggedright\arraybackslash}p{2.4cm}rrrrrr@{}}
\toprule
Review Generation & Factor Proposal & \multicolumn{1}{l}{MB} & \multicolumn{1}{l}{NMB} & \multicolumn{1}{l}{OT} & \multicolumn{1}{l}{Recall} & \multicolumn{1}{l}{Precision} & \multicolumn{1}{l}{F1} \\ \midrule
\multirow{3}{2.2cm}{DeepSeek-V3-0324} & DeepSeek-V3-0324 & 3.75\std{0.66} & 0.00\std{0.00} & 0.44\std{0.51} & 0.75\std{0.13} & 0.91\std{0.10} & 0.82\std{0.12} \\
 & GPT-OSS-120B & 5.00\std{0.00} & 0.00\std{0.00} & 0.00\std{0.00} & 1.00\std{0.00} & 1.00\std{0.00} & 1.00\std{0.00} \\
 & Qwen3-Next-80B-A3B & 2.94\std{0.10} & 0.33\std{0.58} & 0.00\std{0.00} & 0.59\std{0.02} & 0.92\std{0.14} & 0.71\std{0.04} \\
\midrule
\multirow{3}{2.2cm}{GPT-OSS-120B} & DeepSeek-V3-0324 & 4.08\std{1.18} & 0.33\std{0.58} & 0.00\std{0.00} & 0.82\std{0.24} & 0.91\std{0.15} & 0.86\std{0.20} \\
 & GPT-OSS-120B & 5.00\std{0.00} & 0.00\std{0.00} & 0.00\std{0.00} & 1.00\std{0.00} & 1.00\std{0.00} & 1.00\std{0.00} \\
 & Qwen3-Next-80B-A3B & 4.33\std{1.15} & 0.00\std{0.00} & 0.00\std{0.00} & 0.87\std{0.23} & 1.00\std{0.00} & 0.92\std{0.14} \\
\midrule
\multirow{3}{2.2cm}{Qwen3-Next-80B-A3B} & DeepSeek-V3-0324 & 4.33\std{0.47} & 0.00\std{0.00} & 0.00\std{0.00} & 0.87\std{0.09} & 1.00\std{0.00} & 0.93\std{0.05} \\
 & GPT-OSS-120B & 4.72\std{0.25} & 0.00\std{0.00} & 0.00\std{0.00} & 0.94\std{0.05} & 1.00\std{0.00} & 0.97\std{0.03} \\
 & Qwen3-Next-80B-A3B & 3.25\std{0.43} & 0.00\std{0.00} & 0.00\std{0.00} & 0.65\std{0.09} & 1.00\std{0.00} & 0.79\std{0.06} \\
\bottomrule
\end{tabular}
}
\caption{Cross-Model Validation of AppleGastronome Construction.}
\label{tab:cross-ag-variant}
\end{table}

\begin{figure}[t]
    \centering
    \includegraphics[width=0.5\linewidth]{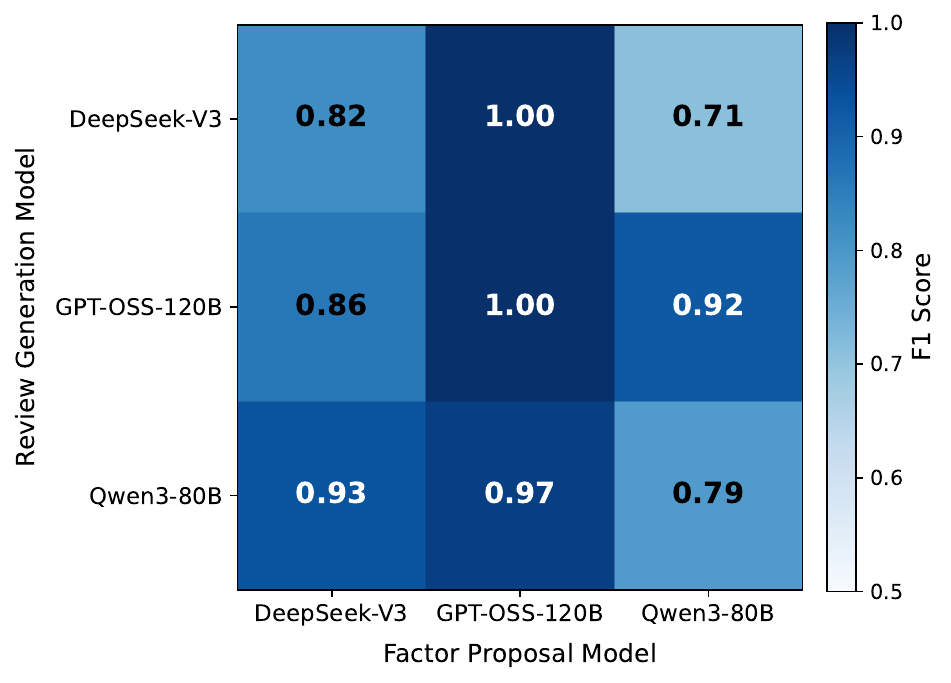}
    \caption{\edit{Heatmap for the cross-model validation of synthetic data.}}
    \label{fig:ag-var-heatmap}
\end{figure} 
\runin{Cross-Model Validation of Synthetic Data.} 
The \apple data set is constructed through a two-phase process: (1) the tabular data (i.e., apple attributes and scores) are generated via Bernoulli distributions according to a human-specified causal structure; (2) \gptf is prompted to generate textual reviews based on each sample's attributes and score. In Table~\ref{tab:full_results_on_apple}, to mitigate potential correlated errors between the LLM used for text generation and those used for factor proposal, we report results across 10 distinct LLMs for the proposal stage.

To further validate the robustness of our evidence, we conduct a cross-model validation study: we generate three additional variants of the textual reviews using different LLMs (\texttt{DeepSeek-V3-0324}, \texttt{GPT-OSS-120B}, and \texttt{Qwen3-Next-80B-A3B}) based on the same underlying tabular data, and then evaluate each LLM as a factor proposer on all three review data sets.

As shown in Table~\ref{tab:cross-ag-variant}, the factor proposal performance remains robust across different combinations of generation and proposal models. Notably, the precision scores are consistently high across nearly all configurations, often reaching $1.00$. 
In terms of recall and F1, \texttt{GPT-OSS-120B} consistently achieves the highest recall and near-perfect F1 scores ($0.97$--$1.00$) as a proposer across all three review variants, and achieves a perfect F1 score on self-generated reviews. Meanwhile, \texttt{DeepSeek-V3-0324} and \texttt{Qwen3-Next-80B-A3B} achieve second and third F1 scores on their respective self-generated reviews. 
This indicates that the performance is largely driven by the factor-proposer LLM.
These results suggest that the \apple data set is robust against model-specific biases.

\begin{table}[t]
    \centering
    {\small
    \begin{tabular}{@{}llp{8cm}@{}}
    \toprule
    \textbf{Factor} & \textbf{Type} & \textbf{Factor Definition} \\
    \midrule
    \multirow{3}{*}{Juiciness} & \multirow{3}{*}{Correlated} & \textbf{1}: Abundant or moist texture described in text \\
     & & \textbf{0}: Not mentioned or ambiguous description \\
     & & \textbf{-1}: Dry or lacking moisture \\
    \midrule
    \multirow{3}{*}{Color} & \multirow{3}{*}{Irrelevant} & \textbf{1}: Vibrant or appealing color mentioned \\
     & & \textbf{0}: Not mentioned or neutral description \\
     & & \textbf{-1}: Uneven, dull, or unappealing color \\
    \bottomrule
    \end{tabular}
    }
    \caption{\edit{Definitions of injected factors for robustness stress testing. \textit{Juiciness} serves as a disturbing candidate (correlated with outcome but not in the true Markov Blanket), while \textit{Color} represents a completely irrelevant factor.}}
    \label{tab:injected-factors}
\end{table}

\begin{table}[t]
    \centering
    {\small
    \begin{tabular}{cccccc}
    \toprule 
    Injection & Iteration & MB & NMB/OT & Injected Factors & F1 \\
    \midrule
    \multirow{4}{*}{juiciness} 
    & 1 & 0.00\std{0.00} & 1.00\std{0.00} & 1.00\std{0.00} & 0.00\std{0.00} \\
    & 2 & 3.20\std{0.98} & 0.20\std{0.40} & 0.00\std{0.00} & 0.75\std{0.16} \\
    & 3 & 3.60\std{0.80} & 0.00\std{0.00} & 0.00\std{0.00} & 0.83\std{0.10} \\
    & 4 & 4.00\std{0.89} & 0.00\std{0.00} & 0.00\std{0.00} & 0.88\std{0.11} \\
    \midrule
    \multirow{4}{*}{color} 
    & 1 & 0.00\std{0.00} & 0.00\std{0.00} & 0.00\std{0.00} & 0.00\std{0.00} \\
    & 2 & 2.60\std{0.49} & 0.60\std{0.80} & 0.00\std{0.00} & 0.64\std{0.14} \\
    & 3 & 3.40\std{0.80} & 0.40\std{0.49} & 0.00\std{0.00} & 0.77\std{0.12} \\
    & 4 & 4.00\std{0.63} & 0.40\std{0.49} & 0.00\std{0.00} & 0.85\std{0.09} \\
    \bottomrule
    \end{tabular}
    }
    \caption{\edit{Stress test result with non-MB factor injected in the initial set. For the \textit{juiciness} factor, it was preserved in the first round, and then was removed after 1 iteration. For the \textit{color} factor, since it is not relevant to the textual data, it was rejected in the first round by statistical testing.}}
    \label{tab:stress-test-mb}
\end{table} 
\runin{\oursMB with Unexpected Factor Injection} To assess whether COAT-MB can recover from wrong factor proposals, we conduct a stress test by manually injecting ``unexpected factors'' into the initial round: (1) \textit{juiciness}—a factor correlated with the target but not part of the true Markov Blanket, and (2) \textit{color}—a completely irrelevant factor with no statistical relationship to the data. To conduct the injection, we put their factor definitions in Table~\ref{tab:injected-factors} as the factor proposal of the first round in \oursMB. 

As shown in Table~\ref{tab:stress-test-mb}, when \textit{juiciness} is injected, the algorithm initially preserves it in round 1 (Injected Factors $=1.00$). Critically, as the \oursMB continues to propose and verify new factors, it starts to identify the true MB factors (MB size increases to $3.20$), and this wrong factor is quickly discarded by round 2 (Injected Factors drops to $0.00$). For the completely irrelevant \textit{color} factor, it is dropped immediately in round 1 (Injected Factors $=0.00$) due to the independence between its factor values and the target's ones.
These results demonstrate that \oursMB is robust to possibly wrong factor proposals with the help of the iterative feedback loops.

\begin{figure}[t]
    \centering
    \def\imgH{0.2\textheight}
    \subfigure[Two noises]{
        \begin{minipage}[c]{0.45\linewidth}
            \centering
            \includegraphics[height=\imgH]{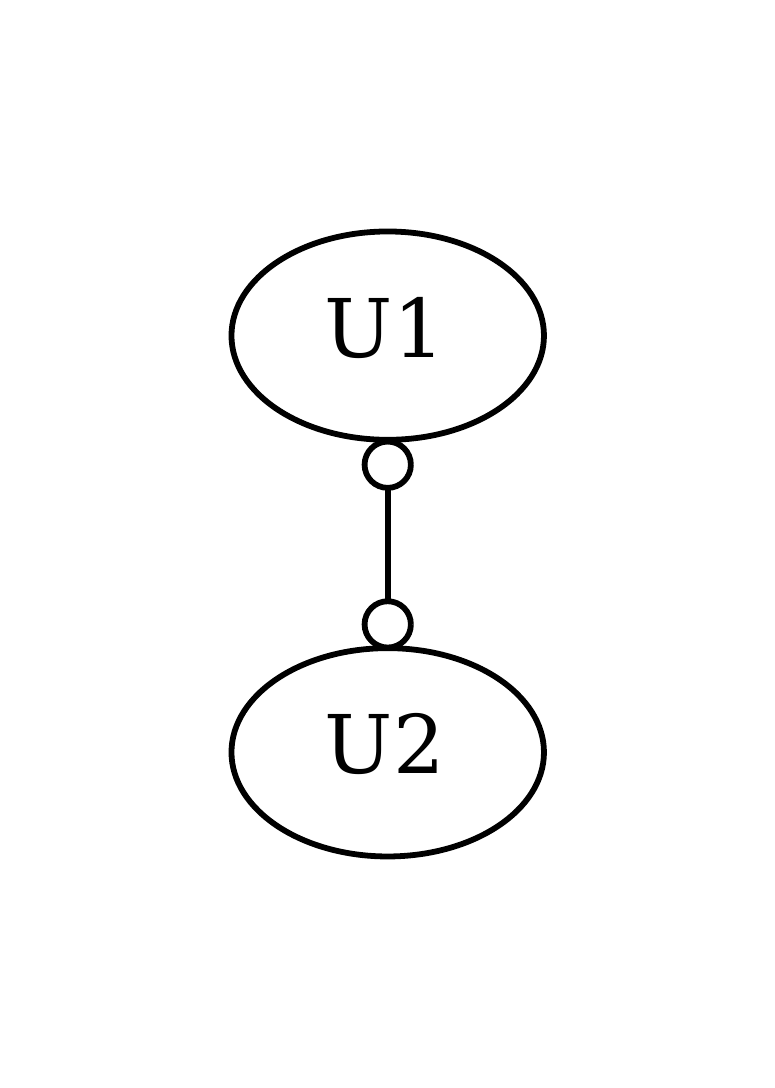}
        \end{minipage}
        \label{fig:stress-pag-2noise}
    }%
    \subfigure[Three noises]{
        \begin{minipage}[c]{0.45\linewidth}
            \centering
            \includegraphics[height=\imgH]{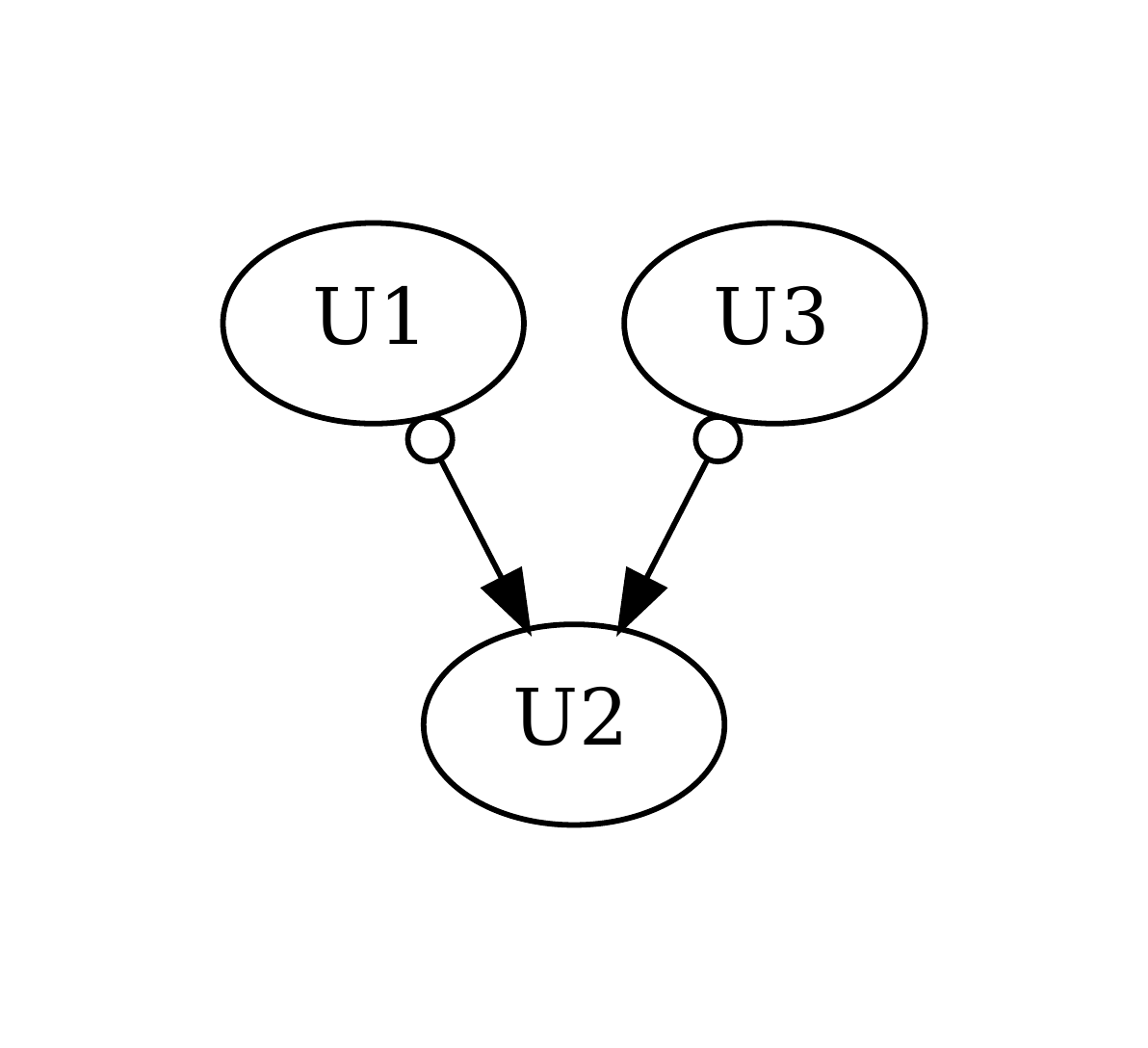}
        \end{minipage}
        \label{fig:stress-pag-3noise}
    }
    \caption{\edit{Stress test result with a variable set of pure noises. The initial variables $U_1, U_2, U_3$ are not relevant to the textual data, and therefore \oursPAG returns no additional factors.}}
    \label{fig:stress-test-pag}
\end{figure}
 
\runin{\oursPAG with Pure Noise Variables.} 
To evaluate whether \oursPAG can recover when the initial variable set $\gU$ contains only noise variables with no relationship to the textual data, we conduct a stress test with two synthetic noise structures: (1)~\textit{two correlated noises} ($U_1 \rightarrow U_2$), where $U_2$ is constructed as a linear function of $U_1$ with Gaussian noise; and (2)~\textit{three noises with collider structure} ($U_1 \rightarrow U_2 \leftarrow U_3$), where $U_2$ is a common effect of two independent noise sources $U_1$ and $U_3$. For each setting, we apply \oursPAG on $\gU$ and  the five textual data sets in Figure~\ref{fig:setting2_graph_5_graphs} respectively.

As illustrated in Figure~\ref{fig:stress-test-pag}, when COAT-PAG is initialized with these pure noise variables over the textual data set, the algorithm correctly rejects all the factors proposed from the text $\mX$ based on independence testing. The resulting graphs, which are consistent across five data sets,  preserve only the initial variable set  $\gU$ and the causal structures among the noise variables themselves.
This demonstrates that \oursPAG can abstain from hallucinating factors in the extreme case.

\begin{figure}[t]
    \def\imgW{0.45\linewidth}
    \def\imgH{0.2\textheight}  
    \centering
    \subfigure[fully observable case - ground truth]{
        \begin{minipage}[c]{\imgW}
            \centering
            \includegraphics[height=\imgH]{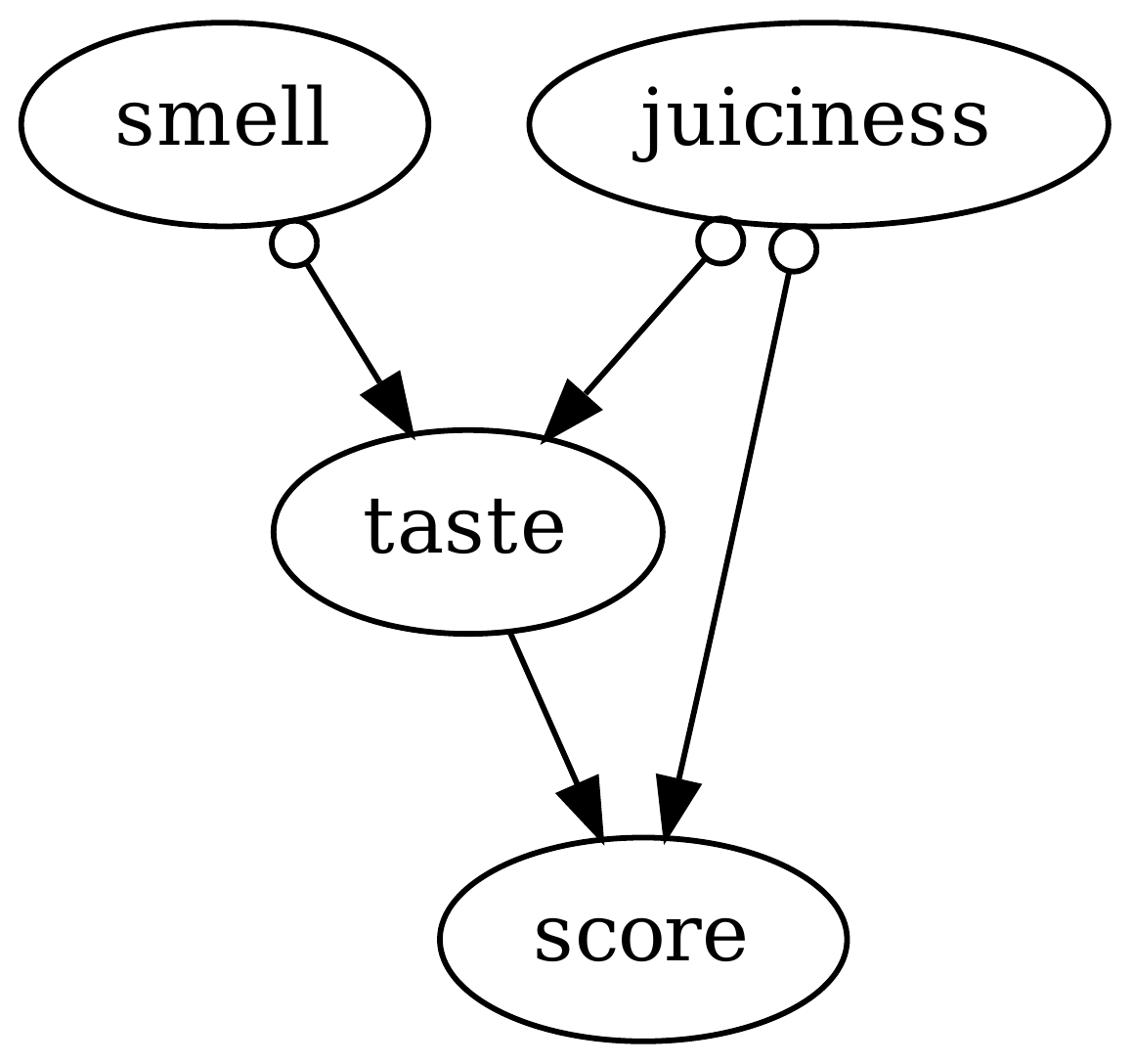}
        \end{minipage}
        \label{fig:stress-adj-visible-gt}
    }%
    \subfigure[fully observable case - result]{
        \begin{minipage}[c]{\imgW}
            \centering
            \includegraphics[height=\imgH]{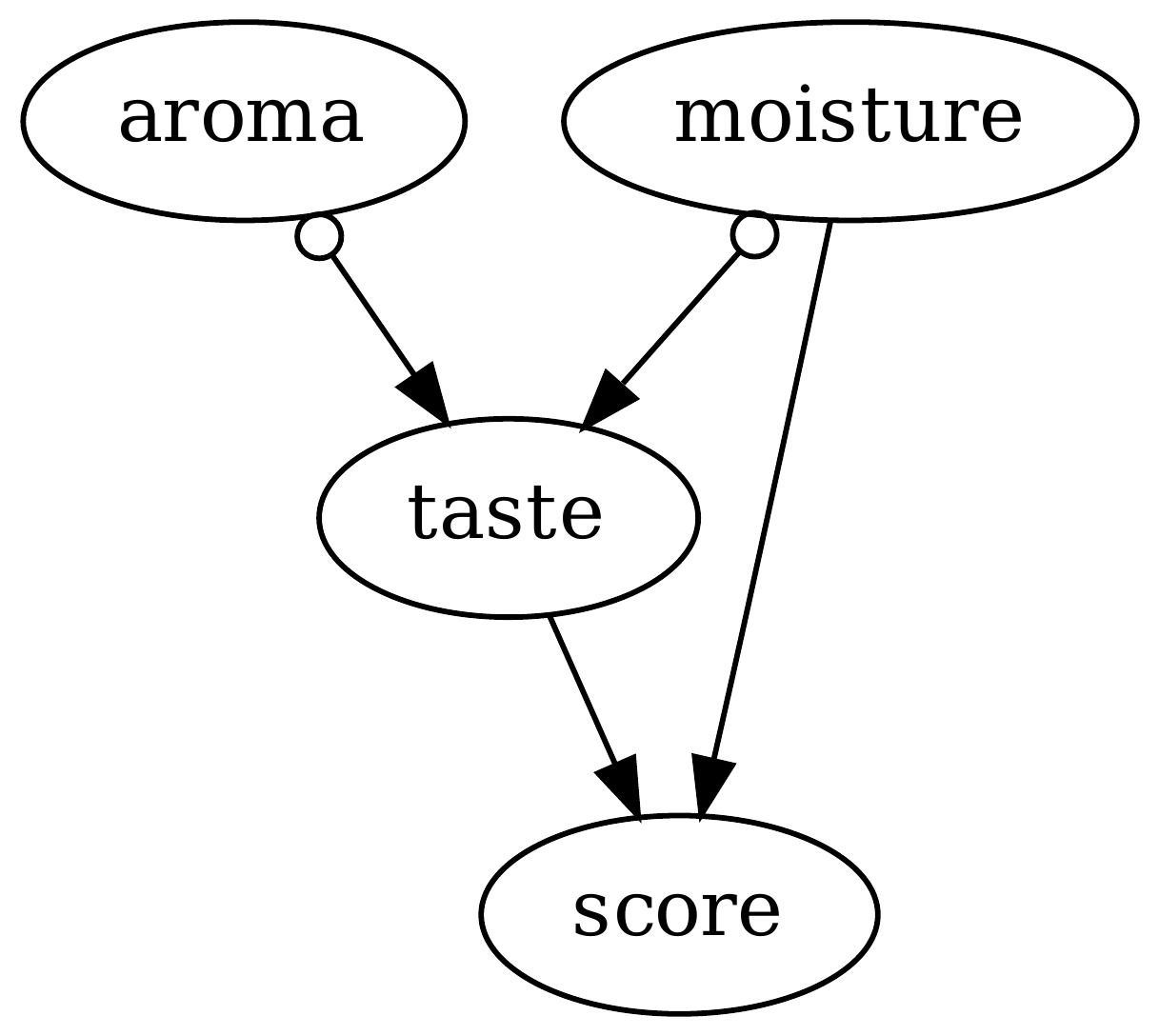}
        \end{minipage}
        \label{fig:stress-adj-visible-result}
    }
    
    \vspace{0.5em}  %
    
    \subfigure[partially observable case - ground truth]{
        \begin{minipage}[c]{\imgW}
            \centering
            \includegraphics[height=\imgH]{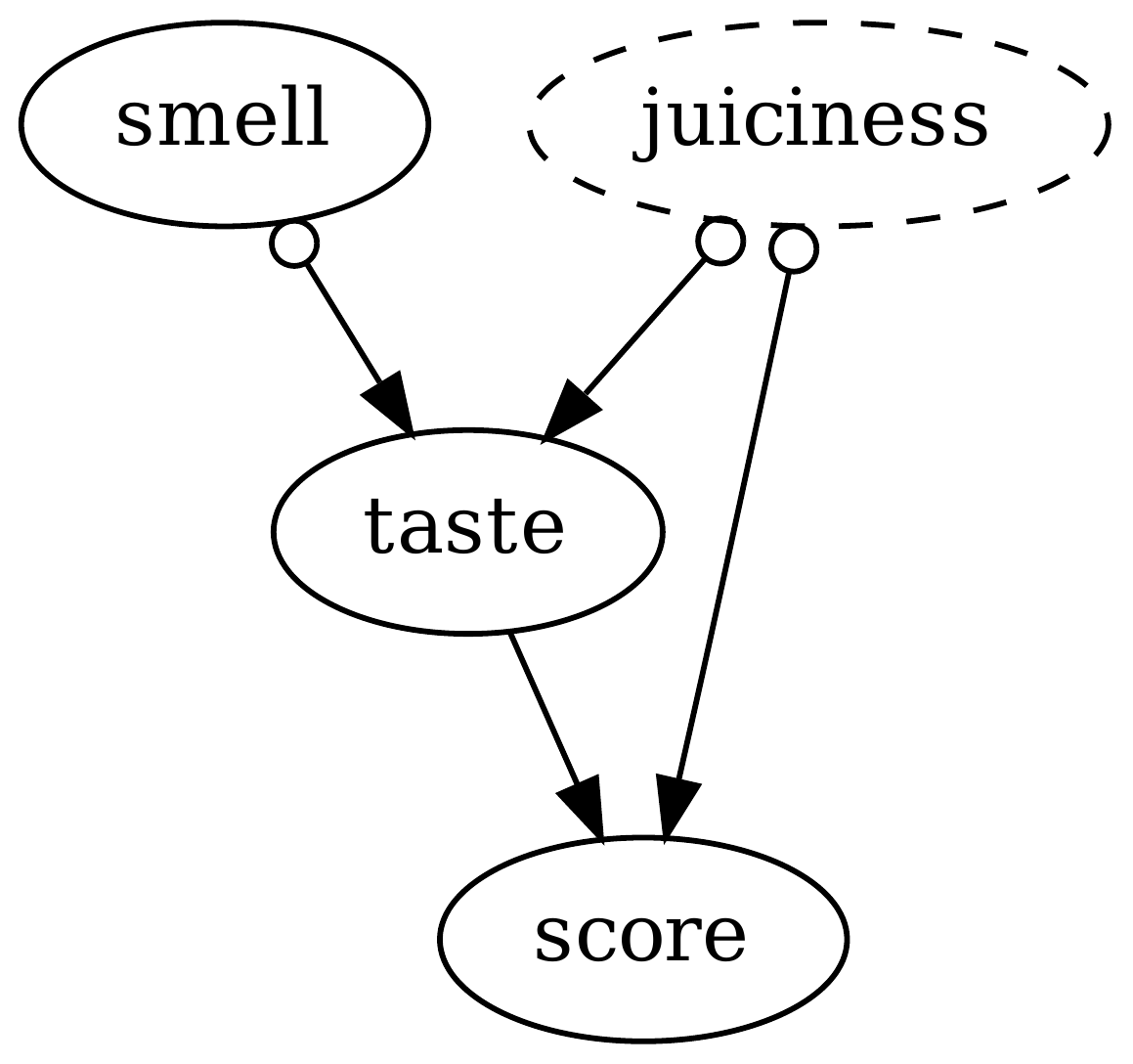}
        \end{minipage}
        \label{fig:stress-adj-invisible-gt}
    }%
    \subfigure[partially observable case - result]{
        \begin{minipage}[c]{\imgW}
            \centering
            \includegraphics[height=\imgH]{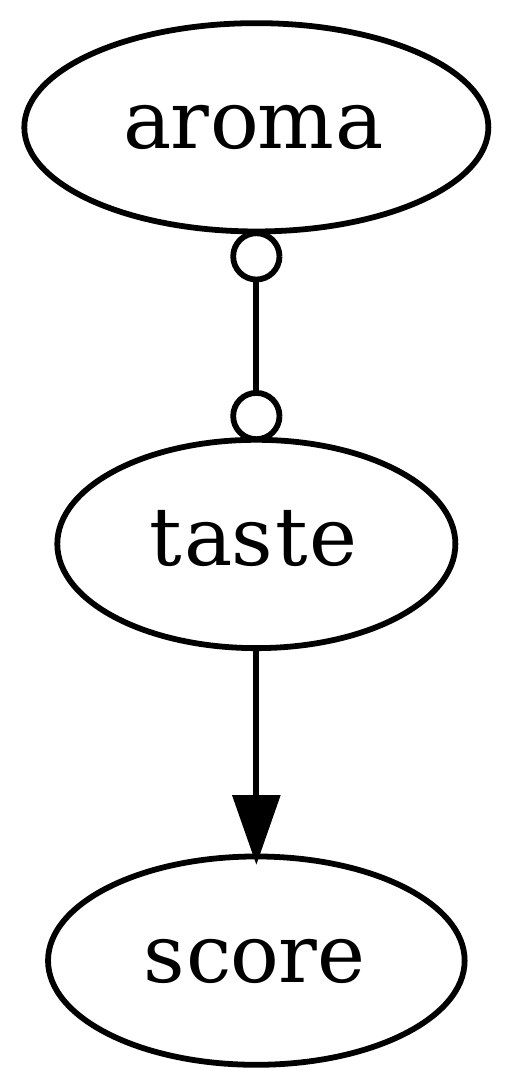}
        \end{minipage}
        \label{fig:stress-adj-invisible-result}
    }

    \caption{\edit{Stress test result with latent confounder for the treatment effect variables.
For (a) and (b), there is no latent variable, and the result from \oursADJ indicates the edge $T \rightarrow Y$ is \textit{visible} since there exists another node (aroma) that is not adjacent to $Y$ and has an edge into $T$.  For (c) and (d), the edge $T \rightarrow Y$ is invisible due to the latent confounder (juiciness).
    }}
    \label{fig:stress-test-adj}
\end{figure}
 
\runin{\oursADJ with Unobserved Confounders}
To evaluate how \oursADJ behaves when some confounding factors cannot be expressed as a function of unstructured data $\mX$, we conduct a stress test comparing two scenarios: (1)~\textit{fully observable case} where all confounders are extractable from $\mX$ (Figure~\ref{fig:stress-test-adj}~(a--b)), and (2)~\textit{partially observable case} where a latent confounder (\textit{juiciness}, shown as dashed oval) affects both $T$ and $Y$ but is not captured in $\mX$ (Figure~\ref{fig:stress-test-adj}~(c--d)). We apply \oursADJ to both data sets with initial graph $T \rightarrow Y$ (\textit{taste} $\to$ \textit{score}).

In the fully observable case, the algorithm correctly identifies the edge $T \rightarrow Y$ as \textit{visible}  (see Definition~\ref{def:visibility})  because there exists a node (\textit{aroma}) that is not adjacent to $Y$ yet has a directed edge into $T$. Consequently, \oursADJ would proceed to find valid adjustment sets (in this case, it can be $\{\text{moisture}\}$).
In the partially observable case, the edge $T \rightarrow Y$ is rendered \textit{invisible} due to the unobserved \textit{juiciness}. In such situation, \oursADJ would report the invisibility and abstain from causal estimation. 
These results demonstrate that \oursADJ behaves conservatively when faced with possible confounders that cannot be elicited from the $\mX$.

}

\section{Real-World Case Studies}
\label{sec:realwold-casestudy}

In this section, we use  \ours \edit{ framework (\oursMB Algorithm)} to analyze unstructured data from three real-world scenarios: (1) brain tumor detection with MRI images~\citep{bhuvaji2024brain}, where \edit{\oursMB} elicits visual factors; (2) climatic reanalysis data with fine-grained time and space coverage~\citep{compo2011twentieth}, where \edit{ LLM can} write code to interact with a database. (3) three-year news summary about one stock from the New York Times~\citep{BidecInnovations2023}, where \ours is applied to analyze the sequential data.

\subsection{Eliciting Tumor-Related Factors from Brain MRI Images}
\label{casestudy:brain_tumor}
In this section, we utilize the \ours to explore the image data set. Magnetic Resonance Imaging (MRI) is an important technique for detecting tumors in the human brain. The images are from an open Kaggle data set (\href{https://www.kaggle.com/datasets/sartajbhuvaji/brain-tumor-classification-mri/data}{kaggle/brain-tumor-classification-mri}) with an open-sourced project~\citep{bhuvaji2024brain}. Each sample is a scanning MRI of a human brain. In this case study, the interesting variable is the tumor type. We consider three types of MRI images: glioma, meningioma, and no tumor. We include 20 images for each category and the total sample size is 60.

\begin{figure}[t]
    \centering
    \includegraphics[width=0.75\textwidth]{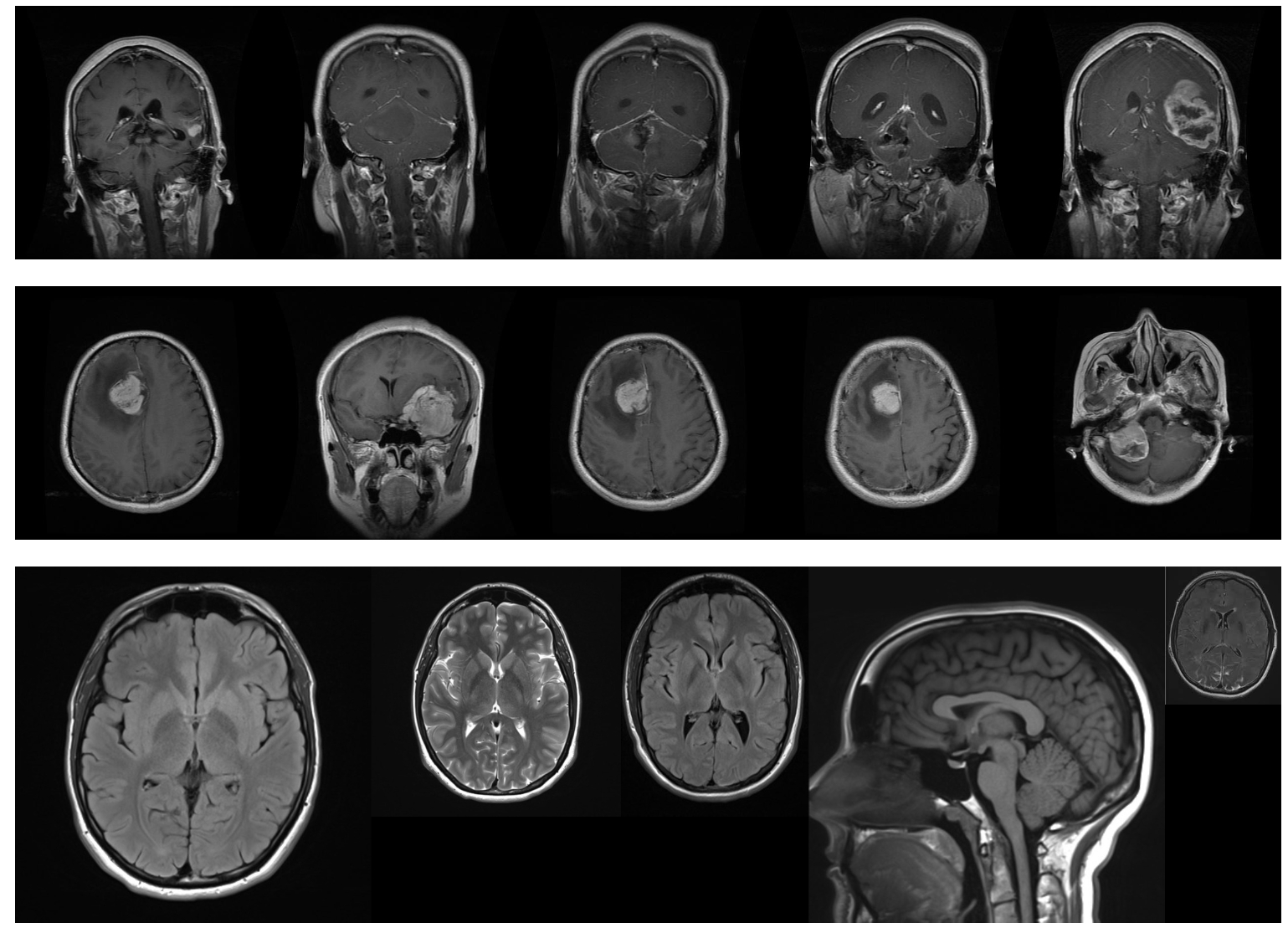}
    \caption{An example input from Brain Tumor data. Each row contains random samples from the same category (top-down: glioma, meningioma, and no tumor).}
    \label{fig:realworld_casestudy_tumer_iter1}
\end{figure}

\runin{Data Processing}  We use gpt-4-vision-preview to handle image samples. As the current GPT-4 cannot process multiple images simultaneously, we concatenate samples from different categories into one picture, as shown in Figure~\ref{fig:realworld_casestudy_tumer_iter1}, and provide additional instructions in the prompt to explain the format. For each proposed factor, the LLMs will go through all 60 sample images individually to evaluate the factor value.

\begin{figure}[tbp]
    \centering
    \includegraphics[width=\columnwidth]{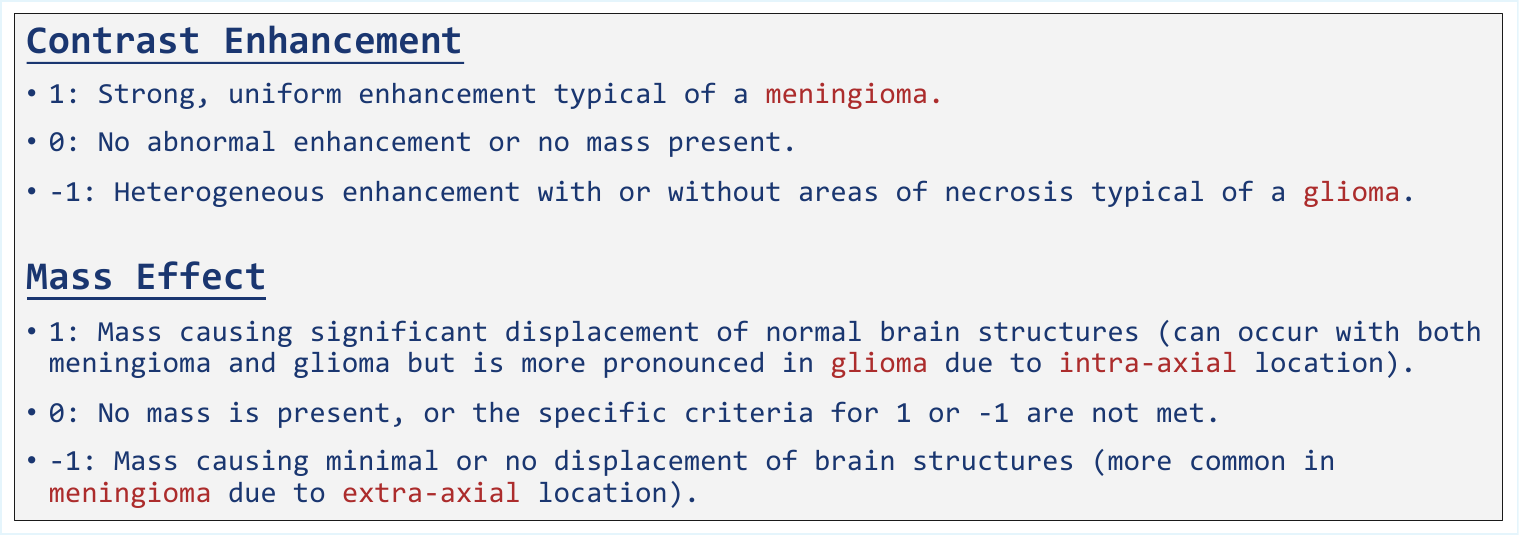}
    \caption{Detailed descriptions of these two final factors in the Brain Tumor case study.}
    \label{fig:factor_details_braintumor}
\end{figure}

\begin{figure}[tbp]
    \centering
    \includegraphics[width=0.5\columnwidth]{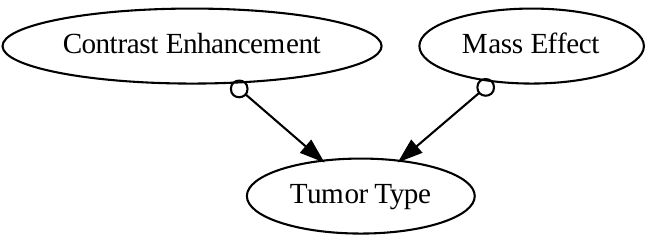}
    \caption{Final causal graph by \edit{\oursMB} in the Brain Tumor case study.}
    \label{fig:braintumor_causal_graph}
\end{figure}

\runin{Result} As shown in Figure~\ref{fig:braintumor_causal_graph}, there are two visual factors that appeared in the final causal graph: \emph{contrast enhancement} and \emph{mass effect}. One may refer to Figure~\ref{fig:factor_details_braintumor} to see detailed descriptions of these two factors.
We verify the proposed factors by searching medical literature according to keywords in factor descriptions. We found that both factors are aligned with existing research
~\citep{lyndon2019dural,haydar2022role,tuncc2021modeling,watts2014magnetic}. Therefore, we believe \ours framework can provide a reasonable starting point for domain experts.

\subsection{Eliciting Climate Factors with LLMs Using Tools}
\label{casestudy:enso}
The NOAA 20th Century Reanalysis V3 data set~\citep{compo2011twentieth,slivinski2019reanalysis} contains high-dimensional information about Earth's atmosphere with a fine-grained time-and-spatial coverage. We use \ours to explore El Niño-Southern Oscillation (ENSO) phenomenon: irregular fluctuations in sea surface temperatures~\citep{mcphaden1999genesis} in the Pacific Ocean.
The focus of this study is the future change in monthly SST in the Nino3 region, which could be an important indicator of ENSO events.

\runin{LLMs using tools} The data set is in the NetCDF format (network Common Data Format), so it is not convenient to be directly fed to LLMs. Therefore, the LLM is required to write code to define specific factors. An example is provided in Lst.~\ref{lst:enso_factor_code}. 
The \emph{measurement} specifies the climate variables like precipitation rate or temperature; the \emph{level} specifies the vertical location, like surface or specific pressure level; the region is a rectangle about the area concerned. This function will output a single time series about the measurement on the specified level averaged within the specified region.

\begin{figure}[t]
   \begin{lstlisting}[caption={LLM can be prompted to propose factors using complex tools.}, label={lst:enso_factor_code}]
    def get_air_temperature_at_cloud_layer_top_nino3():
        '''
        Air Temperature at Cloud Layer Top (Nino3 Region)
        '''
        # Declare observation
        observation_name = Observation(
            measurement="Air Temperature",
            level="Isentropic Levels",
            region={'latitude_min': -5, 'latitude_max': 5, 'longitude_min': 210, 'longitude_max': 270},
            detailed_level=300
        )
    
        # Return the observation values directly
        return observation_name
    
    factor_dict['Air Temperature at Cloud Layer Top (Nino3 Region)'] = get_air_temperature_at_cloud_layer_top_nino3()

\end{lstlisting}
\end{figure}

\runin{Feedback construction} 
As shown in Figure~\ref{fig:enso_feedback_example}, two types of information are included in the feedback: \emph{intermediate causal graph} and \emph{OLS regression result} of factors in the current estimation of a Markov blanket on the target variable. These two types of information can be drawn naturally from the \ours's intermediate results.

\begin{figure}[tbp]
    \centering
    \includegraphics[width=\columnwidth]{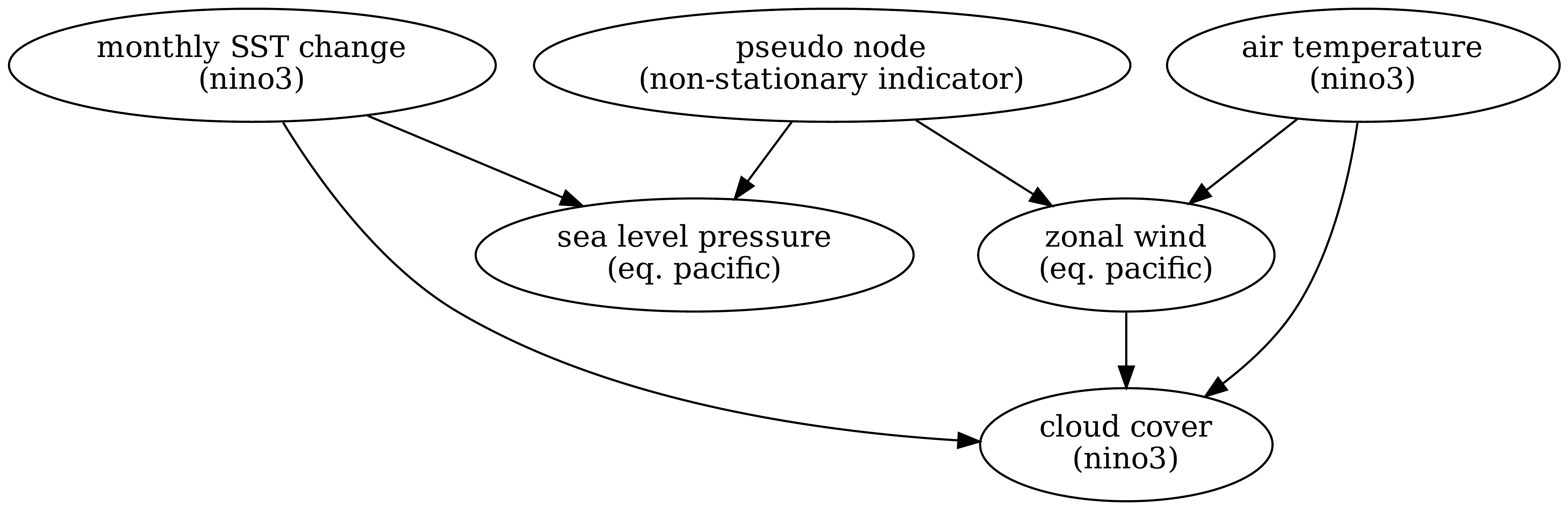} \\
    \includegraphics[width=\columnwidth]{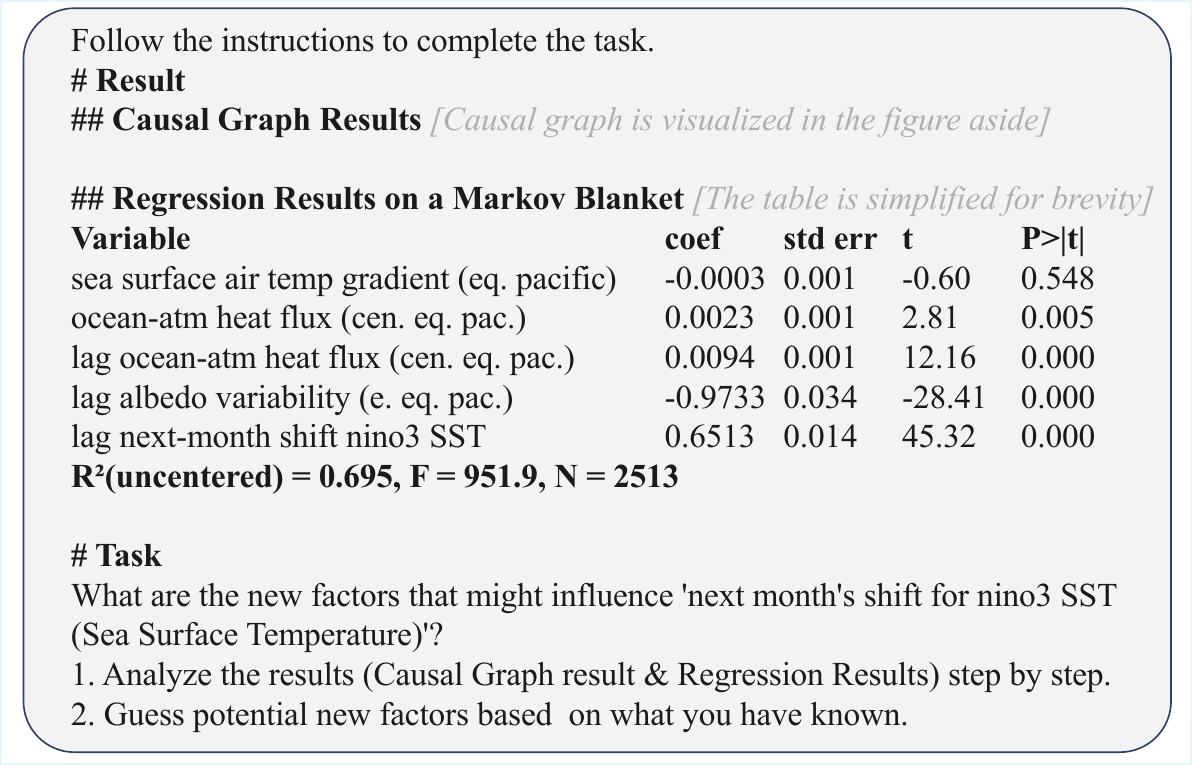}
    \caption{The prompt with feedback based on previous loop in the ENSO case study.}
    \label{fig:enso_feedback_example}
\end{figure}

\runin{Result interpretation} 
In this case study, we are treating non-stationary climate data, while assuming the causal structure is invariant. Under this assumption, the non-stationarity is actually helpful for causal structure learning \citep{huang2020causal,mooij2020joint}. To this end, we utilize the CD-NOD algorithm \citep{huang2020causal} to fully utilize the changing causal modules for better identifiability. CD-NOD will first identify the non-stationary nodes, whose conditional distribution given the causal parents is changing with time, and then use them for better causal structure recovery. Four factors are identified to be non-stationary, as shown in Figure~\ref{fig:enso_causal_graph}. 

We have adjusted the nodes' names, shapes, and colors for better visualization. There are 14 nodes in total, with 13 factors identified by \ours. Each factor is a time series about a certain climate measurement above a specific level averaged over a specific region. For simplicity, we only considered instantaneous causal relations among those time series. 

There are three regions identified to be relevant to the ENSO phenomenon: 
\begin{itemize}
    \item \textbf{Equatorial Pacific Region {\color{ensoorange}(Orange Nodes)}.} This region (5N-5S, 120W-280W) is one of the most active places about ENSO. It becomes significantly warm during the El Niño phase and becomes significantly cold during the La Niña phase \citep{ZHANG20171}.
    \item \textbf{Nino3 Region {\color{ensoblue}(Blue Nodes)}.} This region (5N-5S, 150W-90W) is one of the most classical regions to monitor the El Niño events by scientists. It is also used by humans to construct the target variables \citep{mcphaden1999genesis}. 
    \item \textbf{South American Coastal Region {\color{ensogreen}(Green Nodes)}.}  This region (0N-20S, 80W-60W) includes the Peruvian coastal up-welling system, and is noticed to be relevant to the ENSO cycle \citep{Tarazona2001}.
\end{itemize}

Some insights are delivered by paths in the output causal graph:
\begin{itemize}
    \item \textbf{{\color{ensoorange}Sea level Pressure.}} This factor is about the sea-level pressure on the equatorial Pacific region. The pressure gradient will influence the movement of warm water, and thus influence the sea surface temperature (SST) change \citep{bjerknes1969atmospheric}. In addition, pressure can influence the water evaporation and thus regulate through the water circulation. This also matches another indirect path $\text{\color{ensoorange}Sea level Pressure} \rightarrow \text{\color{ensoorange}Sensible Heat Net Flux} \rightarrow \text{\color{ensogreen}Volumetric Soil Moisture} \rightarrow \text{\color{ensoblue}Cloud Cover} \rightarrow \text{\color{ensoblue}SST Change}$.
    \item \textbf{{\color{ensoorange}Momentum Flux, V-component.}} This factor is about the vertical movement of air. It is crucial in driving atmospheric convection \citep{bjerknes1969atmospheric}, and it is related to the Walker Circulation, which is an important component in the ENSO system\citep{wang2004enso}. Also, it could influence the change in sea level pressure and indirectly influence the SST change. 
    \item \textbf{{\color{ensoblue}Cloud Cover.}} The factor is the fraction of the sky covered by clouds in the NINO3 region. It could influence the SST Change through solar radiation as well as water circulation. It is confirmed to have a significant correlation \citep{liu2016variation} with ENSO events and plays an important role in the atmospheric circulation and hydrological cycle~\citep{mishra2019investigating}.
    \item \textbf{{\color{ensogreen}Soil Temperature.}} This might be a novel hypothesis proposed by \ours, since we found no sufficient research to confirm this point to the best of our knowledge. This causal graph has also suggested two possible indirect mechanisms: (1) through \emph{Volumetric Soil Moisture} and  \emph{Cloud Cover}; and (2) through \emph{Convective Precipitation Rate} and \emph{Sea Level Pressure}. Therefore, this finding encourages more serious investigations of these hypotheses. 
\end{itemize}

\begin{figure}[tbp]
    \centering
    \includegraphics[width=\columnwidth, trim=35 110 35 40, clip]{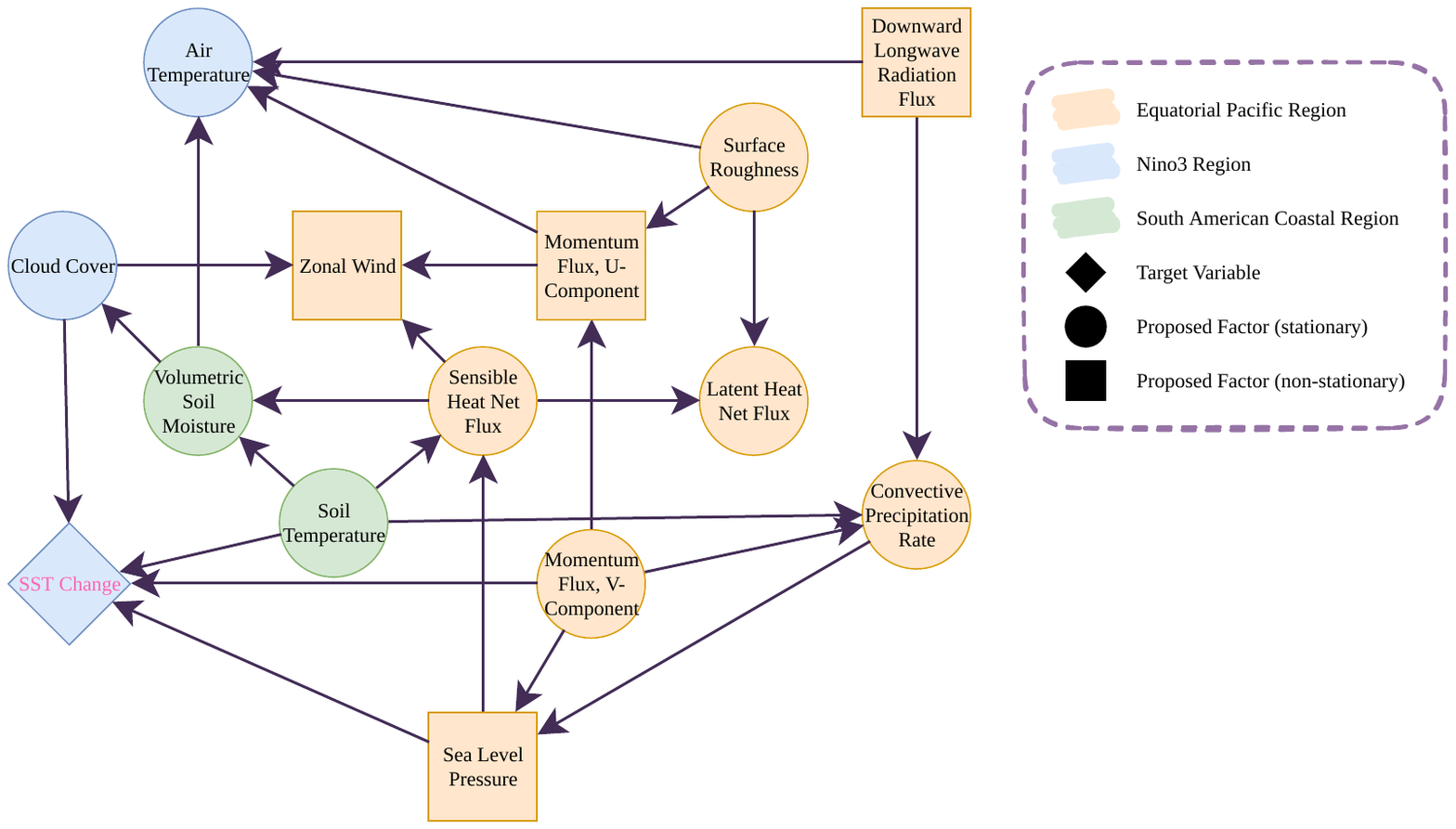}
    \caption{The final causal graph found by \edit{\oursMB} in the ENSO case study.}
    \label{fig:enso_causal_graph}
\end{figure}

\subsection{Eliciting Financial Factors with Sequential News Data}
\label{casestudy:stocknews}
In this section, we utilize \ours to explore time series data with text.
The data set consists of the stock value of Microsoft (MSFT) from 2006 to 2009 and its news summary (only include news in \emph{the New York Times}). This is a subset of an open Kaggle data set (\href{https://www.kaggle.com/datasets/BidecInnovations/stock-price-and-news-realted-to-it/data}{kaggle/stock-price-and-news-realted-to-it}).
Each sample is one trading day with the company's close stock price and news. 
The target value is the future return rate, and we are curious about factors in the related news.
We fed data during the first 200 trading days to \ours, and we used the following 400 trading days for evaluation.

The target value is a binary variable: $1$ for positive return rate in next week, 
$0$ otherwise.
News fed to \ours are grouped according to the value of $Y_t$. To keep the whole prompt within the context limit, only 3 samples will be randomly included in each group.

\runin{Factor processing} Given one proposed factor, denote its annotated value at the $t$-th trading day as $a_t \in \{-1, 0, 1\}$. Each factor is averaged over a rolling window of the past $M$ days:
\begin{equation}
    S_t = \frac{1}{M} \sum_{t-M+1 \le i \le t } a_i.
\end{equation}

Each variable, including return rate, is normalized by the rolling standard deviation over the past $M$ days:
\begin{equation}\label{equ:stocknews_def_Ft}
    F_t = \frac{1}{\text{Sd}\left(\left\{S_i\right\}_{t-M+1 \le i \le t }\right)+1} S_t.
\end{equation}

An example of a processed factor can be seen in Figure~\ref{fig:trading_signal_examples}.

\begin{figure}[tbp]
    \centering
    \includegraphics[width=0.6\columnwidth]{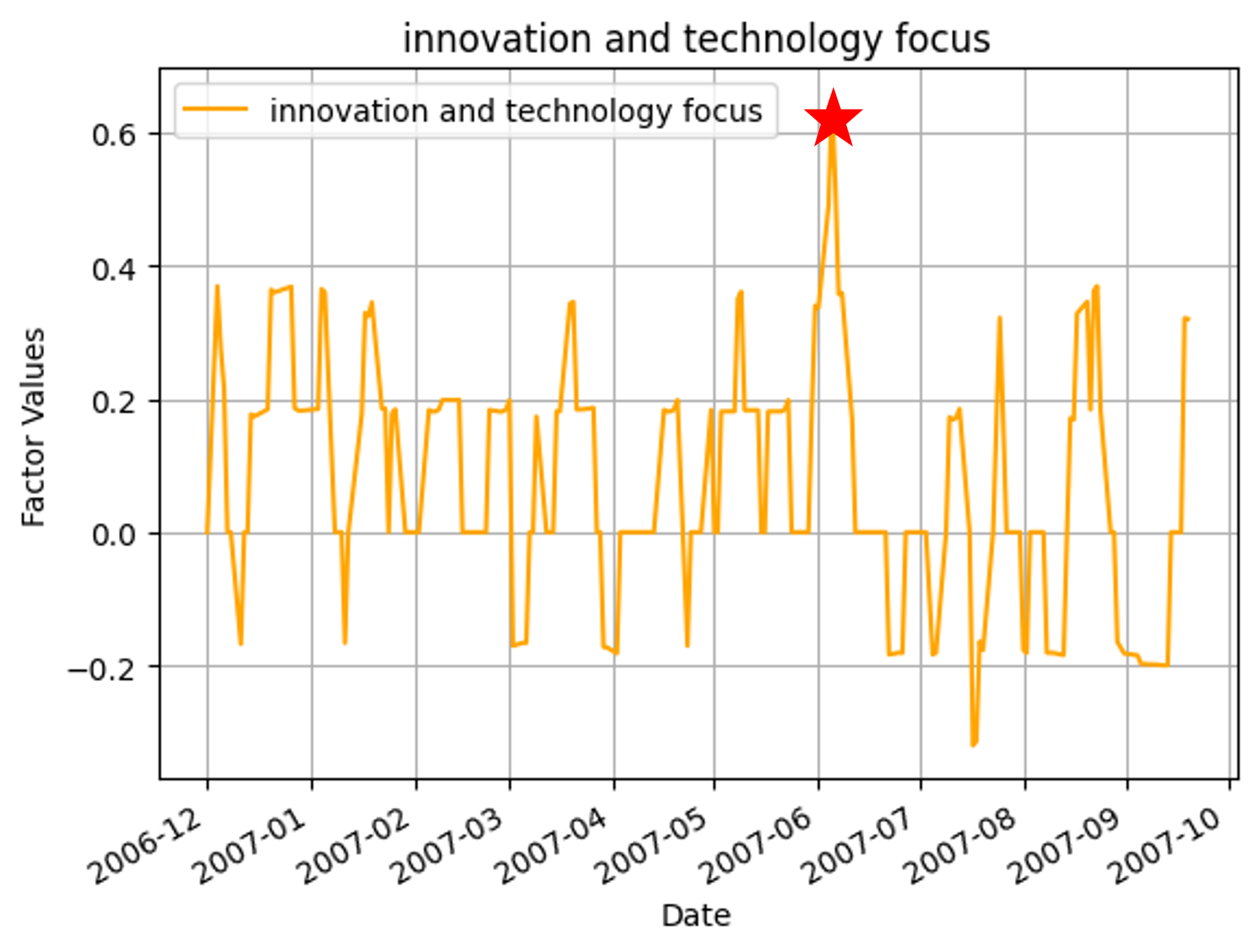}
    \caption{Example of one processed factor \emph{Innovation and Technology Focus} during the first 200 trading days. The red star marks the highest value. In June 2007, there was a discussion about future competition in desktop operating systems and the trend towards web-based services~\cite{Markoff2007}.}
    \label{fig:trading_signal_examples}
\end{figure}

\runin{Resulting causal graph} As shown in Figure~\ref{fig:stocknews_causal_graph}, four factors are identified in the final causal graph to form a Markov blanket of return rate:
\emph{Product Focus}, 
\emph{Legal and Regulatory Issues}, 
\emph{Market Strategy}, and 
\emph{Innovation and Technology Focus}.
One can refer to Figure~\ref{fig:stocknews_factor_des} for more detailed descriptions of these factors.
One factor (\emph{Innovation and technology focus}) is identified as a possible cause of the return rate; this matches the nature of the company's type and reflects people's expectation of the company to keep creating innovative computer software. It is also interesting to see that \ours captures the structure between factors and \emph{market strategy}, where \emph{product focus} and \emph{legal and regulatory issues} are identified as potential causes. It also implies the existence of a latent confounder between \emph{market strategy} and \emph{return rate} that may not be significantly reflected in news text. 

\begin{figure}[tbp]
    \centering
    \includegraphics[width=0.6\columnwidth, trim=40 40 35 35, clip]{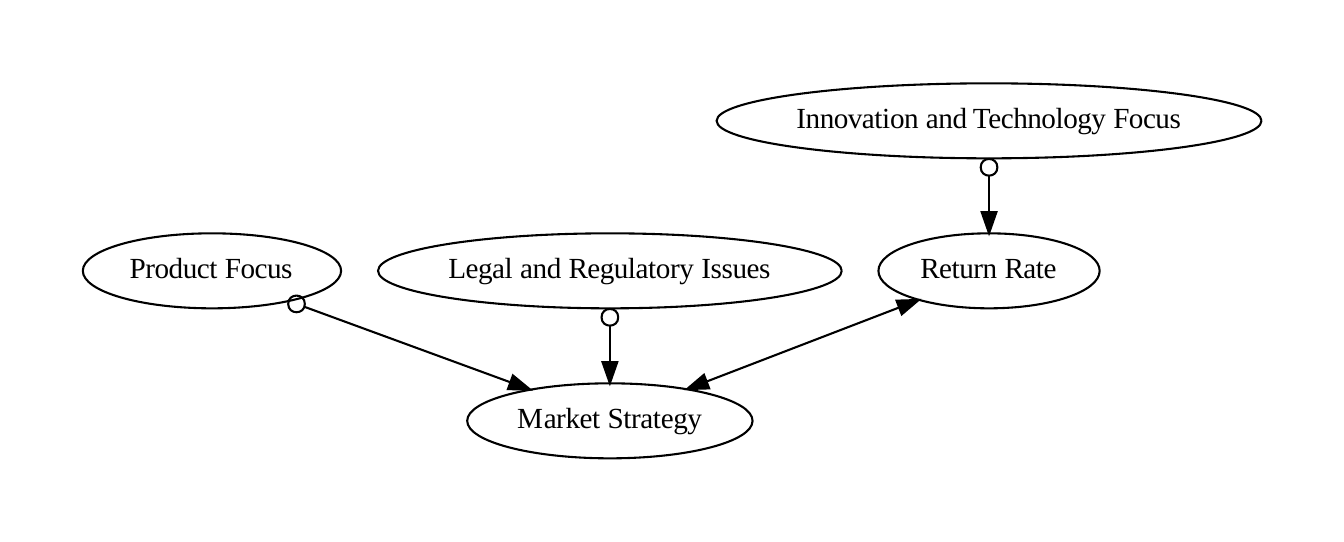}
    \caption{The final causal graph by \edit{\oursMB} in the Stock News case study.}
    \label{fig:stocknews_causal_graph}
\end{figure}

\begin{figure}[tbp]
    \centering
    \includegraphics[width=\columnwidth]{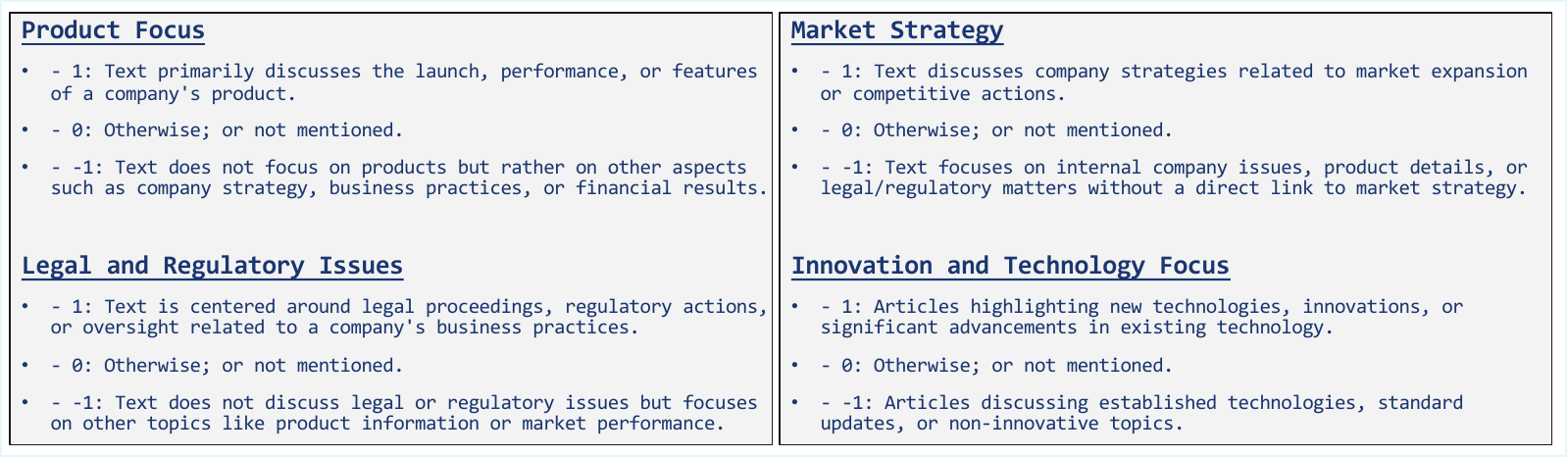}
    \caption{The descriptions of the proposed factors in the Stock News data.}
    \label{fig:stocknews_factor_des}
\end{figure}

\runin{Evaluation by trading strategies} %
For each factor $\{F_t\}$ processed in Equation~\ref{equ:stocknews_def_Ft}, we establish a trading strategy on it to see the performance in the out-of-sample trading days. At each trading day $t$, we fit the model $r_t = \alpha + \beta F_t$, 
using samples $\{r_i, F_i\}_{t-200 \le i<t}$.  
After the estimation, we make decision based on $\hat{r}_t := \hat{\alpha} + \hat{\beta} F_t $. 
If  $\hat{r}_t > 0$, we \emph{go long} with 1 unit of capital, i.e., purchase that amount of stock, the gain will be $r \times 1$ units of capital. If  $\hat{r}_t < 0$, we \emph{go short} with 1 unit of capital,  i.e., borrow and sell that amount of stocks immediately and buy them back next time, the gain will be $-r \times 1$ units of capital.
We introduce an additional \emph{Buy and Hold} baseline to always go long one unit of capital. The trading evaluation is after the $200$-th day and thus is not accessible by \ours. 

The cumulative return plot is shown in Figure~\ref{fig:stocknews_factor_performance}, and the metrics commonly used in economic literature \citep{eugene1970efficient,wachter2013can,bybee2023business} are shown in Table~\ref{table:stocknews_performance}. One important metric to see a trading strategy's effectiveness is the Sharpe ratio: a measure of risk-adjusted return, showing the excess return (over the risk-free rate, we set it to be $2\%$) per unit of standard deviation. A higher Sharpe ratio indicates better performance per unit of risk. We see that the \emph{Innovation and Technology Focus} factor yields the highest Sharpe ratio and outperforms other non-causal factors.

\begin{figure}[tbp]
    \centering
    \includegraphics[width=\columnwidth]{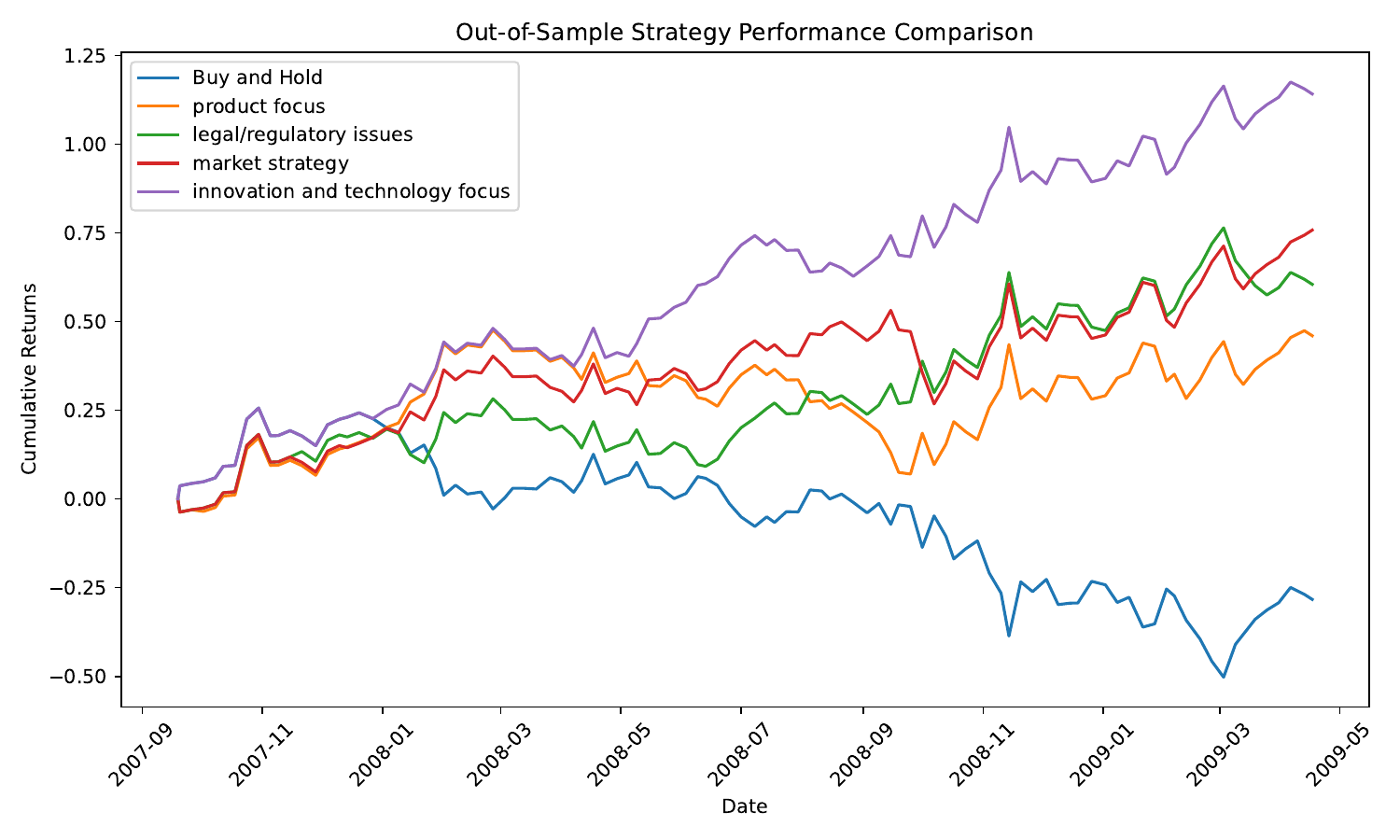}
    \vspace{-0.2in}
    \caption{The out-of-sample cumulative return of trading strategies based on different factors in the Stock News case study.}
    \label{fig:stocknews_factor_performance}
    \vspace{-0.2in}
\end{figure}

\begin{table}[tbp]
    \centering
    {\small
    \setlength{\tabcolsep}{3pt}
    \resizebox{\textwidth}{!}{
            \begin{tabular}{@{}lccccc@{}}
                \toprule
                                    & Buy and Hold  & \makecell{product\\focus} & \makecell{legal/regulatory\\issues}   & \makecell{market\\strategy}   & \makecell{innovation and\\technology focus} \\
                \midrule
                Expected Return     & -3.50         & 0.29          & 0.38                      & 0.48              & 0.72 \\
                Sharpe Ratio        & -2.38         & 0.70          & 0.94                      & 1.20              & {1.86}  \\
                t-Stat              & -3.00         & 0.89          & 1.19                      & 1.51              & 2.35 \\
                Information Ratio   & -             & 2.45          & 2.50                      & 2.56              & 2.73 \\
                $\alpha$            & -             & 3.79          & 3.88                      & 3.98              & 4.22 \\
                $\alpha$ t-Stat     & -             & 3.09          & 3.15                      & 3.23              & 3.44 \\
                Max Loss            & -0.50         & -0.15         & -0.15                     & -0.15             & -0.15 \\
                Skew                & -0.43         & -0.14         & -0.15                     & -0.41             & -0.38 \\
                \bottomrule
            \end{tabular}
    }
    }
    \caption{Performance of trading strategy according to each factor.}
    \label{table:stocknews_performance}
\end{table}

\section{Discussion and Conclusions} 
\label{sec:discussion-conclusion}

\edit{
In this work, we reformulate and extend our previous work~\citep{causalcoat2024} that is only on the Markov blanket setting, and extend the \ours framework for \textit{partial ancestral graphs} (\oursPAG) and \textit{adjustment sets} (\oursADJ).  
In general,
}
\ours aims to leverage the rich knowledge learned by LLMs to elicit causal representation from unstructured data, like images or text.
The key module of \ours is a propose-and-verify procedure where the advantages of LLMs are employed on understanding unstructured data and actively proposing possible candidate high-level factors, and these factors will be further verified by causal discovery methods. 
First, we propose an iterative method that can elicit Markov Blanket for a given target variable from its paired unstructured data. Based on this method, we further present an approach to extending a given Partial Ancestral Graph to maximally refine its skeleton and arrowhead orientation. Furthermore, we extend our framework to find adjustment sets for cause-effect estimation.

For evaluation, we construct \apple benchmark in the two critical settings: (1) Eliciting Markov Blanket; and (2) Extending Partial Ancestral Graph. By comparing with LLMs' reasoning approach, \ours presents robust performance when the reality deviates from \minor{LLMs' priors}, and is a reliable framework to integrate the advantages of LLMs into causality tasks. 
We define two metrics: perception score and capacity score, to measure the LLMs' causal ability under the \ours framework; the empirical results over ten predominant LLMs support their competence. 
We show \ours is useful and \minor{reliable} in semi-realistic benchmarks and real-world case studies. In \pain benchmark, \ours can elicit factors in radiculopathy and pathophysiology with only symptom descriptions, where LLMs utilize their intrinsic domain knowledge with verification from causal discovery methods. In ATE estimation data sets, \ours\ \minor{presents} lowest RMSE loss compared with baselines, and \minor{is} competitive with results from human experts.
In real-world case studies, \ours\ \minor{elicits} two visual factors related to tumor type by analyzing tumor images data; \ours can also propose factors by writing codes to interact with climate database. These results are consistent with domain literature. In the financial case study, the proposed factors present high Sharpe ratios in out-of-distribution trading simulations.

The intensive empirical results, as well as the identifiability results of our proposed methods, demonstrate the insight that LLMs and causal discovery methods are mutually beneficial. On the one hand, LLMs that learn rich world knowledge about the world can assist with eliciting high-level hidden variables from low-level observational data. On the other hand, causal discovery methods can provide verification and feedback to LLMs, improving the robustness when \minor{LLMs' priors don't} fit the specific data.

There are several open questions that can be further investigated in the future. First, in this paper we consider the observational data, which can be extended to the setting where LLMs are able to choose and conduct interventions on certain target. This will bring more information on the causal mechanism. Second, in this paper we mainly focus on the constraint-based methods in causal discovery; one can explore how to fully utilize other methods like ICA-based approach, and also causal representation learning methods. Third, the \ours framework can also integrate experimental design methods in new perspectives: (1) by expanding the subgraph around an intervention target, \ours can resolve undetermined edges using unstructured data, reducing the need for redundant experiments; and (2) with the expanded node set and the refined causal structure, \ours delivers richer input to methods for experimental design\minor{, and} yields more valuable experimental data sets.

\clearpage
\newpage

\acks{CXL and BH were supported by NSFC Major Research Plan No. 92570109, RGC Young Collaborative Research Grant No. C2005-24Y, RGC Research Fellow Scheme No. RFS2627-2S02, RGC General Research Fund No. 12202026, RGC General Research Fund No. 12200725, NSFC General Program No. 62376235, Guangdong and Hong Kong Universities “1+1+1” Joint Research Collaboration Scheme, HKBU Faculty Niche Research Areas No. RC-FNRA-IG/25-26/SCI/05, and HKBU CSD Departmental Incentive Scheme. TLL was partially supported by the following Australian Research Council projects: DP260102466, FT220100318, and DP220102121. MG was supported by ARC DP240102088 and WIS–MBZUAI 14257. JC was supported by RGC Young Collaborative Research Grant No. C2005-24Y. KZ would like to acknowledge the support from NSF Award No. 2229881, AI Institute for Societal Decision Making (AI-SDM), the National Institutes of Health (NIH) under Contract R01HL159805, and grants from MBZUAI-WIS Joint Program, and the Al Deira Causal Education project. Support for the NOAA-CIRES-DOE Twentieth Century Reanalysis Project version 3 dataset is provided by the U.S. Department of Energy, Office of Science Biological and Environmental Research (BER), by the National Oceanic and Atmospheric Administration Climate Program Office, and by the NOAA Physical Sciences Laboratory. The authors declare no competing interests.}

\clearpage
\newpage

\appendix

\begin{center}
    \LARGE \bf {Appendix of \ourst}
\end{center}
\etocdepthtag.toc{mtappendix}
\etocsettagdepth{mtchapter}{none}
\etocsettagdepth{mtappendix}{subsection}
\tableofcontents

\clearpage
\newpage

\section{Table of Notations}
\label{appdx:notation}

\begin{table}[H]
\renewcommand{\arraystretch}{1.5}
\centering
{\small
\resizebox{\textwidth}{!}{
\begin{tabular}{c|l}
\toprule
\textbf{Notation} & \textbf{Description} \\ \hline
$\Psi$ & The LLM for factor proposal.  \\ \hline
$\vp$ ($\vp^t$) & Prompt used (at the $t$-th \ours round) for factor proposal.  \\ \hline
$\Psi_s$ & The LLM for factor parsing.  \\ \hline
$\vp_p$ & Prompt used for factor parsing.  \\ \hline
$\gA$ & Algorithm for causal discovery. In this paper we mainly use FCI. \\ \hline
$\mX$ & Unstructured data; treated as a high-dimensional random vector. \\ \hline
$\varset{U}$ & The set of given structured random variables; and $\varset{U}=\{U_1, U_2, \cdots, U_m\}$ \\ \hline
$\mU$ & The random vector consists of variables in $\varset{U}$, i.e.,  $\mU=[U_1, U_2, \cdots, U_m]^\top$ \\ \hline
$Y$ & The target variable from $\varset{U}$; or the effect variable in the causal inference setting. \\ \hline
$\gD$ & Full sample set $\gD=\{\vx^{(j)}, \vu^{{j}}\}_{j=1}^n$. $(\vx^{(j)},\vu^{(j)})$ is a realization of $(\mX,\mU)$.\\ \hline
$\vw_i$ & The $i$-th factor proposed by LLMs, like \emph{sweetness}, \emph{size}, or \emph{scent}   \\ \hline
$\varset{W}$ & The set of proposed factors from $\mX$; and $\varset{W}=\{W_1, W_2, \cdots, W_k\}$, where $W_i\triangleq \vw_i(\mX)$ \\ \hline
$\mW$ & The random vector consists of variables in $\varset{W}$, i.e.,  $\mW=[W_1, W_2, \cdots, W_k]^\top$ \\ \hline
$h_{\gS}(\cdot)$ & LLM-induced representation over index $\gS$, e.g., $h_{\{1,3\}}(\cdot) = \left( \vw_1(\cdot), \vw_3(\cdot) \right)$. $h(\cdot) \triangleq h_{[k]}(\cdot)$\\ \hline
$\gC$ & The predefined value space for LLM-proposed factors, e.g., set $\gC$ as $\{-1,0,1\}$ \\ \hline
$\varset{V}$  & The augmented structured variable set by \ours, i.e.,  $\varset{V} \triangleq \varset{W} \cup \varset{U}$. \\ \hline
$\mV$  & The random vector consists of variables in $\varset{V}$, i.e.,  $\mV=[U_1, U_2, \cdots, U_m, W_1, W_2, \cdots, W_k]^\top$\\ \hline
$\vv^{(j)}$  & The realization of random vector $\varset{V}$ in the $j$-th sample. \\ \hline
$\varset{O}$ & The set of all possible factors that can be defined based on $\mX$. \\ \hline
$\varset{E}$ & The maximal observable variable set $\varset{E} \triangleq \varset{O} \cup \varset{U}$. \\ \hline
$T$ & The treatment variable from $\varset{U}$ in the causal inference setting. \\ \hline
$\gP_\varset{V}$  & The partial ancestral graph over node set $\varset{V}$; similar for $\gP_\varset{E}$ \\ \hline
$\widehat{\gD}^t$ & The small subset of sample used at the $t$-th \ours round. \\ \hline
$\gW^t$ & The set of factors proposed by LLMs at the $t$-th round. $\gW^{\le t} \triangleq \gW^{1} \cup \cdots \cup \gW^{t}$\\ \hline
$C_\Psi$ & Capacity Score of LLM $\Psi$ \\ \hline
$p_\Psi$ or $p$ & Perception Score of LLM $\Psi$ \\ 
\bottomrule
\end{tabular}}
}
\label{tab:notation}
\end{table}
 
\clearpage

\section{Proofs for Theoretical Results}
\label{appdx:proof}

\begin{proof}[Proof for Proposition~\ref{prop:mutual_info_down}]
From Eq.~\ref{eq:mutual_info_down:notind}, we have 
    \begin{equation}
        H(Y \mid h_{[k]}(\mX); \vw_{k+1}(\mX)) < H(Y \mid h_{[k]}(\mX) );
    \end{equation}
Since $ Y \ind h_{[k+1] \setminus \gS}(X)  \mid  h_{\gS}(X)$, we have 
    \begin{equation}
        H(Y \mid h_{\gS}(\mX);  h_{[k+1] \setminus \gS}(\mX) ) = H(Y \mid h_{\gS}(\mX) ).
    \end{equation}
Therefore:
    \begin{equation}
        \begin{aligned}
            H(Y \mid h_{\gS}(\mX) ) & = H(Y \mid h_{\gS}(\mX),  h_{[k+1] \setminus \gS}(\mX) ) \\
            & = H(Y \mid   h_{[k+1]}(\mX) ) \\ 
            & = H(Y \mid h_{[k]}(\mX), \vw_{k+1}(\mX)) \\
            & < H(Y \mid h_{[k]}(\mX) )
        \end{aligned}
    \end{equation}
Also note that 
    \begin{equation}
        H (Y \mid h_{\gS}(\mX), \mX) = H (Y \mid \mX)=H (Y \mid h_{[k]}(\mX), \mX),
    \end{equation}
therefore:
    \begin{equation}
            H (Y \mid h_{\gS}(\mX)) - H (Y \mid h_{\gS}(\mX), \mX) < H (Y \mid h_{[k]}(\mX)) - H (Y \mid h_{[k]}(\mX), \mX), 
        \end{equation}
which is Eq.~\ref{eq:mutual_info_down:result}.
\end{proof}

\begin{proof}[Proof for Proposition~\ref{prop:cha_f_i}]
Let $n_s$ be the number of rounds that LLM proposed at least one usable factor satisfying Assumption \ref{asp:LLM_good}. Since $t$ is large, its Binomial distribution $\text{Binom}(t,p)$ can be approximated by Gaussian distribution $\gN\left(tp,tp(1-p)\right)$.
To enforce 
    \begin{equation}
        \begin{aligned}
            \Pr\left(  \frac{I\left( Y; \mX \mid h_{\le t}(\mX) \right)}{I\left( Y; \mX  \right)} < \epsilon \right) 
        & \ge \Pr\Big(  \left(1-C_\Psi\right)^{n_s} < \epsilon \Big) \\
        & = \Pr\Big(  n_s > \frac{\log{\epsilon}}{\log{(1-C_\Psi)}}\Big) \\
        & = \Pr\Big(  \frac{n_s-tp}{\sqrt{tp(1-p)}} > 
        \frac{1}{\sqrt{tp(1-p)}} \left( \frac{\log{\epsilon}}{\log{(1-C_\Psi)}} -tp  \right)
        \Big)\\
        &\ge 1-\delta,
        \end{aligned}
    \end{equation}
we have 
    \begin{equation}
        \frac{1}{\sqrt{tp(1-p)}} \left( \frac{\log{\epsilon}}{\log{(1-C_\Psi)}} -tp  \right) < z_\delta = - |z_\delta|
    \end{equation}
with Gaussian distribution approximation.

Isolating $\sqrt{t}>0$ from the above inequality, we would have the desired result.
\end{proof}

\edit{
\begin{proof}[Proof for Proposition~\ref{prop:finite-sample-mb}]
Under the Gaussian assumption, the sample partial correlation \(r\) between \(Y\) and \(W\) given \(\mathcal{M}\) can be transformed via Fisher's \(z\)-transformation:
\[
z = \frac{1}{2} \ln \frac{1+r}{1-r}.
\]
When the true partial correlation is \(\rho\), the statistic \(z\) is approximately normally distributed with mean \(\zeta = \frac{1}{2} \ln \frac{1+\rho}{1-\rho}\) and variance \(1/(n - |\mathcal{M}| - 3)\)~\citep{fisher1921probable}. The null hypothesis \(H_0: \rho = 0\) is rejected at level \(\alpha\) if
\[
|z| > \frac{z_{1-\alpha/2}}{\sqrt{n - |\mathcal{M}| - 3}}.
\]
Under the alternative hypothesis \(H_1: \rho \neq 0\), the power of this two-sided test is
\[
1-\beta \;=\; \Pr\!\left( \bigl| z \bigr| > \frac{z_{1-\alpha/2}}{\sqrt{n - |\mathcal{M}| - 3}} \;\Big|\; \rho \right)
\;\approx\; \Phi\!\Bigl( \sqrt{n - |\mathcal{M}| - 3}\;|\zeta| - z_{1-\alpha/2} \Bigr),
\]
where \(\Phi\) is the cumulative distribution function of the standard normal distribution. Setting the non-centrality parameter to guarantee power at least \(1-\beta\) yields
\[
\sqrt{n - |\mathcal{M}| - 3}\;|\zeta| - z_{1-\alpha/2} \;\ge\; z_{1-\beta}.
\]
Solving for \(n\) and substituting \(\zeta = \frac{1}{2}\ln\frac{1+\rho}{1-\rho}\) gives the claimed bound.
\end{proof}
}

\begin{proof}[Proof for Proposition~\ref{prop:arrowhead-preserve}] \label{proof:arrowhead-preserve}
First we would show that the adjacency within $\varset{U}'$ are preserved given $\MB(\varset{U}')$. Then, we show the arrow heads would be preserved by analyzing the FCI orientation rules $\gR 0$--$\gR 4$ of \citet{ZHANG20081873}. 

    \textbf{Adjacency.} Suppose $\alpha, \beta \in \varset{U}'$ are adjacent in $\gP_\varset{E}$, then they cannot be m-separated by any subset of $\varset{E} \setminus \{\alpha, \beta\}$. Since $\varset{V} \setminus \{\alpha, \beta\} \subseteq \varset{E} \setminus \{\alpha, \beta\}$, they are adjacent in $\gP_\varset{V}$.
    
    Suppose $\alpha, \beta \in \varset{U}'$ are not adjacent in $\gP_\varset{E}$. We can construct a set $\varset{C} \subseteq \varset{E}$ as follows: Begin with $\varset{C}=\emptyset$, if $\alpha \ind \beta \mid \varset{C}$, stop the construction. For each path m-connecting $\alpha$ and $\beta$ conditional on $\varset{C}$, if it consists of colliders, then removing any of these nodes from $\varset{C}$ leads to m-connection by previous processed path, thus $\alpha$ and $\beta$ must be adjacent, this is a contradiction. Let $\theta$ be the non-collider closest to $\alpha$, and put it into $\varset{C}$. We show $\theta \in \MB(\alpha)$ by induction: (1) if it is adjacent to $\alpha$, then we are done; (2) if not, it would be connected with colliders from $\MB(\alpha)$, thus it has to be included in $\MB(\alpha)$. By this construction, we have $\alpha \ind \beta \mid \varset{C}$ and $\varset{C} \subseteq \MB(\alpha) \subseteq \varset{V}$, which means $\alpha, \beta \in \varset{U}'$ are not adjacent in $\gP_\varset{V}$.

    \textbf{Arrow heads.} %
    We construct $\widehat{\gP}_\varset{E}$ from $\gP_\varset{V}$ as follows: (1) we extend node set from $\varset{V}$ to $\varset{E}$; (2) Let the edges between $\varset{U}'$ and $\varset{E} \setminus \varset{U}'$ in $\widehat{\gP}_\varset{E}$ be the same as $\gP_\varset{E}$. By this construction, the skeleton in $\widehat{\gP}_\varset{E}$ is same as $\gP_\varset{E}$. If for some nodes $\alpha, \beta \in \varset{U}'$, the edge $\alpha \leftarrow \!\! * \ \beta$ occurs in  $\gP_\varset{E}$ but not in $\gP_\varset{V}$, some orientation rules in $\gR 0 - \gR 4$ should be applicable in $\widehat{\gP}_\varset{E}$ for some edges among $\varset{U}'$, and we will show this is impossible. If for some nodes $\alpha, \beta \in \varset{U}'$, the edge $\alpha \leftarrow \!\! * \ \beta$ occurs in  $\gP_\varset{V}$, some rules in $\gR 0 - \gR 4$ are applied.

    If $\gR 0$ is applicable to $\widehat{\gP}_\varset{E}$, then there exists an unshielded collider $<\alpha, \gamma, \beta>$ with $\gamma, \beta\in \varset{U}'$, and $\gamma$ should have been processed by the loop in Algorithm~\ref{alg:coat-search}. Therefore, since $\alpha \in \MB(\gamma) \subseteq \MB(\varset{U}')$, the triple is also unshielded in $\gP_\varset{V}$, which means  $\gR 0$ is applicable in $\gP_\varset{V}$, contradicting with that $\gP_\varset{V}$ is maximally informative. 
    The similar argument can be applied on $\gR 1 - \gR 4$. 

    If $\gR 1$ is applicable to $\widehat{\gP}_\varset{E}$, the antecedent implies an unshielded triple $<\alpha, \beta, \gamma>$, and $\beta, \gamma \in \varset{U}'$. Since there is an uncertain mark $\circ$ on $\beta$ (i.e., $\beta \circarrowstar \gamma$), $\beta$ should have been processed by the loop in Algorithm~\ref{alg:coat-search}. Therefore, we have $\alpha \in \MB(\beta) \subseteq \MB(\varset{U}')$,  $\alpha$ should be included in $\varset{V}$ and the triple should also be unshielded in $\gP_\varset{V}$. Thus, $\gR 1$ is also applicable to $\gP_\varset{V}$, which is a contradiction. 
    
    If $\gR 2$ is applicable to $\widehat{\gP}_\varset{E}$, then there exist $\alpha \rightarrow \beta \starrightarrow \gamma$ or $\alpha \starrightarrow  \beta \rightarrow \gamma$, and $\alpha \stararrowcirc \gamma$, and $\alpha, \gamma \in \varset{U}'$. Then $\gamma$ should have been processed by the loop in Algorithm~\ref{alg:coat-search} due to its uncertain mark $\circ$. Therefore, $\alpha, \beta \in \MB(\gamma) \subseteq \MB(\varset{U}')$. Thus, $\gR 2$ is also applicable to $\gP_\varset{V}$, which is a contradiction. 

    If $\gR 3$ is applicable to $\widehat{\gP}_\varset{E}$, then there exist $\alpha \starrightarrow \beta \starleftarrow \gamma$, $\alpha \stararrowcirc\theta \circarrowstar\gamma$, $\theta \stararrowcirc \beta$ and $\theta, \beta \in \varset{U}'$. Therefore, $\beta$ should have been processed by the loop in Algorithm~\ref{alg:coat-search}, which means $\alpha, \gamma, \theta \in \MB(\beta) \subseteq \MB(\varset{U}')$ and thus $\gR 3$ is also applicable to $\gP_\varset{V}$, i.e.,  contradiction. 
    
    If $\gR 4$ is applicable to $\widehat{\gP}_\varset{E}$, its antecedent implies a discriminating path $p=<\theta, \cdots, \alpha, \beta, \gamma>$ where $\beta, \gamma \in \varset{U}'$ and $\beta \circ \! \! - \! \! * \gamma $ would be oriented as $\beta \rightarrow \gamma $ or $\beta \leftrightarrow \gamma $, $\theta$ is not adjacent to $\gamma$, and each nodes between $\theta$ and $\beta$ is a collider and a parent of $\gamma$. Note that $\gamma \in \MB(\beta)$, thus all nodes between $\theta$ and $\beta$ are included in $\MB(\beta)$ due to a path like $\alpha \rightarrow \gamma \leftarrow \! \! * \beta$. Then the path $\theta * \! \! \rightarrow \cdots \leftrightarrow \gamma \leftarrow \! \! * \beta$, we have $\theta \in \MB(\beta)$. Therefore $p \subseteq \MB(\beta) \subseteq\MB(\varset{U}')$, and hence $\gP_\varset{V}$ is not maximally informative, contradiction.

\end{proof}

\begin{definition}[Higher order of Markov Blanket] For a integer $p \in \N _+$ and a set of random variables $\varset{Z}$, $\MB(X)$ is the Markov Blanket of $X$ relative to $\varset{Z}$. If $p=1$, then $\MB^{(p)}(X):=\MB(X)$ as usual; if $p>1$, then $\MB^{(p)}(X):=\bigcup_{X' \in \MB^{(p-1)}(X)} \MB(X')$.
\end{definition}

\begin{proof}[Proof for Proposition~\ref{prop:amenable-preservation}]
Let $\vu$ be a path from $T$ to $Y$ in $\gP_\varset{E}$. We have $\vu \subseteq \MB^{(2)}(\vu) \subseteq  \varset{V}$ by the Loop at step 3 in Algorithm~\ref{alg:coat-t-y}. Therefore, its adjacency and arrow heads are preserved in $\gP_\varset{V}$ according to Proposition~\ref{prop:arrowhead-preserve}. Thus $\vu$ is a possibly directed path from $T$ to $Y$ in $\gP_\varset{V}$ if and only if it is possibly directed in $\gP_\varset{E}$.

Since $\gP_\varset{E}$ is amenable relative to $(T, Y)$, each possibly directed path $\vu$ from $T$ to $Y$ starts with a visible directed edge $T \rightarrow V_1$. By the definition of visible edge, we need to consider two cases. In the first case, there is an unshielded triple $V_0 \starrightarrow T \rightarrow V_1$. We have $V_0 \in \MB(T)$, thus its adjacency and arrow heads would preserved by Proposition~\ref{prop:arrowhead-preserve}, and the arrow tail would be preserved by applying the FCI orientation rule $\gR 1$ of \citet{ZHANG20081873}. In the second case, there is a node $V_0$ which is not adjacent to $V_1$, and a path $V_0 \starrightarrow \cdots \leftrightarrow \alpha_k \leftrightarrow T \rightarrow V_1 $ that each non-endpoint node ($\alpha_1 \cdots \alpha_k$) is a collider and a parent of $V_1$. This is a discriminating path, and we have shown this path is inside $\MB(T)$ in the proof for Proposition~\ref{prop:arrowhead-preserve}, and its arrow tails are preserved in $\gP_\varset{V}$ according to Corollary~\ref{cor:arrow-tails-in-discriminating-path}.
\end{proof}

\begin{proof}[Proof for Proposition~\ref{prop:arrow-tails-in-path-T-Y}]
Since $\alpha$ and $\gamma$ are two adjacent nodes in a possibly directed path from $T$ to $Y$, we have $\{\alpha, \gamma\} \subseteq \MB^{(2)}(\{\alpha, \gamma\}) \subseteq  \varset{V}$ by the Loop 1 in Algorithm~\ref{alg:coat-t-y}.

We construct $\widehat{\gP}_\varset{E}$ from $\gP_\varset{V}$ as follows: (1) we extend node set from $\varset{V}$ to $\varset{E}$; (2) Let the edges between $\varset{U}'$ and $\varset{E} \setminus \varset{U}'$ in $\widehat{\gP}_\varset{E}$ be the same as $\gP_\varset{E}$. By this construction, the skeleton in $\widehat{\gP}_\varset{E}$ is same as $\gP_\varset{E}$. By Proposition~\ref{prop:arrowhead-preserve}, the arrow heads are also preserved.

If the edge $\alpha \rightarrow \gamma$ occurs in  $\gP_\varset{E}$ but not in $\gP_\varset{V}$, i.e., $\alpha \circrightarrow \gamma$ in $\gP_\varset{V}$, then some of the FCI orientation rules $\gR 8$--$\gR 10$ from \citet{ZHANG20081873} should be applicable in $\widehat{\gP}_\varset{E}$ on this edge, and we will show that this is impossible.

If $\gR 8$ is applicable to $\alpha \circrightarrow \gamma$ in $\gP_\varset{E}$, then we have $\alpha \rightarrow \beta \rightarrow \gamma$, since there is an uncertain mark $\circ$ on $\alpha$, it has been processed by the Loop 2 in Algorithm~\ref{alg:coat-t-y}. Therefore, $\gamma \in \MB(\alpha) \subseteq \MB^{(2)}(\{\alpha, \gamma\})  \subseteq \varset{V}$, thus $\gR 8$ could have been applied to $\gP_\varset{V}$, and violates the fact that $\gP_\varset{V}$ is maximally informative. 

If $\gR 9$ is applicable to $\alpha \circrightarrow \gamma$ in $\gP_\varset{E}$, then there is a uncovered possibly directed path $\vp$ from $\alpha$ to $\gamma$. Since $\vp$ is a part of possibly directed path from $T$ to $Y$, it holds that $\vp \subseteq \MB^{(2)}(\vp) \subseteq \varset{V}$ by the Loop 1 in Algorithm~\ref{alg:coat-t-y}, thus $\gR 9$ could have been applied to $\gP_\varset{V}$, and thus leads to a contradiction.

If $\gR 10$ is applicable to $\alpha \circrightarrow \gamma$ in $\gP_\varset{E}$, then we have $\beta \rightarrow \gamma \leftarrow \theta$, $\vp_1$ and $\vp_2$ are uncovered possibly directed paths from $\alpha$ to $\beta$, and from $\alpha$ to $\theta$ respectively. Note that $\beta, \theta \in \MB(\alpha)$, and $\vp_1$ and $\vp_2$ are in part of possibly directed path from $T$ to $Y$. Therefore, with same argument, this would lead to a contradiction.
\end{proof}

\begin{proof}[Proof for Proposition~\ref{prop:Adjustment-Sets-Preservation}]
    By the construction of Algorithm~\ref{alg:coat-t-y}, we have a set $\varset{U}'$ that $\MB(\varset{U}') \subseteq \varset{V}$ and $\MB\left(\left\{T,Y\right\}\right) \subseteq \varset{V}$, and for any node $V$ on a path between $T$ and $Y$, $V \in \varset{U}'$.
    To show $\varset{Z}$ is also an adjustment set relative to $(T, Y)$ in $\gP_\varset{V}$, we need to check whether it satisfies the three conditions in Definition~\ref{def:Generalized-adjustment-criterion}. 

    (1) {Amenability.} Since $\gP_\varset{E}$ has an adjustment set, then it has to be amenable, therefore, according to Proposition~\ref{prop:amenable-preservation}, $\gP_\varset{V}$ is amenable.

    (2) {Forbidden set.} %
    We need to show that: if a node $\alpha \in \varset{U}'$ satisfies $\alpha \notin \text{Forb}(T,Y,\gP_\varset{V})$, then $\alpha \notin \text{Forb}(T,Y,\gP_\varset{E})$. By definition, if $\alpha \in \text{Forb}(T,Y,\gP_\varset{E})$, then there is a possible directed path from $T$ to $\alpha$. Thus, by Proposition~\ref{prop:arrowhead-preserve}, there is no such a path in $\gP_\varset{V}$ containing an arrowhead towards $T$, thus we have $\alpha \in \text{Forb}(T,Y,\gP_\varset{V})$. On the other hand, if $\alpha \in \text{Forb}(T,Y,\gP_\varset{V})$ there is a possibly directed path $\vu$ from $\beta \in \text{S}(T,Y, \gP_\varset{V})$ to $\alpha$. By the construction of Algorithm~\ref{alg:coat-t-y},  $\beta \in \text{S}(T,Y, \gP_\varset{E})$ and $\vu$ is also a possibly directed path in $\gP_\varset{E}$, thus $\alpha \in \text{Forb}(T,Y,\gP_\varset{E})$.

    (3) {Blocking.} By Proposition~\ref{prop:arrow-tails-in-path-T-Y} and Corollary~\ref{cor:definite-status-path}, each definite status paths between $T$ and $Y$ are preserved. Thus if a node can m-separate some such paths in $\gP_\varset{E}$, it m-separates them in $\gP_\varset{V}$.
\end{proof}

\clearpage
\section{More Details about Experiments}
\label{appendix:More Details about Experiments}
\subsection{More Details on Constructing \apple}
\label{appdx:apple_construction}
In the \apple benchmark, we consider the target variable as a rating score of the apple by several gastronomes.
Each apple has its own attributes, including size, smell, and taste (or sweetness). Each gastronome has a unique preference for some attributes of the apple. They will give a rating as well as write a review according to the matchness of the apple with respect to their preference. We generate the review using \gptf by fetching \gptf the preferences and the apple attributes.

The prompts for generating the unstructured inputs are given in Figure~\ref{fig:apple_gen_prompt}.

\begin{figure}[ht]
    \centering
    \includegraphics[width=0.6\textwidth]{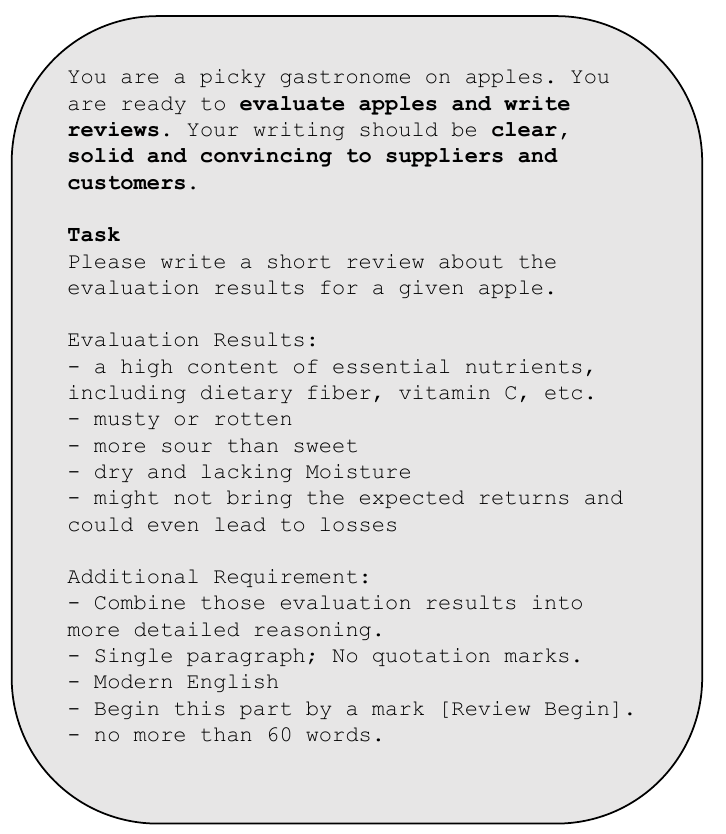}
    \caption{Illustration of prompts for generating \apple.}
    \label{fig:apple_gen_prompt}
\end{figure}

Examples of \apple are given in Figure~\ref{fig:apple_example_appdx}.

\begin{figure}[ht]
    \centering
    \includegraphics[width=1\textwidth]{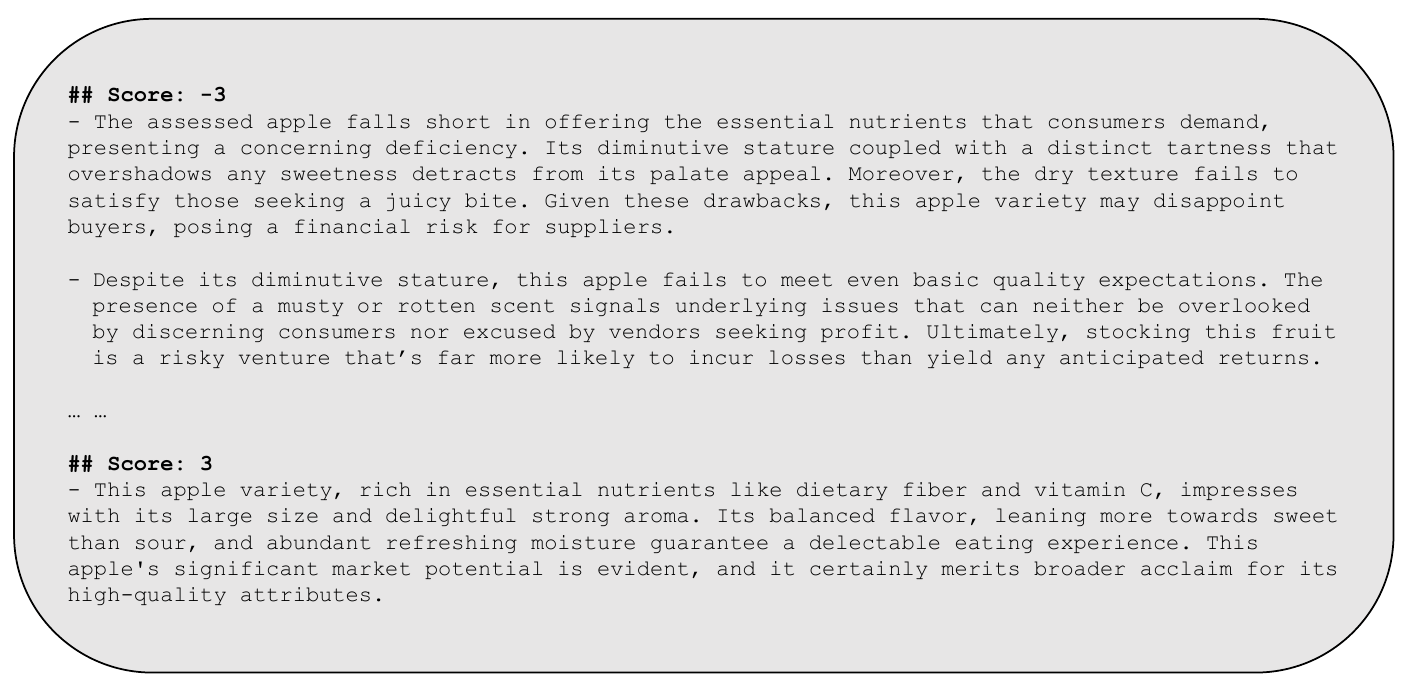}
    \caption{Illustration of examples in \apple.}
    \label{fig:apple_example_appdx}
\end{figure}

\subsection{More Details on Prompts for \apple}
\label{appdx:apple_prompts}

The prompts for factor proposal are given in Figure~\ref{fig:prompt_example_appdx}.

\begin{figure}[t]
    \centering
    \includegraphics[width=0.6\columnwidth]{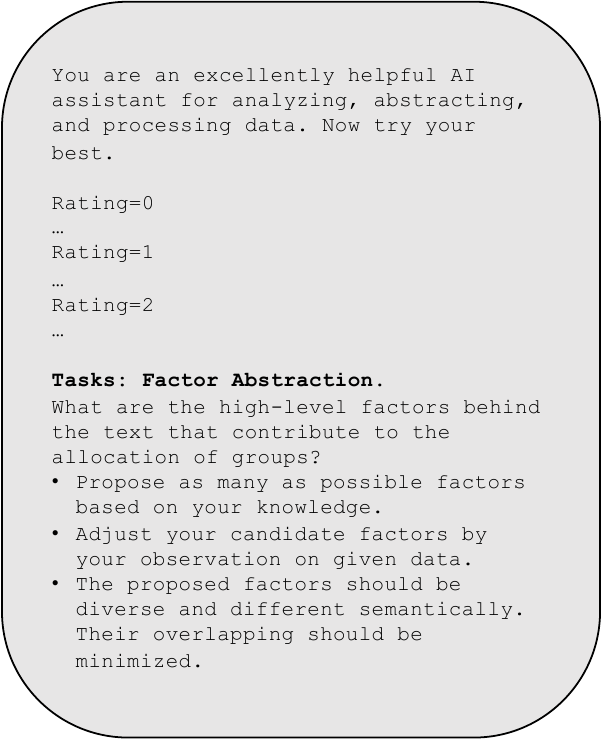}
    \caption{Illustration of the prompt for factor proposal.}
    \label{fig:prompt_example_appdx}
\end{figure}

The prompt for factor annotation is given in Figure~\ref{fig:apple_anno_p1}.%

\begin{figure}[ht]
    \centering
    \includegraphics[width=0.6\columnwidth]{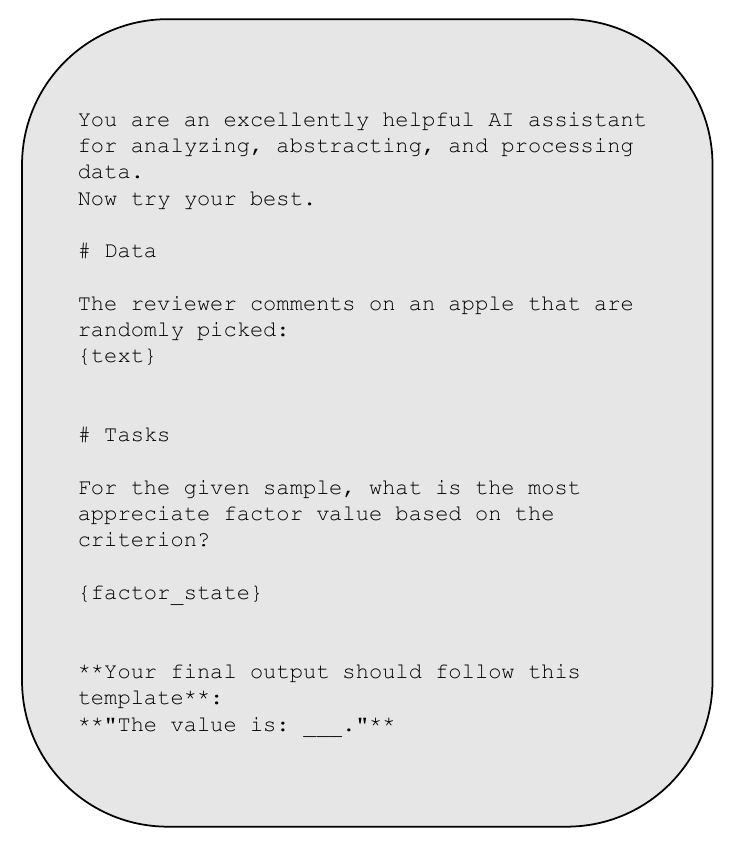}
    \caption{Illustration of the prompt for factor annotation.}
    \label{fig:apple_anno_p1}
\end{figure}

The prompts for constructing feedback are given in Figure~\ref{fig:feedback_prompt}.

\begin{figure}[ht]
    \centering
    \includegraphics[width=1\columnwidth]{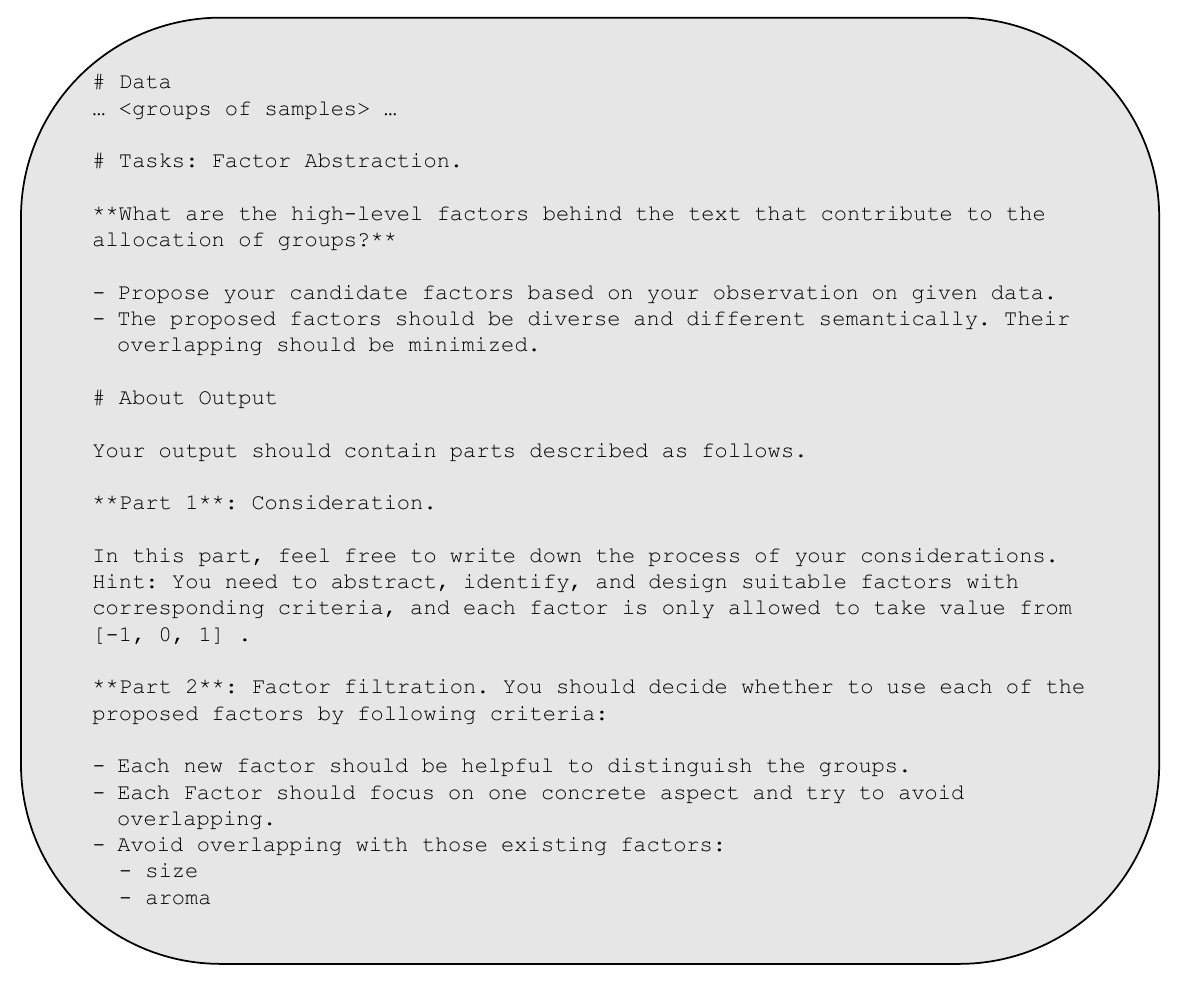}
    \caption{Illustration of the prompt for feedback.}
    \label{fig:feedback_prompt}
\end{figure}

\subsection{More Details of Results on \apple}
\label{appdx:apple_results}
The detailed causal graph results are given from Figure~\ref{fig:causal_graph_apple_appdx} to Figure~\ref{fig:causal_graph_apple_mistral_appdx}.
Independent tests about annotation on the Apple Gastronome benchmark are shown in Table~\ref{tab:ind_test_annotation}.
Full Results of Causal Metrics each Round of each LLM on Apple Gastronome Benchmark are shown in Table~\ref{tab:full_results_causal_metrics_on_apple}.

\begin{table}[t]
    \centering
    {\small
    \resizebox{\textwidth}{!}{
\begin{tabular}{lcccc}
\toprule
LLMs           & \ours Round & Perception Score   & Capacity Score & $I(y;x|h_S)$ \\ \midrule
GPT-4o         & iter 1     & 0.82\std{0.02} & 0.24\std{0.01} & 0.28\std{0.02} \\
               & iter 2     & 0.44\std{0.08} & 0.26\std{0.02} & 0.21\std{0.03} \\
               & iter 3     & 0.38\std{0.12} & 0.19\std{0.09} & 0.24\std{0.05} \\ \midrule
GPT-4          & iter 1     & 0.78\std{0.02} & 0.18\std{0.02} & 0.45\std{0.04} \\
               & iter 2     & 0.61\std{0.08} & 0.26\std{0.05} & 0.28\std{0.05} \\
               & iter 3     & 0.39\std{0.28} & 0.22\std{0.00} & 0.29\std{0.02} \\ \midrule
GPT-3.5        & iter 1     & 0.81\std{0.02} & 0.21\std{0.01} & 0.42\std{0.05} \\
               & iter 2     & 0.52\std{0.11} & 0.23\std{0.02} & 0.31\std{0.04} \\
               & iter 3     & 0.38\std{0.10} & 0.21\std{0.03} & 0.33\std{0.11} \\ \midrule
Mistral-Large  & iter 1     & 0.78\std{0.02} & 0.20\std{0.01} & 0.41\std{0.04} \\
               & iter 2     & 0.61\std{0.08} & 0.20\std{0.03} & 0.29\std{0.02} \\
               & iter 3     & 0.33\std{0.33} & 0.32\std{0.00} & 0.25\std{0.05} \\ \midrule
Mistral-Medium & iter 1     & 0.78\std{0.02} & 0.20\std{0.02} & 0.39\std{0.05} \\
               & iter 2     & 0.39\std{0.15} & 0.19\std{0.10} & 0.32\std{0.04} \\
               & iter 3     & 0.36\std{0.10} & 0.37\std{0.14} & 0.25\std{0.08} \\ \midrule
\llama-3-70b    & iter 1     & 0.75\std{0.00} & 0.19\std{0.02} & 0.47\std{0.08} \\
               & iter 2     & 0.36\std{0.31} & 0.19\std{0.00} & 0.38\std{0.02} \\
               & iter 3     & 0.27\std{0.38} & 0.18\std{0.00} & 0.36\std{0.03} \\ \midrule
\llama-2-70b    & iter 1     & 0.77\std{0.02} & 0.21\std{0.05} & 0.43\std{0.13} \\
               & iter 2     & 0.08\std{0.12} & 0.35\std{0.00} & 0.35\std{0.03} \\
               & iter 3     & 0.28\std{0.04} & 0.14\std{0.07} & 0.33\std{0.03} \\ \midrule
Qwen-1.5-110B  & iter 1     & 0.77\std{0.02} & 0.17\std{0.01} & 0.45\std{0.09} \\
               & iter 2     & 0.56\std{0.16} & 0.21\std{0.12} & 0.32\std{0.08} \\
               & iter 3     & 0.28\std{0.21} & 0.18\std{0.02} & 0.33\std{0.11} \\ \midrule
DeepSeek-V2    & iter 1     & 0.89\std{0.08} & 0.17\std{0.02} & 0.42\std{0.05} \\
               & iter 2     & 0.62\std{0.10} & 0.20\std{0.02} & 0.34\std{0.02} \\
               & iter 3     & 0.67\std{0.24} & 0.40\std{0.12} & 0.29\std{0.06} \\ \midrule
Claude-3-Opus  & iter 1     & 0.77\std{0.02} & 0.18\std{0.00} & 0.43\std{0.06} \\
               & iter 2     & 0.56\std{0.08} & 0.39\std{0.11} & 0.24\std{0.06} \\
               & iter 3     & 0.75\std{0.25} & 0.15\std{0.00} & 0.22\std{0.05} \\ \bottomrule
\end{tabular}}
    }
    \caption{Full Results of Causal Metrics in Each Round for Each LLM on the Apple Gastronome Benchmark.}
    \label{tab:full_results_causal_metrics_on_apple}
\end{table}

\subsection{Implementation of the FCI Algorithm}
We use a third-party open-sourced Python library to perform the FCI algorithm: \url{https://causal-learn.readthedocs.io/en/latest/}.

We set \codeblock{$\alpha=0.05$}, and \codeblock{independence\_test\_method="fisherz"} throughout all experiments. Other parameters are kept as the default.

\begin{figure}[th]
    \centering
    \subfigure[Ground truth]{
        \includegraphics[height=0.2\textheight, trim=40 40 40 35, clip]{figures/apple/Apple_GroundTruth.pdf}
    }
    \subfigure[FCI results]{
        \includegraphics[height=0.2\textheight, trim=40 40 40 35, clip]{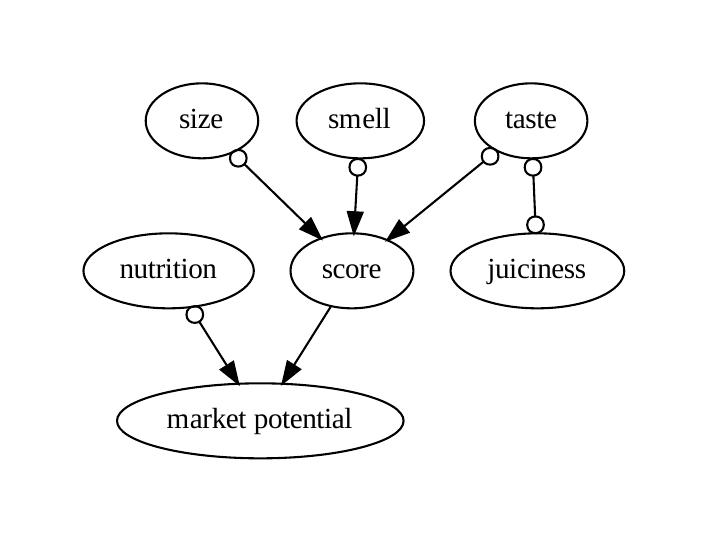}
    }
    \caption{Ground truth and faithful (via FCI algorithm) causal graphs in \apple.}
    \label{fig:causal_graph_apple_appdx}
\end{figure}

\begin{figure}[th]
    \centering
    \subfigure[\gptf reasoning]{
        \includegraphics[height=0.3\textheight, trim=40 40 40 35, clip]{figures/apple/gpt-4_Apple_LLM_revised.pdf}
    }
    \subfigure[\gptf \ours]{
        \includegraphics[height=0.2\textheight, trim=40 40 40 35, clip]{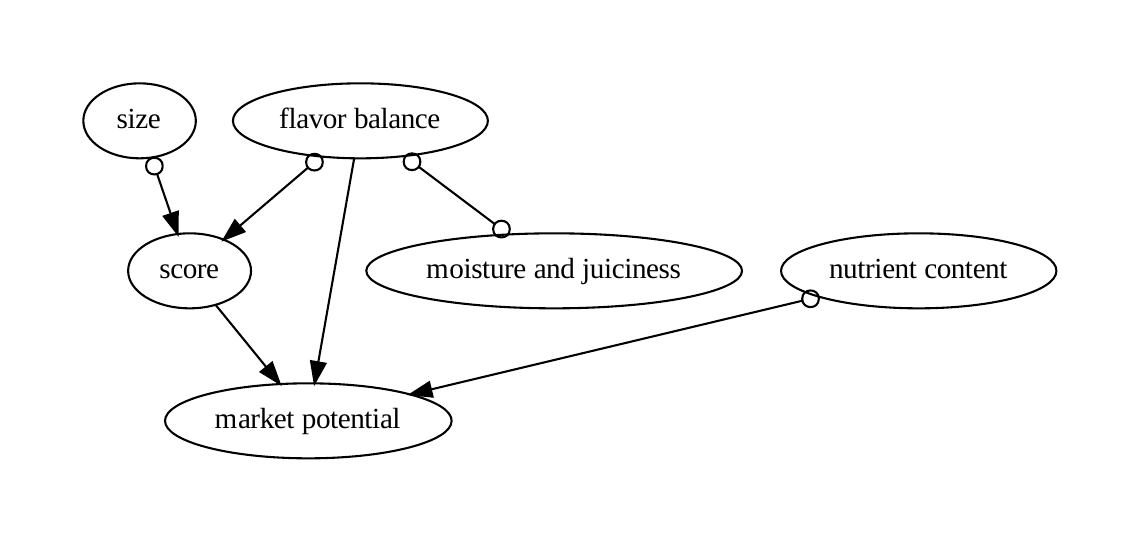}
    }
    \caption{Causal graphs with \gptf in \apple.}
    \label{fig:causal_graph_apple_gpt4_appdx}
\end{figure}

\begin{figure}[th]
    \centering
    \subfigure[\chatgpt reasoning]{
        \includegraphics[height=0.25\textheight, trim=40 40 40 35, clip]{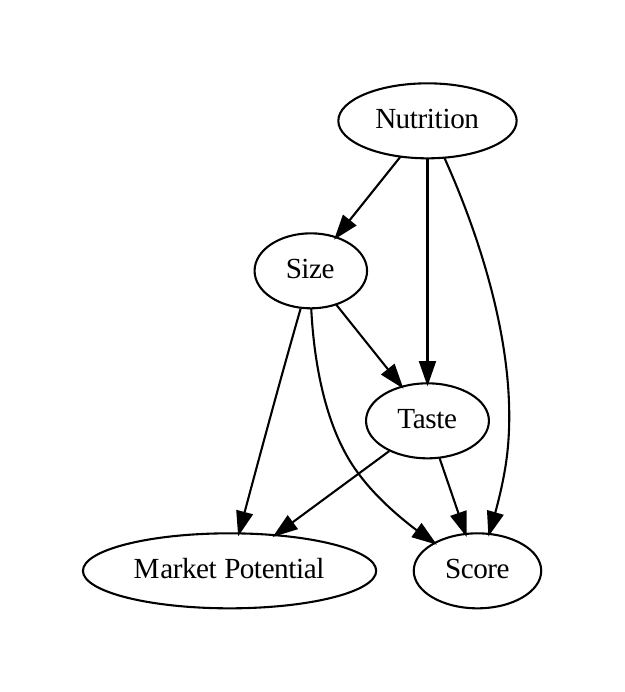}
    }
    \subfigure[\chatgpt \ours]{
        \includegraphics[height=0.2\textheight, trim=40 40 40 35, clip]{figures/apple/gpt-3.5_Apple.pdf}
    }
    \caption{Causal graphs with \chatgpt in \apple.}
    \label{fig:causal_graph_apple_gpt3_appdx}
\end{figure}

\begin{figure}[th]
    \centering
    \subfigure[\llama reasoning]{
        \includegraphics[height=0.25\textheight, trim=20 40 15 35, clip]{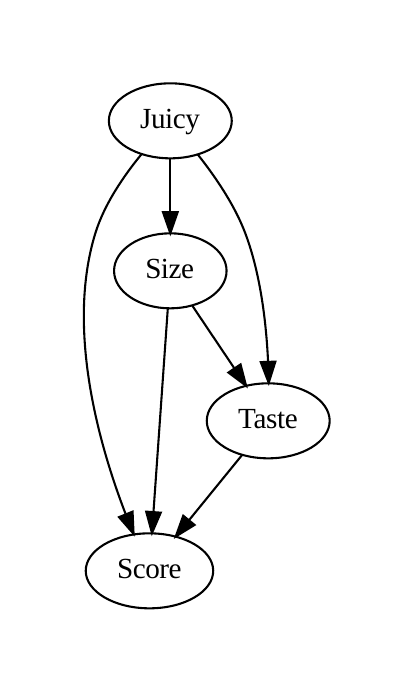}
    }
    \subfigure[\llama \ours]{
        \includegraphics[height=0.2\textheight, trim=40 40 35 35, clip]{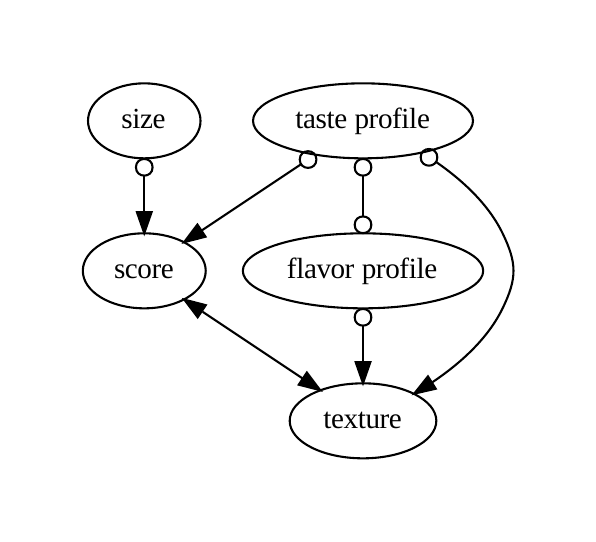}
    }
    \caption{Causal graphs with Llama-2 in \apple.}
    \label{fig:causal_graph_apple_llama_appdx}
\end{figure}

\begin{figure}[th]
    \centering
    \subfigure[\mistral reasoning]{
        \includegraphics[height=0.25\textheight, trim=40 40 40 35, clip]{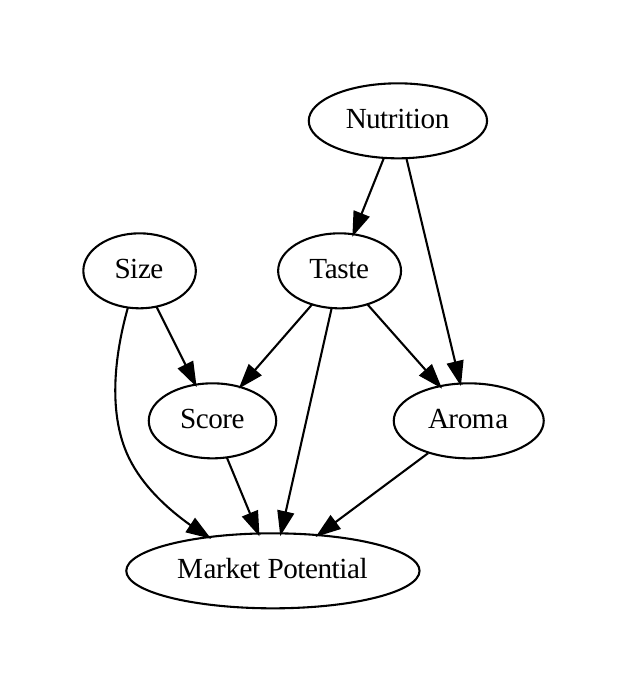}
    }
    \subfigure[\mistral \ours]{
        \includegraphics[height=0.25\textheight, trim=40 40 40 35, clip]{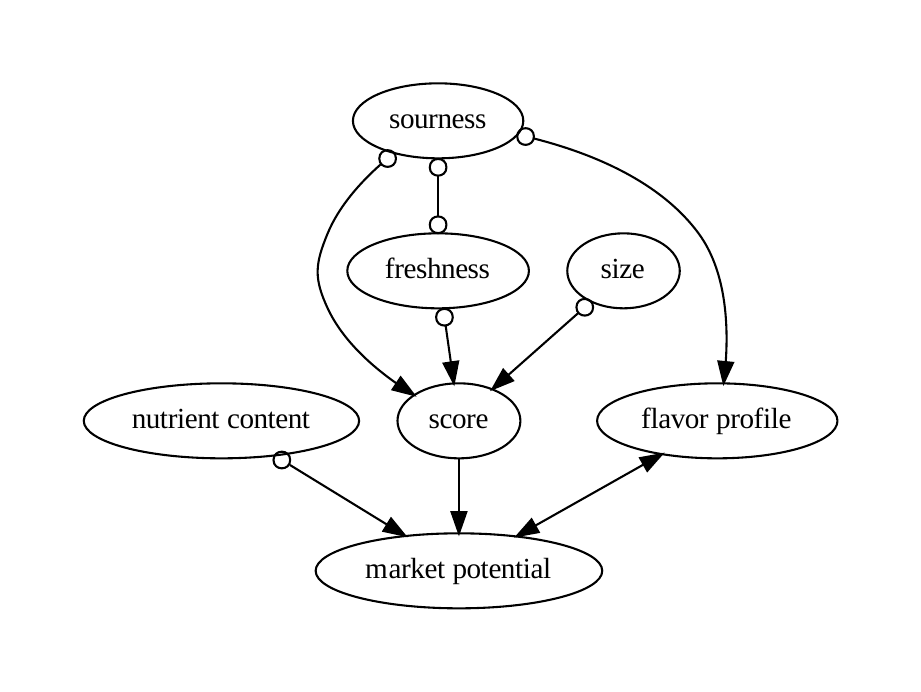}
    }
    \caption{Causal graphs with mistral Medium in \apple.}
    \label{fig:causal_graph_apple_mistral_appdx}
\end{figure}

\begin{figure}[th]
    \centering
    \subfigure[Claude-3-Opus reasoning]{
        \includegraphics[height=0.3\textheight, trim=40 40 40 35, clip]{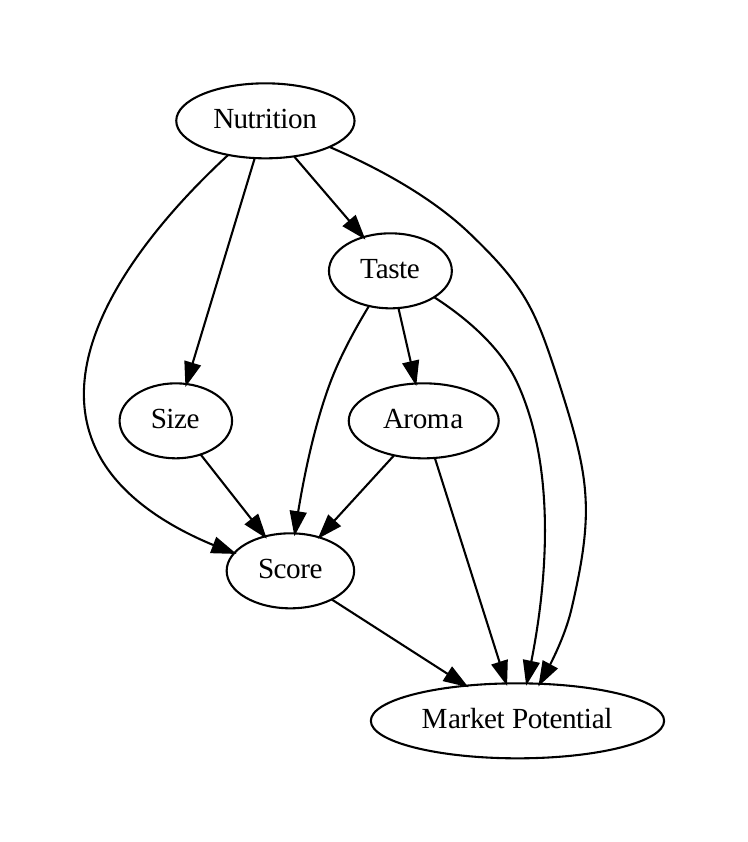}
    }
    \subfigure[Claude-3-Opus \ours]{
        \includegraphics[height=0.25\textheight, trim=40 40 40 35, clip]{figures/apple/claude_3_opus.pdf}
    }
    \caption{Causal graphs with Claude-3-Opus in \apple.}
    \label{fig:causal_graph_apple_Claude_3_Opus}
\end{figure}

\begin{figure}[th]
    \centering
    \subfigure[DeepSeek-V2 reasoning]{
        \includegraphics[height=0.2\textheight, trim=40 40 40 35, clip]{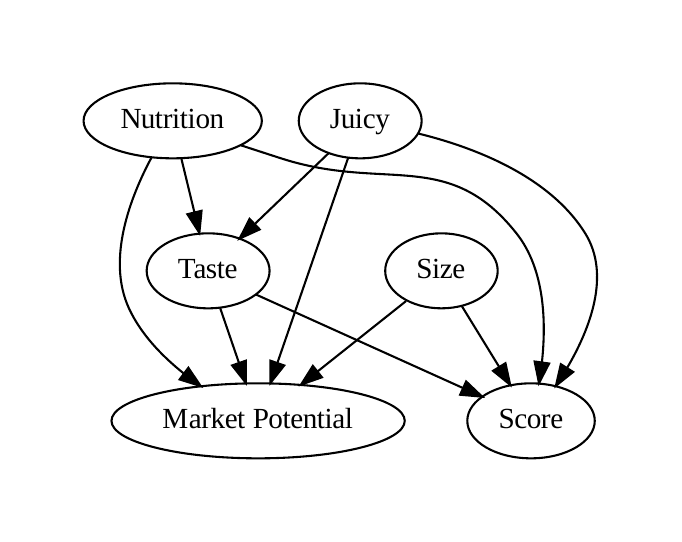}
    }
    \subfigure[DeepSeek-V2 \ours]{
        \includegraphics[height=0.25\textheight, trim=40 40 40 35, clip]{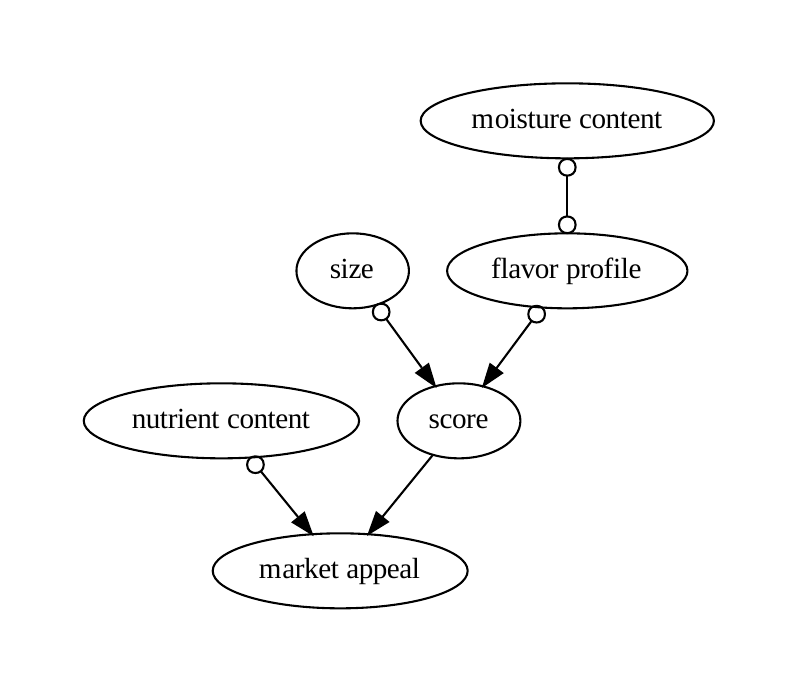}
    }
    \caption{Causal graphs with DeepSeek-V2 in \apple.}
    \label{fig:causal_graph_apple_DeepSeek-V2}
\end{figure}

\begin{figure}[th]
    \centering
    \subfigure[Llama-3-70b reasoning]{
        \includegraphics[height=0.2\textheight, trim=40 40 35 35, clip]{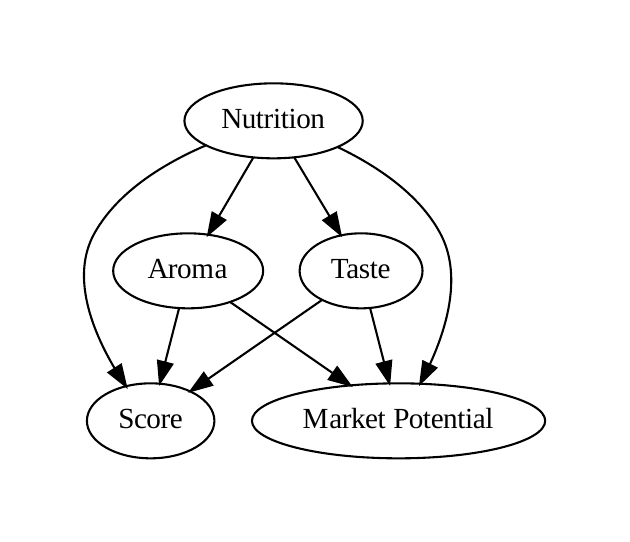}
    }
    \subfigure[Llama-3-70b \ours]{
        \includegraphics[height=0.25\textheight, trim=40 40 35 35, clip]{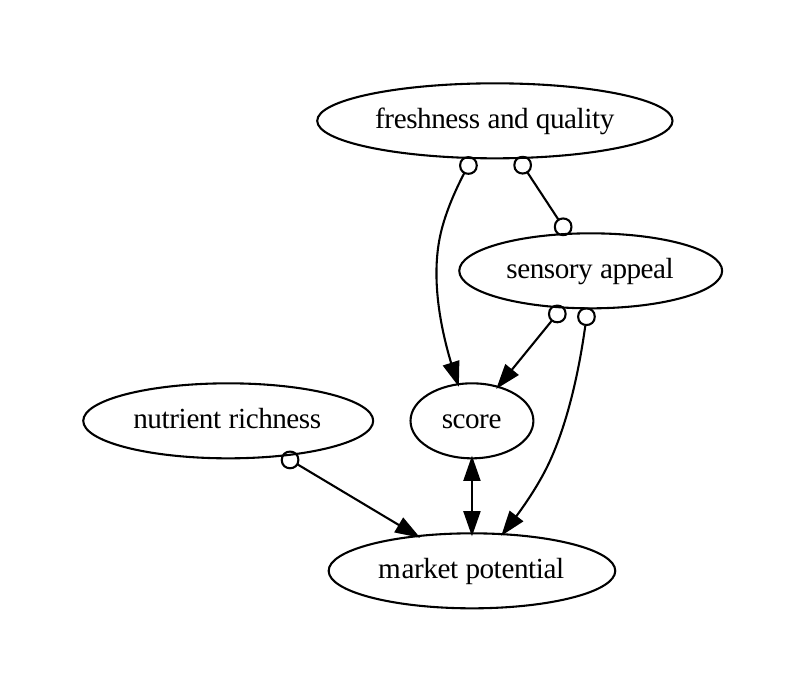}
    }
    \caption{Causal graphs with Llama-3-70b in \apple.}
    \label{fig:causal_graph_apple_llama_3_appdx}
\end{figure}

\begin{figure}[th]
    \centering
    \subfigure[mistral-Large reasoning]{
        \includegraphics[height=0.34\textheight, trim=40 40 35 35, clip]{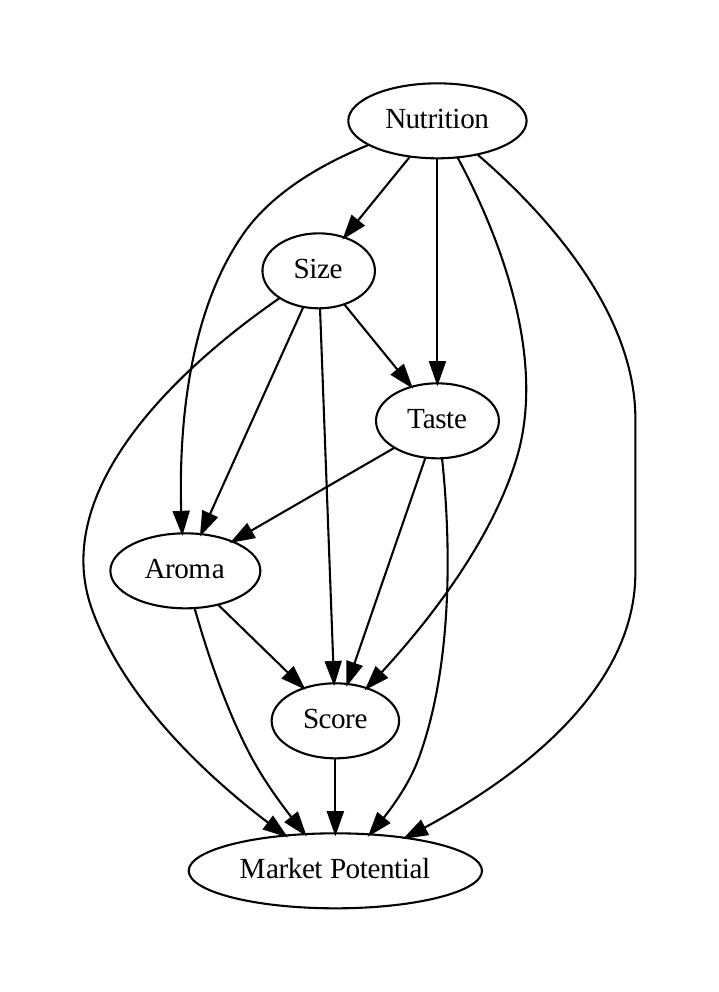}
    }
    \subfigure[mistral-Large \ours]{
        \includegraphics[height=0.2\textheight, trim=40 40 35 35, clip]{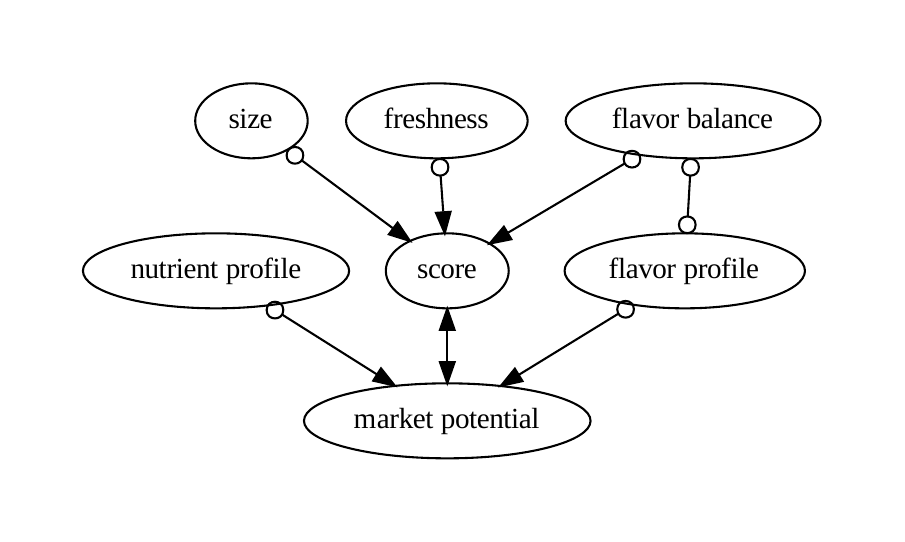}
    }
    \caption{Causal graphs with mistral-Large in \apple.}
    \label{fig:causal_graph_apple_mistral-Large_appdx}
\end{figure}

\begin{figure}[th]
    \centering
    \subfigure[qwen-1.5-110B reasoning]{
        \includegraphics[height=0.34\textheight, trim=40 40 35 35, clip]{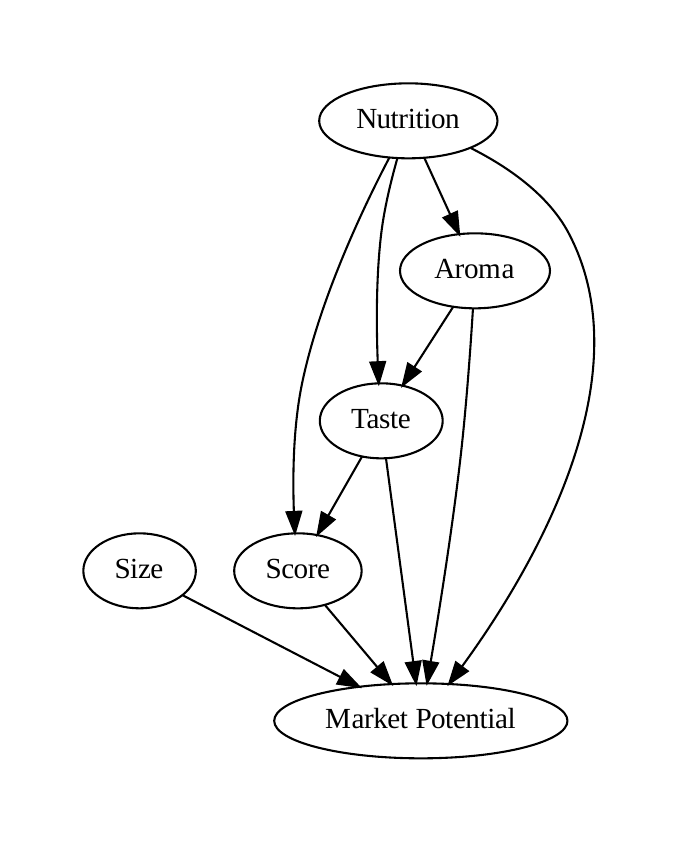}
    }
    \subfigure[qwen-1.5-110B \ours]{
        \includegraphics[height=0.25\textheight, trim=40 40 35 35, clip]{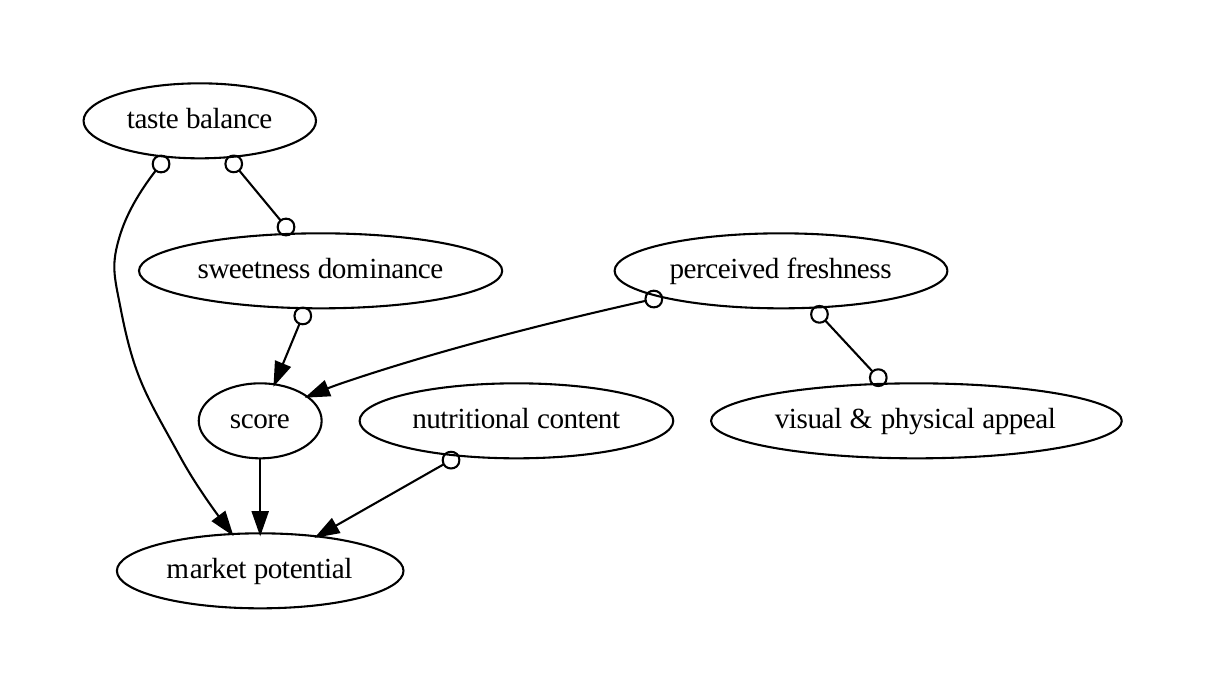}
    }
    \caption{Causal graphs with qwen-1.5-110B in \apple.}
    \label{fig:causal_graph_apple_qwen-1.5-110B_appdx}
\end{figure}

\begin{figure}[th]
    \centering
    \subfigure[GPT-4o reasoning]{
        \includegraphics[height=0.25\textheight, trim=40 40 35 35, clip]{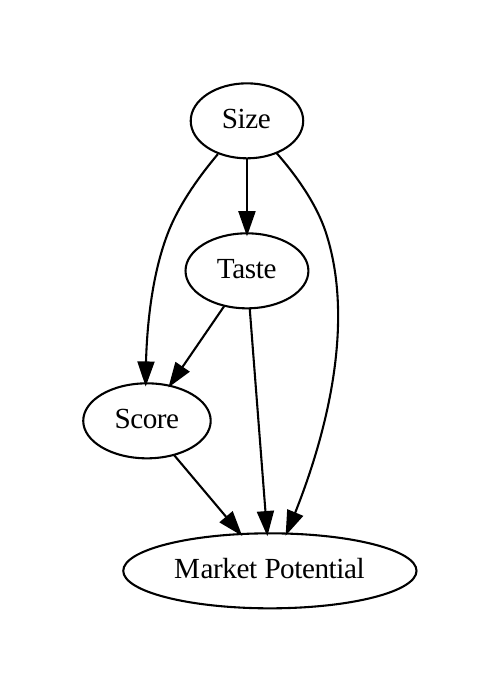}
    }
    \subfigure[GPT-4o \ours]{
        \includegraphics[height=0.25\textheight, trim=40 40 35 35, clip]{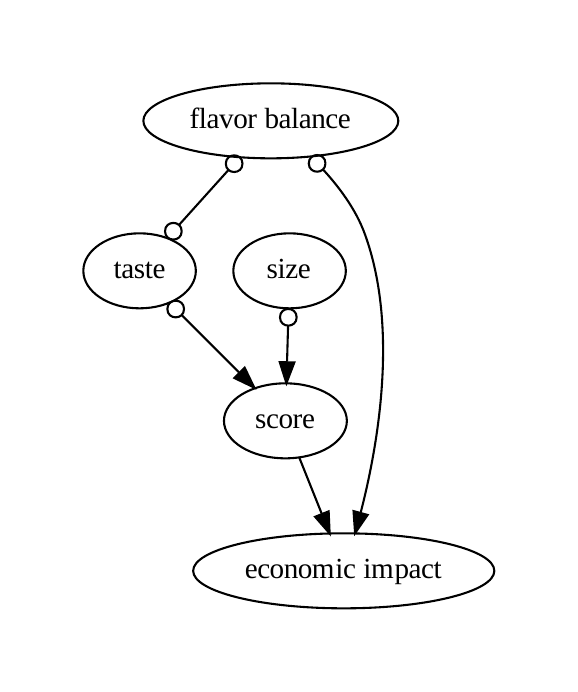}
    }
    \caption{Causal graphs with GPT-4o in \apple.}
    \label{fig:causal_graph_apple_GPT-4o_appdx}
\end{figure}

\clearpage
\newpage
\subsection{More Details on Constructing \pain}
\label{appdx:pain_construction}

\begin{figure}[ht]
    \centering
    \includegraphics[width=0.8\textwidth]{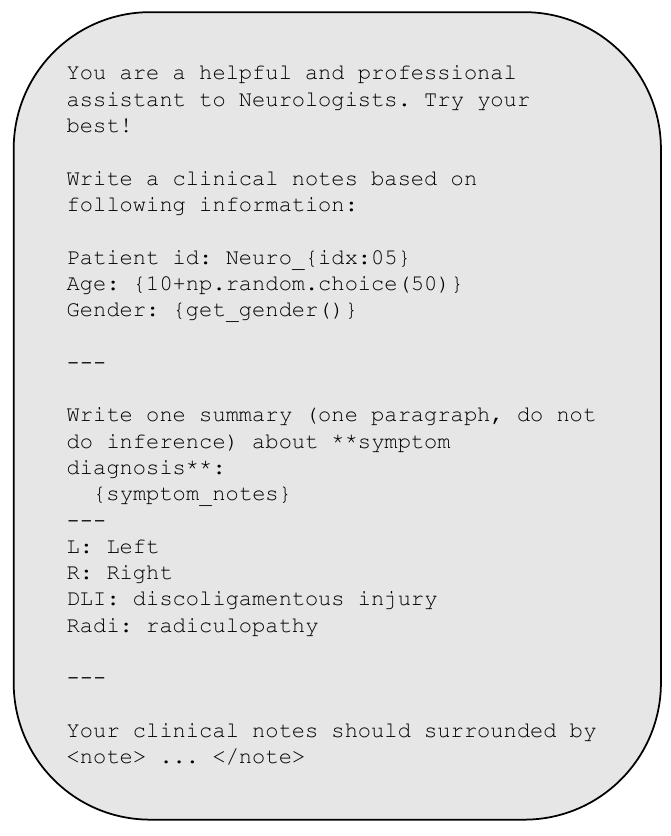}
    \caption{Illustration of prompts for generating \pain.}
    \label{fig:neuro_gen_prompt}
\end{figure}

In the \pain benchmark, we convert the data set into a clinical diagnosis task. In the original data set, there are three levels of causal variables, including the symptom-level, radiculopathy-level and the pathophysiology-level.
In experiments, we mainly consider the target variable of right shoulder impingement.
When generating the clinical diagnosis notes as $\vx$ using \gptf, we will avoid any mentioning of variables other than symptoms.

As we intend to leverage the \pain benchmark to simulate the real-life diagnosis, after the factor proposal stage, we directly incorporate external experts that measure the values of the candidate factors.
The prompts to generate the diagnosis records are given in Figure~\ref{fig:neuro_gen_prompt}.

Examples of \pain are given in Figure~\ref{fig:neuro_example_appdx}.

\begin{figure}[ht]
    \centering
    \includegraphics[width=\textwidth]{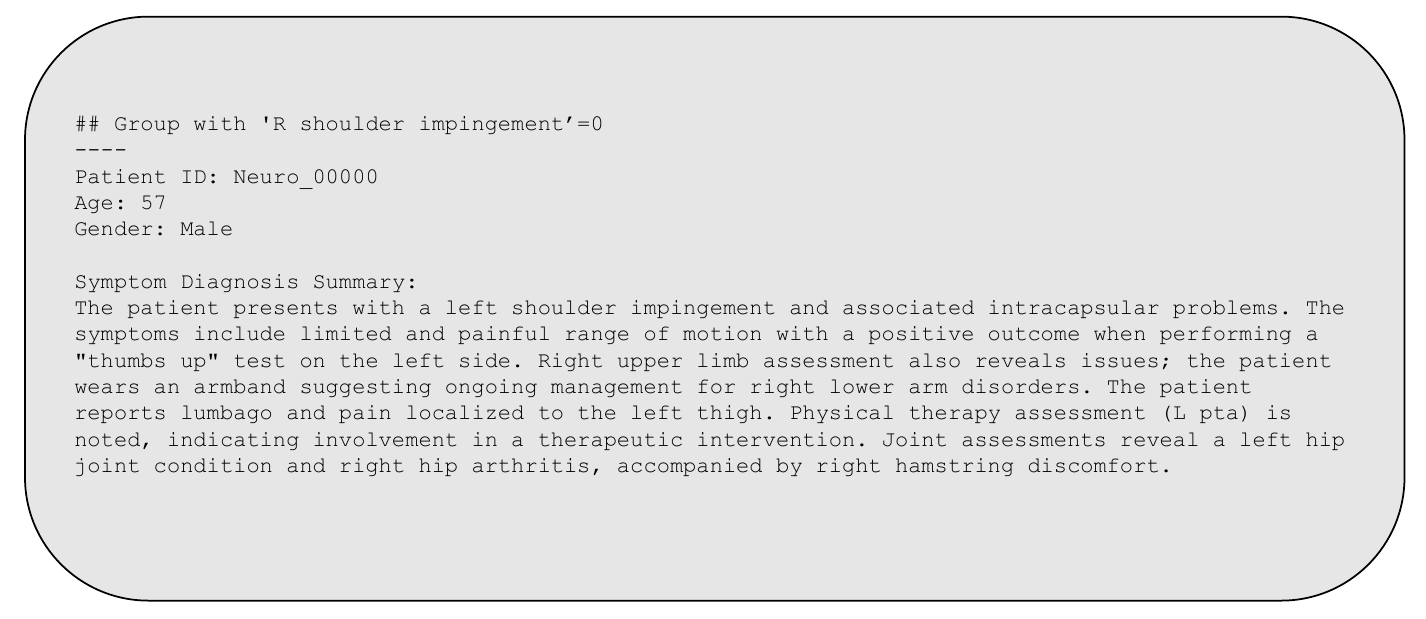}
    \caption{Illustration of examples in \pain.}
    \label{fig:neuro_example_appdx}
\end{figure}

\subsection{More Details of Results on \pain}
\label{appdx:pain_results}
The detailed causal graph results are given from Figure~\ref{fig:causal_graph_neuro_appdx} to Figure~\ref{fig:causal_graph_neuro_mistral_appdx}.

\begin{figure}[t]
    \centering
    \subfigure[Ground truth]{
        \includegraphics[height=0.237\textheight, trim=40 40 35 35, clip]{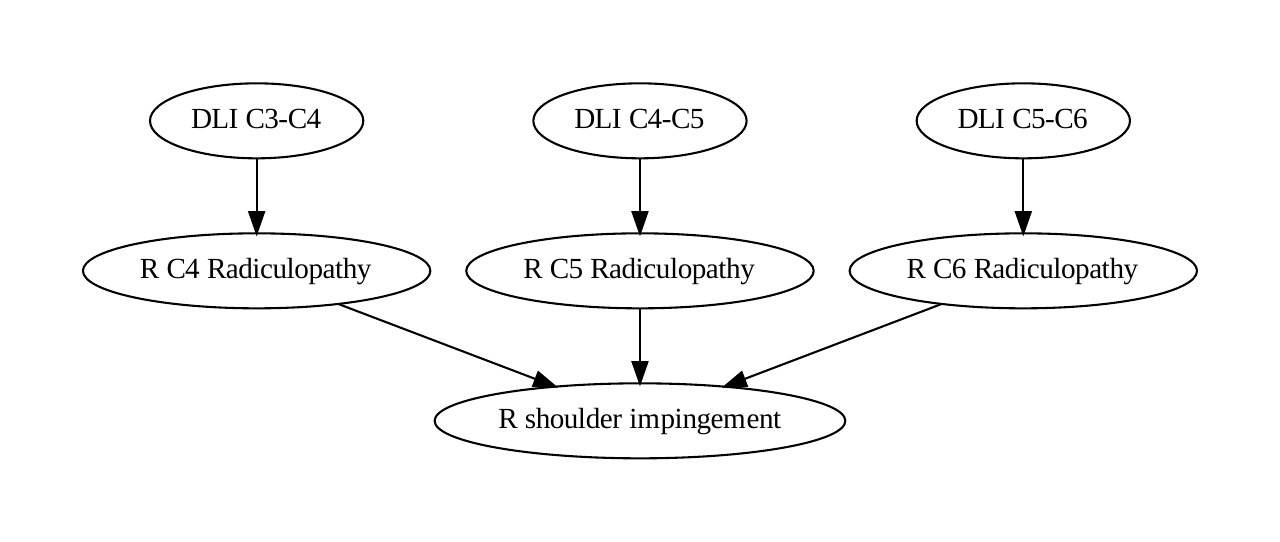}
    }
    \subfigure[FCI results]{
        \includegraphics[height=0.237\textheight, trim=40 40 35 35, clip]{figures/neuro/Neuro_FCI.pdf}
    }
    \caption{Ground truth and faithful (via FCI algorithm) causal graphs in \pain.}
    \label{fig:causal_graph_neuro_appdx}
\end{figure}

\begin{figure}[t]
    \centering
    \subfigure[\gptf reasoning]{
        \includegraphics[height=0.127\textheight, trim=40 40 35 35, clip]{figures/neuro/gpt-4_LLM_Neuro.pdf}
    }
    \subfigure[\gptf \ours]{
        \includegraphics[height=0.227\textheight, trim=40 40 35 35, clip]{figures/neuro/gpt-4-turbo_Neuro.pdf}
    }
    \caption{Causal graphs with \gptf in \pain.}
    \label{fig:causal_graph_pain_gpt4_appdx}
\end{figure}

\begin{figure}[t]
    \centering
    \subfigure[\chatgpt reasoning]{
        \includegraphics[height=0.674\textheight, trim=40 40 35 35, clip]{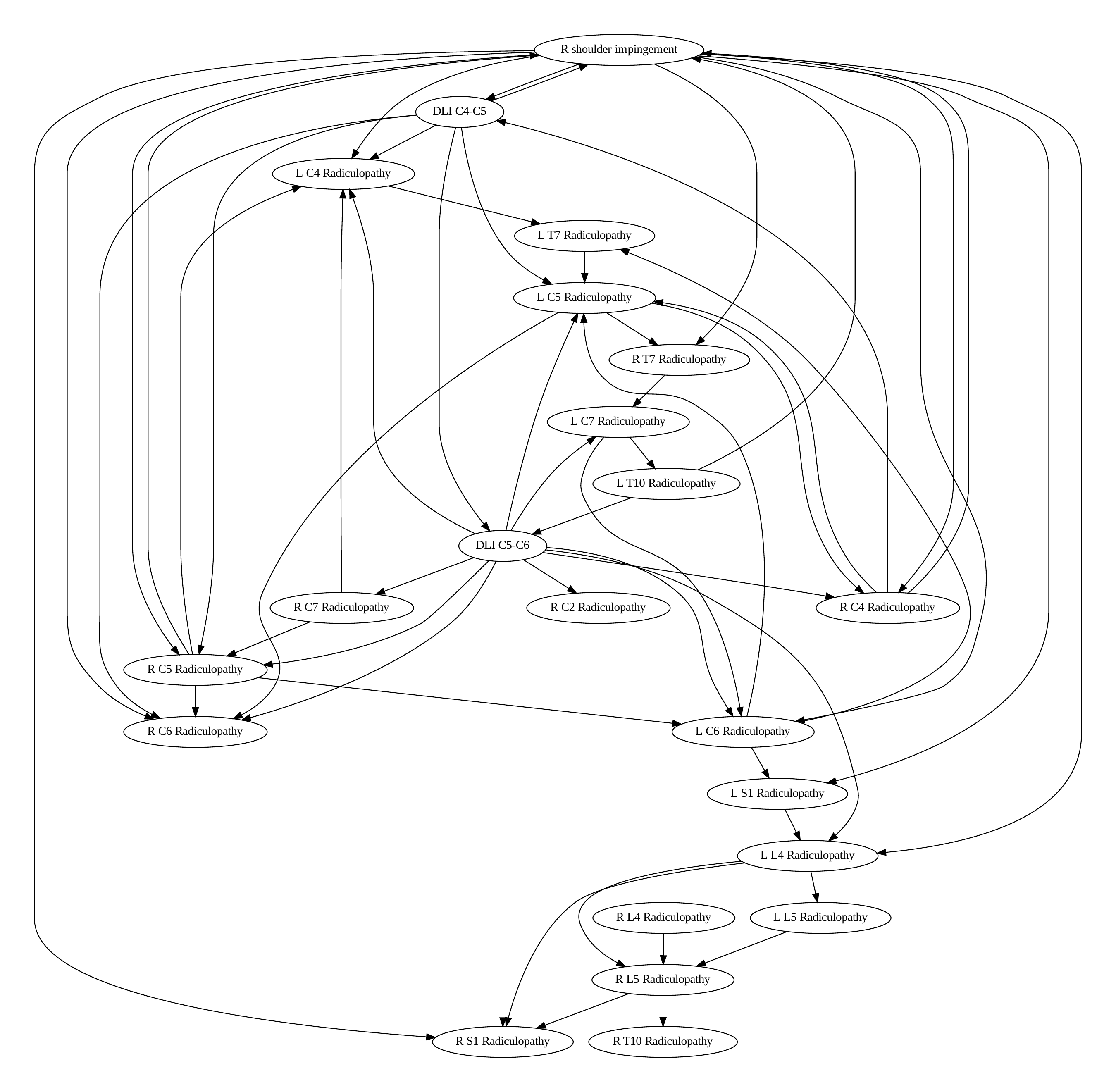}
    }
    \subfigure[\chatgpt \ours]{
        \includegraphics[height=0.192\textheight, trim=40 40 35 35, clip]{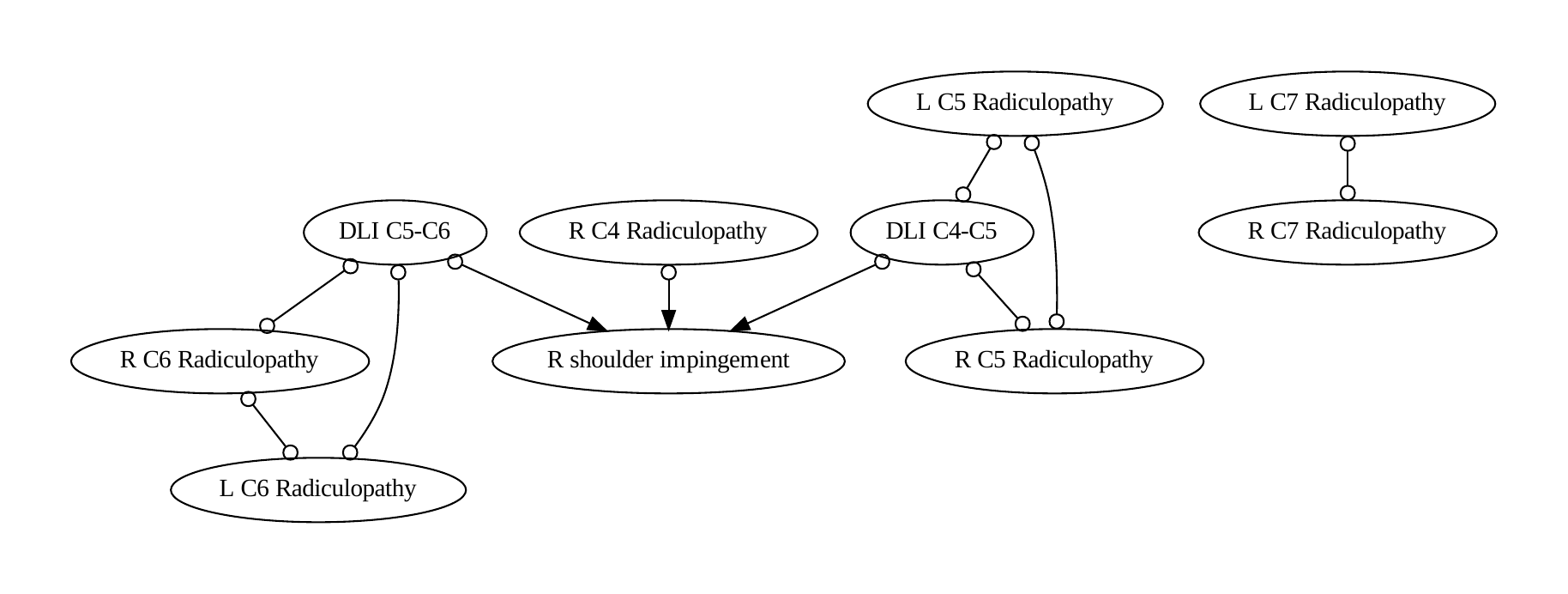}
    }
    \caption{Causal graphs with \chatgpt in \pain. }
    \label{fig:causal_graph_neuro_gpt3_appdx}
\end{figure}

\begin{figure}[t]
    \centering
    \subfigure[LLaMA-2-70b reasoning]{
        \includegraphics[height=0.198\textheight, trim=40 40 35 35, clip]{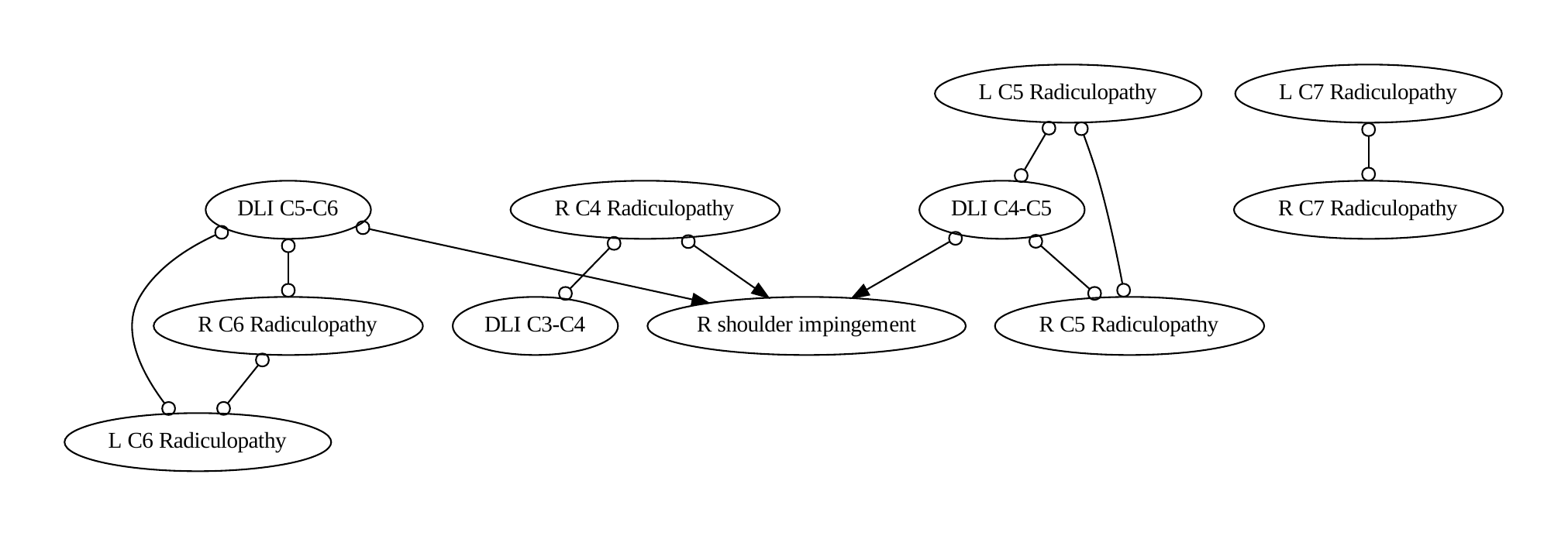}
    }
    \subfigure[LLaMA-2-70b \ours]{
        \includegraphics[height=0.198\textheight, trim=40 40 35 35, clip]{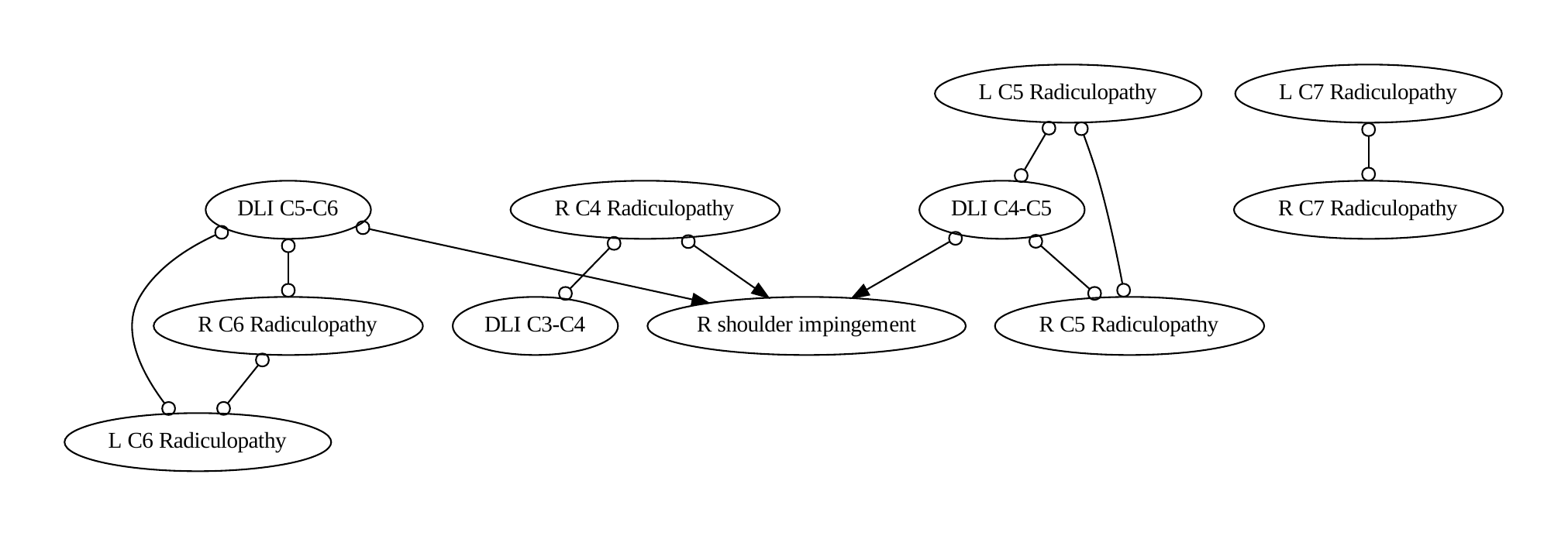}
    }
    \caption{Causal graphs with LLaMA-2-70b in \pain.}
    \label{fig:causal_graph_neuro_llama_appdx}
\end{figure}

\begin{figure}[t]
    \centering
    \subfigure[Mistral-Medium reasoning]{
        \includegraphics[height=0.133\textheight, trim=40 40 35 35, clip]{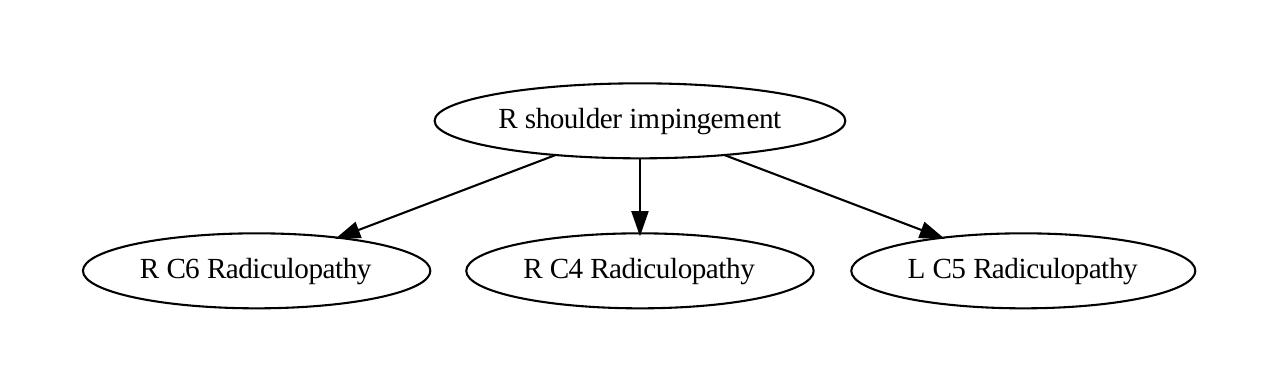}
    } 
    
    \subfigure[Mistral-Medium \ours]{
        \includegraphics[height=0.286\textheight, trim=40 40 35 35, clip]{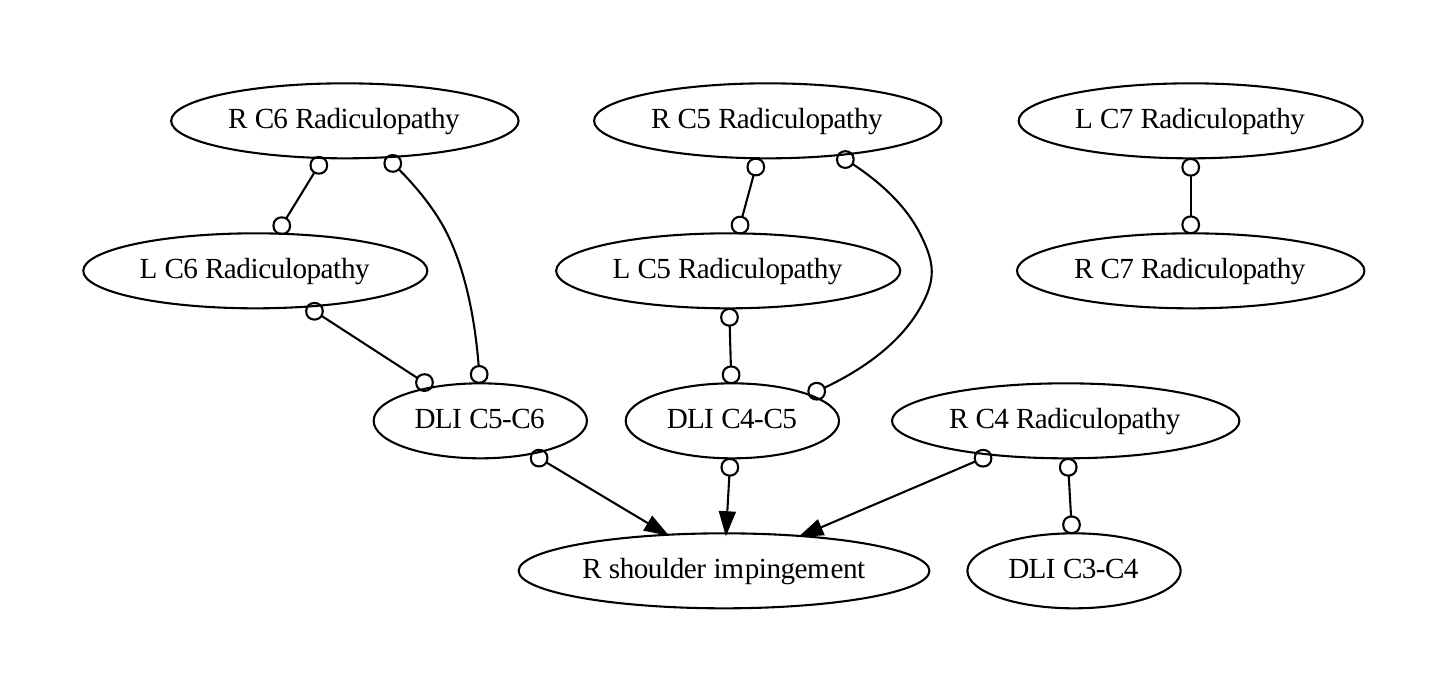}
    }
    \caption{Causal graphs with Mistral-Medium in \pain. }
    \label{fig:causal_graph_neuro_mistral_appdx}
\end{figure}

\subsection{Discussion on the Time Complexity}

Assume there are $m$ samples with $n$ possible factors.

\runin{For Factor Proposal:} META only needs to interact once with LLM, so it is $O(1)$.
DATA runs one single COAT round, so it is $O(1)$.
COAT interacts with LLMs multiple rounds. In each round, at least one new factor should be proposed; otherwise, the loop will stop. So it is $O(n)$.

\runin{For Factor Annotation:} META: Not applicable.
DATA: At most $n$ factors would be proposed in a single round. And each of them needs $m$ times annotations by LLMs for all samples. So it is $O(nm)$.
COAT: At most $n$ factors would be proposed during all rounds. And each of them needs $m$ times annotations by LLMs for all samples. So it is $O(nm)$.

\runin{For Causal Discovery:} Pair-wise reasoning by LLMs: $O(n^2)$.
COAT: LLM is not involved. The computational cost of it depends on the numeric methods. The FCI algorithm used by COAT has a time complexity that grows exponentially with $n$. Learning the causal graph over a large number of nodes effectively is still an open problem in causal discovery literature.

\subsection{Resources}
\label{appdx:Compute Resources}
We utilized a system comprising two Intel Xeon E5-2630v4 processors with 2.2GHz, two NVIDIA Tesla P40 GPUs, and 256 GB of memory. For conversations with large language models (LLMs), we leveraged the poe.com platform, while annotations were facilitated using the OpenAI API and the Mistral API.

\clearpage

\section{\ours with Different Causal Discovery Algorithms}
\label{appdx:ligam_coat}

\begin{table}[tbp]
    \centering
    {\small
    \resizebox{0.8\textwidth}{!}{
        \begin{tabular}{ll|ccccc}
            \toprule
            \multirow{2}{*}{LLM}          & \multirow{2}{*}{Method} & \multicolumn{5}{c}{Factor Proposal}                         \\
                                          &                         & PA                                  & AN & OT & Acc  & F1   \\
            \midrule
            \multirow{5}{*}{GPT-4}        & Meta                    & 3                                   & 5  & 6  & 0.91 & 0.59 \\
                                          & Data                    & 2                                   & 2  & 0  & 0.95 & 0.50 \\
                                          & \ours                   & 3                                   & 6  & 3  & 0.96 & 0.80 \\
                                          & Data (LiNGAM)           & 3                                   & 3  & 0  & 0.96 & 0.67 \\
                                          & \ours (LiNGAM)          & 3                                   & 6  & 4  & 0.95 & 0.75 \\
            \midrule
            \multirow{5}{*}{GPT-3.5}     & Meta                    & 3                                   & 5  & 6  & 0.91 & 0.59 \\
                                          & Data                    & 3                                   & 5  & 4  & 0.94 & 0.67 \\
                                          & \ours                   & 3                                   & 5  & 2  & 0.96 & 0.77 \\
                                          & Data (LiNGAM)           & 2                                   & 3  & 4  & 0.91 & 0.46 \\
                                          & \ours (LiNGAM)          & 3                                   & 5  & 4  & 0.94 & 0.67 \\
            \midrule
            \multirow{5}{*}{\llama-70b}   & Meta                    & 2                                   & 4  & 5  & 0.91 & 0.53 \\
                                          & Data                    & 3                                   & 3  & 1  & 0.95 & 0.60 \\
                                          & \ours                   & 3                                   & 6  & 2  & 0.97 & 0.86 \\
                                          & Data (LiNGAM)           & 3                                   & 4  & 4  & 0.92 & 0.57 \\
                                          & \ours (LiNGAM)          & 3                                   & 6  & 4  & 0.95 & 0.75 \\
            \midrule
            \multirow{5}{*}{\mistral-Med} & Meta                    & 3                                   & 6  & 3  & 0.96 & 0.80 \\
                                          & Data                    & 3                                   & 3  & 2  & 0.94 & 0.66 \\
                                          & \ours                   & 3                                   & 6  & 2  & 0.97 & 0.86 \\
                                          & Data (LiNGAM)           & 3                                   & 3  & 3  & 0.92 & 0.50 \\
                                          & \ours (LiNGAM)          & 3                                   & 6  & 3  & 0.96 & 0.80 \\
            \bottomrule
        \end{tabular}}
    }
    \caption{Causal discovery results in \pain. PA, AN, and OT refer to the parents, ancestors, and others, respectively. Accuracy and F1 measure the recovery of the causal ancestors.}
    \label{table:neuropathic_lingam_factor_proposal}
\end{table}

\begin{figure}[ht]
\vspace{-0.2in}
    \centering
    \subfigure[GPT-4 \ours with LiNGAM]{
        \includegraphics[height=0.2\textheight, trim=40 40 35 35, clip]{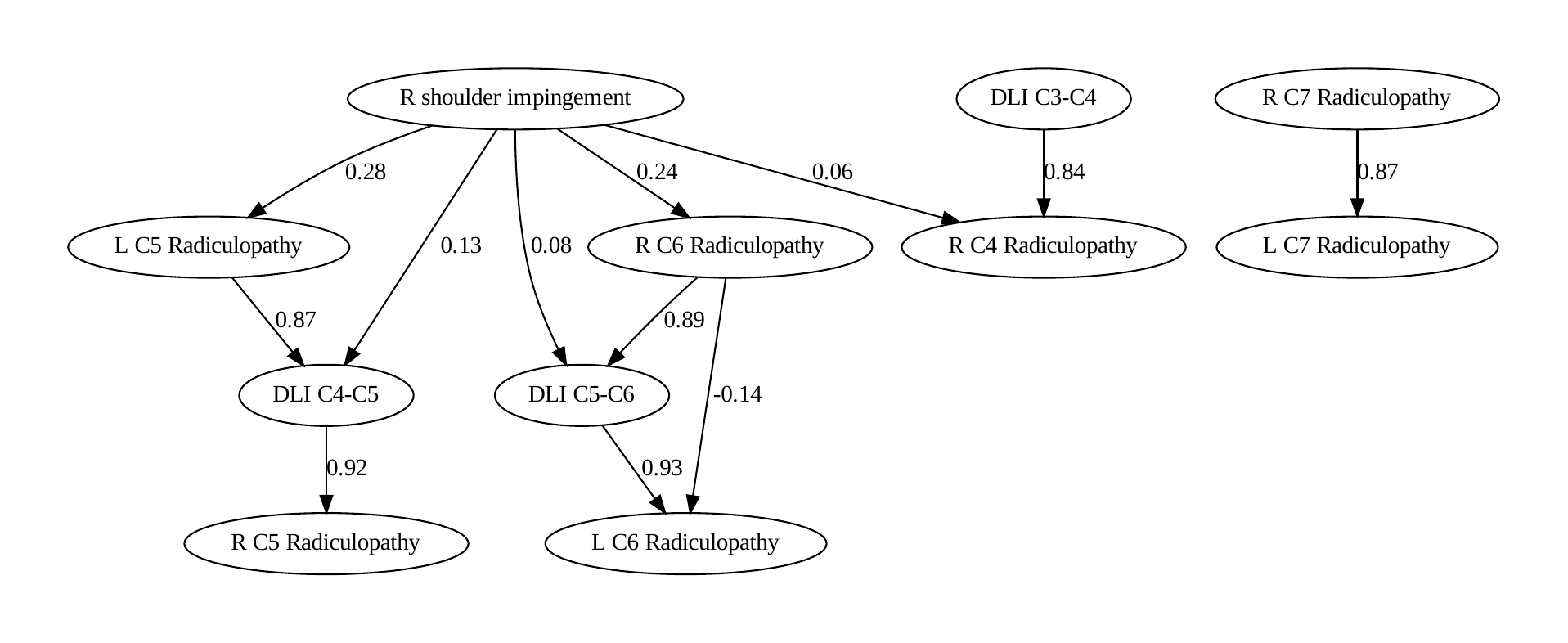}
    } 
    \subfigure[GPT-3.5 \ours with LiNGAM]{
        \includegraphics[height=0.2\textheight, trim=40 40 35 35, clip]{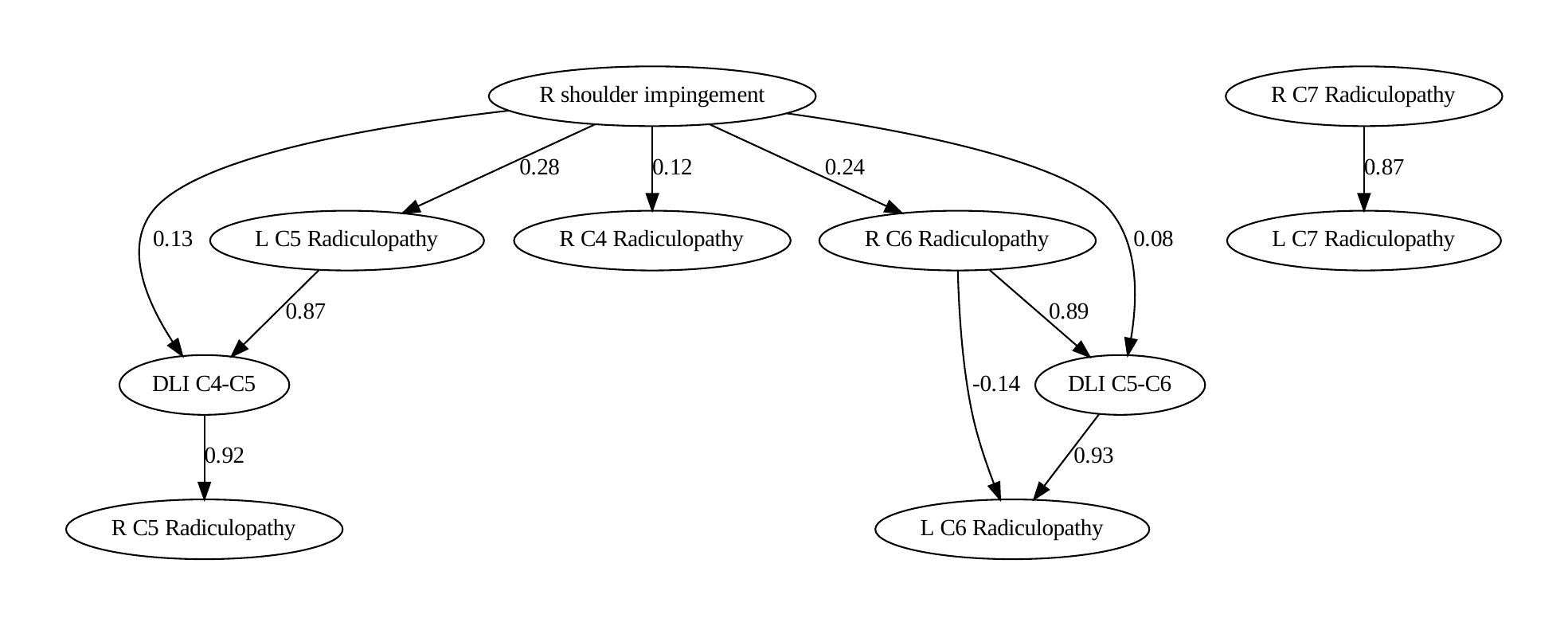}
    } \\ \vspace{0.15in} 
    \subfigure[llama-2-70b \ours with LiNGAM]{
        \includegraphics[height=0.2\textheight, trim=40 40 35 35, clip]{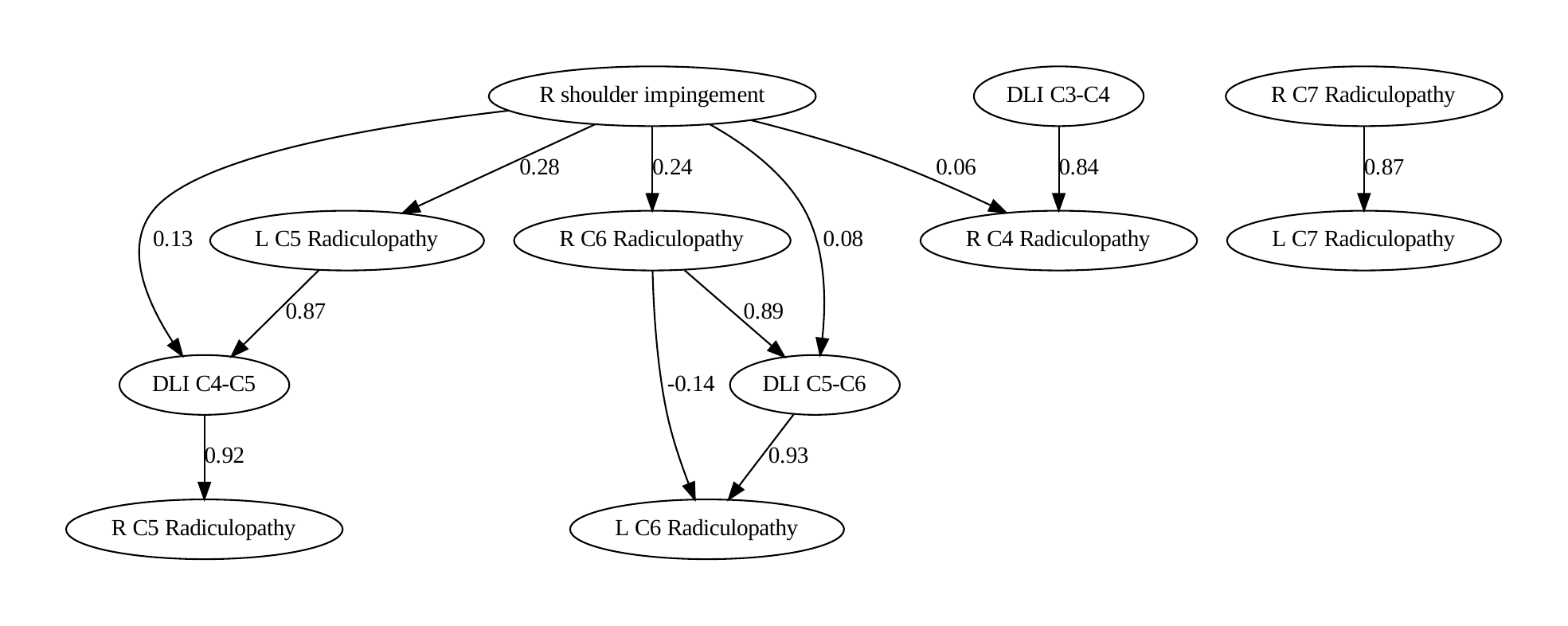}
    }
    \subfigure[Mistral-med \ours with LiNGAM]{
        \includegraphics[height=0.2\textheight, trim=40 40 35 35, clip]{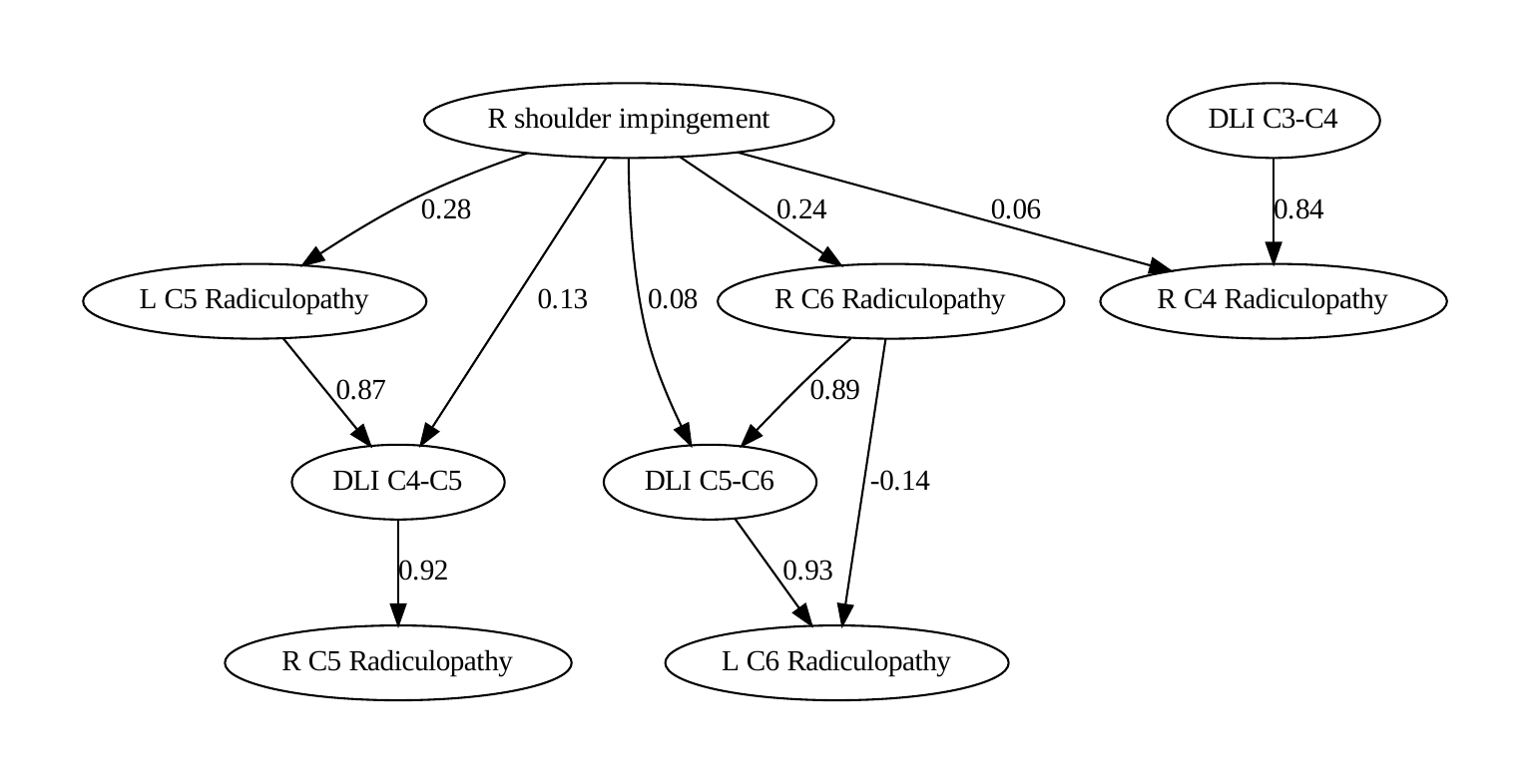}
    }
    \caption{The discovered causal graphs in \pain with LiNGAM.}
    \label{fig:causal_graph_neuro_lingam}
\end{figure}

\clearpage

\newpage

\section{Summary of Benchmark Data}
\label{appendex:summary_of_benchmark_data}

\begin{table}[H]
\centering
{\footnotesize
\setlength{\tabcolsep}{4pt}%
\begin{tabular}{>{\raggedright\arraybackslash}p{1.9cm} >{\raggedright\arraybackslash}p{1.7cm} >{\raggedright\arraybackslash}p{3.9cm} >{\raggedright\arraybackslash}p{1.4cm} >{\raggedright\arraybackslash}p{3.0cm} >{\raggedright\arraybackslash}p{1.3cm}}
\toprule
\textbf{Benchmark Name} & \textbf{Type} & \textbf{Source} & \textbf{Sample} & \textbf{Sample Size} & \textbf{Ground Truth} \\
\midrule
Apple Gastronome & Synthetic Data & This paper. It will be open-sourced under \href{https://creativecommons.org/licenses/by/4.0/deed.en}{CC-BY 4.0}. & Textual & 200 textual reviews. & Yes \\
\midrule
Neuropathic & Semi Real-world Data & Tabular samples from~\citet{neuropanic} on their \href{https://github.com/TURuibo/Neuropathic-Pain-Diagnosis-Simulator}{GitHub repo} under \href{https://creativecommons.org/licenses/by/4.0/deed.en}{CC-BY 4.0}. Textual samples are generated by this paper. & Textual & 100 synthetic textual data for prompt construction; 1000 tabular data for CI tests. & Yes \\
\midrule
Brain Tumor & Real-world Data & An open Kaggle data set (\href{https://www.kaggle.com/datasets/sartajbhuvaji/brain-tumor-classification-mri/data}{kaggle/brain-tumor-classification-mri}) with an open-sourced project.~\citep{bhuvaji2024brain}. \href{https://www.mit.edu/~amini/LICENSE.md}{MIT License.} & Image & 60 MRI samples for COAT and CI tests. The sample size is small due to the expensive API cost. One can easily construct a much larger data set by following our instructions. & No \\
\midrule
Stock News & Real-world Data & An open Kaggle data set (\href{https://www.kaggle.com/datasets/BidecInnovations/stock-price-and-news-realted-to-it/data}{kaggle/stock-price-and-news-realted-to-it})~\citep{BidecInnovations2023}. \href{https://www.mit.edu/~amini/LICENSE.md}{MIT License.} & Textual with time index & 200 samples for COAT experiments. 400 samples for empirical validation. & No \\
\midrule
ENSO & Real-world Data & The NOAA 20th Century Reanalysis V3 data set~\citep{compo2011twentieth} from their website at \href{https://psl.noaa.gov}{https://psl.noaa.gov} (with \href{https://creativecommons.org/publicdomain/zero/1.0/?ref=chooser-v1}{CC0 1.0 License}). We developed a set of simple tools for LLMs to utilize the official database. & NetCDF & Monthly observation for 1836/01 to 2015/12; 1.0 degree latitude x 1.0 degree longitude global grid (360x181); 90+ climate variables. & No \\
\bottomrule
\end{tabular}}
\caption{Summary of Benchmark Data}
\end{table}

\clearpage
\section{COAT Output for Users in Practice}
\label{app:coat-output-schematic}

\begin{figure}[t]
    \centering
    \fbox{\includegraphics[width=\dimexpr\linewidth-2\fboxsep-2\fboxrule\relax]{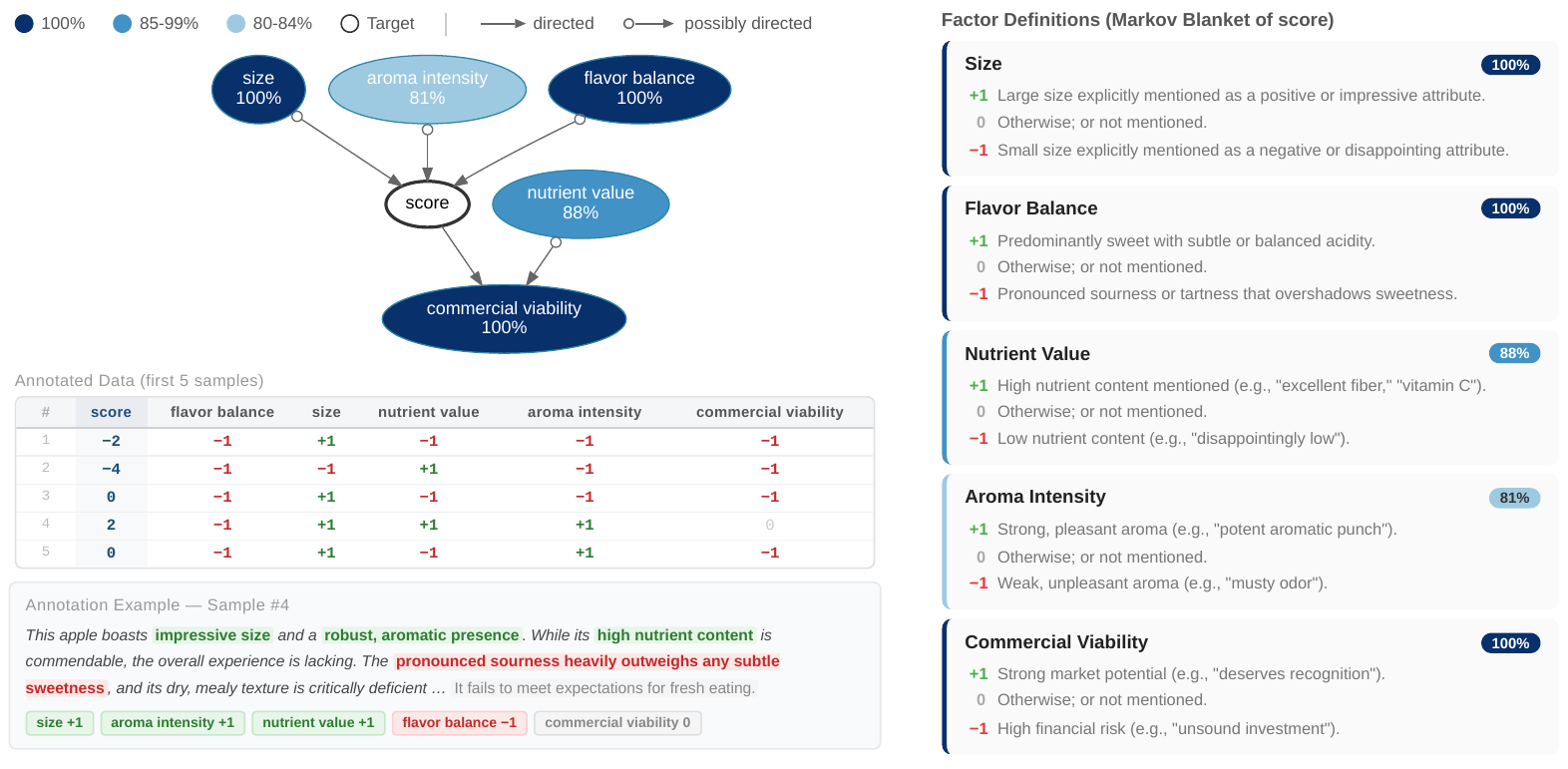}}
    \caption{Schematic overview of \oursMB output components for end users. It has four integrated modules (conceptual layout): factor definitions with uncertainty indicators (right), the inferred causal structure (upper-left), structured annotation results (middle-left), and annotation details (bottom-left).}
    \label{fig:coat-user-demo}
\end{figure}

\edit{
Figure~\ref{fig:coat-user-demo} presents a conceptual layout of \oursMB's practical output. 
The schematic illustrates four core components presented to users: \textit{(i)~Factor definitions} (right): the final factors identified by the system with uncertainty quantification; \textit{(ii)~Causal graph} (upper-left): the learned causal structure among the factors and target variable; \textit{(iii)~Annotation tables} (middle-left): structured data with factor values after the annotation process on raw samples; and \textit{(iv)~Annotations details } (bottom-left): the annotation details for raw samples, possibly with the rationale of the decisions about each factor's value. 
} 

\clearpage
\newpage

\vskip 0.2in
\pdfbookmark[1]{References}{sec:references}
\bibliography{references/references.bib,references/causality.bib,references/llm.bib,references/ood.bib,references/causalllm.bib,references/braintumor.bib,references/stocknews.bib,references/enso.bib}

\end{document}